%% file: template.tex
\documentclass{article}

\usepackage{arxiv}

\usepackage[utf8]{inputenc} 
\usepackage[T1]{fontenc}    
\usepackage{hyperref}       
\usepackage{url}            
\usepackage{booktabs}       
\usepackage{amsfonts}       
\usepackage{nicefrac}       
\usepackage{microtype}      
\usepackage{lipsum}		
\usepackage{graphicx}
\usepackage{natbib}
\usepackage{doi}
\usepackage{multirow}
\usepackage{wrapfig}

\usepackage{appendix, subcaption}
\usepackage{amsmath}
\usepackage{amsthm}
\usepackage{amssymb}
\input{defs}

\usepackage{algorithm}
\usepackage{algorithmic}

\newtheorem{example}{Example}
\newtheorem{theorem}{Theorem}
\newtheorem{lemma}{Lemma}
\newtheorem{assumption}{Assumption}
\newtheorem{corollary}{Corollary}
\newtheorem{remark}{Remark}

\newtheorem{proposition}[theorem]{Proposition}

\theoremstyle{definition}
\newtheorem{definition}{Definition}

\title{Stability and Generalization for Stochastic Recursive Momentum-based Algorithms for (Strongly-)Convex One to $K$-Level Stochastic Optimizations}

\author{{Xiaokang Pan} \\
	Central South University\\
	\texttt{224712176@csu.edu.cn} \\
        \And
        {Xingyu Li} \\
	Tulane University\\
	\texttt{xli82@tulane.edu} \\
        \And
        {Jin Liu} \\
	Central South University\\
	\texttt{liujin06@csu.edu.cn} \\
        \And
	{Tao Sun} \\
	National University of Defense Technology\\
	\texttt{suntao.saltfish@outlook.com} \\
        \And
	{Kai Sun} \\
	Xi’an Jiaotong University\\
	\texttt {sunkai@xjtu.edu.cn} \\
        \And
        {Lixing Chen} \\
	Shanghai Jiao Tong University\\
	\texttt {lxchen@sjtu.edu.cn} \\
        \And
	{Zhe Qu}\\
	Central South University\\
	\texttt{zhe\_qu@csu.edu.cn} \\
}

\date{}

\renewcommand{\headeright}{}
\renewcommand{\undertitle}{}
\renewcommand{\shorttitle}{Stability and Generalization for Stochastic Recursive Momentum-based One to $K$-Level Stochastic Optimizations}

\hypersetup{
pdftitle={Stability and Generalization for Stochastic Recursive Momentum-based Algorithms for (Strongly-)Convex One to $K$-Level Stochastic Optimizations}
}

\begin{document}
\maketitle

\begin{abstract}
STOchastic Recursive Momentum (STORM)-based algorithms have been widely developed to solve one to $K$-level ($K \geq 3$) stochastic optimization problems. Specifically, they use estimators to mitigate the biased gradient issue and achieve near-optimal convergence results. However, there is relatively little work on understanding their generalization performance, particularly evident during the transition from one to $K$-level optimization contexts. This paper provides a comprehensive generalization analysis of three representative STORM-based algorithms: STORM, COVER, and SVMR, for one, two, and $K$-level stochastic optimizations under both convex and strongly convex settings based on algorithmic stability. Firstly, we define stability for $K$-level optimizations and link it to generalization. Then, we detail the stability results for three prominent STORM-based algorithms. Finally, we derive their excess risk bounds by balancing stability results with optimization errors. Our theoretical results provide strong evidence to complete STORM-based algorithms: (1) Each estimator may decrease their stability due to variance with its estimation target. (2) Every additional level might escalate the generalization error, influenced by the stability and the variance between its cumulative stochastic gradient and the true gradient. (3) Increasing the batch size for the initial computation of estimators presents a favorable trade-off, enhancing the generalization performance.
\end{abstract}

\section{Introduction}
In stochastic optimization problems, variance reduction techniques \cite{fang2018spider, zhoustochastic, wen2018flipout, qi2021stochastic, liu2019variance, liu2023breaking} can significantly mitigate the negative impact of inherent variance due to the stochastic gradients. In particular, Stochastic Recursive Momentum (STORM) \cite{cutkosky2019momentum} stands out for its simple implementation and near-optimal convergence results. STORM carefully designs momentum-based estimators for model updating, which can dynamically adapt to the optimization challenge without a large batch or checkpoint gradient computations. Due to these advantages, STORM has been extensively used in various practical applications: reinforcement learning \cite{hu2019efficient, mao2022improving}, model-agnostic meta-learning \cite{ji2022theoretical, qu2023prevent}, risk-averse portfolio optimization \cite{tran2020hybrid, jiang2022optimal}, and deep AUC maximization \cite{yuan2021compositional, liu2023breaking}.

Subsequently, various STORM-based algorithms \cite{hu2019efficient, yuan2021compositional, chen2021solving, jiang2022optimal, li2023fedlga} have extended this methodology to address stochastic two-level and $K$-level (where $K \geq 3$) optimization problems. In their definitions, two-level stochastic optimizations are equivalent to stochastic compositional optimizations and similar to $K$-level stochastic optimizations \cite{wang2017stochastic, ghadimi2020single, chen2021solving}, which pose a challenge in obtaining a biased estimate of the objective function and gradients \cite{dann2014policy, wang2017stochastic}. By leveraging the high-precision estimations, STORM-based algorithms have successfully addressed the corresponding challenge.

In particular, in two-level optimizations, one of the most popular STORM-based algorithms COVER \cite{qi2021stochastic} employs estimators for both the value of the inner function and the value of the gradient. When increasing to $K$-level optimizations, inherent variances can be magnified, leading to significant gradient deviations and potential explosions. To mitigate this, the near-optimal algorithm SVMR \cite{jiang2022optimal} employs estimators for all function values and gradients, except the outer function value, and applies gradient projection techniques to the function gradient estimator.

Although STORM-based algorithms have achieved many breakthroughs in algorithmic convergence, their effect on generalization performance is less understood \cite{hardt2016train, yang2023stability}, i.e., how the model trained by the training samples would generalize to test samples, especially for optimizations with higher levels. To clearly understand the generalization of these algorithms, we consider the following two key questions.

Specifically, as the success of STORM lies in leveraging estimators to tackle biased gradient issues, exploring the influence of these estimators on generalization performance enriches the study \cite{yuan2019stochastic, hu2019efficient, ghadimi2020single, balasubramanian2022stochastic, qu2023convergence}. Additionally, in $K$-level optimization, the gradient estimator at each level is influenced by the function value estimator at the preceding level, which, in turn, indirectly affects the function value estimator at the subsequent level \cite{chen2021solving, jiang2022optimal}. Therefore, addressing the second question can offer guidance for designing corresponding estimators in more complex and general scenarios.

To answer the above two questions, this paper leverages the algorithmic stability to systematically explore the generalization of STORM-based algorithms from one to $K$-level stochastic optimizations. We believe that this exploration is important to gain insights into STORM's scalability and effectiveness across different tiers of stochastic optimization. In particular, our contributions are summarized as follows. 

\begin{itemize}
    \item To achieve our goal, we first introduce a novel definition of uniform stability, specifically for $K$-level optimizations. Leveraging this definition, we establish a quantitative relationship between generalization error and stability in the context of \( K \)-level optimization. Then, we analyze the stability and optimization errors for three distinct algorithms: STORM, COVER, and SVMR, corresponding to one, two, and $K$-level stochastic optimizations in both convex and strongly convex settings. Finally, by analyzing the interplay between stability and optimization errors, we ascertain their excess risks in these settings.
    \item Our theoretical results indicate that fewer iterations and proper step sizes will improve algorithm stability of stability in the convex setting. For the excess risk, our results demonstrate that we need about $T \asymp \max(n_k^{5/2})$, $\forall k \in [1,K]$, iterations to achieve the ideal excess risk rate. In the strongly convex setting, a proper step size will not necessarily make the algorithm stable enough, which must be combined with expanding the batch size to ensure stability. Moreover, $T \asymp \max(n_k^{7/6})$ iterations should be used, which is fewer than the convex setting. 
    \item Based on our analysis, we can successfully address the above questions. Firstly, we find that the stability of the algorithm can be compromised by each estimator, due to the variance between the estimator and its estimated target, which degrades the generalization performance. Moreover, as the number of levels increases, two main factors impact the algorithm's generalization error: the first is the influence of the new level on the algorithm's stability, and the second is the variance between the combined stochastic gradient and the true gradient across all levels. There is one more observation in our analysis: employing more samples for the initial computation of estimators may enhance performance without significantly increasing computational costs. This strategy presents a viable approach to improve the efficiency of STORM-based algorithms.

\end{itemize}

\section{Related Work}
\textbf{Algorithmic stability and Generalization.} In learning theory, the stability of an algorithm shows that small changes in the training data result in only minimal differences in the predictions made by the model \cite{kearns1997algorithmic, vapnik2000bounds, cucker2002mathematical}. The landmark work \cite{bousquet2002stability} introduces the notion of uniform stability and establishes the generalization of ERM based on stability, and it has a deep connection with \cite{cesa2004generalization, rakhlin2005stability, kutin2012almost}. Furthermore, \cite{bartlett2002rademacher, poggio2004general, shalev2010learnability} discuss the relationship between algorithmic stability and complexity measures, and use it on general conditions for predictivity. \cite{hardt2016train} contribute significantly to the understanding of algorithmic stability in optimization algorithms, particularly gradient descent. More recently, \cite{li2023transformers} presents in-context learning, showing its effectiveness and stability in different data scenarios. \cite{sakaue2023improved} demonstrates that coordinate estimation leads to tighter generalization bounds.

\textbf{Stochastic Compositional Optimization.} Extensive studies have mitigated the issue of bias in gradient estimation due to combination functions. \cite{wang2017stochastic} uses stochastic gradients for internal function value computation. Variance reduction techniques can accelerate the efficiency of Stochastic Compositional Gradient Descent (SCGD). Algorithms such as SAGA \cite{zhang2019composite}, SPIDER \cite{fang2018spider}, and STORM \cite{cutkosky2019momentum} have been integrated into SCGD. Later, some studies \cite{yuan2019stochastic, zhang2021multilevel, tarzanagh2022fednest} have successfully linked stochastic two-level or $K$-level optimization challenges. In $K$-level optimization, \cite{yang2019multilevel} leads to the creation of an accelerated technique (A-TSCGD). Subsequently, \cite{balasubramanian2022stochastic} introduces the NLASG method, which expands the scope of the NASA \cite{ghadimi2020single} algorithm to broader applications. In a similar vein, \cite{chen2021solving, jiang2022optimal} extend STORM for estimating function values to $K$ levels. However, all the above works only focus on convergence analysis.

\section{Preliminaries and Warm Up}
In this section, we begin by introducing three optimization problems that we address, accompanied by three popular STORM-based algorithms designed for these specific problems. Then, we will present the concept of stability as applied in statistical learning theory \cite{james2013introduction}. To this end, we present the first theorem in this paper.

\subsection{One to $K$-level Stochastic Optimziations}
In this paper, we extend algorithmic stability analysis to the most popular STORM-based algorithms: STORM \cite{cutkosky2019momentum}, COVER \cite{qi2021stochastic}, and SVMR \cite{jiang2022optimal} for stochastic optimization problems with levels 1, 2, and $K \geq 3$, respectively. Detailed update rules for these algorithms are presented in Appendix~\ref{sec:algorithm}, Algorithms~\ref{alg_storm}-\ref{alg_svmr_multi}. Their optimization formulations are introduced subsequently.

\textit{One-level optimization}. Typically, the one-level stochastic optimization problem \cite{hardt2016train, cutkosky2019momentum, bousquet2020sharper, levy2021storm+} can be formulated as follows
\begin{equation}\label{eq:sco-single}
	\min_{x\in \X} \Bigl\{F(x)  = \EX_{\nu}[ f_\nu (x) ]\Bigr\},
\end{equation}
where $f: \mathbb{R}^{d} \to \mathbb{R}^{d_1}$ on a convex domain $\mathcal{X} \in \mathbb{R}^d$, $\nu$ is an independent random data sample, and $F$ is the empirical risk $\min_{x\in \X}\{F_S(x) := f_S(x) = \frac{1}{n}\sum_{i=1}^nf_{\nu_i}(x)\}$. Let $S = \{\nu_1, \cdots, \nu_n\}$ be a dataset from which the samples are drawn independently and identically (i.i.d.). To facilitate the expansion below, we give more symbol definitions: $S'$ is the i.i.d copy of $S$, where $S' = \{ \nu'_1, \cdots, \nu'_n \}$, and $S^{i}$ is the i.i.d. copy of $S$ where only $i$-th data point $\nu_i$ in $S$ in change to $\nu_i'$.  Compared with SGD which directly uses stochastic gradients for updates, the main part of STORM \cite{cutkosky2019momentum} is to leverage the corrected momentum variance reduction estimator for updates.

\textit{Two-level optimization}. We consider the two-level stochastic optimization problem \cite{yuan2019stochastic, yang2019multilevel, balasubramanian2022stochastic} as follows
\begin{equation}\label{eq:sco-two} 
	\min_{x\in \X} \Bigl\{F(x) = {f}\circ g(x) = \EX_\nu[ f_\nu ( \EX_\omega[ g_\omega(x) ] ) ]\Bigr\},
\end{equation}
where $f: \mathbb{R}^{d_1} \to \mathbb{R}^{d_2}$ and $g: \mathbb{R}^{d} \to \mathbb{R}^{d_1}$ on a convex domain $\mathcal{X} \in \mathbb{R}^d$, $\nu$ and $\omega$ are independent random variables. Let $S = S_{\nu} \cup S_{\omega}$, where $S_\nu = \{\nu_1, \cdots, \nu_n\}$ and $S_\omega = \{\omega_1, \cdots, \omega_m\}$, and the empirical risk is defined as $\min_{x\in \X}\{F_S(x) := f_S(g_S(x)) = \frac{1}{n}\sum_{i=1}^n f_{\nu_i}( \frac{1}{m} \sum_{j=1}^m g_{\omega_j}(x) )\}$. In this scenario, altering a single data point can affect either \(S_\nu\) or \(S_\omega\). For \(\forall i \in [1, n]\) and \(\forall j \in [1, m]\), \(S^{i,\nu}\) denotes the version of \(S\) where only the \(i\)-th point in \(S_\nu\) is replaced by \(\nu'_i\), with \(S_\omega\) remaining unchanged. \(S^{j,\omega}\) is defined similarly. The i.i.d. copied dataset \(S'\) is represented as \(S' = S'_\nu \cup S'_\omega\), where \(S'_\nu = \{\nu'_1,\ldots, \nu'_n\}\) and \(S'_\omega = \{ \omega'_1,\ldots, \omega'_m\}\). Note that the two-level optimization problem in \eqref{eq:sco-two} can also be considered as the compositional optimization \cite{yuan2019stochastic, yang2019multilevel, balasubramanian2022stochastic, hu2023non}. Among the STORM-based algorithms for two-level stochastic optimization, we will analyze the stability and generalization of the most popular algorithms, COVER \cite{qi2021stochastic}. Specifically, COVER utilizes two estimators for both the function and gradient values of the inner function, namely \(u_t\) and \(v_t\).

\textit{$K$-level optimization}. The $K$-level stochastic optimization problem \cite{chen2021solving, jiang2022optimal} can be formulated as follows 
\begin{equation}\label{eq:sco-multi}
	\begin{aligned}
		\min _{x \in \X}  \Bigl\{ F(x)&=f_{K} \circ f_{K-1} \circ \cdots \circ f_{1}(x) =\EX_{\nu^{(K)}}[f_K^{\nu^{(K)}}(\cdots\EX_{\nu^{(1)}}[f_{1}^{\nu^{(1)}}(x)])]\Bigr\},
	\end{aligned}
\end{equation}
where $f_k: \mathbb{R}^{d_{k-1}} \to \mathbb{R}^{d_k}$ on a convex domain $\mathcal{X} \in \mathbb{R}^d$, $k \in[1,k]$ and $d_0 = d$. $\nu^{(k)}$ are independent random variables, where $k \in [1,K]$. Similarly, let  $S = \cup_{k=1}^K  S_{k}$, where $S_{k} = \{\nu_1^{(k)}, \cdots, \nu_{n_k}^{(k)}\}$, the empirical risk is defined as $\min_{x\in \X}\{F_S(x) := f_{K,S} \circ f_{K-1,S} \cdots f_{1,S} =  \frac{1}{n_K}\sum_{i_K=1}^{n_K}f_K^{\nu_{i_K}^{(K)}}(\cdots (\frac{1}{n_1}\sum_{i_1 = 1}^{n_1}f_1^{\nu_{i_1}^{(1)}}(A(S))))\}$. In the \(K\)-level optimization, where changing one sample data can occur in any layer of the function, we define:
$S^{l,k}$ be the i.i.d. copy of $S$ where only the $l$-th data point $\nu_l^{k}$ in $S_k$ is replaced with $\nu_l^{k'}$ , where $k \in [1,k]$ and $l \in[1,n_k]$. Moreover, we denote $S' = \cup_{i=1}^K S^{(i)}$, where $S^{(i)} = \{\nu^{(i)'}_1, \cdots,\nu^{(i)'}_{n_i}\}$. In this scenario, we consider SVMR \cite{jiang2022optimal} with multiple estimators, which obtains the best convergence result. In particular, $u^{(k)}$ represents the estimate of the $k$-th layer function value and $v^{(k)}$ represents the estimate of the $k$-th layer function's gradient value.

\subsection{Concept of Excess Risk}
As we all know, excess risk is an evaluation for the generalization performance \cite{bousquet2002stability, james2013introduction, charles2018stability}, which is used to analyze the three tackled STORM-based algorithms in this paper. For a randomized algorithm $A$, denote by $A(S)$ its output model based on the training data $S$. By denoting $F(x_*) = \inf_{x\in \X}F(x)$ and $F(x_*^S) = \inf_{x\in \X}F_S(x)$, then the excess risk is $\EX_{S,A}[F(A(S) - F(x_*)]$. According to the decomposition in \cite{bousquet2002stability} and $F_S(x_*^S)\leq F_S(x_*)$ by the definition of $x_*^S$, we can obtain the excess risk as follows
\begin{align*}
    \EX_{S, A}[F(A(S))& - F(x_* )] \leq \EX_{S,A}[F(A(S)) - F_S(A(S)) ]\nonumber+ \EX_{S,A}[F_S(A(S)) - F_S(x_*^S)].
\end{align*}
We refer to the term $\EX_{S,A}[F(A(S)) - F_S(A(S))]$ as the generalization error, as it quantifies the generalization shift from training to testing behavior. Similarly, $\EX_{S,A}[F_S(A(S)) - F_S(x_*^S)]$ is termed the optimization error, measuring how effectively the algorithm minimizes empirical risk. The generalization error in this paper is informed by analyses from prior studies \cite{cutkosky2019momentum, qi2021stochastic, jiang2022optimal}. Unlike these works, which primarily focus on convergence analysis, our main objective is to estimate the generalization error through the algorithmic stability approach \cite{bousquet2002stability}. Next, we provide the definitions of stability.

\begin{definition}[Uniform Stability]\label{def:stability}
The uniform stability of the three stochastic optimizations is defined as follows
\begin{enumerate}[label=({\roman*})]
\item In the one-level optimization, an algorithm $A$ is uniformly stable for \eqref{eq:sco-single} if $\forall i\in [1,n]$, there holds $\EX_A [\| A(S) - A(S^{i})\|] \le \gep$.
\item In the two-level optimization, an algorithm $A$ is uniformly stable for \eqref{eq:sco-two}, if $\forall i\in [1,n]$ and $\forall j \in [1,m]$, there holds $\EX_A [\| A(S) - A(S^{i,\nu})\|] \le \gep_\nu$ and $\EX_A[\| A(S) - A(S^{j,\omega})\|] \le \gep_\omega$.
\item In the $K$-level optimization, an algorithm $A$ is uniformly stable for \eqref{eq:sco-multi}, if $\forall k \in [1,K]$ and $\forall l \in [1,n_k]$,  there holds $\EX_A [\| A(S) - A(S^{l,k})\|] \le \gep_{k}$.
\end{enumerate}
\end{definition}
The expectation $\EX_A[\cdot]$ is taken w.r.t. the internal randomness of $A$ not the data points for the above definition.

We aim to elucidate the connection between uniform stability (as outlined in Definition \ref{def:stability}) and the generalization error, a relation applicable across all randomized algorithms. To achieve this, we state the following assumption.

\begin{assumption}[Lipschitz Continuity] \label{ass:Lipschitz continuous} 
The Lipschitz continuity of our focused problems is proposed as follows
\begin{enumerate}[label=({\roman*})]
    \item In the one-level optimization problem, there exists a constant $L_f$, such that $f_\nu$ is Lipschitz continuous with parameters $L_f$, i.e., $\sup _{\nu}\|f_{\nu}(x)-f_{\nu}(\hat{x})\| \leq L_{f}\|x-\hat{x}\|,~\text{for all}~x, \hat{x} \in \mathbb{R}^d$.
    \item In the two-level optimization problem, there exist two constants $L_f$ and $L_g$, such that $f_\nu$ and $g_\omega$ are Lipschitz continuous with parameters $L_f$ and $L_g$, respectively, i.e., $\sup _{\nu}\|f_{\nu}(y)-f_{\nu}(\hat{y})\| \leq L_{f}\|y-\hat{y}\|  ~\text{for all} ~   y, \hat{y} \in \mathbb{R}^{d_1},$ and $\sup _{\omega}\|g_{\omega}(x)-g_{\omega}(\hat{x})\| \leq L_{g}\|x-\hat{x}\| ~ \text{for all} ~ x, \hat{x} \in \mathbb{R}^{d}$.
    \item In the $K$-level optimization problem, there exists a constant $L_f$, such that $\forall k \in [1,K] $, $f_k^{\nu^{(k)}}$ are Lipschitz continuous with parameter $L_f$, respectively, i.e., $	\sup _{\nu^{(k)}}\|f_k^{\nu^{(k)}}(y)-f_k^{\nu^{(k)}}(\hat{y})\| \leq L_{f}\|y-\hat{y}\|,  ~\forall ~   y, \hat{y} \in \mathbb{R}^{d_{k-1}}$.
\end{enumerate}
\end{assumption}

\subsection{Generalization of the $K$-level Optimization}
Although existing studies have established relationships between the generalization error and the stability under one-level \cite{hardt2016train} and two-level \cite{yang2023stability} stochastic optimizations, the more complex and general $K$-level stochastic optimization remains unexplored. Therefore, by integrating the stability concept, we specifically define the following theorem for the $K$-level optimization, which aims to show the quantitative relationship between the generalization error and the stability.

\begin{theorem}\label{theorem:general_multi_level}
If Assumption \ref{ass:Lipschitz continuous} (iii) holds true and the randomized algorithm $A$ is uniformly stable, then for $K\geq 3$, $\EX_{S,A}[F(A(S))-F_S(A(S))]$ is bounded by
\begin{equation*}
	\begin{aligned}
		L_f^K\epsilon_K + \sum_{k=1}^{K-1}\Big( 4L_f^K\epsilon_{k}+ L_f\sqrt{\frac{ \mathbb{E}_{S, A} [\operatorname{Var}_{k}(A(S)]}{n_{k}}}\Big),
	\end{aligned}
\end{equation*}
where $\operatorname{Var}_{k}(A(S)) = \EX_{v^{(k)}}[\|f_{k} \circ f_{k-1}\circ \cdots \circ f_1(A(S)-  f_{k}^{v^{(k)}} \circ f_{k-1}\circ \cdots \circ f_1(A(S)\|^{2}]$.
\end{theorem}

\begin{remark}
Theorem~\ref{theorem:general_multi_level} establishes the quantitative relationship between the generalization and the uniform stability for any randomized algorithm applied to $K$-level stochastic optimizations. In particular, when $K=1$, i.e., the one-level stochastic optimization, where $F(x) = \mathbb{E}_\nu[f_{\nu}(x)]$ and $F_S(x) = \frac{1}{n}\sum_{i=1}^n f_{\nu_i}(x)$, we can see the absence of randomness with respect to $\epsilon_k$, $\forall k \in [2, K]$. Consequently, we derive $\mathbb{E}_{S,A}[F(A(S))-F_S(A(S))] \leq L_f\epsilon$, consistent with the findings in \cite{hardt2016train}. For the two-level scenario, i.e., $k = 2$, we obtain $L_{f}^2 \epsilon_{2}+4 L_{f}^2 \epsilon_{1}+L_{f} \sqrt{\mathbb{E}_{S, A}[\operatorname{Var}_{1}(A(S))]/n_1}$. Here, the variance term $\mathbb{E}_{S, A}[\operatorname{Var}_{1}(A(S))]$ arises from the estimator used for the inner function values. We only need to alter the notations in Assumption \ref{ass:Lipschitz continuous} (iii) to obtain results consistent with \cite{yang2023stability}.

\end{remark}
\begin{remark}
In Theorem \ref{theorem:general_multi_level}, we can find the generalization error depends not only on stability but also on the variance term, i.e., $\sqrt{\mathbb{E}_{S, A} [\operatorname{Var}_{k}(A(S)]/n_{k}}$ due to the estimators. An interesting observation is that the variance term is not only determined by the current layer function but also by the combined function of the total number of layers, i.e., for $\operatorname{Var}_{k}(A(S))$, which is determined by $f_{k} \circ f_{k-1}\circ \cdots \circ f_1$, instead of $f_{k}$. This implies that with an increasing number of levels, we should enlarge the sample size in order to achieve a better generalization error. 
\end{remark}

After establishing the quantitative relationship between the generalization error and the stability bound, the next goal is to establish stability bounds for these corresponding algorithms, i.e., STORM, COVER, and SVMR. In next section, we will introduce how to approach this in detail.

\section{Stability and Generalization}
In this section, we present the main results for various optimization problems, which include stability bounds and optimization errors, and ultimately derive the excess risks. Different results for the convex and strongly convex settings will be shown in separate subsections. Before giving the theoretical results, we state the following assumptions to facilitate our proofs.

\begin{assumption}[Empirical Variance]\label{ass:bound variance}
With probability $1$ w.r.t. $S$, there exist constants to bound the following:
\begin{enumerate}[label=({\roman*})]
    \item In the one-level optimization problem, there exist two constants $\sigma_f$ and $\sigma_J$, such that $ \sup _{x \in \mathcal{X}} \frac{1}{n} \sum_{i=1}^{n}[\|f_{\nu_{i}}(x)-f_{S}(x)\|^{2}] \leq \sigma_f^2$ and  $\sup _{x \in \mathcal{X}} \frac{1}{n} \sum_{i=1}^{n}[\|\nabla f_{\nu_{i}}(x)-\nabla f_{S}(x)\|^{2}] \leq \sigma_{J}^2$. 
    \item In the two-level optimization problem, there exist three constants $\sigma_g$, $\sigma_g'$ and $\sigma_f$, such that $ \sup _{x \in \mathcal{X}} \frac{1}{m} \sum_{j=1}^{m}[\|g_{\omega_{j}}(x)-g_{S}(x)\|^{2}] \leq \sigma_g^2$, $\sup _{x \in \mathcal{X}} \frac{1}{m} \sum_{j=1}^{m}[\|\nabla g_{\omega_{j}}(x)-\nabla g_{S}(x)\|^{2}] \leq \sigma_{g'}^2$ and $\sup _{y \in \mathbb{R}^d} \frac{1}{n} \sum_{i=1}^{n}[\|\nabla f_{\nu_{i}}(y)-\nabla f_{S}(y)\|^{2}] \leq \sigma_f^2$.
    \item In the $K$-level optimization problem, there exist two constants $\sigma_f$ and $\sigma_J$, such that for $1 \leq k \leq K$, there holds $\sup _{y \in \mathbb{R}_{d_{k-1}}} \frac{1}{n_k} \sum_{j=1}^{n_k}[\|f_{k}^{\nu^{(j)}}(y) -$ $ f_{k,S}(y)\|^{2}] \leq \sigma_f^2$ and $ \sup _{y \in \mathbb{R}_{d_{k-1}}} \frac{1}{n_k} \sum_{j=1}^{n_k}[\| \nabla f_{k}^{\nu^{(j)}}(y) - \nabla f_{k,S}(y)\|^{2}] \leq \sigma_J^2$.
\end{enumerate}
\end{assumption}

\begin{assumption}[Smoothness and Lipschitz continuous gradient]\label{ass:Smoothness and Lipschitz continuous gradient} With probability 1 w.r.t. $S$, there exist constants to make following conditions hold true.

\begin{enumerate}[label=({\roman*})]
    \item In the one-level optimization, the problem $f_{S}(\cdot)$ is $L$-smooth, i.e., $\|\nabla f_{\nu}(x)-\nabla f_{\nu}(x')\| \leq L\|x-x'\|$, $\forall x, x^{\prime} \in \mathcal{X}$.
    \item In the two-level optimization, the problem $f_{S}(g_{S}(\cdot))$ is $L$-smooth, i.e., $\|\nabla g_{S}(x) \nabla f_{S}(g_{S}(x))-\nabla g_{S}(x^{\prime}) \nabla f_{S}(g_{S}(x^{\prime}))\| \leq L\|x-x^{\prime}\| $, $\forall x, x^{\prime} \in \mathcal{X}$. Also, $f_{S}(\cdot)$ has Lipschitz continuous gradients, i.e.,$\|\nabla f_{S}(y)-\nabla f_{S}(\bar{y})\| \leq C_{f}\|y-\bar{y}\|$ for all $y, \bar{y} \in \mathbb{R}^{d}$. 
    \item In the $K$-level optimization, the problem $F_S(\cdot)$ is $L$-smooth, i.e., $\|\Pi_{i=1}^{K}\nabla F_{i,S}(x) -\Pi_{i=1}^{K}\nabla F_{i,S}(x') \| \leq L\|x-x^{\prime}\|$, $\forall x, x^{\prime} \in \mathcal{X}$, where $\nabla F_{k,S}(x) = \nabla f_{k,S}(f_{k-1,S}(\cdots (f_{1,S}(x)))))$ and $\nabla F_{1,S}(x) = \nabla f_{1,S}(x)$.  Additionally, $\forall k \in [1,K]$, the $k$-level function has Lipschitz continuous gradients, i.e., $\|\nabla f_{k,S}(y)-\nabla f_{k,S}(\bar{y})\| \leq L_f\|y-\bar{y}\|$ for all $y, \bar{y} \in \mathbb{R}^{d_{k-1}}$.
\end{enumerate}
\end{assumption}

Assumptions~\ref{ass:bound variance}-\ref{ass:Smoothness and Lipschitz continuous gradient} are widely used in convergence and generalization analysis \cite{charles2018stability, cutkosky2019momentum, zhang2021generalization, qi2021stochastic, jiang2022optimal, yang2023stability}, which ensure the convergence and stability. It is important to note that Assumption~\ref{ass:bound variance} in generalization analysis shows the difference between the stochastic gradient and the empirical risk gradient \(\nabla f_S(x)\). We also present the following definition for our focused settings, i.e., convex and strongly convex.

\begin{definition}\label{def:StronglyConvex}
    A function $F$ is $\mu$-strongly convex if for all $x$, $x' \in \mathcal{X} $, we have $F(x) \geq F(x') + \langle \nabla F(x'), x-x'\rangle + \frac{\mu}{2}\|x-x'\|^2$, and if $\mu = 0$, we say that $F$ is convex. 
\end{definition}

\subsection{Convex setting}\label{section:convex}

\textbf{Stability Results.} The following theorems establish the uniform stability for the three optimizations under the convex setting, i.e., convex $F_S$. All the theoretical results in this subsection are under Assumptions \ref{ass:Lipschitz continuous}-\ref{ass:Smoothness and Lipschitz continuous gradient}.

\begin{theorem}[One-level, Stability, Convex]\label{thm:sta_convex_single_level}
Consider STORM in Algorithm \ref{alg_storm} with $\eta_t = \eta \leq \frac{2}{3L}$ and $\beta_t = \beta \in (0,1) $, $\forall t \in [0,T-1]$. Then, the outputs $A(S) =x_{T}$ at iteration $T$ are uniformly stable with
\begin{equation*}
	\begin{aligned}
		&\epsilon = O\Big(\sup_{S} \eta \sum_{j=0}^{T-1}\operatorname{Var}(v_j) + \frac{L_f\eta T}{n}\Big),
	\end{aligned}
\end{equation*}
where $\operatorname{Var}(v_j) = (\EX_{A}[\|v_j -\nabla f_S(x_j)\|^2 ])^{1/2}$.
\end{theorem}

\begin{remark}
We can find that in \cite{hardt2016train}, the uniform stability for SGD with the same setting is of the order $O(\frac{L_f \eta T}{n})$. However, using STORM adds another term $\sup_{S} \eta \sum_{j=0}^{T-1}\operatorname{Var}(v_j)$ caused by the estimator. This new term is determined by the difference between the estimate $v_j$ and the gradient of the empirical risk $\nabla f_S(x_j)$. In other words, STORM may not be as stable as SGD.
\end{remark}

\begin{theorem}[Two-level, Stability, Convex]\label{thm:sta_convex_two_level}
Consider COVER in Algorithm \ref{alg_cover} with $\eta_t = \eta \leq \frac{1}{4L}$ and $\beta_t = \beta \in (0,1)$, $\forall t\in [0,T-1]$. Then, the outputs $A(S) =x_{T}$ at iteration $T$ are uniformly stable with
\begin{equation*}
\begin{aligned}
    \epsilon_\nu + \epsilon_\omega = O\Big(  L_gC_f\sup_S\eta\sum_{j=0}^{T-1}(  \operatorname{Var}(u_j) + \operatorname{Var}(v_j)) + L_f\sigma_f\eta \sqrt{T} + \frac{L_gL_f\eta T}{m} +\frac{L_gL_f\eta T}{n}  \Big).
\end{aligned}
\end{equation*}
where $\operatorname{Var}(u_j) = (\EX_A[\|u_j-g_S(x_j)\|^2])^{1/2} $ and $\operatorname{Var}(v_j) = (\EX_A[\|v_j- \nabla g_S(x_j)\|^2])^{1/2}$
\end{theorem}

\begin{remark}
When comparing the stability of COVER in Theorem~\ref{thm:sta_convex_two_level} with STORM, particularly under the condition where $n=m$, COVER in the two-level stochastic optimization is characterized by two additional terms: $L_f\sigma_f\eta \sqrt{T}$ and $L_gC_f\sup_S\eta\sum_{j=0}^{T-1}\operatorname{Var}(v_j)$. The first term emerges due to the empirical error of the outer function. The second term is generated by the provided estimator from COVER for the inner function values, which accounts for the difference between the inner function estimator and the empirical risk of the inner function value.

\end{remark}

\begin{theorem}[$K$-level, Stability, Convex]\label{thm:sta_convex_mul_level}
Consider SVMR in Algorithm \ref{alg_svmr_multi} with $\eta_t = \frac{2}{LK(K+2)}$ and $\beta_t =\beta \in (0,1)$, $\forall t\in [0, T-1]$. Then, the outputs $A(S) =x_{T}$ at iteration $T$ are uniformly stable with
\begin{equation*}
	\begin{aligned}
		& \sum_{k=1}^K\epsilon_k = O\Big(\sup_S \eta \sum_{s=1}^{T-1} \sum_{i=1}^{K}\sum_{j=1}^{i-1}  L_f^{K-j+\frac{(i-1)i}{2}} \operatorname{Var}_{j,s}(u)   +  \sup_S \eta \sum_{s=1}^{T-1} \sum_{i=1}^K  L_f^{K + \frac{(i-3)i}{2}} \operatorname{Var}_{i,s}(v) + \sum_{k=1}^K\frac{\eta L_f^{K} T}{n_k}\Big),
	\end{aligned}
\end{equation*}
where $\operatorname{Var}_{j,s}(u) = (\EX_A\|u_{s}^{(j)} - f_{j,S}(u_{s}^{(j-1)})\|^2)^{1/2}$ and $\operatorname{Var}_{i,s}(v) = (\EX_A\|v_s^{(i)} - \nabla f_{i,S}(u_s^{(i-1)})\|^2)^{1/2}$.
\end{theorem}

\begin{remark}
Compared to the stability of COVER, especially when $n_k$ is equal $\forall k \in [1, K]$, SVMR introduces additional terms due to its estimators. Let us discuss the term introduced by the function gradient estimator $\sup_S \eta \sum_{s=1}^{T-1} \sum_{i=1}^K L_f^{K + \frac{(i-3)i}{2}} \operatorname{Var}_{i,s}(v)$, accumulating an extra factor of $K$ due to the need for $K$ estimators to estimate the function gradient at each level. As for the term from the function value estimator $\sup_S \eta \sum_{s=1}^{T-1} \sum_{i=1}^{K}\sum_{j=1}^{i-1} L_f^{K-j+\frac{(i-1)i}{2}} \operatorname{Var}_{j,s}(u)$, it becomes more complex in $K$-level optimization, involving three cumulative summations. This complexity arises from interactions between multiple levels, where estimators at different levels have influence instead of them at the same level. The derivatives of the function at the each level are affected by the function value estimator at the previous level and, in turn, impact the function value estimator at the next level, indicating their increased importance. The omitted term relates to the use of the gradient value estimator for the outer function and is equal to $L_f\sigma_f\eta \sqrt{T}$ in Theorem~\ref{thm:sta_convex_mul_level}. This omission transforms the empirical variance of the outer function into a discrepancy between the gradient estimator and the empirical gradient value of the outer function.
\end{remark}

\begin{remark}
Regardless of any algorithm, i.e., SGD or STORM-based, or any number of levels, the choice of step size $\eta$ will affect the stability bound, which indicates proper selection of $\eta$. In addition, we can find that using fewer iterations can make the algorithms more stable, which may be a potential approach to enhance the generalization of STORM-based algorithms.
\end{remark}

Combining Theorems \ref{theorem:general_multi_level} and \ref{thm:opt_convex}, we have established generalization results for the three algorithms. To get excess risk bounds, we also need the optimization error results, i.e., $\EX [F_S(A(S) - F_S(x_*^S)]$.

\textbf{Generalization results.}  Before giving the theorems, we give some clarification. We use the assumption that the $\X$ domain is bounded in $\mathbb{R}^d$ to give the upper bound, i.e., $\EX_A[ \|x_t- x_*^S\|^2] \leq D_x$, $\forall t\in [0,T-1]$. Let $c$ be an arbitrary constant, the following three theorems hold.

\begin{theorem}[Optimization, Convex]\label{thm:opt_convex}
Let \(A(S)= \frac{1}{T}\sum_{t= 1}^T x_t\) be the solution produced by STORM, COVER, and SVMR in Algorithms~\ref{alg_storm}-\ref{alg_svmr_multi}, respectively. The following results bound the optimization error $\EX[F_S(A(S)) - F_S(x_*^S)]$. 

\textnormal{(One-level).} For the problem in \eqref{eq:sco-single}, by selecting $\eta_t=\eta$ and $\beta_t=\beta$, then it holds
\begin{equation*}
        O\Big( \frac{D_x}{\eta T} + (D_x + \sigma_J^2) \beta^{\frac{1}{2}}+ L_f^2\eta  + V(T\beta)^{-c}\beta^{-\frac{1}{2}}+  \frac{L_f^2\eta^2}{\beta^{3/2}}\Big),
\end{equation*}
where $\EX_A \|v_0-\nabla f_S(x_0)\|^2 \leq V$.

\textnormal{(Two-level).} For the problem in \eqref{eq:sco-two}, by selecting  $\eta_t=\eta$ and $\beta_t=\beta$, then it holds
\begin{equation*}
        O\Big( \frac{D_x}{\eta T} + \Phi_1 \beta^{\frac{1}{2}} + \Phi_2 \eta  + \Phi_3 (T \beta )^{-c}\beta^{-\frac{1}{2}} + \frac{\Phi_4\eta^2}{\beta^{3/2}} \Big),
\end{equation*}
where $\Phi_1 = L_g C_f \sigma_g^2 + L_f \sigma_{g'}^2 + (L_f+L_gC_f) D_x$, $\Phi_2 = L_g^2 L_f^2$, $\Phi_3 = L_gC_f U + L_f V$, $\Phi_4 = L_g^5L_f^2C_f + L_g^4 L_f^3$, $\EX_A \|u_0- g_S(x_0)\|^2 \leq U$, and $\EX_A \|v_0-\nabla g_S(x_0)\|^2 \leq V$.

\textnormal{($K$-level).} For the problem in \eqref{eq:sco-multi}, by selecting  $\eta_t=\eta$ and $\beta_{t} =\beta < \max{\Big(\frac{1}{8K\sum_{i=1}^K(2L_f^2)^{i}}, 1\Big)}$, then it holds
\begin{equation*}
	O\Big(\frac{D_x}{\eta T} + \Phi_5 \beta^{\frac{1}{2}} + L_f^K\eta + \Phi_6 (T \beta )^{-c}\beta^{-\frac{1}{2}} + \frac{\Phi_7\eta^2}{\beta^{3/2}} \Big).
\end{equation*}
where $\Phi_5 =  L_f^{m}(\sigma_{f}^{2} + \sigma_{J}^{2} + \sigma_{f}^2\sum_{i=1}^{K}L_f^{2i} +D_x ) +  D_x$, $\Phi_6 = L_f^{m}(\sum_{i=1}^K U_{i}+V_{i})$, $\Phi_7 = L_f^{m} \sum_{i=1}^{K}L_f^{2i}$, $L_f^m = \max(L_f^{K-j+\frac{(i-1)i}{2}}, L_f^{K+\frac{(i-3)i}{2}})$ for any $i,j \in [1,K]$, \(\EX_A \|u_1^{(i)}- f_{i,S}(u_0^{(i-1)})\|^2 
 \leq U_i\), and \(\EX_A \|v_1^{(i)}-\nabla f_{i,S}(u_{0}^{(i-1)})\|^2 \leq V_{i}\), $\forall i \in [1,K]$. 
\end{theorem}


\begin{remark}
In Theorem~\ref{thm:opt_convex}, we can see that various factors affect optimization errors. Note that selecting $\beta_t$ and $\eta_t$ should be tailored to the specific requirements of different problems. In particular, when adjusting $\eta_t$ to minimize the optimization error in one-level optimizations, $\eta_t$ impacts $\frac{D_x}{\eta T}, L_f^K \eta$, and $\frac{L_f^2\eta^2}{\beta^{3/2}}$. Unfortunately, the unknown value of \( L_f \) during training complicates determining the optimal \( \eta \). In addition, each theorem features a term influenced by the first estimation error, i.e., $V(T\beta)^{-c}\beta^{-\frac{1}{2}}, \Phi_3(T\beta)^{-c}\beta^{-\frac{1}{2}}$, and $\Phi_6(T\beta)^{-c}\beta^{-\frac{1}{2}}$, where $V, \Phi_3$, and $\Phi_6$ all include the discrepancy between the estimators and the empirical risk at the first iteration. This suggests that employing a larger batch size to compute the estimators in the first iteration could effectively reduce the optimization error of the algorithm without significantly increasing computational costs.
\end{remark}

By combining Theorems \ref{theorem:general_multi_level}-\ref{thm:sta_convex_mul_level}, we obtain the generalization error. Further, integrating this with the optimization error outlined in Theorem~\ref{thm:opt_convex} allows us to derive the following excess risk bounds.

\begin{theorem}[Excess Risk Bound, Convex]\label{thm:Excess_Risk_Bound_convex}
Let \(A(S)= \frac{1}{T}\sum_{t= 1}^T x_t\) be the solution produced by STORM, COVER, and SVMR in Algorithms~\ref{alg_storm}-\ref{alg_svmr_multi}, respectively.

\textnormal{(One-level).} For the problem in \eqref{eq:sco-single}, by selecting $T \asymp n^{\frac{5}{2}}$, $\eta=T^{-\frac{4}{5}}$, and $\beta=T^{-\frac{4}{5}}$, we can obtain that $$\EX_{S,A}[F(A(S)) - F(x_*)] = O\Big(\frac{1}{{\sqrt{n}}}\Big)$$.

\textnormal{(Two-level).} For the problem in \eqref{eq:sco-two}, by selecting $T \asymp \max (n^{5/2}, m^{5/2})$, $\eta=T^{-\frac{4}{5}}$, and $\beta=T^{-\frac{4}{5}}$, we can obtain that $$\EX_{S,A}[F(A(S)) - F(x_*)] =O\Big(\frac{1}{\sqrt{n}}+ \frac{1}{\sqrt{m}}\Big)$$.

\textnormal{($K$-level).} For the problem in \eqref{eq:sco-multi}, by selecting $T \asymp \max(n_k^{5/2})$, $\forall k\in [1,K]$, $\eta=T^{-\frac{4}{5}}$, and $\beta=T^{-\frac{4}{5}}$, we can obtain that $$\EX_{S,A}[F(A(S)) - F(x_*)] =O\Big(\sum_{k=1}^K \frac{1}{\sqrt{n_k}}\Big)$$.
\end{theorem}



\begin{remark}
Theorem \ref{thm:Excess_Risk_Bound_convex} demonstrates that STORM, by choosing \(T \asymp n^{\frac{5}{2}}\) and appropriately selecting iteration number \(T\) and parameters \(\eta, \beta\), achieves a generalization error rate of \(O(\frac{1}{\sqrt{n}})\) in a convex setting. This is in contrast to SGD, which requires fewer iterations (\(T \asymp n\)) to reach the same bound \cite{hardt2016train}. This difference may be caused by the estimator in STORM, potentially leading to increased generalization error and excess risk due to reduced algorithm stability. This contrast is further highlighted when comparing with Theorem \ref{thm:Excess_Risk_Bound_convex}, where each additional level, denoted as \(K+1\), requires reassessing iterations and selecting the maximum sample size \(T \asymp \max(n_{k}^{5/2})\), \(\forall k\in [1,K+1]\), which results in an incremental excess risk increase of \(O(\frac{1}{\sqrt{n_{K+1}}})\) with each level while maintaining constant settings for \(\eta\) and \(\beta\) relative to \(T\).
\end{remark}

\begin{remark}
It should be noted that in Theorems~\ref{thm:sta_convex_single_level}-\ref{thm:sta_convex_mul_level}, we discuss the stability of the final iterate \(A(S)= x_T\). Conversely, in Theorem~\ref{thm:opt_convex}, we address the generalization bound of \(A(S)= \frac{1}{T}\sum_{t= 1}^T x_t\), representing the average of the intermediate iterates \(x_1, \ldots, x_T\). This distinction arises from the understanding that generalization encompasses both stability and optimization. In the convex setting, the primary focus of optimization is often on the average of intermediate iterates, as exemplified in sources such as \citep{wang2017stochastic, yang2023stability}. 
\end{remark}

\subsection{The Strongly Convex Setting}\label{section:strong_convex}
Note that we follow a similar process in the convex setting to analyze the generalization performance in the strongly convex setting.

\textbf{Stability Results.} The following theorem establishes the uniform Stability in the strongly convex setting. Before proceeding, we assume that Assumptions~\ref{ass:Lipschitz continuous}-\ref{ass:Smoothness and Lipschitz continuous gradient} and Definition~\ref{def:StronglyConvex} apply to $F_S$, which is strongly convex at the corresponding level, as outlined in Section \ref{section:strong_convex}.

\begin{theorem}[One-level, Stability, Strongly Convex] \label{thm:sta_sconvex_single_convex}
Consider STORM in Algorithm \ref{alg_storm} with $\eta_t = \eta \leq \frac{2}{3(L+\mu)}$ and $\beta_t=\beta \in (0, 1)$, $\forall t\in [0,T-1]$. Then, the outputs $A(S) =x_{T}$ at iteration $T$ are uniform stable with 
\begin{equation*}
	\begin{aligned}
		\epsilon = O\Big(\eta\sum_{j=0}^{T-1}(1-\frac{2\eta L\mu}{L+\mu})^{T-j-1}\operatorname{Var}(v_j) + \frac{L_f(L+\mu)}{L\mu n}\Big).
	\end{aligned}
\end{equation*}
\end{theorem}

\begin{theorem}[Two-level, Stability, Strongly Convex] \label{thm:sta_sconvex_two_convex}
Consider COVER in Algorithm~\ref{alg_cover} with $\eta_t = \eta \leq \frac{1}{4L+4\mu}$ and $\beta_t=\beta \in (0, 1)$, $\forall t \in [0,T-1]$ and the output $A(S)  =x_{T}$. Then, the outputs $A(S) =x_{T}$ at iteration $T$ are uniform stable with 
\begin{equation*}
\begin{aligned}
	\epsilon_\nu + \epsilon_\omega &= O\Big(  L_g C_f \eta \sup_S \sum_{j=0}^{T-1}(1-\frac{2L\mu \eta}{L+\mu})^{T-j-1}\operatorname{Var}(u_j) + L_f\eta \sup_S\sum_{j=0}^{T-1}(1-\frac{2L\mu \eta}{L+\mu})^{T-j-1}\operatorname{Var}(v_j)\\
	& + \frac{(L+ \mu)L_gL_f}{L\mu m} + \frac{(L+ \mu)L_gL_f}{L\mu n} + L_g\sigma_f \sqrt{\frac{L+\mu}{L\mu}}\sqrt{\eta}\Big).
\end{aligned}
\end{equation*}
\end{theorem}

\begin{theorem}[$K$-level, Stability, Strongly Convex] \label{thm:sta_mul_s_convex}
Consider SVMR in Algorithm \ref{alg_svmr_multi} with $\eta_t = \eta \leq \frac{2}{(L+\mu)K(K+2)}$ and $\beta_t=\beta \in (0, 1)$, $\forall t \in [0,T-1]$ and the output $A(S)  =x_{T}$. Then, the outputs $A(S) =x_{T}$ at iteration $T$ are uniform stable with 
\begin{equation*}
	\begin{split}
		 \sum_{k=1}^K\epsilon_k &= O\Big(\sum_{s=1}^{T-1} \sum_{i=1}^K \Big(1-\frac{2\eta L\mu}{L+\mu}\Big)^{T-s} \eta L_f^{K + \frac{(i-3)i}{2}}  \operatorname{Var}_{i,s}(v) \\
  & + \sum_{s=1}^{T-1} \sum_{i=1}^{K}\sum_{j=1}^{i-1} \Big(1-\frac{2\eta L\mu}{L+\mu}\Big)^{T-s} \eta L_f^{K-j+\frac{(i-1)i}{2}}\operatorname{Var}_{j,s}(u) + \sum_{k=1}^K\frac{ L_f^{K}(L+\mu)}{L\mu n_k}  \Big).
	\end{split}
\end{equation*}
\end{theorem}

\begin{remark}
Many conclusions from the strongly convex setting align with the convex setting, and we analyze them individually. First, in the one-level stochastic optimization, the stability of SGD is of the order \(O(\frac{1}{\mu n})\) in \cite{hardt2016train}. Compared to SGD, our results include an additional term, \(\eta\sum_{j=0}^{T-1}(1-\frac{2\eta L\mu}{L+\mu})^{T-j-1}\operatorname{Var}(v_j)\), which is the same as in the convex setting. This implies that STORM may also be less stable under the strongly convex setting than SGD. Second, in the two-level scenario, considering \(m=n\), COVER introduces two additional terms. The reasons for these terms are the same as under the convex setting, stemming from the additional estimator used and the empirical variance of the outer function. Lastly, in \(K\)-level optimization, SVMR includes only one additional coefficient, \((1-\frac{2\eta L\mu}{L+\mu})^{T-s}\), due to the strongly convex property.
\end{remark}

\begin{remark}
Note that there are some significant differences in the strongly convex setting compared to the convex setting. Under the strongly-convex setting, each situation includes an item, such as \(\sum_{k=1}^K\frac{L_f^{K}(L+\mu)}{L\mu n_k}\), that is independent of the step size but depends on the sample size used by each layer function. Therefore, in strongly convex settings, achieving satisfactory stability may require more than just selecting an appropriate step size; it becomes imperative to increase the sample size simultaneously to improve stability.
\end{remark}

\textbf{Generalization results.} Let $c$ be an arbitrary constant, the following theorems hold, which aim to show the optimization errors in the strongly convex setting. 

\begin{theorem}[Optimization, Strongly Convex] \label{thm:opt_sconvex} 
Let \(A(S)= (\sum_{t= 1}^T (1- \frac{\mu\eta}{2})^{T- t} x_t) / (\sum_{t= 1}^T (1- \frac{\mu\eta}{2})^{T- t})\) be the solution produced by STORM, COVER, and SVMR in Algorithms~\ref{alg_storm}-\ref{alg_svmr_multi}, respectively. The following results bound the optimization error $\EX[F_S(A(S)) - F_S(x_*^S)]$. 

\textnormal{(One-level).} For the problem in \eqref{eq:sco-single}, by selecting $\eta_t = \eta \leq \frac{2}{3(L+\mu)}$ and $\beta_t=\beta \in (0,1)$, then it holds
\begin{equation*}
    O\Big( \frac{D_x + U\eta}{(\eta T)^{c}}  + L_f^2L\eta + \frac{V}{(\beta T)^{c}}  + \sigma_J^2\beta + \frac{L_f^2\eta^2}{\beta}\Big).
\end{equation*}

\textnormal{(Two-level).} For the problem in \eqref{eq:sco-two}, by selecting $\eta_t=\eta$, and $\beta_t=\beta < \min{\Big(\frac{1}{8C_f^2}, 1\Big)}$, then it holds
\begin{equation*}
	O\Big(\frac{D_x + \Psi_1\eta}{(\eta T)^{c}} + LL_g^2L_f^2\eta + \frac{\Psi_2}{(\beta T)^{c}} + \Psi_3 \beta + \frac{\Psi_3\eta^2}{\beta}\Big),
\end{equation*}
where $\Psi_1 = L_g^2C_f^2U+L_f^2V$, $\Psi_2 =  L_g^2C_f^2\sigma_g^2 + L_f^2\sigma_{g'}^2$, and $\Psi_3 = L_g^6 C_f^2 L_f^2 +L_f^4 L_g^4$.

\textnormal{($K$-level).} For the problem in \eqref{eq:sco-multi}, by selecting $\eta_t=\eta$ and $\beta_{t} =\beta < \max{\Big(\frac{1}{8K\sum_{i=1}^K(2L_f^2)^{i}}, 1\Big)}$, then it holds
\begin{equation*}
    O\Big(\frac{D_x + \Psi_4\eta}{(\eta T)^{c}}+\eta L_f^K + \frac{\Psi_5}{ (\beta T)^{c}} + \Psi_6 \beta + \frac{\Psi_7\eta^2}{\beta} \Big).
\end{equation*}
where $\Psi_4 =  L_f^{m}\sum_{j=1}^{K-1}(U_{i} + V_{i})$,  $\Psi_5= L_f^{m}\sum_{i=1}^K (U_{i} + V_{i})$, $\Psi_6 = L_f^{m}  ( \sigma_{f}^{2} + \sigma_{J}^{2} + \sigma_{f}^2(\sum_{i=1}^{K}( L_f^{2})^{i}))$, and $\Psi_7 = L_f^{m} \sum_{i=1}^{K}(L_f^{2})^{i}$.
\end{theorem}



Now, we come to derive the following excess risk bounds for the strongly convex setting.

\begin{theorem}[Excess Risk Bound, Strongly Convex] \label{thm:Excess_Risk_Bound_s_convex}
Let \(A(S)= (\sum_{t= 1}^T (1- \frac{\mu\eta}{2})^{T- t} x_t) / (\sum_{t= 1}^T (1- \frac{\mu\eta}{2})^{T- t})\) be the solution produced by STORM, COVER, and SVMR in Algorithms~\ref{alg_storm}-\ref{alg_svmr_multi}, respectively.

\textnormal{(One-level).} For the problem in \eqref{eq:sco-single}, by selecting $T \asymp n^{\frac{7}{6}}$, $\eta=T^{-\frac{6}{7}}$, and $\beta=T^{-\frac{6}{7}}$, we can obtain that $\EX_{S,A}[F(A(S)) - F(x_*)] = O\Big(\frac{1}{\sqrt{n}}\Big)$.

\textnormal{(Two-level).} For the problem in \eqref{eq:sco-two}, by selecting $T \asymp \max( n^{7/6}, m^{7/6})$ and $\eta=\beta = T^{-\frac{7}{6}}$, we can obtain that $\EX_{S,A}[F(A(S)) - F(x_*)] = O\Big(\frac{1}{\sqrt{n}} + \frac{1}{\sqrt{m}}\Big)$.

\textnormal{($K$-level).} For the problem in \eqref{eq:sco-multi}, by selecting $T \asymp \max(n_k^{7/6})$, $\forall k \in [1,K]$ and $\eta=\beta = T^{-\frac{7}{6}}$, we can obtain that $\EX_{S,A}[F(A(S)) - F(x_*)] =O\Big(\sum_{k=1}^K \frac{1}{\sqrt{n_k}}\Big)$.
\end{theorem}

\begin{remark}
Theorem \ref{thm:Excess_Risk_Bound_s_convex} demonstrates that, in the case of strong convexity, the generalization error for STORM can attain a rate of $O(\frac{1}{\sqrt{n}})$ by carefully choosing the iteration number $T$, along with constant step sizes $\eta$ and $\beta$. We can find that under the strongly convex setting, we only need iteration $T \asymp n^{\frac{7}{6}}$, however, under the convex setting, we need more iteration $T \asymp n^{\frac{5}{2}}$. Summarizing these three theorems, we can easily discern the relationship between the excess risk bound and the number of levels. This conclusion is very similar to that in the convex setting. Specifically, for each additional level, denoted as \(K+1\), it is necessary to reassess iterations and select the maximum sample size \(T \asymp \max(n_{k}^{7/6})\), $\forall k\in [1,K+1]$. This results in an incremental excess risk increase of \(O(\frac{1}{\sqrt{n_{K+1}}})\) with each level, while \(\eta\) and \(\beta\) remain constant relative to \(T\).
\end{remark}

To make our paper easy to understand, Table~\ref{Tab:Summary} lists all of our theoretical results in Appendix~\ref{sec:algorithm}.

\section{Experiments}
In this section, we carried out a series of experiments using simulated data to validate our theoretical findings, consisting of four separate tests.

First, we examined the performance of STORM versus SGD in fitting a univariate quintic polynomial. We generated 2000 data points based on this polynomial and introduced Gaussian noise with a mean of 0 and variance of 3. The data was divided into a training and testing split of 60/40. Throughout 500 iterations, using a step size of 0.001 and a batch size of 128, we monitored both training and testing losses using the mean squared error metric. Although STORM demonstrated poorer generalization, indicated by a larger discrepancy between training and testing losses, it outperformed SGD in overall loss metrics.

Second, we investigated how varying the number of levels, \(k\), affects generalization error within a two-level optimization framework. We represented our target function as \(F(x) = f(g(\cdot))\), creating two sets of data points, \(S_1\) and \(S_2\), each contaminated with Gaussian noise (mean 0, variance 3). The dataset was split into a 60/40 train-test ratio. The goal was to optimize \(g(\cdot)\) to fit \(S_1\) and \(f(\cdot)\) to fit \(S_2\) using SVMR as the optimizer, with a step size of 0.01, a projection operation \(L_f\) set at 50, and a batch size of 128 over 500 iterations. We recorded the average generalization error during the last 10 iterations while incrementally increasing the level count from 1 to 50. Our results showed a steady rise in generalization error as the number of levels increased, particularly intensifying beyond 35 levels.

\begin{figure}[htp]
\centering
\begin{minipage}{0.49\columnwidth}
  \centering
\includegraphics[width=\columnwidth]{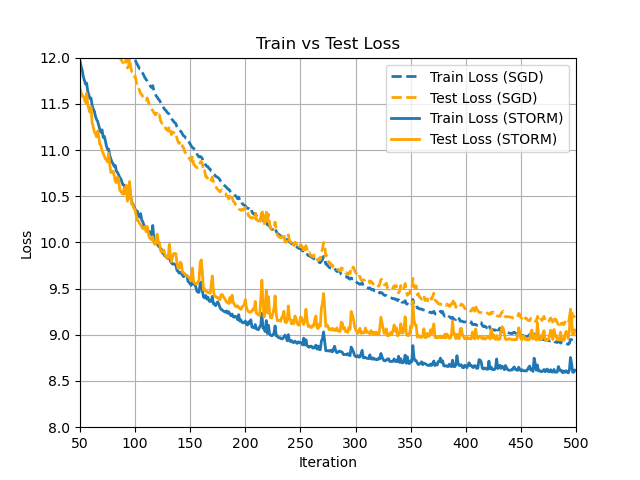}
\caption{SGD VS STORM.}
\label{fig:SGD_VSSTROM}
\end{minipage}%
\hfill
\begin{minipage}{0.49\columnwidth}
  \centering
\includegraphics[width=\columnwidth]{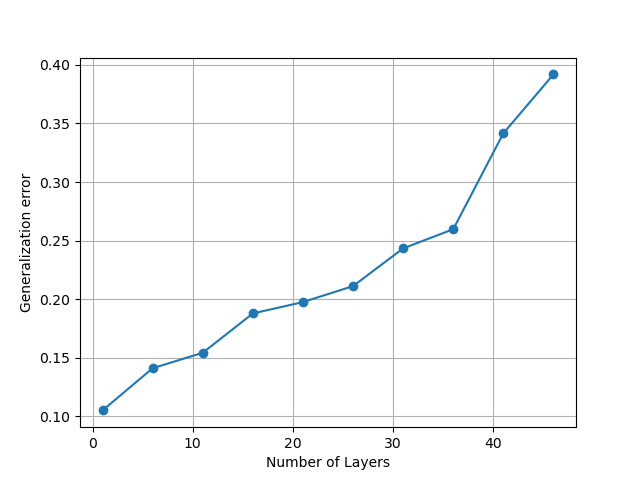}
\caption{Effect of Level.}
\label{fig:k_level_result}
\end{minipage}
\end{figure}

Third, we explored the impact of the initial iteration batch size on generalization. In this experiment, we maintained a fixed number of levels \(k = 10\), with other parameters consistent with above, and varied only the batch size during the first five iterations before stabilizing it at 128. We observed that when the initial batch size is smaller than the standard value of 128, the generalization error is higher than at 128. Conversely, setting the initial batch size to 256 and 512 significantly improved the generalization error. This finding supports our observation that under the same initial conditions, increasing the batch size in the initial few iterations can enhance the generalization performance of SVMR.

Fourth, we investigated the impact of noise on generalization. In this experiment, while keeping the settings consistent with Experiment 2, we set \(k = 10\) and maintained the batch size at 128. However, we varied the variance of Gaussian distribution noise. Specifically, we incrementally increased the Gaussian noise variance from 0.1 to 3 in steps of 0.1 to observe its effects on generalization. Noise can improve generalization by 1) aiding the model in escaping local minima to find lower values, and 2) preventing the model from overfitting the training data. The drawback of noise in terms of generalization is that it challenges an algorithm's stability; excessive noise can compromise this stability, thereby diminishing generalization performance. Our results indicated that when the noise variance does not exceed 1.5, it positively impacts generalization. However, beyond a variance of 1.5, the detrimental effects on algorithm stability outweigh the benefits, leading to poorer generalization outcomes.

\begin{figure}[htp]
\centering
\begin{minipage}{0.49\columnwidth}
  \centering
\includegraphics[width=\textwidth]{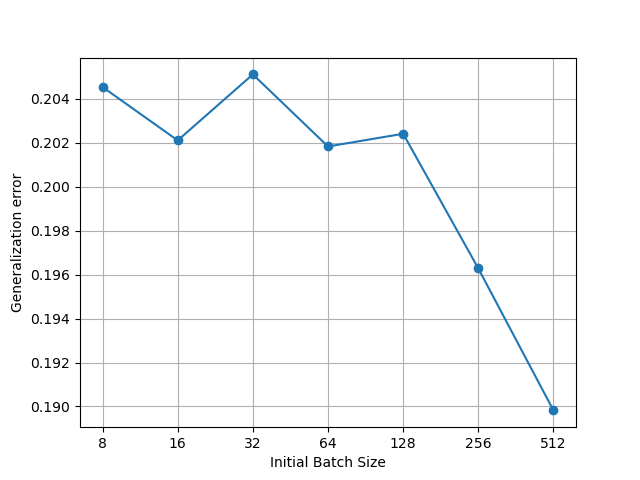}
\caption{Effect of batch size.}
\label{fig:initial_batch_size}
\end{minipage}%
\hfill
\begin{minipage}{0.49\columnwidth}
  \centering
\includegraphics[width=\textwidth]{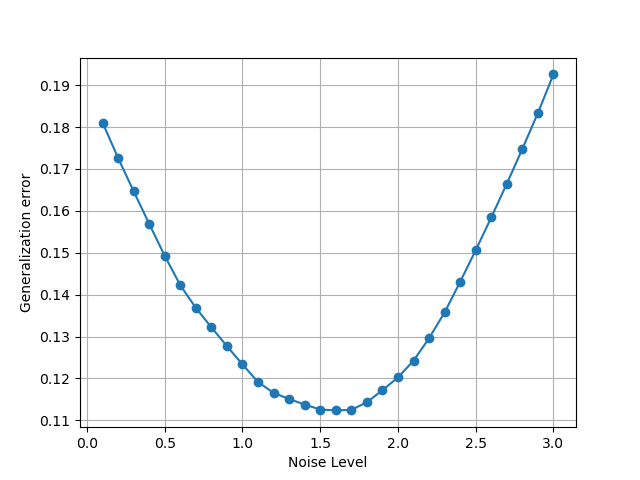}
\caption{Effect of noise.}
\label{fig:noisy_result}
\end{minipage}
\end{figure}

\section{Conclusion}\label{Sec:Conclusion}
This paper conducts a thorough generalization analysis of STORM-based algorithms: STORM, COVER, and SVMR, for one, two, and $K$-level stochastic optimizations. Firstly, for the $K$-level optimization, we introduce a tailored stability notion, paving the way for deeply understanding the relationship between generalization error, stability, and the number of levels. We further investigate their stability and excess risk bounds in both convex and strongly convex settings. Based on our analysis, we have found three observations for STORM-based algorithms: (1) Individual estimators can compromise algorithm stability due to target variances, harming generalization performance. (2) Increasing the number of levels also affects the algorithm's generalization error through stability and gradient variances. (3) Using more initial samples for estimation can boost performance without significantly raising computational costs.

\bibliographystyle{unsrtnat}
\bibliography{template}  

\clearpage
\setcounter{page}{1}
\appendices

\input{algorithms_and_useful_lemmas}
\input{appdenix_theorem_bound_single_level}
\input{appendix_theorem_two_level}

\input{appendix_theorem_multi_level}

\end{document}

%% file: defs.tex
\usepackage{bm}
\usepackage{enumitem,amsthm,amsmath,amsfonts,amssymb}
\usepackage{indentfirst}
\usepackage{float}
\usepackage{mathtools}
\usepackage{changepage}
\usepackage{algorithmic}
\usepackage{algorithm}
\usepackage{xcolor}
\usepackage{graphicx}
\usepackage{tikz}
\usepackage{multirow}
\usepackage{multicol}
\usepackage{hyperref}
\usepackage{mathrsfs, fancyhdr}
\usepackage{amscd}
\usepackage{bbding}
\usepackage{pifont}
\usepackage{booktabs}

\newcommand{\F}{\mathcal{F}}

\def\gep{\epsilon}

\def\X{\mathcal{X}}

\def\EX{\mathbb{E}}

\def\gep{\epsilon}

\def\0{\mathbf{0}}

\def\iv{\dot{v}}

\def\1{\textbf{1}}

\def\begeqn{\begin{equation}}
\def\endeqn{\end{equation}}
\def\begth{\begin{theorem}}
\def\endth{\end{theorem}}
\def\begprop{\begin{proposition}}
\def\endprop{\end{proposition}}
\def\begcor{\begin{corollary}}
\def\endcor{\end{corollary}}
\def\begdef{\begin{definition}}
\def\enddef{\end{definition}}
\def\beglemm{\begin{lemma}}
\def\endlemm{\end{lemma}}
\def\begexm{\begin{example}}
\def\endexm{\end{example}}
\def\begrem{\begin{remark}}
\def\endrem{\end{remark}}

\def\begdef{\begin{definition}}
\def\enddef{\end{definition}}

\def\EX{\mathbb{E}}

\def\X{\mathcal{X}}

%% file: algorithms_and_useful_lemmas.tex
\section{Results Summary and Corresponding Algorithms}\label{sec:algorithm}

\subsection{Summary of Results}
\begin{table*}[htp]
\centering
\caption{Summary of our results.}
\resizebox{\textwidth}{!}{%
\begin{tabular}{c|c|c|c|clll}
\cline{1-6}
Setting           & Bound                        & Level & Reference                                                  & \multicolumn{2}{c}{Result} &  &  \\ \cline{1-6}
\multirow{3}{*}{} &
  \multirow{3}{*}{Generation} &
  $1$ &
  \cite{hardt2016train} &
  \multicolumn{2}{c}{$L_f\epsilon$} &
   &
   \\\cline{3-6}
                  &                              & $2$     &   \cite{yang2023stability}                                       & \multicolumn{2}{c}{$L_{f}^2 \epsilon_{2}+4 L_{f}^2 \epsilon_{1}+L_{f} \sqrt{ \mathbb{E}_{S, A}[\operatorname{Var}_{1}(A(S))]/n_1}$}       &  &  \\\cline{3-6}
                  &                              & $K$     &  Theorem \ref{theorem:general_multi_level}                                                           & \multicolumn{2}{c}{$L_f^K\epsilon_K + \sum_{k=1}^{K-1}\Big( 4L_f^K\epsilon_{k}+ L_f\sqrt{\mathbb{E}_{S, A} [\operatorname{Var}_{k}(A(S)]/n_{k}}\Big)$}       &  &  \\ \cline{1-6}
\multirow{6}{*}{C} &
  \multirow{3}{*}{Stability} &
  $1$ &
  Theorem \ref{thm:sta_convex_single_level} &
  \multicolumn{2}{c}{$O\Big( \eta \sum_{j=0}^{T-1}\operatorname{Var}(v_j) + \frac{L_f\eta T}{n}\Big)$} &
   &
   \\\cline{3-6}
                  &                              & $2$     & Theorem \ref{thm:sta_convex_two_level} & \multicolumn{2}{c}{$O\Big( \eta\sum_{j=0}^{T-1}(  \operatorname{Var}(u_j) + \operatorname{Var}(v_j)) + \eta \sqrt{T} + \frac{\eta T}{m} +\frac{\eta T}{n}  \Big)$}       &  &  \\\cline{3-6}
                  &                              & $K$     & Theorem \ref{thm:sta_convex_mul_level} & \multicolumn{2}{c}{$O\Big( \tilde{L}_f^{K,i}\sum_{j=1}^{i-1}  L_f^{i-j} \operatorname{Var}^T (u, v)  + \sum_{k=1}^K\frac{\eta L_f^{K} T}{n_k}\Big)$}       &  &  \\ \cline{2-6}
                  & \multirow{3}{*}{Excess Risk} & $1$     &  Theorem \ref{thm:Excess_Risk_Bound_convex}                                                          & \multicolumn{2}{c}{$O(\frac{1}{\sqrt{n}})$, $~T \asymp n^{5/2}$ }       &  &  \\\cline{3-6}
                  &                              & $2$     & Theorem \ref{thm:Excess_Risk_Bound_convex}                                                            & \multicolumn{2}{c}{$O\Big(\frac{1}{\sqrt{n}} + \frac{1}{\sqrt{m}}\Big)$, $~T \asymp \max(n^{5/2}, m^{5/2})$}       &  &  \\\cline{3-6}
                  &                              & $K$     & Theorem \ref{thm:Excess_Risk_Bound_convex}                                                            & \multicolumn{2}{c}{$O\Big(\sum_{k=1}^K\frac{1}{\sqrt{n_k}}\Big)$, $~T \asymp \max(n_k^{5/2})$, $~\forall k \in [1, K]$}       &  &  \\ \cline{1-6}
\multirow{6}{*}{SC} & \multirow{3}{*}{Stability}   & $1$     &                  Theorem \ref{thm:sta_sconvex_single_convex}                                          & \multicolumn{2}{c}{$O\Big(\eta\sum_{j=0}^{T-1}\tilde{L}^{T-j-1}\operatorname{Var}(v_j) + \frac{L_f(L+\mu)}{L\mu n}\Big)$}       &  &  \\\cline{3-6}
                  &                              & $2$     & Theorem \ref{thm:sta_sconvex_two_convex}                                                           & \multicolumn{2}{c}{$O\Big(  \eta \sum_{j=0}^{T-1}\tilde{L}^{T-j-1} (\operatorname{Var}(u_j) + \operatorname{Var}(v_j)) + \frac{(L+ \mu)L_gL_f}{L\mu m} + \frac{(L+ \mu)L_gL_f}{L\mu n}\Big)$}       &  &  \\\cline{3-6}
                  &                              & $K$     &  Theorem \ref{thm:sta_mul_s_convex}                                                           & \multicolumn{2}{c}{ $O\Big(\eta \sum_{s=1}^{T-1} \tilde{L}^{T-s} \tilde{L}_f^{K,i} \sum_{j=1}^{i-1}  L_f^{i-j}\operatorname{Var}^T (u,v)+ \sum_{k=1}^K\frac{ L_f^{K}(L+\mu)}{L\mu n_k}  \Big)$ }       &  &  \\ \cline{2-6}
                  & \multirow{3}{*}{Excess Risk} & $1$     & Theorem \ref{thm:Excess_Risk_Bound_s_convex}                                                           & \multicolumn{2}{c}{$O\Big(\frac{1}{\sqrt{n}}\Big)$, $~T \asymp n^{7/6}$}       &  &  \\\cline{3-6}
                  &                              & $2$     & Theorem \ref{thm:Excess_Risk_Bound_s_convex}                                                            & \multicolumn{2}{c}{$O\Big(\frac{1}{\sqrt{n}} + \frac{1}{\sqrt{m}})$, $~T \asymp \max(n^{7/6}, m^{7/6})$}       &  &  \\\cline{3-6}
                  &                              & $K$     & Theorem \ref{thm:Excess_Risk_Bound_s_convex}                                                            & \multicolumn{2}{c}{$O\Big(\sum_{k=1}^K\frac{1}{\sqrt{n_k}}\Big)$, $~T \asymp \max(n_k^{5/2})$, $~\forall k \in [1, K]$}       &  &  \\ \cline{1-6}
                  \multicolumn{6}{l}{We use the following parameters to simplify the notations: $\operatorname{Var}^T(u, v) = \sum_{s=1}^{T-1}(\operatorname{Var}_{j,s}(u) + \operatorname{Var}_{i,s}(v)), \tilde{L} = (1-\frac{2\eta L \mu}{L+\mu})$,} \\
                  \multicolumn{6}{l}{and $\tilde{L}_f^{K,i} = \sum_{i=1}^{K}L_f^{K+\frac{(i-3)i}{2}}$}
\end{tabular}%
}
\label{Tab:Summary}
\end{table*}

\subsection{Description of Algorithms}
\begin{algorithm}[htp]
\caption{STORM.}	\label{alg_storm}
\renewcommand{\algorithmicrequire}{\textbf{Inputs:}}
\renewcommand{\algorithmicensure}{\textbf{Output:}}
\begin{algorithmic}[1]
    \REQUIRE Training data $S = \{\nu_i: i =1,\cdots,n\}$; Number of iterations $T$; Parameter $\eta_t$, $\beta_t$
    \STATE Initialize $x_0 \in \X$, $v_0 \in \mathbb{R}^d$
    \STATE Draw a sample $j_{0} \in [1, n]$, obtain $\nabla f _{\nu_{j_{0}}}(x_0)$.
    \FOR{$t=0$ to $T-1$ }
    \STATE $x_{t+1} = x_t - \eta_t v_t$
    \STATE Draw a sample $j_{t+1} \in [1, n]$, obtain $\nabla f _{\nu_{j_{t+1}}}(x_t) $
    \STATE Compute estimators $v_{t+1} = \Pi_{L_f}[\nabla f _{\nu_{j_{t+1}}}(x_{t+1})  +   (1-\beta_{t+1})(v_t - \nabla f_{\nu_{j_{t+1}}}(x_t))] $
    \ENDFOR
    \STATE {\bf Outputs:} $A(S)  =x_{T}$ or $x_\tau \sim \texttt{Unif}(\{x_t\}_{t=1}^T)$
\end{algorithmic}
\end{algorithm}

\begin{algorithm}[htp]
	\caption{COVER.}	\label{alg_cover}
	\renewcommand{\algorithmicrequire}{\textbf{Inputs:}}
	\renewcommand{\algorithmicensure}{\textbf{Output:}}
	\begin{algorithmic}[1]
		\REQUIRE Training data $S_\nu = \{\nu_i: i =1,\cdots,n\}$, $S_\omega = \{\omega_j: j =1,\cdots,m\}$; Number of iterations $T$, Parameter $\eta_t$, $\beta_t$
		\STATE Initialize $x_0 \in \X$, $u_0, v_0 \in \mathbb{R}^d$
		\STATE Draw a sample $j_{0}\in [1, n]$ and $,i_{0} \in [1,m]$, obtain $\nabla g _{\omega_{j_{0}}}(x_0)$ and $\nabla f_{\nu_{i_{0}}}(u_0) $.
		\FOR{$t=0$ to $T-1$ }
		\STATE $x_{t+1} = x_t - \eta_t v_t \nabla f_{\nu_{i_t}}(u_t)$
		\STATE Draw a sample $j_{t+1} \in [1, m]$, obtain $g _{\omega_{j_{t+1}}}(x_{t+1})$ and $ g _{\omega_{j_{t+1}}}(x_{t})$
		\STATE Compute estimators $u_{t+1} = g _{\omega_{j_{t+1}}}(x_{t+1}) + (1-\beta_{t+1})(u_t - g_{\omega_{j_{t+1}}}(x_t))$
		\STATE Draw a sample $j_{t+1} \in [1, m]$ , obtain $\nabla g _{\omega_{j_{t+1}}}(x_{t+1})$ and $\nabla g _{\omega_{j_{t+1}}}(x_t)$
		\STATE Compute estimators $v_{t+1} = \Pi_{L_f}[\nabla g_{\omega_{j_{t+1}}}(x_{t+1})  +   (1-\beta_{t+1})(v_t - \nabla g_{\omega_{j_{t+1}}}(x_{t})]$
		\STATE Draw  samples $i_{t+1} \in [1, n]$,          obtain $\nabla f_{\nu_{i_{t+1}}}(u_{t+1})$
		\ENDFOR
		\STATE {\bf Outputs:} $A(S)  =x_{T}$ or $x_\tau \sim \texttt{Unif}(\{x_t\}_{t=1}^T)$
	\end{algorithmic}
\end{algorithm}

\begin{algorithm}[htp]
	\caption{SVMR.}	\label{alg_svmr_multi}
	\renewcommand{\algorithmicrequire}{\textbf{Inputs:}}
	\renewcommand{\algorithmicensure}{\textbf{Output:}}
	\begin{algorithmic}[1]
		\REQUIRE Training data $S = \{\nu^{(1)}_1,\cdots, \nu^{(1)}_{n_1}, \cdots, \nu^{(K)}_1,\cdots, \nu^{(K)}_{n_K} \}.$; Number of iterations $T$; Parameter $\eta_t$, $\beta_t$
		\STATE Initialize $x_0 \in \X$, $u_0^{(i)}, v_{0}^{(i)} \in \mathbb{R}^d$ for all $i \in [0,K]$
		\STATE Draw a sample $j_{0}\in [1, n]$ and $,i_{0} \in [1,m]$, obtain $\nabla g _{\omega_{j_{0}}}(x_0)$ and $\nabla f_{\nu_{i_{0}}}(u_0) $.
		\FOR{$t=0$ to $T-1$ }
		\STATE $x_{t+1} = x_t - \eta_t \prod_{i=1}^Kv_{t}^{(i)}$ and set $u_t^{(0)} = x_t$
		\FOR{ level $i=1$ to $K$}
		\STATE Draw a sample $\nu_{t+1}^{(i)} \in [1, n_i]$, obtain $f_{\nu_{t+1}^{(i)}}(u_{t+1}^{(i-1)})$, $f_{\nu_{t+1}^{(i)}}(u_{t}^{(i-1)})$, $\nabla f_{\nu_{t+1}^{(i)}}(u_{t+1}^{(i-1)})$ and $\nabla f_{\nu_{t+1}^{(i)}}(u_{t}^{(i-1)})$
		\STATE Compute estimators $u_{t+1}^{(i)} = f_{\nu_{t+1}^{(i)}}(u_{t+1}^{(i-1)}) + (1-\beta_{t+1})(u_t^{(i)} - f_{\nu_{t+1}^{(i)}}(u_{t}^{(i-1)}))$
		\STATE Compute estimators $v_{t+1}^{(i)} = \Pi_{L_f}[\nabla f_{\nu_{t+1}^{(i)}}(u_{t+1}^{(i-1)}) + (1-\beta_{t+1})(u_t^{(i)} - \nabla f_{\nu_{t+1}^{(i)}}(u_{t}^{(i-1)}))]$
		\ENDFOR
		\ENDFOR
		\STATE {\bf Outputs:} $A(S)  =x_{T}$ or $x_\tau \sim \texttt{Unif}(\{x_t\}_{t=1}^T)$
	\end{algorithmic}
\end{algorithm}

\newpage

\section{Useful Lemmas}
lxchen@sjtu.edu.cnBefore giving the detailed proof, we first give some useful lemmas. 
\begin{lemma}[Lemma 4 in \cite{yang2023stability}]\label{lemma:weighted_avg}
Let \(\{a_i\}_{i= 1}^T, \{b_i\}_{i= 1}^T\) be two sequences of positive real numbers such that \(a_i\leq a_{i+ 1}\) and \(b_i\geq b_{i+ 1}\) for all \(i\). Then we have $\frac{\sum_{i= 1}^T a_ib_i}{\sum_{i= 1}^T a_i}\leq \frac{\sum_{i= 1}^T b_i}{T}$.
\end{lemma}

\begin{lemma}\label{lemma:general_recursive}
	Consider a sequence $\{\beta_t\}_{t \geq 0} \in (0,1]$ and define $\Upsilon_t = \prod_{i=1}^{t}(1-\beta_i)$, then we can get for any $q_t \leq (1-\beta_t)q_{t-1}+p_{t}$, $
		q_t \leq \Upsilon_t(q_0+ \sum_{i=1}^t\frac{p_i}{\Upsilon_i})$.
\end{lemma}

\begin{proof}
	We divide both side of $q_t \leq (1-\beta_t)q_{t-1}+p_{t}$ by $\Upsilon_t$, then we have $
	\frac{q_t}{\Upsilon_t} \leq \frac{q_{t-1}}{\Upsilon_{t-1}} + \frac{p_t}{\Upsilon_t},~ t \geq 1$.
	Summing up the above inequalities, we have $
		q_t \leq \Upsilon_k(q_0+ \sum_{i=1}^t\frac{p_i}{\Upsilon_i})$.
\end{proof}

\begin{lemma}[Lemma 2 in \cite{yang2023stability}] \label{lemma:recursion lemma}
	Assume that the non-negative sequence ${u_t: t\in\mathbb{N} }$ satisfies the following recursive inequality for all $t \in \mathbb{N}$,
	\begin{align*}
		u_t^2 \leq S_t+\sum_{\tau=1}^{t-1} \alpha_\tau u_\tau.
	\end{align*}
	where $\{S_\tau: \tau \in \mathbb{N}\}$ is an increasing sequence, $S_0 \geq u_0^2$ and $\alpha_\tau$ for any $\tau \in \mathbb{N}. $ Then, the following inequality holds true:
	\begin{align*}
		u_t \leq \sqrt{S_t}+\sum_{\tau=1}^{t-1} \alpha_\tau.
	\end{align*}
\end{lemma}

%% file: appdenix_theorem_bound_single_level.tex
\section{One-level Stochastic Optimizations}\label{App:Single}

\begin{lemma}[Theorem 3.7 in \cite{hardt2016train}]\label{theorem:general_single_level}
If Assumption \ref{ass:Lipschitz continuous}(i), \ref{ass:bound variance} (i) and \ref{ass:Smoothness and Lipschitz continuous gradient} (i) holds true and the randomized algorithm $A$ is $\epsilon$-uniformly stable then
\begin{equation*}
\begin{aligned}
	&\EX_{S,A}[F(A(S)) - F_S(A(S))]\leq L_f\epsilon.
\end{aligned}
\end{equation*}
\end{lemma}

\begin{lemma}[Lemma 2 in \cite{cutkosky2019momentum}]\label{lemma:v_t_recursive_single}
Let Assumption \ref{ass:Lipschitz continuous}(i), \ref{ass:bound variance} (i) and \ref{ass:Smoothness and Lipschitz continuous gradient} (i) holds hold for the empirical risk $F_S$ , and $x_t,v_t$ is generated by Algorithm \ref{alg_storm}, then we have 
\begin{equation*}
    \EX_A [\|v_t -  \nabla f_S(x_t)\|^2 | \F_t] \leq  (1-\beta_{t})\|v_{t-1} - \nabla f_S(x_{t-1})\|^2 | + 2\beta_t^2\sigma_J^2 + 2L_f^2\|x_t-x_{t-1}\|^2.
\end{equation*}
\end{lemma}

\begin{lemma}\label{lemma:v_t_bound_single}
Let Assumption \ref{ass:Lipschitz continuous}(i), \ref{ass:bound variance} (i) and \ref{ass:Smoothness and Lipschitz continuous gradient} (i) holds hold for the empirical risk $F_S$ , and $x_t,v_t$ is generated by Algorithm \ref{alg_storm}, then for any  $c>0$, we have 
\begin{equation*}
    \EX_A [\|v_t -  \nabla f_S(x_t)\|^2] \leq (\frac{c}{e})^{c}(t\beta)^{-c}\EX_A[\| v_0 - \nabla f_S(x_0)\|^2] +  2\beta\sigma_J^2 + \frac{L_f^2\eta^2}{\beta}.
\end{equation*}
\end{lemma}

\begin{proof}[proof of lemma \ref{lemma:v_t_bound_single}]
According to Lemma \ref{lemma:v_t_recursive_single}, and note that $\EX_{A}[\|x_t-x_{t-1}\|^2] \leq L_f^2\eta_{t-1}^2$ we have 
\begin{equation*}
    \EX_A [\|v_t -  \nabla f_S(x_t)\|^2]  \leq (1-\beta_t)\EX_{A}[\|v_{t-1} -\nabla f_S(x_{t-1})\|^2]+ 2\beta_t^2\sigma_J^2 + L_f^2\eta_{t-1}^2.
\end{equation*}
Telescoping the above inequality from 1 to $t$, according to Lemma \ref{lemma:general_recursive}, we have
\begin{equation*}
\begin{aligned}
    \EX_A [\|v_t -  \nabla f_S(x_t)\|^2] &\leq \prod_{j=1}^t(1-\beta_j)\EX_A[\| v_0 - \nabla f_S(x_0)\|^2] + \prod_{j=1}^t(1-\beta_j)(\sum_{j=1}^t\frac{2\beta_j\sigma_J^2}{\prod_{i=1}^j(1-\beta_i)}) \\
    &+ \prod_{j=1}^t(1-\beta_j)(\sum_{j=1}^t\frac{L_f^2\eta_{j-1}^2}{\prod_{i=1}^j(1-\beta_i)}).
\end{aligned}
\end{equation*}
Setting $\beta_t =\beta$ and $\eta_t = \eta$, we have 
\begin{equation*}
    \EX_A [\|v_t -  \nabla f_S(x_t)\|^2] \leq \prod_{j=1}^t(1-\beta_j)\EX_A[\| v_0 - \nabla f_S(x_0)\|^2] + \sum_{j=1}^t(1-\beta)^{t-j}( 2\beta^2\sigma_J^2 + L_f^2\eta^2).
\end{equation*}
Note that for all $K\leq N$ and $\beta_i >0$, we have 
\begin{equation}\label{Eq:prod_(1-beta_i)}
    \prod_{i=K}^{N}(1-\beta_i) \leq \exp(-\sum_{i=K}^N\beta_i),
\end{equation}
then we have
\begin{equation*}
    \EX_A [\|v_t -  \nabla f_S(x_t)\|^2] \leq \exp(-\beta t)\EX_A[\| v_0 - \nabla f_S(x_0)\|^2] + \sum_{j=1}^t(1-\beta)^{t-j}( 2\beta^2\sigma_J^2 + L_f^2\eta^2).
\end{equation*}
According to the fact that for any $c >0$, we have 
\begin{equation}\label{Eq:exp(-t beta)}
    e^{-x} \leq (\frac{c}{e})^{c}x^{-c},
\end{equation}
then we can get for any $c >0$
\begin{equation*}
    \EX_A [\|v_t -  \nabla f_S(x_t)\|^2] \leq (\frac{c}{e})^{c}(t\beta)^{-c}\EX_A[\| v_0 - \nabla f_S(x_0)\|^2] + \sum_{j=1}^t(1-\beta)^{t-j}( 2\beta^2\sigma_J^2 + L_f^2\eta^2).
\end{equation*}
Moreover, according to the fact that 
\begin{equation}\label{Eq:sum_(1-a)^t-j}
    \sum_{j=1}^t(1-\beta)^{t-j} \leq \frac{1}{\beta},
\end{equation}
we have $\EX_A [\|v_t -  \nabla f_S(x_t)\|^2] \leq (\frac{c}{e})^{c}(t\beta)^{-c}\EX_A[\| v_0 - \nabla f_S(x_0)\|^2] +  2\beta\sigma_J^2 + \frac{L_f^2\eta^2}{\beta}$.
\end{proof}

We first give some notations used in the one-level optimization to simplify our proof.

For any $k \in [1,n]$, let $S^{k} = \{\nu_{1}, \ldots, \nu_{k-1}, \nu_{k}^{\prime}, \nu_{k+1}, \ldots, \nu_{n}\}$ be formed from $S$ by replacing the $k$-th element.

Let $\{x_{t+1}\}$, and $\{v_{t+1}\}$ be generated by Algorithm \ref{alg_storm} based on $S$. Similarly, $\{x_{t+1}^{k}\}$  and $\{v_{t+1}^{k}\}$ be generated by Algorithm \ref{alg_storm} based on $S^{k}$. Set $x_{0}=x_{0}^{k}$  as starting points in $\mathcal{X}$.

Next, we give the detailed proof of Theorem \ref{thm:sta_convex_single_level}.
\begin{proof}[proof of Theorem \ref{thm:sta_convex_single_level} ]
We will consider two cases, i.e., $i_t \neq k$ and $i_t = k$.

\textbf{\quad Case 1 ($i_t \neq k$).}  We have 
\begin{equation}\label{Eq:sta_origin_single_c}
\begin{aligned}
    \|x_{t+1} - x_{t+1}^{k}\|^2 & = \|x_t - \eta_t v_t - x_t^k + \eta_t v_t^k\|^2 \\
    & \leq \|x_t - x_t^k\|^2 -2\eta_t \langle v_t - v_t^k , x_t -x_t^k\rangle + \eta_t^2\|v_t - v_t^k\|^2.
\end{aligned}
\end{equation}
For the second term on the RHS of \eqref{Eq:sta_origin_single_c}, we have 
\begin{equation*}
\begin{aligned}
    &-2\eta_t \langle v_t - v_t^k , x_t -x_t^k\rangle\\
    & = -2\eta_t \langle v_t - \nabla f_S(x_t), x_t -x_t^k\rangle - 2\eta_t \langle \nabla f_S(x_t) -  \nabla f_S(x_t^k), x_t -x_t^k\rangle -2\eta_t \langle \nabla f_S(x_t^k)-  v_t^k , x_t -x_t^k\rangle.
\end{aligned}
\end{equation*}
 
Smoothness generally suggests that the gradient update of $F$ is constrained from being excessively large. Additionally, the convexity and $L$-smoothness of $F$ indicate co-coercivity in the gradients, leading to the following conclusion $$\left \langle \nabla F\bigl(x\bigr)-\nabla F\bigl(x'\bigr), x-x'\right \rangle\ge \frac{1}{L}\|\nabla F\bigl(x\bigr)-\nabla F\bigl(x'\bigr)\|^2.$$

Then using Assumption \ref{ass:Smoothness and Lipschitz continuous gradient} (i), i.e., the smoothness of $f_S(\cdot)$, we can get
 
\begin{equation*}
\begin{aligned}
    &-2\eta_t \langle v_t - v_t^k , x_t -x_t^k\rangle\\
    & \leq 2\eta_t \|v_t -\nabla f_S(x_t)\| \cdot \|x_t - x_t^k \| -\frac{2\eta_t}{L}\|\nabla f_S(x_t) -  \nabla f_S(x_t^k)\|^2 + 2\eta_t \|v_t^k-f_S(x_t^k)\|\cdot \|x_t - x_t^k \|.
\end{aligned}
\end{equation*}
For the third term on the RHS of \eqref{Eq:sta_origin_single_c}, we have 
\begin{equation*}
\begin{aligned}
    \eta_t^2\|v_t - v_t^k\|^2 &\leq 3\eta_t^2 \|v_t -\nabla f_S(x_t)\|^2 + 3\eta_t^2\|\nabla f_S(x_t) -  \nabla f_S(x_t^k)\|^2 + 3\eta_t^2\|v_t^k-f_S(x_t^k)\|^2.
\end{aligned}
\end{equation*}
Putting above two inequalities into  \eqref{Eq:sta_origin_single_c}, we have 
\begin{equation*}
\begin{aligned}
    \|x_{t+1} - x_{t+1}^{k}\|^2
    &\leq \|x_t - x_t^k\|^2 +  2\eta_t \|v_t -\nabla f_S(x_t)\| \cdot \|x_t - x_t^k \| + 2\eta_t \|v_t^k-f_S(x_t^k)\|\cdot \|x_t - x_t^k \|\\
    &\quad + (3\eta_t^2 - \frac{2\eta_t}{L})\|\nabla f_S(x_t) -  \nabla f_S(x_t^k)\|^2+  3\eta_t^2\|v_t^k-f_S(x_t^k)\|^2.
\end{aligned}
\end{equation*}
By setting $\eta_t \leq \frac{2}{3L}$, we have 
\begin{equation*}
\begin{aligned}
    &\|x_{t+1} - x_{t+1}^{k}\|^2\\
    & \leq \|x_t - x_t^k\|^2 +  2\eta_t \|v_t -\nabla f_S(x_t)\| \cdot \|x_t - x_t^k \| + 2\eta_t \|v_t^k-f_S(x_t^k)\|\cdot \|x_t - x_t^k \|+  3\eta_t^2\|v_t^k-f_S(x_t^k)\|^2.
\end{aligned}
\end{equation*} 

\textbf{\quad Case 2 ($i_t = k$).}  We have
\begin{equation*}
\begin{aligned}
    \|x_{t+1} - x_{t+1}^{k}\| & = \|x_t - \eta_t v_t - x_t^k + \eta_t v_t^k\|\\
    & \leq \|x_t- x_t^k\| +\eta_t \|v_t -v_t^k\| \leq \|x_t- x_t^k\| +\eta_t L_f.
\end{aligned}
\end{equation*} 
Then we can get 
\begin{equation*}
\begin{aligned}
    &\|x_{t+1} - x_{t+1}^{k}\|^2 \leq \|x_t- x_t^k\|^2 + 2\eta_t L_f \|x_t- x_t^k\| + \eta_t^2L_f^2.
\end{aligned}
\end{equation*}
Combining \textbf{Case 1} and \textbf{Case 2} we have
\begin{equation*}
\begin{aligned}
    \|x_{t+1} - x_{t+1}^{k}\|^2 & \leq  \|x_t - x_t^k\|^2 +  2\eta_t \|v_t -\nabla f_S(x_t)\| \cdot \|x_t - x_t^k \| + 2\eta_t \|v_t^k-f_S(x_t^k)\|\cdot \|x_t - x_t^k \| \\
    &\quad + 3\eta_t^2\|v_t^k-f_S(x_t^k)\|^2 + 2\eta_t L_f \|x_t- x_t^k\| \1_{i_t = k} + \eta_t^2L_f^2\1_{i_t = k}.
\end{aligned}
\end{equation*}
Note that 
\begin{equation}\label{Eq:single_sc_1/n_x_t-x_t^k}
\begin{aligned}
    \mathbb{E}_{A}[\|x_{t}-x_{t}^{k}\| \1_{[i_{t}=k]}]=\mathbb{E}_{A}[\|x_{t}-x_{t}^{k}\| \1_{[i_{t}=k]}]=\frac{1}{n} \mathbb{E}_{A}[\|x_{t}-x_{t}^{k}\|] \leq \frac{1}{n}(\mathbb{E}_{A}[\|x_{t}-x_{t}^{k}\|^{2}])^{1 / 2}.
\end{aligned}
\end{equation}
Then using Cauchy-Schwarz inequality,  we can get 
\begin{equation*}
\begin{aligned}
    \mathbb{E}_{A}[\|x_{t+1} - x_{t+1}^{k}\|^2] & \leq 	\mathbb{E}_{A}[\|x_t - x_t^k\|^2] + 2\eta_t (\EX_{A}[\|v_t -\nabla f_S(x_t)\|^2 ])^{1/2}(\EX_{A}[\|x_t -x_t^k\|^2])^{1/2}\\
    &\quad  + 2\eta_t (\EX_{A}[\|v_t^k -\nabla f_S(x_t^k)\|^2 ])^{1/2}(\EX_{A}[\|x_t -x_t^k\|^2])^{1/2}+ 3\eta_t^2\EX_{A}[\|v_t^k-f_S(x_t^k)\|^2]\\
    &\quad + \frac{2L_f \eta_t}{n}(\mathbb{E}_{A}[\|x_{t}-x_{t}^{k}\|^{2}])^{1 / 2} + \frac{\eta_t^2L_f^2}{n}.
\end{aligned}
\end{equation*}
Telescoping the above inequality from 0 to $t$, and combining with $x_0 = x_0^k$, we have 
\begin{equation*}
\begin{aligned}
    \mathbb{E}_{A}[\|x_{t+1} - x_{t+1}^{k}\|^2] & \leq 2 \sum_{j=1}^t\eta_j (\EX_{A}[\|v_j -\nabla f_S(x_j)\|^2 ])^{1/2}(\EX_{A}[\|x_j -x_j^k\|^2])^{1/2}\\
    &\quad + 2 \sum_{j=1}^t\eta_j (\EX_{A}[\|v_j^k -\nabla f_S(x_j^k)\|^2 ])^{1/2}(\EX_{A}[\|x_j -x_j^k\|^2])^{1/2} + 3\sum_{j=1}^t\eta_j^2\EX_{A}[\|v_j^k-f_S(x_j^k)\|^2]\\
    &\quad + \sum_{j=0}^t \frac{2L_f \eta_j}{n}(\mathbb{E}_{A}[\|x_{j}-x_{j}^{k}\|^{2}])^{1 / 2} + \sum_{j=0}^t \frac{\eta_j^2L_f^2}{n}.
\end{aligned}
\end{equation*}
Denote $u_t = (\mathbb{E}_{A}[\|x_{t} - x_{t}^{k}\|^2])^{1/2}$,  then we can get 
\begin{equation*}
\begin{aligned}
    u_t^2 & \leq 2 \sum_{j=1}^{t-1}\eta_j (\EX_{A}[\|v_j -\nabla f_S(x_j)\|^2 ])^{1/2}u_j+ 2 \sum_{j=1}^{t-1}\eta_j (\EX_{A}[\|v_j^k -\nabla f_S(x_j^k)\|^2 ])^{1/2}u_j\\
    &\quad  + 3\sum_{j=1}^{t-1}\eta_j^2\EX_{A}[\|v_j^k-f_S(x_j^k)\|^2] + \sum_{j=0}^{t-1} \frac{2L_f \eta_j}{n}u_j + \sum_{j=0}^{t-1} \frac{\eta_j^2L_f^2}{n}.
\end{aligned}
\end{equation*}
Define 
\begin{equation*}
\begin{aligned}
    S_t \leq 3\sum_{j=1}^{t-1}\eta_j^2\EX_{A}[\|v_j^k-f_S(x_j^k)\|^2] + \sum_{j=0}^{t-1} \frac{\eta_j^2L_f^2}{n},
\end{aligned}
\end{equation*}
and 
\begin{equation*} 
\begin{aligned}
    \alpha_j  = 2 \eta_j (\EX_{A}[\|v_j -\nabla f_S(x_j)\|^2 ])^{1/2}+ 2 \eta_j (\EX_{A}[\|v_j^k -\nabla f_S(x_j^k)\|^2 ])^{1/2} + \frac{2L_f \eta_j}{n}. 
\end{aligned}
\end{equation*}
using Lemma \ref{lemma:general_recursive} we can get 
\begin{equation*}
\begin{aligned}
    u_t &\leq \sqrt{S_t} + \sum_{j=1}^{t-1}\alpha_j \\
    & \leq 2(\sum_{j=1}^{t-1}\eta_j^2\EX_{A}[\|v_j^k -\nabla f_S(x_j^k)\|^2 ])^{1/2}+ (\sum_{j=0}^{t-1} \frac{\eta_j^2L_f^2}{n})^{1/2}+ 2\sum_{j=1}^{t-1} \eta_j( \EX_{A}[\|v_j -\nabla f_S(x_j)\|^2 ])^{1/2} \\
    &\quad + 2 \sum_{j=1}^{t-1}\eta_j (\EX_{A}[\|v_j^k -\nabla f_S(x_j^k)\|^2 ])^{1/2} + \sum_{j=1}^{t-1}\frac{2L_f \eta_j}{n}.
\end{aligned}
\end{equation*}
Furthermore, setting $\eta_t = \eta $, we can get $ \sum_{j=1}^{t-1} \eta_j (\EX_{A}[\|v_j -\nabla f_S(x_j)\|^2 ])^{1/2} \leq \sup_{S} \eta\sum_{j=1}^{t-1} (\EX_{A}[\|v_j -\nabla f_S(x_j)\|^2 ])^{1/2}$ and $ \sum_{j=1}^{t-1} \eta_j (\EX_{A}[\|v_j^k -\nabla f_S(x_j^k)\|^2 ])^{1/2} \leq \sup_{S} \eta \sum_{j=1}^{t-1} (\EX_{A}[\|v_j -\nabla f_S(x_j)\|^2 ])^{1/2}$. Consequently, with $T$ iterations, we obtain that
\begin{equation}\label{Eq:u_T_bound_single}
\begin{aligned}
    u_T & \leq 6\sup_{S} \eta \sum_{j=0}^{T-1}(\EX_{A}[\|v_j -\nabla f_S(x_j)\|^2 ])^{1/2} +  \frac{\eta L_f\sqrt{T}}{\sqrt{n}} + \frac{2L_f\eta T}{n}.
\end{aligned}
\end{equation}
Because often we have $T \geq n$, and $\EX_{A}[\|x_T - x_T^k\|]\leq u_T = (\EX_{A}[\|x_T-x_T^k\|^2])^{1/2} $, then we can get 
\begin{equation*}
\begin{aligned}
    \EX_{A}[\|x_T - x_T^k\|] \leq O(\sup_{S} \eta \sum_{j=0}^{T-1}(\EX_{A}[\|v_j -\nabla f_S(x_j)\|^2 ])^{\frac{1}{2}} + \frac{L_f\eta T}{n}  ).
\end{aligned}
\end{equation*}
This completes the proof.
\end{proof}

\begin{corollary}\label{cor:1_single_level}
Consider STORM in Algorithm~\ref{alg_storm} with $\eta_t = \eta \leq \frac{2}{3L}$, and $\beta_t = \beta \in (0,1) $, for any $t\in [0,T-1]$. With the output $A(S) =x_{T}$, $\epsilon$ satisfies
\begin{equation*}
	\begin{aligned}
		O\bigg( \eta T \big((\beta T)^{-\frac{c}{2}} + \beta^{1/2} + \eta \beta^{-1/2}\big) + \eta T\frac{1}{n} \bigg).
	\end{aligned}
\end{equation*}
\end{corollary}

Next, we give the proof of Corollary \ref{cor:1_single_level}.
\begin{proof}[proof of Corollary \ref{cor:1_single_level}]
Combining \eqref{Eq:u_T_bound_single} and Lemma \ref{lemma:v_t_bound_single}, we can get 
\begin{equation*}
\begin{aligned}
    \epsilon &\leq 6\sup_S \eta \sum_{j=0}^T((\frac{c}{e})^{c}(t\beta)^{-c}\EX_A[\| v_0 - \nabla f_S(x_0)\|^2] +  2\beta\sigma_J^2 + \frac{L_f^2\eta^2}{\beta})^{\frac{1}{2}}  +  \frac{\eta L_f\sqrt{T}}{\sqrt{n}} + \frac{2L_f\eta T}{n}\\
    & \leq  6\sup_S \eta \Big( (\frac{c}{e})^{c} \EX_A[\| v_0 - \nabla f_S(x_0)\|^2]\beta^{-\frac{c}{2}} \sum_{j=0}^T t^{-\frac{c}{2}} + 2 \sigma_J \sqrt{\beta} T + L_f\eta\beta^{-\frac{1}{2}}T \Big) + \frac{3L_f\eta T}{n}.
\end{aligned}
\end{equation*}
Then according to 
\begin{equation}\label{Eq:sum_t^-z_bound}
\begin{aligned}
    \sum_{t= 1}^{T} t^{-z}= O(T^{1- z}), \forall z\in (-1, 0)\cup (-\infty, -1),~~ \sum_{t= 1}^{T} t^{-1}= O(\log T),
\end{aligned}
\end{equation}
we have 
\begin{equation*}
\begin{aligned}
    \epsilon = O(\eta (\beta T)^{-\frac{c}{2}}T + \eta \beta^{1/2}T + \eta^2\beta^{-1/2}T + \eta T n^{-1}  )
\end{aligned}
\end{equation*}
\end{proof}

Before giving the proof of Theorem \ref{thm:opt_convex}, we first introduce a useful lemma.
\begin{lemma}\label{lemma:single}
Suppose  Assumption \ref{ass:Lipschitz continuous}(i), \ref{ass:bound variance} (i) and \ref{ass:Smoothness and Lipschitz continuous gradient} (i) holds for  the empirical risk $F_S$. By running Algorithm \ref{alg_storm}, we have for any $\gamma_t >0$
\begin{equation*}
\begin{aligned}
    &\EX_A[\|x_{t+1}-x_*^S\|^2|\F_t]\\
    & \leq (1+\eta_t\gamma_t)\EX_A[\|x_t - x_*^S\|^2|\F_t]-2\eta_t(F_S(x_t) - F_S(x_*^S)) + \eta_t^2L_f^2+ \frac{\eta_t}{\gamma_t}\EX_A[\|\nabla f_S(x_t)-v_t\|^2|\F_t],
\end{aligned}
\end{equation*}
where $\F_t $ is the $\sigma$-field generated by $\{v_{i_0}, \cdots ,v_{i_{t-1}}\}$.
\end{lemma}

\begin{proof}[proof of Lemma \ref{lemma:single}]
	According to the update rule of Algorithm \ref{alg_storm}, we have 
	\begin{equation*}
		\begin{aligned}
			\|x_{t+1}-x_*^S\|^2 & =  \| x_t - \eta_t v_t -x_*^S\|^2\\
			& = \|x_t - x_*^S\|^2 - 2\eta_t \langle  v_t , x_t - x_*^S\rangle + \eta_t^2\|v_t\|^2 \\
			& = \| x_t - x_*^S\|^2 - 2\eta_t \langle\nabla f_S(x_t), x_t - x_*^S\rangle + \eta_t^2\|v_t\|^2 + 2\eta_t \langle \nabla f_S(x_t) -v_t , x_t - x_*^S\rangle.
		\end{aligned}
	\end{equation*}
	Let $\F_t$ be the $\sigma$-field generated by $\{v_{i_0}, \cdots ,v_{i_{t-1}}\}$, we have 
	\begin{equation*}
		\begin{aligned}
			&\EX_A[\|x_{t+1}-x_*^S\|^2|\F_t]\\
			&= 	\EX_A[\|x_t - x_*^S\|^2|\F_t] -2\eta_t(F_S(x_t) - F_S(x_*^S)) + \eta_t^2L_f^2 + \EX_{A}[2\eta_t\langle \nabla f_S(x_t) -v_t , x_t - x_*^S\rangle| \F_t]\\
			& \leq \EX_A[\|x_t - x_*^S\|^2|\F_t] -2\eta_t(F_S(x_t) - F_S(x_*^S)) + \eta_t^2L_f^2 + 2\eta_t\EX_{A}[\frac{1}{2\gamma_t}\|\nabla f_S(x_t)-v_t\|^2 + \frac{\gamma_t}{2}\|x_t - x_*^S\|^2|\F_t]\\
			& = (1+\eta_t\gamma_t)\EX_A[\|x_t - x_*^S\|^2|\F_t]-2\eta_t(F_S(x_t) - F_S(x_*^S)) + \eta_t^2L_f^2+ \frac{\eta_t}{\gamma_t}\EX_A[\|\nabla f_S(x_t)-v_t\|^2|\F_t].
		\end{aligned}
	\end{equation*}
	This complete the proof.
\end{proof}

Then we give the proof of Theorem \ref{thm:opt_convex}.

\begin{proof}[proof of Theorem \ref{thm:opt_convex}]
	Setting $\eta_t = \eta $, $\beta_t =\beta$ and $\gamma_t =\sqrt{\beta}$, putting Lemma \ref{lemma:v_t_bound_single} into \ref{lemma:single}  we have 
	\begin{equation*}
		\begin{aligned}
			\EX_A[\|x_{t+1}-x_*^S\|^2] 
			& \leq \EX_A[\|x_t - x_*^S\|^2] + \eta\sqrt{\beta} \EX_A[\|x_t - x_*^S\|^2] -2\eta\EX_{A}[F_S(x_t) - F_S(x_*^S)] + \eta^2L_f^2 \\
			& \quad + \frac{\eta}{\sqrt{\beta}} ( (\frac{c}{e})^{c}(t\beta)^{-c}\EX_A[\| v_0 - \nabla f_S(x_0)\|^2] +  2\beta\sigma_J^2 + \frac{L_f^2\eta^2}{\beta}).
		\end{aligned}
	\end{equation*}
	Re-arranging above inequality and telescoping from 1 to $t$ we have 
	\begin{equation}\label{Eq:F_Sxt - F_Sx*_single}
		\begin{aligned}
			&2\eta \sum_{t=1}^T\EX_{A}[F_S(x_t) - F_S(x_*^S)]\\
			& ~~~~~~~~\leq D_x + D_x\eta \beta^{1/2}T+L_f^2\eta^2T + (\frac{c}{e})^cV\beta^{-\frac{1}{2}-c}\eta \sum_{t=1}^Tt^{-c} + 2\sigma_J^2\eta\beta^{1/2}T + L_f^2\eta^3\beta^{-3/2}T.
		\end{aligned}
	\end{equation}
	Then From the choice of $A(S)$, according to \eqref{Eq:sum_t^-z_bound}, as long as $c>2$, we have 
	\begin{equation*}
		\begin{aligned}
			\EX_{A}[F_S(A(S)) - F_S(x_*^S)] = O( D_x(\eta T)^{-1} + D_x \beta^{1/2} + L_f^2\eta + VT^{-c}\beta^{-1/2-c} +\sigma_J^2 \beta^{1/2}+ L_f^2\eta^2\beta^{-3/2} ).
		\end{aligned}
	\end{equation*}
	This complete the proof.
\end{proof}

Next we give the proof of Theorem \ref{thm:Excess_Risk_Bound_convex}.

\begin{proof}[proof of Theorem \ref{thm:Excess_Risk_Bound_convex}]
	Combining Lemma \ref{lemma:v_t_bound_single}  and  \eqref{Eq:u_T_bound_single}, we have 
	\begin{equation*}
		\begin{aligned}
			\EX_{A}[\|x_t - x_t^k\|] &\leq 6 \eta \sum_{j=0}^{t-1}((\frac{c}{e})^{c}(t\beta)^{-c}\EX_A[\| v_0 - \nabla f_S(x_0)\|^2] +  2\beta\sigma_J^2 + \frac{L_f^2\eta^2}{\beta}  )^{1/2} +  \frac{\eta L_f\sqrt{t}}{\sqrt{n}} + \frac{2L_f\eta t}{n}\\
			& \leq 6 (\frac{c}{e})^{c/2} V \eta \beta^{-c/2}\sum_{j=0}^{t-1}t^{-c/2} + 12\sigma_J \eta\beta^{1/2}t + 6L_f\eta^2\beta^{-1/2}t + \frac{\eta L_f\sqrt{t}}{\sqrt{n}} + \frac{2L_f\eta t}{n}.
		\end{aligned}
	\end{equation*}
	Then according to Theorem \ref{theorem:general_single_level}, we have 
	\begin{equation*}
		\begin{aligned}
			\EX_{S,A}[F(x_t) - F_S(x_t)] & \leq L_f(6 (\frac{c}{e})^c V \eta \beta^{-c/2}\sum_{j=0}^{t-1}t^{-c/2} + 12\sigma_J \eta\beta^{1/2}t + 6L_f\eta^2\beta^{-1/2}t + \frac{\eta L_f\sqrt{t}}{\sqrt{n}} + \frac{2L_f\eta t}{n}).
		\end{aligned}
	\end{equation*}
	Combining above inequality with \eqref{Eq:F_Sxt - F_Sx*_single}, and according to $F_S(x_*^S) \leq  F_S(x_*)$ we have 
	\begin{equation*}
		\begin{aligned}
			&\sum_{t=1}^T\EX_{S,A}[F(x_t) - F(x_*)]\\
			& \leq ( D_x + D_x\eta \beta^{1/2}T+L_f^2\eta^2T + (\frac{c}{e})^cV\beta^{-\frac{1}{2}-c}\eta \sum_{t=1}^Tt^{-c} + 2\sigma_J^2\eta\beta^{1/2}T + L_f^2\eta^3\beta^{-3/2}T)/2\eta \\
			&\quad + L_f\sum_{t=1}^T(6 (\frac{c}{e})^c V \eta \beta^{-c/2}\sum_{j=0}^{t-1}t^{-c/2} + 12\sigma_J \eta\beta^{1/2}\sum_{t=1}^Tt + 6L_f\eta^2\beta^{-1/2}\sum_{t=1}^Tt  + \sum_{t=1}^T\frac{3L_f\eta t}{n}).
		\end{aligned}
	\end{equation*}
	According to \eqref{Eq:sum_t^-z_bound}, we have
	\begin{equation}\label{Eq:double_sum_j^{-c/2}}
		\sum_{t=1}^{T} \sum_{j=1}^{T} j^{-\frac{c}{2}}=O(\sum_{t=1}^{T} t^{1-\frac{c}{2}}(\log t)^{\1_{c=2}})=O(T^{2-\frac{c}{2}}(\log T)^{\1_{c=2}}).
	\end{equation}
	Combining above two inequalities,  we have 
	\begin{equation*}
		\begin{aligned}
			&\sum_{t=1}^T\EX_{S,A}[F(x_t) - F(x_*)]\\
			& = O\Big(\eta^{-1} + \beta^{1/2}T + \eta T + (\beta T)^{-c}\beta^{-1/2}T + \beta^{1/2}T + \eta^2\beta^{-3/2}T + \eta\beta^{-c/2} T^{2-\frac{c}{2}}(\log T)^{\1_{c=2}}\\
			&\quad ~~~~~~~~+ \eta \beta^{1/2} T^2 + \eta^2\beta^{-1/2}T^2 + \eta T^2 n^{-1}\Big).
		\end{aligned}
	\end{equation*}
	Setting $\eta = T^{-a} $ and $\beta = T^{-b}$, dividing both sides of above inequality with $T$, then from the choice of $A(S)$ we get
	\begin{equation*}
		\begin{aligned}
			&\EX_{S,A}[F(A(S)) - F(x_*)]\\
			& \leq O\Big(T^{a-1} + T^{-b/2} + T^{-a} + T^{1/b - c(1-b) + T^{-b/2-1}} + T^{-b/2} + T^{3b/2-2a} + T^{1-a+c/2(b-1) } (\log T)^{\1_{c=2}}\\
			&\quad ~~~~~~~~+T^{1-a-b/2} + T^{1-2a+b/2} + T^{1-a}n^{-1}\Big).
		\end{aligned}
	\end{equation*}
	As long as $c > 4$,  the dominating terms are $ O(T^{1- a- \frac{b}{2}})$, $\quad O(T^{1+ \frac{b}{2}- 2a})$, $\quad O(n^{-1}T^{1- a}), \quad  O(T^{a-1})$,  and $ O(T^{\frac{3}{2}b- 2a}).$
	Setting $a = b =4/5$,  then we have  
	\begin{equation*}
		\begin{aligned}
			\mathbb{E}[F(A(S))-F(x_{*})]=O(T^{-\frac{1}{5}}+\frac{T^{\frac{1}{5}}}{n}).
		\end{aligned}
	\end{equation*}
	Choosing $T=O(n^{2.5})$, we have the following bound
	\begin{equation*}
		\mathbb{E}[F(A(S))-F(x_{*})]=O(\frac{1}{\sqrt{n}}).
	\end{equation*}
	This completes the proof.
\end{proof}

\subsection{Strongly-convex-setting}

\begin{proof}[proof of Theorem \ref{thm:sta_sconvex_single_convex}]
	Similar to the proof for convex setting, we use the same notations.
	
	We will consider two cases, i.e., $i_t \neq k$ and $i_t = k$.
	
	\textbf{\quad Case 1 ($i_t \neq k$).}  We have 
	\begin{equation}\label{Eq:sta_origin_single_sc}
		\begin{aligned}
			\|x_{t+1} - x_{t+1}^{k}\|^2 & = \|x_t - \eta_t v_t - x_t^k + \eta_t v_t^k\|^2 \\
			& \leq \|x_t - x_t^k\|^2 -2\eta_t \langle v_t - v_t^k , x_t -x_t^k\rangle + \eta_t^2\|v_t - v_t^k\|^2.
		\end{aligned}
	\end{equation}
	For the second term on the RHS of \eqref{Eq:sta_origin_single_sc}, we have 
	\begin{equation*}
		\begin{aligned}
			&-2\eta_t \langle v_t - v_t^k , x_t -x_t^k\rangle\\
			& = -2\eta_t \langle v_t - \nabla f_S(x_t), x_t -x_t^k\rangle - 2\eta_t \langle \nabla f_S(x_t) -  \nabla f_S(x_t^k), x_t -x_t^k\rangle -2\eta_t \langle \nabla f_S(x_t^k)-  v_t^k , x_t -x_t^k\rangle.
		\end{aligned}
	\end{equation*}

Note that if $F$ is $\mu$ strongly convex, then $\varphi\bigl(x\bigr)=F\bigl(x\bigr)-\frac{\sigma}{2}\|x\|^2$ is convex with $\bigl(L-\mu\bigr)$-smooth. Then, applying above to $\varphi$ yields the following inequality 
\begin{equation*}
    \begin{aligned}
        \langle \nabla F\left(x\right)-\nabla F(x'), x-x' \rangle \ge\frac{L\mu}{L+\mu}\|x-x'\|^2+\frac{1}{L+\mu}\|\nabla F(x)-\nabla F(x')\|^2.
    \end{aligned}
\end{equation*}

Then using Assumption  \ref{ass:Smoothness and Lipschitz continuous gradient} (i), i.e., the smoothness, and combining with the strong convexity of $f_S(\cdot)$ we can get 
\begin{equation*}
\begin{aligned}
    &-2\eta_t \langle v_t - v_t^k , x_t -x_t^k\rangle\\
    & \leq 2\eta_t \|v_t -\nabla f_S(x_t)\| \cdot \|x_t - x_t^k \| -2\eta_t(\frac{1}{L+\mu}\|\nabla f_S(x_t) -  \nabla f_S(x_t^k)\|^2 + \frac{L\mu}{L+\mu}\|x_t-x_t^k\|^2)\\
    &\quad + 2\eta_t \|v_t^k-f_S(x_t^k)\|\cdot \|x_t - x_t^k \|.
\end{aligned}
\end{equation*}
	
For the third term on the RHS of \eqref{Eq:sta_origin_single_sc}, we have 
\begin{equation*}
\begin{aligned}
    \eta_t^2\|v_t - v_t^k\|^2 &\leq 3\eta_t^2 \|v_t -\nabla f_S(x_t)\|^2 + 3\eta_t^2\|\nabla f_S(x_t) -  \nabla f_S(x_t^k)\|^2 + 3\eta_t^2\|v_t^k-f_S(x_t^k)\|^2.
\end{aligned}
\end{equation*}
Putting above two inequalities into  \eqref{Eq:sta_origin_single_sc}, we have 
\begin{equation*}
\begin{aligned}
    &\|x_{t+1} - x_{t+1}^{k}\|^2\\
    &\leq (1-\frac{2\eta_tL\mu}{L+\mu})\|x_t - x_t^k\|^2 +  2\eta_t \|v_t -\nabla f_S(x_t)\| \cdot \|x_t - x_t^k \| + 2\eta_t \|v_t^k-f_S(x_t^k)\|\cdot \|x_t - x_t^k \|\\
    &\quad + (3\eta_t^2 - \frac{2\eta_t}{L+\mu})\|\nabla f_S(x_t) -  \nabla f_S(x_t^k)\|^2+  3\eta_t^2\|v_t^k-f_S(x_t^k)\|^2.
\end{aligned}
\end{equation*}
By setting $\eta_t \leq \frac{2}{3(L+\mu)}$, we have 
\begin{equation*}
\begin{aligned}
    (1-\frac{2\eta_tL\mu}{L+\mu})\|x_{t+1} - x_{t+1}^{k}\|^2
    & \leq (1-\frac{2\eta_tL\mu}{L+\mu})\|x_t - x_t^k\|^2 +  2\eta_t \|v_t -\nabla f_S(x_t)\| \cdot \|x_t - x_t^k \|\\
    &~~~~~~~~~~~~~ + 2\eta_t \|v_t^k-f_S(x_t^k)\|\cdot \|x_t - x_t^k \|+  3\eta_t^2\|v_t^k-f_S(x_t^k)\|^2.
\end{aligned}
\end{equation*} 

\textbf{\quad Case 2 ($i_t = k$).}  We have
\begin{equation*}
\begin{aligned}
    \|x_{t+1} - x_{t+1}^{k}\| & = \|x_t - \eta_t v_t - x_t^k + \eta_t v_t^k\|\\
    & \leq \|x_t- x_t^k\| +\eta_t \|v_t -v_t^k\| \leq \|x_t- x_t^k\| +\eta_t L_f.
\end{aligned}
\end{equation*} 
Then we can get 
\begin{equation*}
\begin{aligned}
    &\|x_{t+1} - x_{t+1}^{k}\|^2 \leq \|x_t- x_t^k\|^2 + 2\eta_t L_f \|x_t- x_t^k\| + \eta_t^2L_f^2.
\end{aligned}
\end{equation*}

Combining \textbf{Case 1} and \textbf{Case 2} we have
\begin{equation*}
\begin{aligned}
    \|x_{t+1} - x_{t+1}^{k}\|^2 & \leq  (1-\frac{2\eta_tL\mu}{L+\mu})\|x_t - x_t^k\|^2 +  2\eta_t \|v_t -\nabla f_S(x_t)\| \cdot \|x_t - x_t^k \| + 2\eta_t \|v_t^k-f_S(x_t^k)\|\cdot \|x_t - x_t^k \| \\
    &\quad + 3\eta_t^2\|v_t^k-f_S(x_t^k)\|^2 + 2\eta_t L_f \|x_t- x_t^k\| \1_{i_t = k} + \eta_t^2L_f^2\1_{i_t = k}.
\end{aligned}
\end{equation*}
According to \eqref{Eq:single_sc_1/n_x_t-x_t^k}, we have 
\begin{equation*}
\begin{aligned}
    \EX_{A}[\|x_{t+1} - x_{t+1}^{k}\|^2] & \leq  (1-\frac{2\eta_tL\mu}{L+\mu})  \EX_{A}[\|x_t - x_t^k\|^2] +  2\eta_t (\EX_{A}[\|v_t -\nabla f_S(x_t)\|^2])^{1/2} (\EX_{A}[ \|x_t - x_t^k \|^2])^{1/2}\\
    &\quad + 2\eta_t (\EX_{A}[\|v_t^k -\nabla f_S(x_t^k)\|^2])^{1/2} (\EX_{A}[ \|x_t - x_t^k \|^2])^{1/2} + 3\eta_t^2\EX_{A}[\|v_t^k-f_S(x_t^k)\|^2]\\
    &\quad + \frac{2L_f \eta_t}{n}(\mathbb{E}_{A}[\|x_{t}-x_{t}^{k}\|^{2}])^{1 / 2} + \frac{\eta_t^2L_f^2}{n}.
\end{aligned}
\end{equation*}

According to Lemma \ref{lemma:general_recursive}, setting $\eta_t = \eta$ and $\beta_t = \beta$, we can get 
\begin{equation*}
\begin{aligned}
    \EX_{A}[\|x_{t+1} - x_{t+1}^{k}\|^2 ] 
    & \leq 2\eta\sum_{j=1}^t(1-\frac{2\eta L\mu}{L+\mu})^{t-j} (\EX_{A}[\|v_j -\nabla f_S(x_j)\|^2])^{1/2} (\EX_{A}[ \|x_j - x_j^k \|^2])^{1/2}   \\
    &\quad + 2\eta\sum_{j=1}^t (1-\frac{2\eta L\mu}{L+\mu})^{t-j} (\EX_{A}[\|v_j^k -\nabla f_S(x_j^k)\|^2])^{1/2} (\EX_{A}[ \|x_j - x_j^k \|^2])^{1/2} \\
    &\quad+  \sum_{j=1}^t (1-\frac{2\eta L\mu}{L+\mu})^{t-j} \frac{2L_f \eta}{n}(\mathbb{E}_{A}[\|x_{j}-x_{j}^{k}\|^{2}])^{1 / 2} +  \frac{\eta^2 L_f^2}{n}\sum_{j=0}^t (1-\frac{2\eta L\mu}{L+\mu})^{t-j}\\
    &\quad +   3\eta^2\sum_{j=0}^t (1-\frac{2\eta L\mu}{L+\mu})^{t-j} \EX_{A}[\|v_j^k-f_S(x_j^k)\|^2].
\end{aligned}
\end{equation*}

Denote $u_t = (\mathbb{E}_{A}[\|x_{t} - x_{t}^{k}\|^2])^{1/2}$,  then we can get 
\begin{equation*}
\begin{aligned}
    u_t^2 & \leq 2\eta \sum_{j=1}^{t-1}(1-\frac{2\eta L\mu}{L+\mu})^{t-j-1} (\EX_{A}[\|v_j -\nabla f_S(x_j)\|^2 ])^{1/2}u_j\\
    &\quad+ 2\eta\sum_{j=1}^{t-1}(1-\frac{2\eta L\mu}{L+\mu})^{t-j-1}(\EX_{A}[\|v_j^k -\nabla f_S(x_j^k)\|^2 ])^{1/2}u_j + \frac{2L_f \eta}{n}\sum_{j=0}^{t-1} (1-\frac{2\eta L\mu}{L+\mu})^{t-j-1}u_j \\
    &\quad + 3\eta^2\sum_{j=1}^{t-1}(1-\frac{2\eta L\mu}{L+\mu})^{t-j-1}\EX_{A}[\|v_j^k-f_S(x_j^k)\|^2]  + \frac{\eta^2L_f^2}{n}\sum_{j=0}^{t-1}  (1-\frac{2\eta L\mu}{L+\mu})^{t-j-1}.
\end{aligned}
\end{equation*}
Define $S_t=  3\eta^2\sum_{j=1}^{t-1}(1-\frac{2\eta L\mu}{L+\mu})^{t-j-1}\EX_{A}[\|v_j^k-f_S(x_j^k)\|^2]  + \frac{\eta^2L_f^2}{n}\sum_{j=0}^{t-1}  (1-\frac{2\eta L\mu}{L+\mu})^{t-j-1}$ and $\alpha_j = 2\eta(1-\frac{2\eta L\mu}{L+\mu})^{t-j-1} (\EX_{A}[\|v_j -\nabla f_S(x_j)\|^2 ])^{1/2}+ 2\eta(1-\frac{2\eta L\mu}{L+\mu})^{t-j-1}(\EX_{A}[\|v_j^k -\nabla f_S(x_j^k)\|^2 ])^{1/2} + \frac{2L_f \eta}{n}(1-\frac{2\eta L\mu}{L+\mu})^{t-j-1}$.

using Lemma \ref{lemma:recursion lemma}, we can get
\begin{equation*}
\begin{aligned}
    u_t & \leq \sqrt{S_t} + \sum_{t=1}^{t-1}\alpha_j\\
    & \leq 2\eta (\sum_{j=1}^{t-1}(1-\frac{2\eta L\mu}{L+\mu})^{t-j-1}\EX_{A}[\|v_j^k-f_S(x_j^k)\|^2])^{1/2} + \sqrt{\frac{(L+\mu)\eta L_f^2}{2L\mu n }}\\
    & \quad + 2\eta\sum_{j=1}^{t-1}(1-\frac{2\eta L\mu}{L+\mu})^{t-j-1} (\EX_{A}[\|v_j -\nabla f_S(x_j)\|^2 ])^{1/2}\\
    &\quad+ 2\eta\sum_{j=1}^{t-1}(1-\frac{2\eta L\mu}{L+\mu})^{t-j-1}(\EX_{A}[\|v_j^k -\nabla f_S(x_j^k)\|^2 ])^{1/2} + \frac{L_f(L+\mu)}{L\mu n},
\end{aligned}
\end{equation*}
where the last inequality holds by \eqref{Eq:sum_(1-a)^t-j}. Consequently, with $T$ iterations, because of the inequality $\EX_{A}[\|x_T - x_T^k\|] \leq u_T$, we have 
\begin{equation*}
\begin{aligned}
    \EX_{A}[\|x_T - x_T^k\|]  & \leq 2\eta (\sum_{j=1}^{T-1}(1-\frac{2\eta L\mu}{L+\mu})^{T-j-1}\EX_{A}[\|v_j^k-f_S(x_j^k)\|^2])^{1/2} + \sqrt{\frac{(L+\mu)\eta L_f^2}{2L\mu n }}\\
    & \quad + 2\eta\sum_{j=1}^{T-1}(1-\frac{2\eta L\mu}{L+\mu})^{T-j-1} (\EX_{A}[\|v_j -\nabla f_S(x_j)\|^2 ])^{1/2}\\
    &\quad+ 2\eta\sum_{j=1}^{T-1}(1-\frac{2\eta L\mu}{L+\mu})^{T-j-1}(\EX_{A}[\|v_j^k -\nabla f_S(x_j^k)\|^2 ])^{1/2} + \frac{L_f(L+\mu)}{L\mu n},
\end{aligned}
\end{equation*}
Then we analyze which one of $(\sum_{j=1}^{T-1}(1-\frac{2\eta L\mu}{L+\mu})^{T-j-1}\EX_{A}[\|v_j^k-f_S(x_j^k)\|^2])^{1/2}$ and $\sum_{j=1}^{T-1}(1-\frac{2\eta L\mu}{L+\mu})^{T-j-1}(\EX_{A}[\|v_j^k -\nabla f_S(x_j^k)\|^2 ])^{1/2}$ is the dominant term.

For the first term, according to Lemma \ref{lemma:v_t_bound_single} we have 
\begin{equation*}
\begin{aligned}
    &(\sum_{j=1}^{T-1}(1-\frac{2\eta L\mu}{L+\mu})^{T-j-1}\EX_{A}[\|v_j^k-f_S(x_j^k)\|^2])^{1/2}\\
    & \leq (\sum_{j=1}^{T-1}(1-\frac{2\eta L\mu}{L+\mu})^{T-j-1}((\frac{c}{e})^{c}(j\beta)^{-c}\EX_A[\| v_0 - \nabla f_S(x_0)\|^2] +  2\beta\sigma_J^2 + \frac{L_f^2\eta^2}{\beta}))^{1/2}\\
    & \leq (\frac{c}{e})^{\frac{c}{2}}\beta^{-\frac{c}{2}}\sqrt{V}(\sum_{j=1}^{T-1}(1-\frac{2\eta L\mu}{L+\mu})^{T-j-1} j^{-c})^{\frac{1}{2}} + \sigma_J\sqrt{\frac{(L+\mu)\beta}{\eta L \mu }} + L_f\sqrt{\frac{\eta(L+\mu)}{2\beta L \mu}},
\end{aligned}
\end{equation*}
where the last inequality holds by \eqref{Eq:sum_(1-a)^t-j}, as for $\sum_{j=1}^{T-1}(1-\frac{2\eta L\mu}{L+\mu})^{T-j-1} j^{-c}$, according to Lemma \ref{lemma:weighted_avg}, we have 
\begin{equation}\label{Eq:sum_a^(t-j)j^-c}
\begin{aligned}
    \sum_{j=1}^{T-1}(1-\frac{2\eta L\mu}{L+\mu})^{T-j-1} j^{-c} & \leq \frac{\sum_{j=1}^{T-1}(1-\frac{2\eta L\mu}{L+\mu})^{T-j-1}\sum_{j=1}^{T-1}j^{-\frac{c}{2}}}{T} \leq \frac{(L+\mu)\sum_{j=1}^{T-1}j^{-c}}{2T\eta L\mu},
\end{aligned}
\end{equation}
then according to \eqref{Eq:sum_t^-z_bound} we can get 
\begin{equation*}
\begin{aligned}
    &(\sum_{j=1}^{T-1}(1-\frac{2\eta L\mu}{L+\mu})^{T-j-1}\EX_{A}[\|v_j^k-f_S(x_j^k)\|^2])^{1/2}\\
    & \leq (\frac{c}{e})^{\frac{c}{2}}\sqrt{\frac{V(L+\mu)}{2\eta L\mu}}(T\beta)^{-\frac{c}{2}}+ \sigma_J\sqrt{\frac{(L+\mu)\beta}{\eta L \mu }} + L_f\sqrt{\frac{\eta(L+\mu)}{2\beta L \mu}}.
\end{aligned}
\end{equation*}
For the second term, according to Lemma \ref{lemma:v_t_bound_single} we have 
\begin{equation*}
\begin{aligned}
    &\sum_{j=1}^{T-1}(1-\frac{2\eta L\mu}{L+\mu})^{T-j-1}(\EX_{A}[\|v_j^k -\nabla f_S(x_j^k)\|^2 ])^{1/2}\\
    &\leq \sum_{j=1}^{T-1}(1-\frac{2\eta L\mu}{L+\mu})^{T-j-1}((\frac{c}{e})^{c}(t\beta)^{-c}\EX_A[\| v_0 - \nabla f_S(x_0)\|^2] +  2\beta\sigma_J^2 + \frac{L_f^2\eta^2}{\beta})^{\frac{1}{2}}.
\end{aligned}
\end{equation*}
Similar to the first term, we can get
\begin{equation}\label{Eq:v_t_error_sum_single_sc}
\begin{aligned}
    &\sum_{j=1}^{T-1}(1-\frac{2\eta L\mu}{L+\mu})^{T-j-1}(\EX_{A}[\|v_j^k -\nabla f_S(x_j^k)\|^2 ])^{1/2} \leq  (\frac{c}{e})^{\frac{c}{2}}\sqrt{V}\frac{L+\mu}{\eta L\mu}(T\beta)^{-\frac{c}{2}}+\sigma_J\sqrt{\beta}\frac{L+\mu}{\eta L\mu}+L_f\frac{L+\mu}{2\sqrt{\beta} L\mu}.
\end{aligned}
\end{equation}
It's easily to get the dominating term is the second term $\sum_{j=1}^{T-1}(1-\frac{2\eta L\mu}{L+\mu})^{T-j-1}(\EX_{A}[\|v_j^k -\nabla f_S(x_j^k)\|^2])^{1/2}$. Therefore
\begin{equation}\label{Eq:x_T-x_T^k_single_sc}
\begin{aligned}
    \EX_{A}[\|x_T - x_T^k\|]  & \leq 2\eta (\sum_{j=1}^{T-1}(1-\frac{2\eta L\mu}{L+\mu})^{T-j-1}\EX_{A}[\|v_j^k-f_S(x_j^k)\|^2])^{1/2} + \sqrt{\frac{(L+\mu)\eta L_f^2}{2L\mu n }}\\
    & \quad + 2\eta\sum_{j=1}^{T-1}(1-\frac{2\eta L\mu}{L+\mu})^{T-j-1} (\EX_{A}[\|v_j -\nabla f_S(x_j)\|^2 ])^{1/2}\\
    &\quad+ 2\eta\sum_{j=1}^{T-1}(1-\frac{2\eta L\mu}{L+\mu})^{T-j-1}(\EX_{A}[\|v_j^k -\nabla f_S(x_j^k)\|^2 ])^{1/2} + \frac{L_f(L+\mu)}{L\mu n}\\
    & \leq 6\eta\sum_{j=1}^{T-1}(1-\frac{2\eta L\mu}{L+\mu})^{T-j-1} (\EX_{A}[\|v_j -\nabla f_S(x_j)\|^2 ])^{1/2}+ \sqrt{\frac{(L+\mu)\eta L_f^2}{2L\mu n }}+ \frac{L_f(L+\mu)}{L\mu n}\\
    & \leq 6\eta\sum_{j=1}^{T-1}(1-\frac{2\eta L\mu}{L+\mu})^{T-j-1} (\EX_{A}[\|v_j -\nabla f_S(x_j)\|^2 ])^{1/2} + \frac{2L_f(L+\mu)}{L\mu n},
\end{aligned}
\end{equation}
where the last inequality holds since often we have $\eta \leq \frac{1}{n}$. 
Then we get the final result
\begin{equation*}
\begin{aligned}
    \epsilon \leq O(\eta\sum_{j=1}^{T-1}(1-\frac{2\eta L\mu}{L+\mu})^{T-j-1} (\EX_{A}[\|v_j -\nabla f_S(x_j)\|^2 ])^{1/2} + \frac{L_f(L+\mu)}{L\mu n}).
\end{aligned}
\end{equation*}
This completes the proof.
\end{proof}

\begin{corollary}[One-level Optimization]\label{cor:2_single_level}
Consider STORM in Algorithm \ref{alg_storm} with $\eta_t = \eta \leq \frac{2}{3(L+\mu)}$, and $\beta_t=\beta \in (0,1)$ for any $t\in [0,T-1]$ and the output  $A(S)  =x_{T}$. Then, we have the following results
\begin{equation*}
	\begin{aligned}
		\epsilon \leq O( (T \beta)^{-\frac{c}{2}} + \beta^{\frac{1}{2}}  + \eta\beta^{-\frac{1}{2}  }+ n^{-1}).
	\end{aligned}
\end{equation*}
\end{corollary}

Next, we give the proof of Corollary \ref{cor:2_single_level}.

\begin{proof}[proof of Corollary \ref{cor:2_single_level}]
Combining Theorem \ref{thm:opt_sconvex} and \eqref{Eq:v_t_error_sum_single_sc}, we have 
\begin{equation*}
\begin{aligned}
    \epsilon \leq O( (T \beta)^{-\frac{c}{2}} + \beta^{\frac{1}{2}}  + \eta\beta^{-\frac{1}{2}  }+ n^{-1}).
\end{aligned}
\end{equation*}
This complete the proof.
\end{proof}

Before give the detailed proof of Theorem \ref{thm:opt_sconvex}, we first give a useful lemma.
\begin{lemma}\label{lemma:single_sc}
Let Assumption \ref{ass:Lipschitz continuous}(i), \ref{ass:bound variance} (i) and \ref{ass:Smoothness and Lipschitz continuous gradient} (i) holds, as $F_S$ is $\mu$-strongly convex. By running Algorithm \ref{alg_storm}, we have 
\begin{equation*}
\begin{aligned}
    \EX_{A}[F_S(x_{t+1})|\F_t] \leq	F_S(x_t) -\frac{\eta_t}{2}\| \nabla F_S(x_t)\|^2+ \frac{L\eta_t^2L_f^2}{2}+2\eta_t  \|v_t- \nabla F_S(x_t)\|^2.
\end{aligned}
\end{equation*}
where $\F_t$ is the $\sigma$-field generated by $\{v_{i_0}, \cdots, v_{i_{t-1}}\}$.
\end{lemma}
\begin{proof}[proof of Lemma \ref{lemma:single_sc}]
According to the smoothness of $F_S(\cdot)$, then we have 
\begin{equation*}
\begin{aligned}
    F_S(x_{t+1}) & \leq F_S(x_t) + \langle \nabla F_S(x_t), x_{t+1}-x_t \rangle +\frac{L}{2}\|x_{t+1}-x_t\|^2\\
    & \leq F_S(x_t) - \eta_t\| \nabla F_S(x_t)\|^2 + \frac{L\eta_t^2 L_f^2}{2} - \eta_t \langle \nabla F_S(x_t), v_t- \nabla F_S(x_t)\rangle\\
    & \leq  F_S(x_t) - \eta_t\| \nabla F_S(x_t)\|^2 + \frac{L\eta_t^2L_f^2}{2}+ \frac{\eta_t}{2}\|\nabla F_S(x_t)\|^2 + 2\eta_t  \|v_t- \nabla F_S(x_t)\|^2,
\end{aligned}
\end{equation*}
where the last inequality holds by Cauchy-Schwartz.  Then we can get 
\begin{equation*}
\begin{aligned}
    \EX_{A}[F_S(x_{t+1})|\F_t] \leq	F_S(x_t) -\frac{\eta_t}{2}\| \nabla F_S(x_t)\|^2+ \frac{L\eta_t^2L_f^2}{2}+2\eta_t  \|v_t- \nabla F_S(x_t)\|^2.
\end{aligned}
\end{equation*}
This complete the proof.
\end{proof}
Now we move on the proof of Theorem \ref{thm:opt_sconvex}.
\begin{proof}[proof of Theorem \ref{thm:opt_sconvex}]
	Satisfying strong convexity also satisfies Polyak-\L ojasiewicz (PL) inequality, then we can get for all $x$
	\begin{equation*}
		\begin{aligned}
			\frac{1}{2}\|\nabla F_S(x)\|^2 \geq \mu (F_S(x) - F_S(x_*^S)).
		\end{aligned}
	\end{equation*}
	According to Lemma \ref{lemma:single_sc}, we have 
	\begin{equation*}
		\begin{aligned}
			\EX_{A}[F_S(x_{t+1}) - F_S(x_*^S)] & \leq (1-\mu \eta_t)\EX_{A}[F_S(x_{t}) - F_S(x_*^S)]+ \frac{L\eta_t^2L_f^2}{2}+2\eta_t  \|v_t- \nabla F_S(x_t)\|^2.
		\end{aligned}
	\end{equation*}
	Setting $\eta_t = \eta$ and $\beta_t = \beta$, using Lemma \ref{lemma:v_t_bound_single}, we have 
	\begin{equation*}
		\begin{aligned}
			\EX_{A}[F_S(x_{t+1}) - F_S(x_*^S)] & \leq (1-\mu \eta)\EX_{A}[F_S(x_{t}) - F_S(x_*^S)]+\frac{L\eta^2L_f^2}{2}\\
			&\quad + 2\eta ((\frac{c}{e})^{c}(t\beta)^{-c}\EX_A[\| v_0 - \nabla f_S(x_0)\|^2] +  2\beta\sigma_J^2 + \frac{L_f^2\eta^2}{\beta}).
		\end{aligned}
	\end{equation*}
	Telescoping the above inequality from 1 to $T$, according to Lemma \ref{lemma:general_recursive}, we can get 
	\begin{equation*}
		\begin{aligned}
			&\EX_{A}[F_S(x_T) - F_S(x_*^S)]\\
			 & \leq (1-\mu \eta)^{T-1}\EX_{A}[F_S(x_1) - F_S(x_*^S)] + \frac{L\eta^2L_f^2}{2}\sum_{t=1}^{T-1}(1-\mu \eta)^{T-t-1}\\
			&\quad + 2\eta V(\frac{c}{e})^c\beta^{-c}\sum_{t=1}^{T-1}t^{-c}(1-\mu\eta)^{T-t-1}+2\sigma_J^2\eta\beta\sum_{t=1}^{T-t-1}(1-\mu\eta)^{T-t-1} + \frac{2L_f^2\eta^3}{\beta}\sum_{t=1}^{T-t-1}(1-\mu\eta)^{T-t-1}.
		\end{aligned}
	\end{equation*}
	For $t=0$, we have 
	\begin{equation*}
		\begin{aligned}
			\EX_{A}[F_S(x_1) - F_S(x_*^S)] \leq (1-\mu \eta)\EX_{A}[F_S(x_{0}) - F_S(x_*^S)]+ \frac{L\eta^2 L_f^2}{2}+2\eta V.
		\end{aligned}
	\end{equation*}
	Combining the above two inequalities, we have 
	\begin{equation*}
		\begin{aligned}
			&\EX_{A}[F_S(x_T) - F_S(x_*^S)]\\
			& \leq (1-\mu \eta)^T\EX_{A}[F_S(x_{0}) - F_S(x_*^S)]+ \frac{L\eta^2L_f^2}{2}\sum_{t=1}^{T}(1-\mu \eta)^{T-t}+2\eta V(1-\mu \eta)^{T-1} \\
			&\quad + 2\eta V(\frac{c}{e})^c\beta^{-c}\sum_{t=1}^{T-1}t^{-c}(1-\mu\eta)^{T-t-1}+2\sigma_J^2\eta\beta\sum_{t=1}^{T-t-1}(1-\mu\eta)^{T-t-1}+ \frac{2L_f^2\eta^3}{\beta}\sum_{t=1}^{T-t-1}(1-\mu\eta)^{T-t-1}.
		\end{aligned}
	\end{equation*}
	According to \eqref{Eq:prod_(1-beta_i)}, \eqref{Eq:exp(-t beta)} and \eqref{Eq:sum_(1-a)^t-j}, we have 
	\begin{equation*}
		\begin{aligned}
			\EX_{A}[F_S(x_T) - F_S(x_*^S)] & \leq  (\frac{c}{e\mu})^c(\eta T)^{-c} D_x + \frac{L\eta L_f^2}{2\mu} + 2\eta (\frac{c}{e\mu})^c(\eta T)^{-c} V\\
			&\quad + 2\eta V(\frac{c}{e})^c\beta^{-c}\sum_{t=1}^{T-1}t^{-c}(1-\mu\eta)^{T-t-1} + \frac{2\sigma_J^2\beta}{\mu} + \frac{2L_f^2\eta^2}{\beta\mu}.
		\end{aligned}
	\end{equation*}
	Then according to \eqref{Eq:sum_a^(t-j)j^-c} we have 
	\begin{equation*}
		\begin{aligned}
			\EX_{A}[F_S(x_T) - F_S(x_*^S)] & \leq (\frac{c}{e\mu})^c(\eta T)^{-c} D_x + \frac{L\eta L_f^2}{2\mu} + 2\eta (\frac{c}{e\mu})^c(\eta T)^{-c} V\\
			&\quad + \frac{2 V(c/e)^c(\beta T)^{-c}}{\mu}+ \frac{2\sigma_J^2\beta}{\mu} + \frac{2L_f^2\eta^2}{\beta\mu}.
		\end{aligned}
	\end{equation*}
	Then we can get 
	\begin{equation*}
		\begin{aligned}
			\EX_{A}[F_S(x_T) - F_S(x_*^S)] & = O(D_x(\eta T)^{-c} + L_f^2L\eta + V\eta(\eta T)^{-c} + V(\beta T)^{-c} + \sigma_J^2\beta + L_f^2\eta^2\beta^{-1}).
		\end{aligned}
	\end{equation*}
	This completes the proof.
\end{proof}

Next, we move on to the proof of Theorem \ref{thm:Excess_Risk_Bound_s_convex}
\begin{proof}[proof of Theorem \ref{thm:Excess_Risk_Bound_s_convex}]
	Combining \eqref{Eq:x_T-x_T^k_single_sc} and Theorem \ref{theorem:general_single_level}, we have 
	\begin{equation*}
		\begin{aligned}
			\EX_{S,A}[F(x_T) - F_S(x_T)] &\leq L_f(6\eta\sum_{j=1}^{T-1}(1-\frac{2\eta L\mu}{L+\mu})^{T-j-1} (\EX_{A}[\|v_j -\nabla f_S(x_j)\|^2 ])^{1/2} + \frac{2L_f(L+\mu)}{L\mu n}).
		\end{aligned}
	\end{equation*}
	Then according to \eqref{Eq:v_t_error_sum_single_sc}, we have 
	\begin{equation*}
		\begin{aligned}
			\EX_{S,A}[F(x_T) - F_S(x_T)] &\leq 6L_f \eta ((\frac{c}{e})^{\frac{c}{2}}\sqrt{V}\frac{L+\mu}{\eta L\mu}(T\beta)^{-\frac{c}{2}}+\sigma_J\sqrt{\beta}\frac{L+\mu}{\eta L\mu}+L_f\frac{L+\mu}{2\beta L\mu}) + \frac{2L_f^2(L+\mu)}{L\mu n }.
		\end{aligned}
	\end{equation*}
	Combining with Theorem \ref{thm:opt_sconvex}, and using the fact that $F_S(x_*^S) \leq F_S(x_*)$ we have 
	\begin{equation*}
		\begin{aligned}
			\EX_{S,A}[F(A(S)) - F(x_*)]& \leq 6L_f \eta ((\frac{c}{e})^{\frac{c}{2}}\sqrt{V}\frac{L+\mu}{\eta L\mu}(T\beta)^{-\frac{c}{2}}+\sigma_J\sqrt{\beta}\frac{L+\mu}{\eta L\mu}+L_f\frac{L+\mu}{2\sqrt{\beta} L\mu}) + \frac{2L_f^2(L+\mu)}{L\mu n }\\
			&\quad  +  (\frac{c}{e\mu})^c(\eta T)^{-c} V + \frac{L\eta L_f^2}{2\mu} + 2\eta (\frac{c}{e\mu})^c(\eta T)^{-c} V\\
			&\quad + \frac{2 V(c/e)^c(\beta T)^{-c}}{\mu}+ \frac{2\sigma_J^2\beta}{\mu} + \frac{2L_f^2\eta^2}{\beta\mu}.
		\end{aligned}
	\end{equation*}
	Setting $\eta = T^{-a}$ and $\beta = T^{-b}$ with $a,b \in (0,1]$, we have
	\begin{equation*}
		\begin{aligned}
			&\EX_{S,A}[F(A(S)) - F(x_*)] \\
			& = O(T^{\frac{c}{2}(b-1)} + T^{-\frac{b}{2}} + T^{\frac{b}{2}-a}  + T^{-c(1-a)} + T^{-a} +T^{-c(1-a)-a} +T^{-c(1-b)} +T^{-b} +T^{b-2a}).
		\end{aligned}
	\end{equation*}
	
	Setting $c =3$, the dominating terms are $O(T^{\frac{b}{2}- a}), \quad O(T^{-\frac{b}{2}}), \quad O(T^{\frac{3}{2}(b- 1)}), \quad O(T^{-\frac{a}{2}}), \quad O(T^{3(a- 1)})$.
	
	Setting $a = b = \frac{6}{7}$, we have 
	\begin{equation*}
		\mathbb{E}_{S, A}\left[F(A(S))-F\left(x_{*}\right)\right]=O\left(T^{-\frac{3}{7}}\right).
	\end{equation*}
	
	Setting $T = O(n^{\frac{7}{6}})$, we have the following 
	\begin{equation*}
		\mathbb{E}_{S, A}\left[F(A(S))-F\left(x_{*}\right)\right]=O\left(\frac{1}{\sqrt{n}}\right).
	\end{equation*}
	The proof is completed.
	
\end{proof}

%% file: appendix_theorem_two_level.tex
\section{Two-level Stochastic Optimizations}\label{App:Two}

\begin{lemma}[Theorem 1 in \cite{yang2023stability}]\label{theorem:general_two_level}
	If Assumption \ref{ass:Lipschitz continuous} (ii) holds true and the randomized algorithm $A$ is $\epsilon$-uniformly stable then
	\begin{equation*}
		\begin{aligned}
			\mathbb{E}_{S, A}[F(A(S))-F_{S}(A(S))] \leq L_{f} L_{g} \epsilon_{\nu}+4 L_{f} L_{g} \epsilon_{\omega}+L_{f} \sqrt{m^{-1} \mathbb{E}_{S, A}[\operatorname{Var}_{\omega}(g_{\omega}(A(S)))]},
		\end{aligned}
	\end{equation*}
	where the variance term  $	\operatorname{Var}_{\omega}(g_{\omega}(A(S))) = \mathbb{E}_{\omega}[\|g_{\omega}(A(S))-g(A(S))\|^{2}]$.
\end{lemma}

\begin{lemma}[Lemma 7 in \cite{qi2021stochastic}]\label{lemma:g_value_recursive}
	Let Assumption \ref{ass:Lipschitz continuous}(ii), \ref{ass:bound variance} (ii) and \ref{ass:Smoothness and Lipschitz continuous gradient} (ii) hold  for the empirical risk, and $x_t, u_t$ are generated by Algorithm \ref{alg_cover} \ref{alg_cover}, then we have 
	\begin{equation*}
		\EX [\|u_t - g_S(x_t)\|^2] \leq  (1-\beta_{t})\EX[\|u_{t-1} - g_S(x_{t-1})\|^2] + 2\beta_t^2\sigma_g^2 + 2L_g^2\|x_t-x_{t-1}\|^2.
	\end{equation*}
\end{lemma}

\begin{lemma}[Lemma 7 in \cite{qi2021stochastic}]\label{lemma:g_gradient_recursive}
	Let Assumption \ref{ass:Lipschitz continuous}(ii), \ref{ass:bound variance} (ii) and \ref{ass:Smoothness and Lipschitz continuous gradient} (ii) hold for the empirical risk, and $x_t, v_t$ are generated by Algorithm \ref{alg_cover}, then we have 
	\begin{equation*}
		\EX [\|v_t -\nabla g_S(x_t)\|^2] \leq  (1-\beta_{t})\EX[\|v_{t-1} -\nabla g_S(x_{t-1})\|^2] + 2\beta_t^2\sigma_{g'}^2 + 2L_g^2\|x_t-x_{t-1}\|^2.
	\end{equation*}
\end{lemma}

\begin{lemma}\label{lemma:u_t_bound_two}
	Let Assumption \ref{ass:Lipschitz continuous}(ii), \ref{ass:bound variance} (ii) and \ref{ass:Smoothness and Lipschitz continuous gradient} (ii) hold and $x_t,u_t$ are generated by Algorithm \ref{alg_cover}, let $ 0 <\eta_t = \eta < 1$ and let $ 0<\beta_t = \beta < 1$, for any $c>0$, we have 
	\begin{equation*}
		\begin{aligned}
			\EX [\|u_t - g_S(x_t)\|^2] &\leq (\frac{c}{e})^c(t\beta)^{-c}\EX[\|u_0-g_S(x_0)\|^2] + 2\sigma_g^2\beta + \frac{2L_g^4L_f^2\eta^2}{\beta}.
		\end{aligned}
	\end{equation*}
\end{lemma}
\begin{proof}
	According to the rule of update we have $ \EX[\|x_{t}-x_{t-1}\|^{2}] \leq L_{f}^{2} L_{g}^{2} \eta_{t-1}^{2} $, then using Lemma \ref{lemma:g_value_recursive} and \ref{lemma:general_recursive}, we have 
	\begin{equation*}
		\begin{aligned}
			\EX [\|u_t - g_S(x_t)\|^2] \leq \prod_{i=1}^{t}(1-\beta_i)\EX[\|u_0-g_S(x_0)\|^2] + 2\sigma_g^2\Upsilon_t\sum_{i=1}^t\frac{\beta_i^2}{\Upsilon_i} +  2L_g^4L_f^2\Upsilon_t\sum_{i=1}^t\frac{\eta_{i-1}^2}{\Upsilon_i}.
		\end{aligned}
	\end{equation*}
	For the term $\Upsilon_t\sum_{i=1}^t\beta_i^2/\Upsilon_i$,  according to the setting that $\beta_t = \beta$, we have 
	\begin{equation*}
		\begin{aligned}
			\Upsilon_t\sum_{i=1}^t\frac{\beta_i^2}{\Upsilon_i} = \beta (\Upsilon_t\frac{\beta_1}{\Upsilon_1} +  \Upsilon_t\sum_{i=2}^t\frac{\beta_i}{\Upsilon_i}) = \beta(\Upsilon_t\frac{\beta_1}{\Upsilon_1} + \Upsilon_t\sum_{i=2}^t(\frac{1}{\Upsilon_i}- \frac{1}{\Upsilon_{i-1}})) = \beta(1- \Upsilon_t) =  \beta.
		\end{aligned}
	\end{equation*}
	Then according to the setting that $\eta_t = \eta$, we have 
	\begin{equation*}
		\begin{aligned}
			\EX [\|u_t - g_S(x_t)\|^2] &\leq \prod_{i=1}^{t}(1-\beta_i)\EX[\|u_0-g_S(x_0)\|^2] + 2\sigma_g^2\beta + \frac{2L_g^4L_f^2\eta^2}{\beta}.
		\end{aligned}
	\end{equation*}
	Then using \eqref{Eq:prod_(1-beta_i)} and \eqref{Eq:sum_t^-z_bound}, we can get
	\begin{equation*}
		\begin{aligned}
			\EX [\|u_t - g_S(x_t)\|^2] &\leq (\frac{c}{e})^c(t\beta)^{-c}\EX[\|u_0-g_S(x_0)\|^2] + 2\sigma_g^2\beta + \frac{2L_g^4L_f^2\eta^2}{\beta}.
		\end{aligned}
	\end{equation*}
\end{proof}
And the proof of  $\EX [\|v_t -\nabla g_S(x_t)\|^2]$ is similarly to Lemma \ref{lemma:u_t_bound_two},  we won't repeat it.
\begin{lemma}\label{lemma:iv_t_bound}
	Let Assumption \ref{ass:Lipschitz continuous}(ii), \ref{ass:bound variance} (ii) and \ref{ass:Smoothness and Lipschitz continuous gradient} (ii) hold and $x_t, v_t$ are generated by Algorithm \ref{alg_cover}, let $ 0 <\eta_t \leq \eta < 1$ and let $ 0<\beta_t \leq \beta < 1$, for any $c>0$, we have 
	\begin{equation*}
		\EX [\|v_t -\nabla g_S(x_t)\|^2] \leq (\frac{c}{e})^c(t\beta)^{-c}\EX[\|v_0 - \nabla g_S(x_0)\|^2] + 2\sigma_{g'}^2\beta + \frac{2L_g^4L_f^2\eta^2}{\beta}.
	\end{equation*}
\end{lemma}

We first give some notations used in the two-level optimization to simplify our proof.

For any $k \in [n]$, let $S^{k, \nu} = \{\nu_{1}, \ldots, \nu_{k-1}, \nu_{k}^{\prime}, \nu_{k+1}, \ldots, \nu_{n}, \omega_{1}, \ldots, \omega_{m}\}$ be formed from $S_{\nu}$ by replacing the $k$-th element. Similarly, for any $l \in [m]$, define $S^{l, \omega} = \{\nu_{1}, \ldots, \nu_{n}, \omega_{1}, \ldots, \omega_{l-1}, \omega_{l}^{\prime}, \omega_{l+1}, \ldots, \omega_{m}\}$ as formed from $S_{\omega}$ by replacing the $l$-th element. Let $\{x_{t+1}\}$, \{$u_{t+1}$\} and $\{v_{t+1}\}$ be generated by COVER based on $S$, $\{x_{t+1}^{k, \nu}\}$, \{$u_{t+1}^{k,\nu}$\} and $\{v_{t+1}^{k,\nu}\}$ be generated by COVER based on $S^{k, \nu}$, $\{x_{t+1}^{l, \omega}\}$, \{$u_{t+1}^{l,\omega}$\} and $\{v_{t+1}^{l,\omega}\}$  be generated by COVER based on $S^{l, \omega}$. Set $x_{0}=x_{0}^{k, \nu}$ and $x_{0}=x_{0}^{l, \omega}$ as starting points in $\mathcal{X}$.

\subsection{Convex-setting}

\begin{proof}[Proof of Theorem \ref{thm:sta_convex_two_level}]
	Since a change in one sample data can occur in either $S_{\nu}$ or $S_{\omega}$, we estimate $\mathbb{E}_{A}[\|x_{t+1} - x_{t+1}^{k, \nu}\|]$ and $\mathbb{E}_{A}[\|x_{t+1} - x_{t+1}^{l, \omega}\|]$ as follows.
	
	\noindent{\bf Estimation of $\mathbb{E}_{A}\bigl[\|x_{t+1}-x_{t+1}^{k,\nu}\|\bigr]$}  
	
	We first give the estimation of $\mathbb{E}_{A}[\|x_{t+1} - x_{t+1}^{k, \nu}\|]$.
	For this purpose, we will consider two cases, i.e., $i_t \neq k$ and $i_t = k$.
	
	\textbf{\quad Case 1 ($i_t \neq k$).}  We have 
	\begin{equation}\label{Eq:two_level_case1_origin}
		\begin{aligned}
			&\|x_{t+1} - x_{t+1}^{k,\nu}\|^2 \\
			&\leq \|x_t - \eta_t v_{t} \nabla f_{\nu_{i_t}}(u_t) - x_{t}^{k,\nu} + \eta_t v_{t}^{k,\nu} \nabla f_{\nu_{i_t}}(u_t^{k,\nu})\|^2\\
			& \leq \|x_t - x_t^{k,\nu}\|^2 - 2\eta_t \langle v_{t}^{k,\nu} \nabla f_{\nu_{i_t}}(u_t^{k,\nu})- v_{t} \nabla f_{\nu_{i_t}}(u_t), x_t^{k,\nu} - x_t\rangle + \eta_t^2\|v_{t}^{k,\nu} \nabla f_{\nu_{i_t}}(u_t^{k,\nu}) - v_{t} \nabla f_{\nu_{i_t}}(u_t) \|^2.
		\end{aligned}
	\end{equation}
	We begin to  estimate the second term in \eqref{Eq:two_level_case1_origin}.
	\begin{equation}\label{Eq: decomposed_product_knu}
		\begin{aligned}
			&- 2\eta_t \langle v_{t}^{k,\nu} \nabla f_{\nu_{i_t}}(u_t^{k,\nu})- v_{t} \nabla f_{\nu_{i_t}}(u_t), x_t^{k,\nu} - x_t\rangle\\
 			& = - 2\eta_t \langle v_{t}^{k,\nu} \nabla f_{\nu_{i_t}}(u_t^{k,\nu}) - v_t^{k,\nu} \nabla  f_{\nu_{i_t}}(g_S(x_t^{k,\nu})), x_t^{k,\nu} - x_t\rangle\\
			& \quad  - 2\eta_t \langle v_t^{k,\nu} \nabla  f_{\nu_{i_t}}(g_S(x_t^{k,\nu})) - v_t^{k,\nu} \nabla f_S(g_S(x_t^{k,\nu})) , x_t^{k,\nu} - x_t  \rangle\\
			&\quad- 2\eta_t \langle v_t^{k,\nu} \nabla f_S(g_S(x_t^{k,\nu})) - \nabla g_S(x_t^{k,\nu})\nabla f_S(g_S(x_t^{k,\nu})) , x_t^{k,\nu} - x_t  \rangle \\
			&\quad- 2\eta_t \langle \nabla g_S(x_t^{k,\nu})\nabla f_S(g_S(x_t^{k,\nu})) - \nabla g_S(x_t)\nabla f_S(g_S(x_t)) , x_t^{k,\nu} - x_t  \rangle\\
			&\quad - 2\eta_t \langle \nabla g_S(x_t)\nabla f_S(g_S(x_t)) - v_t\nabla f_S(g_S(x_t)) , x_t^{k,\nu} - x_t  \rangle - 2\eta_t \langle v_t\nabla f_S(g_S(x_t)) - v_t\nabla f_S(u_t), x_t^{k,\nu} - x_t  \rangle\\
			&\quad - 2\eta_t \langle  v_t\nabla f_S(u_t) - v_t \nabla f_{\nu_{i_t}}(u_t), x_t^{k,\nu} - x_t  \rangle.
		\end{aligned}
	\end{equation}
	Now we estimate the terms on the right hand side of \eqref{Eq: decomposed_product_knu} one by one.
	
	For the first term of the RHS, we have 
	\begin{equation}\label{Eq: decomposed_product_knu_1_term}
		\begin{aligned}
			&- 2\eta_t \langle v_{t}^{k,\nu} \nabla f_{\nu_{i_t}}(u_t^{k,\nu}) - v_t^{k,\nu} \nabla  f_{\nu_{i_t}}(g_S(x_t^{k,\nu})), x_t^{k,\nu} - x_t\rangle\\
			& \leq  2\eta_t \| v_{t}^{k,\nu} \nabla f_{\nu_{i_t}}(u_t^{k,\nu}) - v_t^{k,\nu} \nabla  f_{\nu_{i_t}}(g_S(x_t^{k,\nu})) \| \cdot \|x_t^{k,\nu} - x_t \| \leq 2L_g C_f \eta_t \| u_t^{k,\nu} - g_S(x_t^{k,\nu})\|\cdot\|x_t^{k,\nu} - x_t \|.
		\end{aligned}
	\end{equation}
	For the second term of the RHS, according to $\EX_{j_t}[v_t^{k,\nu} \nabla  f_{\nu_{i_t}}(g_S(x_t^{k,\nu}))] = v_t^{k,\nu} \nabla f_S(g_S(x_t^{k,\nu}))$, we have
	\begin{equation}\label{Eq: decomposed_product_knu_2_term}
		\begin{aligned}
			&- 2\eta_t \EX_{j_t}[ \langle v_t^{k,\nu} \nabla  f_{\nu_{i_t}}(g_S(x_t^{k,\nu})) - v_t^{k,\nu} \nabla f_S(g_S(x_t^{k,\nu})) , x_t^{k,\nu} - x_t  \rangle] = 0.
		\end{aligned}
	\end{equation}
	For the third term of the RHS, we have 
	\begin{equation}\label{Eq: decomposed_product_knu_3_term}
		\begin{aligned}
			&- 2\eta_t \langle v_t^{k,\nu} \nabla f_S(g_S(x_t^{k,\nu})) - \nabla g_S(x_t^{k,\nu})\nabla f_S(g_S(x_t^{k,\nu})) , x_t^{k,\nu} - x_t  \rangle\\
			& \leq 2\eta_t \|\nabla f_S(g_S(x_t^{k,\nu})) (v_t^{k,\nu} -  \nabla g_S(x_t^{k,\nu}))\|\cdot \|x_t^{k,\nu} - x_t\| \leq 2\eta_t L_f\|v_t^{k,\nu} -  \nabla g_S(x_t^{k,\nu})\| \cdot \|x_t^{k,\nu} - x_t\|
		\end{aligned}
	\end{equation}
	Then according to Assumption \ref{ass:Smoothness and Lipschitz continuous gradient}, for the fourth term of the RHS, we have 
	\begin{equation} \label{Eq: decomposed_product_knu_4_term}
		\begin{aligned}
			& 2\eta_t \langle \nabla g_S(x_t^{k,\nu})\nabla f_S(g_S(x_t^{k,\nu})) - \nabla g_S(x_t)\nabla f_S(g_S(x_t)) , x_t^{k,\nu} - x_t  \rangle \\
			&\geq \frac{2\eta_t}{L}\|\nabla g_S(x_t^{k,\nu})\nabla f_S(g_S(x_t^{k,\nu})) - \nabla g_S(x_t)\nabla f_S(g_S(x_t))\|^2.
		\end{aligned}
	\end{equation}
	Analogous to the above four terms, we can easily get 
	\begin{align}
		&-2\eta_t\langle \nabla g_S(x_t)\nabla f_S(g_S(x_t)) - v_t\nabla f_S(g_S(x_t)) , x_t^{k,\nu} - x_t  \rangle \leq 2L_f\eta_t\|v_t - \nabla g_S(x_t)\|\cdot\|x_t^{k,\nu} - x_t \|, \label{Eq:decomposed_product_knu_5_term} \\  
		&-2\eta_t \langle v_t\nabla f_S(g_S(x_t)) - v_t\nabla f_S(u_t), x_t^{k,\nu} - x_t  \rangle \leq 2L_gC_f\eta_t\|u_t-g_S(x_t)\|\cdot \|x_t^{k,\nu}-x_t\|, \label{Eq:decomposed_product_knu_6_term}\\
		&- 2\eta_t \EX_{j_t}[ \langle v_t \nabla  f_{\nu_{i_t}}(g_S(x_t)) - v_t \nabla f_S(g_S(x_t)) , x_t^{k,\nu} - x_t  \rangle] = 0. \label{Eq:decomposed_product_knu_7_term}
	\end{align}

	Putting \eqref{Eq: decomposed_product_knu_1_term} - \eqref{Eq:decomposed_product_knu_7_term} into \eqref{Eq: decomposed_product_knu} we have 
	\begin{equation}\label{Eq:two_level_case1_origin_term2}
		\begin{aligned}
			&- 2\eta_t \EX_{j_t}[\langle v_{t}^{k,\nu} \nabla f(u_t^{k,\nu})- v_{t} \nabla f(u_t), x_t^{k,\nu} - x_t\rangle] \\
			& \leq 2L_g C_f \eta_t \| u_t^{k,\nu} - g_S(x_t^{k,\nu})\|\cdot\|x_t^{k,\nu} - x_t \| + 2L_f\eta_t\|v_t^{k,\nu} - \nabla g_S(x_t^{k,\nu})\|\cdot\|x_t^{k,\nu} - x_t \| \\
			&\quad - \frac{2 \eta_t}{L}\|\nabla g_S(x_t^{k,\nu})\nabla f_S(g_S(x_t^{k,\nu})) - \nabla g_S(x_t)\nabla f_S(g_S(x_t))\|^2 \\
			&\quad+ 2L_gC_f\eta_t\|u_t-g_S(x_t)\|\cdot \|x_t^{k,\nu}-x_t\| + 2L_f\eta_t\|u_t-g_S(x_t)\|\cdot \|x_t^{k,\nu}-x_t\|.
		\end{aligned}
	\end{equation}
	Now we begin to bound the third term of the RHS in \eqref{Eq:two_level_case1_origin}.
	\begin{equation*}
		\begin{aligned}
			&\|v_{t}^{k,\nu} \nabla f_{\nu_{i_t}}(u_t^{k,\nu}) - v_{t} \nabla f_{\nu_{i_t}}(u_t)\|\\
			& \leq \|v_{t}^{k,\nu} \nabla f_{\nu_{i_t}}(u_t^{k,\nu}) - v_t^{k,\nu} \nabla f_S(u_t^{k,\nu}) \| +  \| v_t^{k,\nu} \nabla f_S(u_t^{k,\nu}) - v_t^{k,\nu} \nabla f_S(g_S(x_t^{k,\nu}))\| \\
			&\quad  +\| v_t^{k,\nu} \nabla f_S(g_S(x_t^{k,\nu})) - \nabla g_S(x_t^{k,\nu})\nabla f_S(g_S(x_t^{k,\nu}))\| +\| \nabla g_S(x_t^{k,\nu})\nabla f_S(g_S(x_t^{k,\nu})) - \nabla g_S(x_t)\nabla f_S(g_S(x_t))  \|\\
			&\quad  +\| \nabla g_S(x_t)\nabla f_S(g_S(x_t)) - v_t\nabla f_S(g_S(x_t))\|    +\| v_t\nabla f_S(g_S(x_t)) - v_t\nabla f_S(u_t) \| + \|v_t\nabla f_S(u_t) - v_{t} \nabla f_{\nu_{i_t}}(u_t)\| .
		\end{aligned}
	\end{equation*}
	Now because of the fact that $(\sum_{i=1}^ka_i)^2 \leq k \sum_{i=1}^ka_i^2$, we have 
	\begin{equation}\label{Eq:two_level_case1_origin_term3}
		\begin{aligned}
			&\eta_t^2\|v_{t}^{k,\nu} \nabla f(u_t^{k,\nu}) - v_{t} \nabla f(u_t) \|^2\\
			&\leq 7\eta_t^2L_g^2\| \nabla f_{\nu_{i_t}}(u_t^{k,\nu}) -\nabla f_S(u_t^{k,\nu})\|^2  + 7\eta_t^2L_g^2C_f^2\|u_t^{k,\nu}-g_S(x_t^{k,\nu})\|^2+ 7\eta_t^2L_f^2\|v_t^{k,\nu}-\nabla g_S(x_t^{k,\nu})\|^2\\
			& \quad + 7\eta_t^2 \| \nabla g_S(x_t^{k,\nu})\nabla f_S(g_S(x_t^{k,\nu})) - \nabla g_S(x_t)\nabla f_S(g_S(x_t))  \|^2\\
			& \quad + 7\eta_t^2L_f^2\|\nabla g_S(x_t) - v_t\|^2 + 7\eta_t^2L_g^2C_f^2\|u_t-g_S(x_t)\|^2 + 7\eta_t^2L_g^2\| \nabla f_{\nu_{i_t}}(u_t) -\nabla f_S(u_t)\|^2.
		\end{aligned}
	\end{equation}
	Putting \eqref{Eq:two_level_case1_origin_term2} and \eqref{Eq:two_level_case1_origin_term3} into \eqref{Eq:two_level_case1_origin}, we have 
	\begin{equation*}
		\begin{aligned}
			&\EX_{j_t}[\|x_{t+1}-x_{t+1}^{k,\nu}\|^2] \\
			& \leq \|x_t-x_t^{k,\nu}\|^2 +  2L_g C_f \eta_t \| u_t^{k,\nu} - g_S(x_t^{k,\nu})\|\cdot\|x_t^{k,\nu} - x_t \| + 2L_f\eta_t\|v_t^{k,\nu} - \nabla g_S(x_t^{k,\nu})\|\cdot\|x_t^{k,\nu} - x_t \| \\
			&\quad+ 2L_gC_f\eta_t\|u_t-g_S(x_t)\|\cdot \|x_t^{k,\nu}-x_t\| + 2L_f\eta_t\|u_t-g_S(x_t)\|\cdot \|x_t^{k,\nu}-x_t\|\\
			&\quad+ 7\eta_t^2L_g^2\| \nabla f_{\nu_{i_t}}(u_t^{k,\nu}) -\nabla f_S(u_t^{k,\nu})\|^2  + 7\eta_t^2L_g^2C_f^2\|u_t^{k,\nu}-g_S(x_t^{k,\nu})\|^2+ 7\eta_t^2L_f^2\|v_t^{k,\nu}-\nabla g_S(x_t^{k,\nu})\|^2\\
			& \quad + 7\eta_t^2L_f^2\|\nabla g_S(x_t) - v_t\|^2 + 7\eta_t^2L_g^2C_f^2\|u_t-g_S(x_t)\|^2 + 7\eta_t^2L_g^2\| \nabla f_{\nu_{i_t}}(u_t) -\nabla f_S(u_t)\|^2.
		\end{aligned}
	\end{equation*}
	where we use $\eta_t \leq \frac{2}{7L}$ in the inequality.
	
	\textbf{Case 2} ($i_t = k$). We have 
	\begin{equation}\label{Eq:two_level_case2_origin}
		\begin{aligned}
			\|x_{t+1} - x_{t+1}^{k,\nu}\| & = \|x_t - \eta_t v_{t} \nabla f_{\nu_{i_t}}(u_t) - x_{t}^{k,\nu} + \eta_t v_{t}^{k,\nu} \nabla f_{\nu_{i_t}}(u_t^{k,\nu})\|\\
			& \leq \|x_t - x_t^{k,\nu}\|  + \eta_t\|v_{t}^{k,\nu} \nabla f_{\nu_{i_t}}(u_t^{k,\nu}) - v_{t} \nabla f_{\nu_{i_t}}(u_t) \| \leq \|x_t - x_t^{k,\nu}\| + 2\eta_t L_g L_f,
		\end{aligned}
	\end{equation}
	where the first inequality holds by Assumption \ref{ass:Lipschitz continuous}, then we have 
	\begin{equation*}
		\|x_{t+1} - x_{t+1}^{k,\nu}\|^2 \leq \|x_t - x_t^{k,\nu}\|^2 + 4\eta_t L_g L_f\|x_t - x_t^{k,\nu}\| + 4\eta_t^2L_g^2L_f^2.
	\end{equation*}
	Combining above \textbf{Case 1} and \textbf{Case 2}, we can get
	\begin{equation*}
		\begin{aligned}
			&\|x_{t+1}-x_{t+1}^{k,\nu}\|^2 \\
			& \leq \|x_t-x_t^{k,\nu}\|^2 +  2L_g C_f \eta_t \| u_t^{k,\nu} - g_S(x_t^{k,\nu})\|\cdot\|x_t^{k,\nu} - x_t \| + 2L_f\eta_t\|v_t^{k,\nu} - \nabla g_S(x_t^{k,\nu})\|\cdot\|x_t^{k,\nu} - x_t \| \\
			&\quad+ 2L_gC_f\eta_t\|u_t-g_S(x_t)\|\cdot \|x_t^{k,\nu}-x_t\| + 2L_f\eta_t\|u_t-g_S(x_t)\|\cdot \|x_t^{k,\nu}-x_t\|\\
			&\quad+ 7\eta_t^2L_g^2\| \nabla f_{\nu_{i_t}}(u_t^{k,\nu}) -\nabla f_S(u_t^{k,\nu})\|^2  + 7\eta_t^2L_g^2C_f^2\|u_t^{k,\nu}-g_S(x_t^{k,\nu})\|^2+ 7\eta_t^2L_f^2\|v_t^{k,\nu}-\nabla g_S(x_t^{k,\nu})\|^2\\
			& \quad + 7\eta_t^2L_f^2\|\nabla g_S(x_t) - v_t\|^2 + 7\eta_t^2L_g^2C_f^2\|u_t-g_S(x_t)\|^2 + 7\eta_t^2L_g^2\| \nabla f_{\nu_{i_t}}(u_t) -\nabla f_S(u_t)\|^2\\
			&\quad + 4\eta_t L_g L_f\|x_t^{k,\nu} - x_t\| \cdot \1_{i_{t} = k} + 4\eta_t^2L_g^2L_f^2 \cdot \1_{i_{t} = k}.
		\end{aligned}
	\end{equation*}
	
	According to Cauchy-Schwarz inequality, we can get 
	\begin{equation*}
		\begin{aligned}
			\EX_A[\|x_{t+1}-x_{t+1}^{k,\nu}\|^2]
			&\leq \EX_A[ \|x_t-x_t^{k,\nu}\|^2]+  2L_g C_f \eta_t(\EX_A [\| u_t^{k,\nu} - g_S(x_t^{k,\nu})\|^2])^{1/2} \cdot (\EX_A\|x_t^{k,\nu} - x_t \|^2])^{1/2}\\
			&\quad +  2L_f\eta_t(\EX_A[\|v_t^{k,\nu} - \nabla g_S(x_t^{k,\nu})\|^2])^{1/2} \cdot (\EX_A\|x_t^{k,\nu} - x_t \|^2])^{1/2}\\
			&\quad + 2L_gC_f\eta_t(\EX_A[\|u_t-g_S(x_t)\|^2])^{1/2}\cdot (\EX[\|x_t^{k,\nu}-x_t\|^2])^{1/2}\\
			&\quad+ 2L_f\eta_t(\EX_A[\|v_t - \nabla g_S(x_t)\|^2])^{1/2}\cdot(\EX_A[\|x_t^{k,\nu} - x_t \|^2])^{1/2}\\
			&\quad+ 7\eta_t^2L_g^2C_f^2\EX_A[\|u_t^{k,\nu}-g_S(x_t^{k,\nu})\|^2] + 7\eta_t^2L_f^2\EX_A[\|v_t^{k,\nu}-\nabla g_S(x_t^{k,\nu})\|^2]\\
			& \quad + 7\eta_t^2L_f^2\EX_A[\|\nabla g_S(x_t) - v_t\|^2] + 7\eta_t^2L_g^2C_f^2\EX_A[\|u_t-g_S(x_t)\|^2] + 14\eta_t^2L_f^2\sigma_f^2\\
			&\quad+ 4\eta_t L_g L_f\EX_A[\|x_t^{k,\nu} - x_t\| \cdot \1_{i_{t} = k}] + 4\eta_t^2L_g^2L_f^2 \cdot \EX_A[\1_{i_{t} = k}].
		\end{aligned}
	\end{equation*}
	Besides, according to 
	\begin{equation}\label{Eq:case2_ex}
		\mathbb{E}_{A}[\|x_t^{k,\nu} - x_t\| \1_{[i_{t}=k]}]=\mathbb{E}_{A}[\|x_t^{k,\nu} - x_t\| \mathbb{E}_{i_{t}}[\1_{[i_{t}=k]}]]=\frac{1}{n} \mathbb{E}_{A}[\|x_t^{k,\nu} - x_t\|] \leq \frac{1}{n}(\mathbb{E}_{A}[\|x_t^{k,\nu} - x_t\|^{2}])^{1 / 2},
	\end{equation}
	we can get
	\begin{equation*}
		\begin{aligned}
			&\EX_A[\|x_{t+1}-x_{t+1}^{k,\nu}\|^2] \\
			&\leq  2L_g C_f \sum_{j=0}^t \eta_j((\EX_A [\| u_j^{k,\nu} - g_S(x_j^{k,\nu})\|^2])^{1/2} + (\EX_A [\| u_j - g_S(x_j)\|^2])^{1/2}) \cdot (\EX_A\|x_j - x_j^{k,\nu} \|^2])^{1/2} \\
			&\quad+  2L_f \sum_{j=0}^t \eta_j((\EX_A[\|v_j^{k,\nu} - \nabla g_S(x_j^{k,\nu})\|^2])^{1/2} + (\EX_A[\|v_j - \nabla g_S(x_j)\|^2])^{1/2}) \cdot (\EX_A\|x_j - x_j^{k,\nu} \|^2])^{1/2}\\
			&\quad+ 7 \sum_{j=0}^t \eta_j^2L_g^2C_f^2\EX_A[\|u_j-g_S(x_j)\|^2] +  7 \sum_{j=0}^t \eta_j^2L_g^2C_f^2\EX_A[\|u_j^{k,\nu}-g_S(x_j^{k,\nu})\|^2]\\
			&\quad + 7 \sum_{j=0}^t \eta_j^2L_f^2\EX_A[\|v_j^{k,\nu}-\nabla g_S(x_j^{k,\nu})\|^2] + 7 \sum_{j=0}^t \eta_j^2L_f^2\EX_A[\|v_j - \nabla g_S(x_j)\|^2] + 14L_f^2\sigma_f^2 \sum_{j=0}^t\eta_j^2 \\
			&\quad  + \frac{4 L_g L_f}{n} \sum_{j=0}^t \eta_j(\EX_A[\|x_j - x_j^{k,\nu}\|^2])^{1/2} + \frac{4L_g^2L_f^2}{n} \sum_{j=0}^t \eta_j^2.
		\end{aligned}
	\end{equation*}
	For notational convenience, we denote by $ u_{t}=(\mathbb{E}_{A}[\|x_{t}-x_{t}^{k, \nu}\|^{2}])^{1 / 2}$, define 
	\begin{equation*}
		\begin{aligned}
			\alpha_j& = 2L_gC_f\eta_j(\EX_A [\| u_j^{k,\nu} - g_S(x_j^{k,\nu})\|^2])^{1/2} + 2L_gC_f\eta_j(\EX_A [\| u_j - g_S(x_j)\|^2])^{1/2}\\
			&\quad+  2L_f\eta_j(\EX_A[\|v_j - \nabla g_S(x_j)\|^2])^{1/2} + 2L_f\eta_j(\EX_A[\|v_j^{k,\nu} - \nabla g_S(x_j^{k,\nu})\|^2])^{1/2} +  \frac{4L_gL_f}{n}\eta_j,
		\end{aligned}
	\end{equation*}
	and 
	\begin{equation*}
		\begin{aligned}
			S_t & =  7 \sum_{j=0}^{t-1} \eta_j^2L_g^2C_f^2\EX_A[\|u_j^{k,\nu}-g_S(x_j^{k,\nu})\|^2]+ 7 \sum_{j=0}^{t-1} \eta_j^2L_g^2C_f^2\EX_A[\|u_j-g_S(x_j)\|^2] + 14L_f^2\sigma_f^2 \sum_{j=0}^{t-1}\eta_j^2 \\
			&\quad+ 7 \sum_{j=0}^{t-1} \eta_j^2L_f^2\EX_A[\|v_j^{k,\nu}-\nabla g_S(x_j^{k,\nu})\|^2] + 7 \sum_{j=0}^{t-1} \eta_j^2L_f^2\EX_A[\|\nabla g_S(x_j) - v_j\|^2]  + \frac{4L_g^2L_f^2}{n} \sum_{j=0}^{t-1} \eta_j^2.
		\end{aligned}
	\end{equation*}
	
	Using Lemma \ref{lemma:recursion lemma}, we can get 
	\begin{equation*}
		\begin{aligned}
			u_t &\leq \sqrt{S_t}+\sum_{j=1}^{t-1}\alpha_j\\
			& \leq   (7 L_g^2C_f^2 \sum_{j=0}^{t-1} \eta_j^2\EX_A[\|u_j^{k,\nu}-g_S(x_j^{k,\nu})\|^2])^{1/2} + (7L_g^2C_f^2 \sum_{j=0}^{t-1} \eta_j^2\EX_A[\|u_j-g_S(x_j)\|^2])^{1/2} + (14L_f^2\sigma_f^2 \sum_{j=0}^{t-1}\eta_j^2 )^{\frac{1}{2}} \\
			&\quad+ (7 L_f^2\sum_{j=0}^{t-1} \eta_j^2\EX_A[\|v_j^{k,\nu}-\nabla g_S(x_j^{k,\nu})\|^2])^{1/2} + (7L_f^2 \sum_{j=0}^{t-1} \eta_j^2\EX_A[\|\nabla g_S(x_j) - v_j\|^2])^{1/2}+ (\frac{4L_g^2L_f^2}{n} \sum_{j=0}^{t-1} \eta_j^2)^{1/2} \\
			&\quad+ 2L_gC_f\sum_{j=1}^{t-1}\eta_j(\EX_A [\| u_j^{k,\nu} - g_S(x_j^{k,\nu})\|^2])^{1/2} + 2L_gC_f\sum_{j=1}^{t-1}\eta_j(\EX_A [\| u_j - g_S(x_j)\|^2])^{1/2}\\
			&\quad+  2L_f\sum_{j=1}^{t-1}\eta_j(\EX_A[\|v_j - \nabla g_S(x_j)\|^2])^{1/2} + 2L_f\sum_{j=1}^{t-1}\eta_j(\EX_A[\|v_j^{k,\nu} - \nabla g_S(x_j^{k,\nu})\|^2])^{1/2}+ \frac{4L_gL_f}{n}\sum_{j=1}^{t-1}\eta_j\\
		\end{aligned}
	\end{equation*}
	Then according to the inequality that $(\sum_{i=1}^{k} a_{i})^{1 / 2} \leq \sum_{i=1}^{k}(a_{i})^{1 / 2}$ we can get 
	\begin{equation*}
		\begin{aligned}
			u_t &\leq 5L_gC_f \sum_{j=0}^{t}\eta_j (\EX_A[\|u_j^{k,\nu}-g_S(x_j^{k,\nu})\|^2])^{1/2}+  5L_gC_f \sum_{j=0}^{t-1} \eta_j (\EX_A[\|u_j-g_S(x_j)\|^2])^{1/2}\\
			&\quad+ 5L_f  \sum_{j=0}^{t}\eta_j(\EX_A[\|v_j^{k,\nu}-\nabla g_S(x_j^{k,\nu})\|^2])^{1/2}  + 5L_f \sum_{j=0}^{t-1} \eta_j(\EX_A[\|v_j - \nabla g_S(x_j) \|^2])^{1/2} \\
			&\quad+(14L_f^2\sigma_f^2 \sum_{j=0}^{t-1}\eta_j^2 )^{\frac{1}{2}}+(\frac{4L_g^2L_f^2}{n} \sum_{j=0}^{t-1} \eta_j^2)^{1/2} + \frac{4L_gL_f}{n}\sum_{j=1}^{t-1}\eta_j.
		\end{aligned}
	\end{equation*}
	By setting $\eta_t = \eta$, with T iterations, we have
	\begin{equation*}
		\begin{aligned}
			u_T &\leq 10L_gC_f\sup_S\eta\sum_{j=0}^{T-1}(\EX_A [\|u_j-g_S(x_j)\|^2])^{1/2}+10L_f\sup_S\eta\sum_{j=0}^{T-1}(\EX_A [\|v_j-\nabla g_S(x_j)\|^2])^{1/2}\\
			&\quad+ 4L_f\sigma_f\eta \sqrt{T}  + \frac{2L_gL_f\eta\sqrt{T}}{\sqrt{n}} + \frac{4L_gL_f\eta T}{n}\\
			&\leq 10L_gC_f\sup_S\eta\sum_{j=0}^{T-1}(\EX_A [\|u_j-g_S(x_j)\|^2])^{1/2}+10L_f\sup_S\eta\sum_{j=0}^{T-1}(\EX_A [\|v_j-\nabla g_S(x_j)\|^2])^{1/2}\\
			&\quad + 4L_f\sigma_f\eta \sqrt{T} + \frac{6L_gL_f\eta T}{n},
		\end{aligned}
	\end{equation*}
	where the first inequality holds by $\sum_{j=1}^{t}\eta_j (\EX_A[\|u_j^{k,\nu}-g_S(x_j^{k,\nu})\|^2])^{1/2}\leq \sup_S\eta \sum_{j=1}^{t}(\EX_A [\|u_j-g_S(x_j)\|^2])^{1/2}$ and $\sum_{j=1}^t \eta_j (\EX_A[\|u_j-g_S(x_j)\|^2])^{1/2} \leq \sup_S\eta (\EX_A[\|u_j-g_S(x_j)\|^2])^{1/2}$. The other terms to the RHS are treated similarly. And the second inequality follows by the fact that we often have $n \leq T$, therefore $\sqrt{\frac{T}{n}} \leq \frac{T}{n}$. We further get 
	\begin{equation}\label{Eq:x_T-X_Tkv}
		\begin{aligned}
			&\EX_A[\|x_T - x_T^{k,\nu}\|] \leq u_T \\
			&\leq 10L_gC_f\sup_S\eta\sum_{j=0}^{T-1}(\EX_A [\|u_j-g_S(x_j)\|^2])^{1/2}+10L_f\sup_S\eta\sum_{j=0}^{T-1}(\EX_A [\|v_j-\nabla g_S(x_j)\|^2])^{1/2}\\
			&\quad + 4L_f\sigma_f\eta \sqrt{T} + \frac{6L_gL_f\eta T}{n}.
		\end{aligned}
	\end{equation}
	Then we can get the following result
	\begin{equation*}
		\begin{aligned}
			\EX_A[\|x_T - x_T^{k,\nu}\|]= O\Big(\frac{L_gL_f\eta T}{n} & + L_f\sigma_f\eta \sqrt{T} + L_gC_f\sup_S\eta\sum_{j=0}^{T-1}((\EX_A [\|u_j-g_S(x_j)\|^2])^{1/2} \\
   &+(\EX_A [\|v_j-\nabla g_S(x_j)\|^2])^{1/2}\Big ).
		\end{aligned}
	\end{equation*}
	
	\noindent{\bf Estimation of $\mathbb{E}_{A}\bigl[\|x_{t+1} - x_{t+1}^{l,\omega}\|\bigr]$}  
	
	Next we give the  estimation of $\EX_A[\|x_{t+1} - x_{t+1}^{l,\omega}]$. Similarly,  we consider two cases, $j_t \neq l$ and $j_t = l$.
	
	\textbf{\quad Case 1 ($j_t \neq l$).} We have 
	\begin{equation}\label{Eq:two_level_case3_origin}
	\begin{aligned}
		&\|x_{t+1} - x_{t+1}^{l,\omega}\|^2 \\
		&\leq \|x_t - \eta_t v_{t} \nabla f_{\nu_{i_t}}(u_t) - x_{t}^{l,\omega} + \eta_t v_{t}^{l,\omega} \nabla f_{\nu_{i_t}}(u_t^{l,\omega})\|^2\\
		& \leq \|x_t - x_t^{l,\omega}\|^2 - 2\eta_t \langle v_{t}^{l,\omega} \nabla f_{\nu_{i_t}}(u_t^{l,\omega})- v_{t} \nabla f_{\nu_{i_t}}(u_t), x_t^{l,\omega} - x_t\rangle + \eta_t^2\|v_{t}^{l,\omega} \nabla f_{\nu_{i_t}}(u_t^{l,\omega}) - v_{t} \nabla f_{\nu_{i_t}}(u_t) \|^2.
	\end{aligned}
\end{equation}
Similarly to the process of \eqref{Eq:two_level_case1_origin}, we have 
	\begin{equation*}
	\begin{aligned}
		&\EX_{j_t}[\|x_{t+1}-x_{t+1}^{l,\omega}\|^2] \\
		& \leq \|x_t-x_t^{l,\omega}\|^2 +  2L_g C_f \eta_t \| u_t^{l,\omega} - g_S(x_t^{l,\omega})\|\cdot\|x_t^{l,\omega} - x_t \| + 2L_f\eta_t\|v_t^{l,\omega} - \nabla g_S(x_t^{l,\omega})\|\cdot\|x_t^{l,\omega} - x_t \| \\
		&\quad+ 2L_gC_f\eta_t\|u_t-g_S(x_t)\|\cdot \|x_t^{l,\omega}-x_t\| + 2L_f\eta_t\|u_t-g_S(x_t)\|\cdot \|x_t^{l,\omega}-x_t\|\\
		&\quad+ 7\eta_t^2L_g^2\| \nabla f_{\nu_{i_t}}(u_t^{l,\omega}) -\nabla f_S(u_t^{l,\omega})\|^2  + 7\eta_t^2L_g^2C_f^2\|u_t^{l,\omega}-g_S(x_t^{l,\omega})\|^2+ 7\eta_t^2L_f^2\|v_t^{l,\omega}-\nabla g_S(x_t^{l,\omega})\|^2\\
		& \quad + 7\eta_t^2L_f^2\|\nabla g_S(x_t) - v_t\|^2 + 7\eta_t^2L_g^2C_f^2\|u_t-g_S(x_t)\|^2 + 7\eta_t^2L_g^2\| \nabla f_{\nu_{i_t}}(u_t) -\nabla f_S(u_t)\|^2.
	\end{aligned}
\end{equation*}
where we use $\eta_t \leq \frac{2}{7L}$ in the inequality.
	
	\textbf{\quad Case 2 ($j_t = l$).}  We have 
	
	\begin{equation}\label{Eq:two_level_case4_origin}
	\begin{aligned}
		\|x_{t+1} - x_{t+1}^{l,\omega}\| & = \|x_t - \eta_t v_{t} \nabla f_{\nu_{i_t}}(u_t) - x_{t}^{l,\omega} + \eta_t v_{t}^{l,\omega} \nabla f_{\nu_{i_t}}(u_t^{l,\omega})\|\\
		& \leq \|x_t - x_t^{l,\omega}\|  + \eta_t\|v_{t}^{l,\omega} \nabla f_{\nu_{i_t}}(u_t^{l,\omega}) - v_{t} \nabla f_{\nu_{i_t}}(u_t) \| \leq \|x_t - x_t^{l,\omega}\| + 2\eta_t L_g L_f,
	\end{aligned}
\end{equation}
	where the first inequality holds by Assumption \ref{ass:Lipschitz continuous}, then we have 
	\begin{equation*}
		\|x_{t+1} - x_{t+1}^{l,\omega}\|^2 \leq \|x_t - x_t^{l,\omega}\|^2 + 4\eta_t L_g L_f\|x_t - x_t^{l,\omega}\| + 4\eta_t^2L_g^2L_f^2.
	\end{equation*}
	Combining \textbf{Case 1} and\textbf{ Case 2} we have 
	\begin{equation*}
	\begin{aligned}
		&\|x_{t+1}-x_{t+1}^{l,\omega}\|^2 \\
		& \leq \|x_t-x_t^{l,\omega}\|^2 +  2L_g C_f \eta_t \| u_t^{l,\omega} - g_S(x_t^{l,\omega})\|\cdot\|x_t^{l,\omega} - x_t \| + 2L_f\eta_t\|v_t^{l,\omega} - \nabla g_S(x_t^{l,\omega})\|\cdot\|x_t^{l,\omega} - x_t \| \\
		&\quad+ 2L_gC_f\eta_t\|u_t-g_S(x_t)\|\cdot \|x_t^{l,\omega}-x_t\| + 2L_f\eta_t\|u_t-g_S(x_t)\|\cdot \|x_t^{l,\omega}-x_t\|\\
		&\quad + 7\eta_t^2L_g^2C_f^2\|u_t^{l,\omega}-g_S(x_t^{l,\omega})\|^2+ 7\eta_t^2L_f^2\|v_t^{l,\omega}-\nabla g_S(x_t^{l,\omega})\|^2+ 7\eta_t^2L_f^2\|\nabla g_S(x_t) - v_t\|^2\\
		& \quad  + 7\eta_t^2L_g^2C_f^2\|u_t-g_S(x_t)\|^2 + 14\eta_t^2L_g^2\sigma_f^2 + 4\eta_t L_g L_f\|x_t^{l,\omega} - x_t\| \cdot \1_{j_{t} = l} + 4\eta_t^2L_g^2L_f^2 \cdot \1_{j_{t} = l}.
	\end{aligned}
\end{equation*}
	Besides, according to the fact that
	\begin{equation*}
		\mathbb{E}_{A}[\|x_t^{l,\omega} - x_t\| \1_{[j_{t}=l]}]=\mathbb{E}_{A}[\|x_t^{l,\omega} - x_t\| \mathbb{E}_{j_{t}}[\1_{[j_{t}=l]}]]=\frac{1}{m} \mathbb{E}_{A}[\|x_t^{l,\omega} - x_t\|] \leq \frac{1}{m}(\mathbb{E}_{A}[\|x_t^{l,\omega} - x_t\|^{2}])^{1 / 2}.
	\end{equation*}
	Then similarly to the estimation of $\mathbb{E}_{A}[\|x_{t+1} - x_{t+1}^{k, \nu}\|]$, we have 
	
	\begin{equation}\label{Eq:x_T-X_Tlw}
	\begin{aligned}
		&\EX_A[\|x_T - x_T^{l,\omega}\|] \leq u_T \\
		&\leq 10L_gC_f\sup_S\eta\sum_{j=0}^{T-1}(\EX_A [\|u_j-g_S(x_j)\|^2])^{1/2}+10L_f\sup_S\eta\sum_{j=0}^{T-1}(\EX_A [\|v_j-\nabla g_S(x_j)\|^2])^{1/2}\\
		&\quad + 4L_f\sigma_f\eta \sqrt{T} + \frac{6L_gL_f\eta T}{m}.
	\end{aligned}
\end{equation}
Then we can get the following result
\begin{equation*}
	\begin{aligned}
		\EX_A[\|x_T - x_T^{l,\omega}\|]= O\Big(\frac{L_gL_f\eta T}{n} & + L_f\sigma_f\eta \sqrt{T} + L_gC_f\sup_S\eta\sum_{j=0}^{T-1}((\EX_A [\|u_j-g_S(x_j)\|^2])^{1/2}\\
  &+(\EX_A [\|v_j-\nabla g_S(x_j)\|^2])^{1/2}\Big ).
	\end{aligned}
\end{equation*}
	
	Now we combine the above two estimations, we can conclude that 
	\begin{equation*}
		\begin{aligned}
			\epsilon_\nu + \epsilon_\omega =& O\big(  L_f\sigma_f\eta \sqrt{T} + L_gC_f\sup_S\eta\sum_{j=0}^{T-1}((\EX_A [\|u_j-g_S(x_j)\|^2])^{1/2} +(\EX_A [\|v_j-\nabla g_S(x_j)\|^2])^{1/2}\\
			&\quad + \frac{L_gL_f\eta T}{m} +\frac{L_gL_f\eta T}{n}\big).
		\end{aligned}
	\end{equation*}
	This completes the proof.
\end{proof}

\begin{corollary}[Two-level Optimization]\label{cor:1_two_level}
Consider  Algorithm \ref{alg_cover} with $\eta_t=\eta\leq \frac{1}{4L}$ and $\beta_t=\beta< \min{\{1/8C_f^2, 1\}}$, for any $t\in [0,T-1]$. With the output $A(S) =x_{T}$, then $\epsilon_\nu + \epsilon_\omega$ satisfies
\begin{equation*}
	\begin{aligned}
		O\Big( \eta T \big(  (\beta T)^{-\frac{c}{2}} + \beta^{1/2} + \eta \beta^{-1/2}\big)+ \eta T (\frac{1}{m} +  \frac{1}{n})  \Big).
	\end{aligned}
\end{equation*}
\end{corollary}

\begin{proof}[Proof of Corollary \ref{cor:1_two_level}]
	Considering the upadte rule of Algorithm \ref{alg_cover}, according to Lemma \ref{lemma:u_t_bound_two} and \ref{lemma:iv_t_bound} we have 
	\begin{equation*}
		\begin{aligned}
			\epsilon_\nu + \epsilon_\omega & =  O\big(\eta T m^{-1} + \eta T n^{-1} + \eta\sqrt{T} + \eta \sum_{j=0}^{T-1}( (j\beta)^{-c} + \sqrt{\beta} + \sqrt{\frac{\eta^2}{\beta}} ) \\
			& =  O\big(\eta T m^{-1}  + \eta T n^{-1} + \eta\sqrt{T} + \eta T^{-c/2+1}\beta^{-c/2} + \eta T\beta^{1/2} + \eta^2\beta^{-1/2}T \big).
		\end{aligned}
	\end{equation*}
	This complete the proof.
\end{proof}

Before giving the detailed proof  of Theorem \ref{thm:opt_convex}, we first give a useful lemma.
\begin{lemma}\label{lemma:for_two_level_theorem_optimization}
	Let Assumption \ref{ass:Lipschitz continuous}(ii), \ref{ass:bound variance} (ii) and \ref{ass:Smoothness and Lipschitz continuous gradient} (ii) hold for the empirical risk $F_S$, for Algorithm \ref{alg_cover} and any $\gamma_t >0$, we have 
	\begin{equation*}
	\begin{aligned}
		\EX_A [\|x_{t+1} - x_*^S\|^2|\F_t]
		&\leq (1+  \frac{\eta_t(L_f+L_gC_f)}{\gamma_t}) \|x_t -x_*^S\|^2 + L_g^2L_f^2\eta_t^2 - 2\eta_t(F_S(x_t) - F_S(x_*^S))  \\
		&\quad+\eta_t \gamma_tL_f\EX_A[\|v_t - \nabla g_S(x_t)\|^2\F_t] + \eta_t \gamma_tL_gC_f\EX_A[\|u_t-g_S(x_t)\|^2\F_t].
	\end{aligned}
\end{equation*}
	where $\F_t $ is the $\sigma$-field generated by $\{\omega_{j_0},\cdots, \omega_{j_{t-1}},  v_{i_0}, \cdots ,v_{i_{t-1}}\}$.
\end{lemma}

\begin{proof}
	According to the update rule of Algorithm \ref{alg_cover}, we have 
	\begin{equation*}
		\begin{aligned}
			\|x_{t+1} - x_*^{S}\| 
			& \leq \|x_t - \eta_t v_t \nabla f_{\nu_{i_t}}(u_t) -x_*^S\|^2 \\
			& = \|x_t - x_*^S\|^2 + \eta_t^2\| v_t \nabla f_{\nu_{i_t}}(u_t)\|^2 -2\eta_t \langle x_t - x_*^S,v_t \nabla f_{\nu_{i_t}}(u_t) \rangle\\
			& = \|x_t - x_*^S\|^2 + \eta_t^2\|v_t \nabla f_{\nu_{i_t}}(u_t)\|^2 -2\eta_t \langle x_t - x_*^S, \nabla g_S(x_t) \nabla f_{\nu_{i_t}}(g_S(x_t)) \rangle + \theta_t,
		\end{aligned}
	\end{equation*}
	where 
	\begin{equation*}
		\theta_t = 2\eta_t \langle x_t - x_*^S, \nabla g_S(x_t) \nabla f_{\nu_{i_t}}(g_S(x_t)) - v_t \nabla f_{\nu_{i_t}}(u_t) \rangle.
	\end{equation*}
	Let $\F_t$ be the $\sigma$-field generated by $\{\omega_{j_0},\cdots, \omega_{j_{t-1}},  v_{i_0}, \cdots ,v_{i_{t-1}}\}$. Taking expectation  to the above inequality and using Assumption \ref{ass:Lipschitz continuous}, we have 
	\begin{equation*}
		\begin{aligned}
			\EX_A [\|x_{t+1} - x_*^S\|^2|\F_t]
			&\leq \|x_t -x_*^S\|^2 + L_g^2L_f^2\eta_t^2 - 2\eta_t\EX_A[ \langle x_t - x_*^S, \nabla g_S(x_t) \nabla f_{\nu_{i_t}}(g_S(x_t)) \rangle\F_t] + \EX_A[\theta_t|\F_t]\\
			&\leq \|x_t -x_*^S\|^2 + L_g^2L_f^2\eta_t^2 - 2\eta_t \langle x_t - x_*^S, \nabla F_S(x_t) \rangle  +  \EX_A[\theta_t|\F_t]\\
			&\leq \|x_t -x_*^S\|^2 + L_g^2L_f^2\eta_t^2 - 2\eta_t(F_S(x_t) - F_S(x_*^S))  +  \EX_A[\theta_t|\F_t],
		\end{aligned}
	\end{equation*}
	where the last inequality holds by the convexity of $F_S$. As for the term $\EX_A[\theta_t|\F_t]$, we have 
	\begin{equation*}
		\begin{aligned}
			\theta_t &= 2\eta_t \langle x_t - x_*^S, \nabla g_S(x_t)\nabla f_{\nu_{i_t}}(g_S(x_t)) -  v_t \nabla f_{\nu_{i_t}}(u_t) \rangle \\
			& = 2\eta_t \langle x_t - x_*^S,\nabla g_S(x_t)\nabla f_{\nu_{i_t}}(g_S(x_t)) - v_t \nabla f_{\nu_{i_t}}(g_S(x_t)) \rangle+ 2\eta_t \langle x_t - x_*^S, v_t \nabla f_{\nu_{i_t}}(g_S(x_t)) -  v_t \nabla f_{\nu_{i_t}}(u_t) \rangle \\
			&\leq 2\eta_t L_f\|x_t - x_*^S\| \cdot \|v_t - \nabla g_S(x_t)\| + 2\eta_t L_g C_f \|x_t - x_*^S\| \cdot \|u_t - g_S(x_t)\| \\
			&\leq  \frac{\eta_t(L_f+L_gC_f)}{\gamma_t}\|x_t - x_*^S\|^2 + \eta_t \gamma_t L_f\|v_t - \nabla g_S(x_t)\|^2 + \eta_t \gamma_t L_gC_f\|u_t-g_S(x_t)\|^2,
		\end{aligned}
	\end{equation*}
	where the last inequality holds by Cauchy-Schwartz inequality.
	
	Combining above two inequalities, we have 
	\begin{equation*}
		\begin{aligned}
			\EX_A [\|x_{t+1} - x_*^S\|^2|\F_t]
			&\leq (1+  \frac{\eta_t(L_f+L_gC_f)}{\gamma_t}) \|x_t -x_*^S\|^2 + L_g^2L_f^2\eta_t^2 - 2\eta_t(F_S(x_t) - F_S(x_*^S))  \\
			&\quad+\eta_t \gamma_tL_f\EX_A[\|v_t - \nabla g_S(x_t)\|^2\F_t] + \eta_t \gamma_tL_gC_f\EX_A[\|u_t-g_S(x_t)\|^2\F_t].
		\end{aligned}
	\end{equation*}
	This complete the proof.
\end{proof}

\begin{proof}[Proof of Theorem \ref{thm:opt_convex}]
	Now we begin to proof the Theorem \ref{thm:opt_convex}. 
	According to Lemma \ref{lemma:for_two_level_theorem_optimization}, setting  $\eta_t = \eta$, $\beta_t = \beta$ and let $\gamma_t = \frac{1}{\sqrt{\beta}}$, by rearranging and adding up, we get
	\begin{equation*}
		\begin{aligned}
			2\eta \sum_{t=1}^T \EX_{A}[F_S(x_t) - F_S(x_*^S)]
			&\leq \EX_{A}[\|x_1 - x_*^S\|^2] + \eta \sqrt{\beta} \sum_{t=1}^T(L_f+L_gC_f)  \|x_t -x_*^S\|^2 + L_g^2L_f^2\eta^2 T \\
			&\quad + \frac{\eta}{\sqrt{\beta}} \sum_{t=1}^T(L_f\|v_t - \nabla g_S(x_t)\|^2 + L_gC_f\sum_{t=1}^T\|u_t-g_S(x_t)\|^2).
		\end{aligned}
	\end{equation*}
	Then according to the definition that $\EX[\|x_t - x_*^S\|^2]$ is  bounded by
	$D_x$, and Lemma \ref{lemma:u_t_bound_two} and  Lemma \ref{lemma:iv_t_bound}, we have 
	\begin{equation}\label{Eq:fs_t-fst*}
		\begin{aligned}
			&2\eta \sum_{t=1}^T \EX_{A}[F_S(x_t) - F_S(x_*^S)]\\
			& \leq  D_x + \eta \sqrt{\beta} (L_f+L_gC_f) D_x T+ L_g^2L_f^2\eta^2 T\\
			& \quad+ \frac{\eta}{\sqrt{\beta}} L_gC_f(\frac{c}{e})^c\sum_{t=1}^T(t\beta)^{-c}\EX_{A}[\|u_0-g_S(x_0)\|^2] + 2\sigma_g^2\eta \beta^{1/2} L_gC_fT + 2L_g^5L_f^2C_f\frac{\eta^3}{\beta^{3/2}} T\\ 
			&\quad + \frac{\eta}{\sqrt{\beta}} L_f(\frac{c}{e})^c\sum_{t=1}^T(t\beta)^{-c}\EX_{A}[\|v_0 - \nabla g_S(x_0)\|^2] + 2\sigma_{g'}^2\eta \beta^{1/2} L_fT + 2L_g^4L_f^3\frac{\eta^3}{\beta^{3/2}} T.
		\end{aligned}
	\end{equation}
	According to \eqref{Eq:sum_t^-z_bound}, without losing generality, let $ c \neq 1$, we have 
	\begin{equation*}
		\begin{aligned}
			&\EX_{A}[F_S(A(S)) - F_S(x_*^S)]\\
			& \leq O\Big( D_x(\eta T)^{-1} + (L_gC_f U + L_f V)(\beta T)^{-c}\beta^{-\frac{1}{2}} + (L_gC_f\sigma_g^2 + L_f \sigma_{g'}^2 + (L_f+L_gC_f) D_x)\beta^{1/2}  \\
			&\qquad \quad + L_g^2L_f^2\eta + (L_g^5L_f^2C_f + L_g^4L_f^3)\eta^2\beta^{-\frac{3}{2}}  \Big),
		\end{aligned}
	\end{equation*}
	where   $\EX[\|u_0-g_S(x_0)\|^2] \leq U$ and $\EX[\|v_0 - \nabla g_S(x_0)\|^2] \leq V$ .
	
	This complete the proof.
\end{proof}

\begin{proof}[proof of Theorem \ref{thm:Excess_Risk_Bound_convex}]
	Putting   Lemma \ref{lemma:u_t_bound_two} and \ref{lemma:iv_t_bound}  into \eqref{Eq:x_T-X_Tkv} and \eqref{Eq:x_T-X_Tlw}, for any $c > 0$, we have 
	\begin{equation*}
		\begin{aligned}
			&\EX_A[\|x_t - x_t^{k,\nu}\|] \\
			&\leq 10L_gC_f\sup_S\eta\sqrt{(\frac{c}{e})^cU}\beta^{-\frac{c}{2}}\sum_{j=0}^{t-1}\sqrt{j^{-c}} + 10L_gC_f\eta  \sqrt{ 2\sigma_g^2 \beta + \frac{2L_g^4L_f^2\eta^2}{\beta}}t   \\
			&\quad + 10L_f\sup_S\eta \sqrt{(\frac{c}{e})^cV}\beta^{-\frac{c}{2}} \sum_{j=0}^{t-1}\sqrt{j^{-c}}+ 10L_f\eta \sqrt{2\beta\sigma_{g'}^2 + \frac{2L_g^4L_f^2\eta^2}{\beta}} t + \frac{6L_gL_f\eta t}{n} + 4L_f\sigma_f\eta \sqrt{t}.
		\end{aligned}
	\end{equation*}
	Similarly, we can get 
	\begin{equation*}
		\begin{aligned}
			&\EX_A[\|x_t - x_t^{l,\omega}\|] \\
			&\leq 10L_gC_f\sup_S\eta\sqrt{(\frac{c}{e})^cU}\beta^{-\frac{c}{2}}\sum_{j=0}^{t-1}\sqrt{j^{-c}} + 10L_gC_f\eta  \sqrt{ 2\sigma_g^2 \beta + \frac{2L_g^4L_f^2\eta^2}{\beta}}t   \\
			&\quad + 10L_f\sup_S\eta \sqrt{(\frac{c}{e})^cV}\beta^{-\frac{c}{2}} \sum_{j=0}^{t-1}\sqrt{j^{-c}}+ 10L_f\eta \sqrt{2\beta\sigma_{g'}^2 + \frac{2L_g^4L_f^2\eta^2}{\beta}} t + \frac{6L_gL_f\eta t}{m} + 4L_f\sigma_f\eta \sqrt{t}.
		\end{aligned}
	\end{equation*}
	Combining above two inequalities, we have 
	\begin{equation*}
		\begin{aligned}
			&\EX_A[\|x_t - x_t^{k,\nu}\|]  + 4\EX_A[\|x_t - x_t^{l,\omega}\|]\\
			&\leq   \frac{24L_gL_f\eta t}{m} + \frac{6L_gL_f\eta t}{n} +  50L_gC_f\sup_S\eta\sqrt{(\frac{c}{e})^cU}\beta^{-\frac{c}{2}}\sum_{j=0}^{t-1}\sqrt{j^{-c}} + 50L_gC_f\eta \sqrt{ 2\sigma_g^2 \beta + \frac{2L_g^4L_f^2\eta^2}{\beta}}t  \\
			&\quad + 50L_f\sup_S\eta \sqrt{(\frac{c}{e})^cV}\beta^{-\frac{c}{2}} \sum_{j=0}^{t-1}\sqrt{j^{-c}}+ 50L_f\eta \sqrt{2\beta\sigma_{g'}^2 + \frac{2L_g^4L_f^2\eta^2}{\beta}}t + 20L_f\sigma_f\eta\sqrt{t}.
		\end{aligned}
	\end{equation*}
	Putting above inequality into Lemma \ref{theorem:general_two_level}, we have 
	\begin{equation}\label{Eq:F_x_t-F_Sx_t}
		\begin{aligned}
			\EX_{S,A} [F(x_t) - F_S(x_t)]
			& \leq 50L_g^2L_fC_f\sup_S\eta\sqrt{(\frac{c}{e})^cU}\beta^{-\frac{c}{2}}\sum_{j=0}^{t-1}\sqrt{j^{-c}} + 50L_g^2L_fC_f\eta\sqrt{ 2\sigma_g^2 \beta + \frac{2L_g^4L_f^2\eta^2}{\beta}} t  \\
			&\quad + 50L_gL_f^2\sup_S\eta \sqrt{(\frac{c}{e})^cV}\beta^{-\frac{c}{2}} \sum_{j=0}^{t-1}\sqrt{j^{-c}}+ 50L_gL_f^2\eta\sqrt{2\beta\sigma_{g'}^2 + \frac{2L_g^4L_f^2\eta^2}{\beta}} t\\
			&\quad+ \frac{24L_g^2L_f^2\eta t}{m} + \frac{6L_g^2L_f^2\eta t}{n} +  L_{f} \sqrt{m^{-1} \mathbb{E}_{S, A}[\operatorname{Var}_{\omega}(g_{\omega}(A(S)))]}.
		\end{aligned}
	\end{equation}
	Due to 
	\begin{equation*}
		\begin{aligned}
			\sum_{t=1}^{T} \mathbb{E}_{S,A}[F(x_{t})-F(x_{*})] \leq \sum_{t=1}^{T} \mathbb{E}_{S,A}[F(x_{t})-F_S(x_{*})] = \sum_{t=1}^{T} \EX_{S,A} [F(x_t) - F_S(x_t) + F_S(x_t) - F_S(x_*^S)],
		\end{aligned}
	\end{equation*}
	then combining \eqref{Eq:F_x_t-F_Sx_t} and \eqref{Eq:fs_t-fst*}, we have 
	\begin{equation*}
		\begin{aligned}
			&\sum_{t=1}^{T} \mathbb{E}_{S,A}[F(x_{t})-F(x_{*})] \\
			& \leq  \frac{1}{2\sqrt{\beta}} L_gC_f(\frac{c}{e})^c\sum_{t=1}^T(t\beta)^{-c}\EX[\|u_0-g_S(x_0)\|^2] + \sigma_g^2 \beta^{1/2} L_gC_fT + L_g^5L_f^2C_f\frac{\eta^2}{\beta^{3/2}} T\\ 
			&\quad + \frac{1}{2\sqrt{\beta}} L_gL_f(\frac{c}{e})^c\sum_{t=1}^T(t\beta)^{-c}\EX[\|v_0 - \nabla g_S(x_0)\|^2] + \sigma_{g'}^2 \beta^{1/2} L_gL_fT + L_g^5L_f^2C_f\frac{\eta^2}{\beta^{3/2}} T\\
			&\quad + \frac{D_x}{2\eta } +  \sqrt{\beta} (L_gL_f+L_gC_f + L_g) D_x T /2 + L_g^2L_f^2\eta T/2+ \frac{24L_g^2L_f^2\eta}{m}\sum_{t=1}^T t + \frac{6L_g^2L_f^2\eta}{n} \sum_{t=1}^Tt\\
			&\quad + 50L_g^2L_fC_f\sup_S\eta\sqrt{(\frac{c}{e})^cU}\beta^{-\frac{c}{2}}\sum_{t=1}^T\sum_{j=0}^{t-1}\sqrt{j^{-c}} + 50L_g^2L_fC_f\eta \sqrt{ 2\sigma_g^2 \beta + \frac{2L_g^4L_f^2\eta^2}{\beta}} \sum_{t=1}^T t  \\
			&\quad + 50L_gL_f^2\sup_S\eta \sqrt{(\frac{c}{e})^cU_{\iv}}\beta^{-\frac{c}{2}} \sum_{t=1}^T\sum_{j=0}^{t-1}\sqrt{j^{-c}}   +50L_gL_f^2\eta \sqrt{2\beta\sigma_{g'}^2 + \frac{2L_g^4L_f^2\eta^2}{\beta}} \sum_{t=1}^T t\\
			&\quad  +  L_{f} T \sqrt{m^{-1} \mathbb{E}_{S, A}[\operatorname{Var}_{\omega}(g_{\omega}(A(S)))]}.
		\end{aligned}
	\end{equation*}
	Using \eqref{Eq:double_sum_j^{-c/2}} we have, for any $c>0$,
	\begin{equation*}
		\begin{aligned}
			\sum_{t=1}^{T} \mathbb{E}_{S,A}[F(x_{t})-F(x_{*})] 
			& \leq O\Big(\beta^{-\frac{1}{2}-c}T^{1-c}(\log T)^{\1_{c=1}} + \beta^{\frac{1}{2}}T  + \frac{\eta^2 T}{\beta^{\frac{3}{2}}} + \frac{1}{\eta} +   \eta T+ T^2 \eta (\frac{1}{n} + \frac{1}{m})  \\
			&\qquad~~~~~~~ +  \eta \beta^{-\frac{c}{2}} T^{2-\frac{c}{2}}(\log T)^{\1_{c=2}}  + T^2\eta \sqrt{\beta}+ \frac{\eta^2 T^2}{\sqrt{\beta}}  + \frac{T}{\sqrt{m}}\Big).
		\end{aligned}
	\end{equation*}
	Dividing both sides of the above inequality, setting $\eta = T^{-a}$ and $\beta = T^{-b}$, and from the choice of $A(S)$, we have 
	\begin{equation*}
		\begin{aligned}
			&\mathbb{E}_{S,A}[F(A(S))-F(x_{*})] \\
			& \leq O\Big(T^{- (1- b)c+ \frac{b}{2}} (\log T)^{\1_{c= 1}} + T^{-\frac{b}{2}} + T^{\frac{3}{2}b-2a} + T^{a-1} + T^{-a}+ T^{1-a}(\frac{1}{n}+\frac{1}{m} )  \\
			&\qquad + T^{1-a-\frac{c}{2}(b-1) }(\log T)^{\1_{c=2}}+ T^{1-a-\frac{b}{2}} + T^{1+\frac{b}{2}-2a}  + \frac{1}{\sqrt{m}}\Big). 
		\end{aligned}
	\end{equation*}
	Since $a,b \in (0,1]$, as long as we have $c >4$, the  dominating terms are the following $ O(T^{1- a- \frac{b}{2}})$, $\quad O(T^{1+ \frac{b}{2}- 2a})$, $\quad O(n^{-1}T^{1- a}), \quad O(m^{-1}T^{1- a}), \quad  O(T^{a-1})$,  and $ O(T^{\frac{3}{2}b- 2a}).$
	Setting $a = b =4/5$,  then we have  
	\begin{equation*}
		\begin{aligned}
			\mathbb{E}_{S,A}[F(A(S))-F(x_{*})]=O(T^{-\frac{1}{5}}+\frac{T^{\frac{1}{5}}}{n}+\frac{T^{\frac{1}{5}}}{m}+\frac{1}{\sqrt{m}}).
		\end{aligned}
	\end{equation*}
	Choosing $T=O(\max \{n^{2.5}, m^{2.5}\})$, we have the following bound
	\begin{equation*}
		\mathbb{E}_{S,A}[F(A(S))-F(x_{*})]=O(\frac{1}{\sqrt{n}}+\frac{1}{\sqrt{m}}).
	\end{equation*}
	This complete the proof.
\end{proof}

\subsection{Strongly-convex-setting}

\begin{proof}[proof of Theorem \ref{thm:sta_sconvex_two_convex}]
	Similar to the proof for convex setting, we use the same notations. Since changing one sample data can happen in either $S_\nu$ or $S_\omega$, we estimate $\EX[\|x_{t+1}-x_{t+1}^{k,\nu}\|]$ and $\EX[\|x_{t+1}-x_{t+1}^{l,\omega}\|]$.
	
	\noindent{\bf Estimation of $\mathbb{E}_{A}\bigl[\|x_{t+1}-x_{t+1}^{k,\nu}\|\bigr]$}  
	
	we will consider two cases: $i_t \neq k$ and $i_t = k$.
	
	\textbf{\quad Case 1 ($i_t \neq k $). } We have 
	
		\begin{equation}\label{Eq:two_level_case1_origin_sc}
		\begin{aligned}
			&\|x_{t+1} - x_{t+1}^{k,\nu}\|^2 \\
			&\leq \|x_t - \eta_t v_{t} \nabla f_{\nu_{i_t}}(u_t) - x_{t}^{k,\nu} + \eta_t v_{t}^{k,\nu} \nabla f_{\nu_{i_t}}(u_t^{k,\nu})\|^2\\
			& \leq \|x_t - x_t^{k,\nu}\|^2 - 2\eta_t \langle v_{t}^{k,\nu} \nabla f_{\nu_{i_t}}(u_t^{k,\nu})- v_{t} \nabla f_{\nu_{i_t}}(u_t), x_t^{k,\nu} - x_t\rangle + \eta_t^2\|v_{t}^{k,\nu} \nabla f_{\nu_{i_t}}(u_t^{k,\nu}) - v_{t} \nabla f_{\nu_{i_t}}(u_t) \|^2.
		\end{aligned}
	\end{equation}
	We begin to  estimate the second term in \eqref{Eq:two_level_case1_origin_sc}.
	\begin{equation}\label{Eq: decomposed_product_knu_sc}
		\begin{aligned}
			&- 2\eta_t \langle v_{t}^{k,\nu} \nabla f_{\nu_{i_t}}(u_t^{k,\nu})- v_{t} \nabla f_{\nu_{i_t}}(u_t), x_t^{k,\nu} - x_t\rangle\\
			& = - 2\eta_t \langle v_{t}^{k,\nu} \nabla f_{\nu_{i_t}}(u_t^{k,\nu}) - v_t^{k,\nu} \nabla  f_{\nu_{i_t}}(g_S(x_t^{k,\nu})), x_t^{k,\nu} - x_t\rangle\\
			& \quad  - 2\eta_t \langle v_t^{k,\nu} \nabla  f_{\nu_{i_t}}(g_S(x_t^{k,\nu})) - v_t^{k,\nu} \nabla f_S(g_S(x_t^{k,\nu})) , x_t^{k,\nu} - x_t  \rangle\\
			&\quad- 2\eta_t \langle v_t^{k,\nu} \nabla f_S(g_S(x_t^{k,\nu})) - \nabla g_S(x_t^{k,\nu})\nabla f_S(g_S(x_t^{k,\nu})) , x_t^{k,\nu} - x_t  \rangle \\
			&\quad- 2\eta_t \langle \nabla g_S(x_t^{k,\nu})\nabla f_S(g_S(x_t^{k,\nu})) - \nabla g_S(x_t)\nabla f_S(g_S(x_t)) , x_t^{k,\nu} - x_t  \rangle\\
			&\quad - 2\eta_t \langle \nabla g_S(x_t)\nabla f_S(g_S(x_t)) - v_t\nabla f_S(g_S(x_t)) , x_t^{k,\nu} - x_t  \rangle - 2\eta_t \langle v_t\nabla f_S(g_S(x_t)) - v_t\nabla f_S(u_t), x_t^{k,\nu} - x_t  \rangle\\
			&\quad - 2\eta_t \langle  v_t\nabla f_S(u_t) - v_t \nabla f_{\nu_{i_t}}(u_t), x_t^{k,\nu} - x_t  \rangle.
		\end{aligned}
	\end{equation}

	Changing the setting from convex to strongly convex will only affect the fourth item on the RHS of \eqref{Eq: decomposed_product_knu}, and the other items will remain the same as before. Now we estimate the fourth term on the RHS of \eqref{Eq: decomposed_product_knu_sc}.
	\begin{equation*}
		\begin{aligned}
			&\langle  \nabla g_{S}(x_{t}^{k, \nu}) \nabla f_S(g_{S}(x_{t}^{k, \nu})  -\nabla g_{S}(x_{t}) \nabla f_S(g_{S}(x_{t})), x_{t}^{k, \nu} - x_{t}\rangle \\
			&\geq \frac{L \mu}{L+\mu}\|x_{t}^{k, \nu} - x_{t}\|^{2}+\frac{1}{L+\mu}\|\nabla g_{S}(x_{t}^{k, \nu}) \nabla f_S(g_{S}(x_{t}^{k, \nu}))- \nabla g_{S}(x_{t}) \nabla f_S(g_{S}(x_{t}))\|^{2} .
		\end{aligned}
	\end{equation*}
	Then substituting above inequality into \eqref{Eq:two_level_case1_origin_sc} we have 
	\begin{equation*}
		\begin{aligned}
			&\|x_{t+1} - x_{t+1}^{k,\nu}\|\\
			& \leq (1- \frac{2L\mu \eta_t}{L+\mu})\|x_t-x_t^{k,\nu}\|^2  +  2L_g C_f \eta_t \| u_t^{k,\nu} - g_S(x_t^{k,\nu})\|\cdot\|x_t^{k,\nu} - x_t \| + 2L_f\eta_t\|v_t^{k,\nu} - \nabla g_S(x_t^{k,\nu})\|\cdot\|x_t^{k,\nu} - x_t \| \\
			&\quad+ 2L_gC_f\eta_t\|u_t-g_S(x_t)\|\cdot \|x_t^{k,\nu}-x_t\| + 2L_f\eta_t\|u_t-g_S(x_t)\|\cdot \|x_t^{k,\nu}-x_t\|\\
			&\quad+ 7\eta_t^2L_g^2\| \nabla f_{\nu_{i_t}}(u_t^{k,\nu}) -\nabla f_S(u_t^{k,\nu})\|^2  + 7\eta_t^2L_g^2C_f^2\|u_t^{k,\nu}-g_S(x_t^{k,\nu})\|^2+ 7\eta_t^2L_f^2\|v_t^{k,\nu}-\nabla g_S(x_t^{k,\nu})\|^2\\
			& \quad + 7\eta_t^2L_f^2\|\nabla g_S(x_t) - v_t\|^2 + 7\eta_t^2L_g^2C_f^2\|u_t-g_S(x_t)\|^2 + 7\eta_t^2L_g^2\| \nabla f_{\nu_{i_t}}(u_t) -\nabla f_S(u_t)\|^2.
		\end{aligned}
	\end{equation*}
	where the inequality holds by $\eta_t \leq \frac{2}{7(L+\mu)}$.
	
	\textbf{\quad Case 2 ($i_t = k $). } We have 
	\begin{equation}\label{Eq:two_level_case4_origin_sc}
	\begin{aligned}
		\|x_{t+1} - x_{t+1}^{l,\omega}\| & = \|x_t - \eta_t v_{t} \nabla f_{\nu_{i_t}}(u_t) - x_{t}^{l,\omega} + \eta_t v_{t}^{l,\omega} \nabla f_{\nu_{i_t}}(u_t^{l,\omega})\|\\
		& \leq \|x_t - x_t^{l,\omega}\|  + \eta_t\|v_{t}^{l,\omega} \nabla f_{\nu_{i_t}}(u_t^{l,\omega}) - v_{t} \nabla f_{\nu_{i_t}}(u_t) \| \leq \|x_t - x_t^{l,\omega}\| + 2\eta_t L_g L_f,
	\end{aligned}
\end{equation}
where the first inequality holds by Assumption \ref{ass:Lipschitz continuous}, then we have 
\begin{equation*}
	\|x_{t+1} - x_{t+1}^{l,\omega}\|^2 \leq \|x_t - x_t^{l,\omega}\|^2 + 4\eta_t L_g L_f\|x_t - x_t^{l,\omega}\| + 4\eta_t^2L_g^2L_f^2.
	\end{equation*}
	Combining above two cases, we have 
	\begin{equation*}
		\begin{aligned}
			\EX_A[\|x_{t+1} - x_{t+1}^{k,\nu}\|^2] 
			&\leq (1- \frac{2L\mu \eta_t}{L+\mu})\EX_A[\|x_t-x_t^{k,\nu}\|^2]  +  2L_g C_f \eta_t \EX_A[\| u_t^{k,\nu} - g_S(x_t^{k,\nu})\|\cdot\|x_t^{k,\nu} - x_t \|]\\
			&\quad + 2L_f\eta_t\EX_A[\|v_t^{k,\nu} - \nabla g_S(x_t^{k,\nu})\|\cdot\|x_t^{k,\nu} - x_t \|] \\
			&\quad+ 2L_gC_f\eta_t\EX_A[\|u_t-g_S(x_t)\|\cdot \|x_t^{k,\nu}-x_t\|] + 2L_f\eta_t\EX_A[\|u_t-g_S(x_t)\|\cdot \|x_t^{k,\nu}-x_t\|]\\
			&\quad  + 7\eta_t^2L_g^2C_f^2\EX_A[\|u_t^{k,\nu}-g_S(x_t^{k,\nu})\|^2]+ 7\eta_t^2L_f^2\EX_A[\|v_t^{k,\nu}-\nabla g_S(x_t^{k,\nu})\|^2]\\
			& \quad + 7\eta_t^2L_f^2\EX_A[\|\nabla g_S(x_t) - v_t\|^2] + 7\eta_t^2L_g^2C_f^2\EX_A[\|u_t-g_S(x_t)\|^2] + 14\eta_t^2L_g^2 \sigma_f^2\\
			&\quad + 4 \eta_{t} L_{g} L_{f} \mathbb{E}_{A}[\|x_{t}-x_{t}^{k, \nu}\| \1_{[i_{t}=k]}]+4 \eta_{t}^{2} L_{f}^{2} L_{g}^{2} \mathbb{E}_{A}[\1_{[i_{t}=k]}].
		\end{aligned}
	\end{equation*}
	By setting $\eta_t = \eta$, we have 
	\begin{equation*}
		\begin{aligned}
			&\EX_A[\|x_{t+1} - x_{t+1}^{k,\nu}\|^2] \\
			&\leq  2L_gC_f\eta\sum_{j=0}^{t-1}(1-\frac{2L\mu \eta}{L+\mu})^{t-j}((\EX_A[\| u_j^{k,\nu} - g_S(x_j^{k,\nu})\|^2])^{1/2} + (\EX_A[\| u_j - g_S(x_j)\|^2])^{1/2}) (\EX_A[\|x_j - x_j^{k,\nu}\|^2])^{1/2}\\
			&\quad + 2L_f\eta\sum_{j=0}^{t-1}(1-\frac{2L\mu \eta}{L+\mu})^{t-j}((\EX_A\|v_j^{k,\nu} - \nabla g_S(x_j^{k,\nu})\|^2)^{1/2} + (\EX_A\|v_j - \nabla g_S(x_j)\|^2)^{1/2})(\EX_A[\|x_j - x_j^{k,\nu}\|^2])^{1/2}\\
			&\quad  + 7\eta^2 L_g^2 C_f^2\sum_{j=1}^t(1-\frac{2L\mu \eta}{L+\mu})^{t-j} \EX_A[\|u_j^{k,\nu}-g_S(x_j^{k,\nu})\|^2]+7\eta^2 L_g^2 C_f^2\sum_{j=1}^t(1-\frac{2L\mu \eta}{L+\mu})^{t-j} \EX_A[\|u_j-g_S(x_j)\|^2]\\
			&\quad + 7\eta^2 L_f^2\sum_{j=1}^t(1-\frac{2L\mu \eta}{L+\mu})^{t-j} \EX_A[\|v_j^{k,\nu}-\nabla g_S(x_j^{k,\nu})\|^2] +  7\eta^2 L_f^2 \sum_{j=1}^t(1-\frac{2L\mu \eta}{L+\mu})^{t-j} \EX_A[\|\nabla g_S(x_j) - v_j\|^2]\\
			&\quad+\frac{4\eta L_gL_f}{n}\sum_{j=1}^t(1-\frac{2L\mu \eta}{L+\mu})^{t-j} \EX_A[\|x_j - x_{j}^{k,\nu}\|] + \frac{4\eta^2 L_f^2 L_g^2}{n}\sum_{j=1}^t(1-\frac{2L\mu \eta}{L+\mu})^{t-j} + 14\eta^2 L_g^2\sigma_f^2\sum_{j=1}^{t-1}(1-\frac{2L\mu \eta}{L+\mu})^{t-j},
		\end{aligned}
	\end{equation*}
where the inequality holds by Lemma \ref{lemma:general_recursive},  Cauchy-Schwarz inequality, and the fact that $x_0 = x_0^{k,\nu}$. Define  $u_t = (\EX_A[\|x_t - x_t^{k,\nu}\|^2])^{1/2}$, we have 
\begin{equation*}
\begin{aligned}
    u_t^2 &\leq  2L_gC_f\eta\sum_{j=0}^{t-1}(1-\frac{2L\mu \eta}{L+\mu})^{t-j}((\EX_A[\| u_j^{k,\nu} - g_S(x_j^{k,\nu})\|^2])^{1/2} + (\EX_A[\| u_j - g_S(x_j)\|^2])^{1/2}) u_j\\
    &\quad + 2L_f\eta\sum_{j=0}^{t-1}(1-\frac{2L\mu \eta}{L+\mu})^{t-j}((\EX_A\|v_t^{k,\nu} - \nabla g_S(x_j^{k,\nu})\|^2)^{1/2} + (\EX_A\|v_j - \nabla g_S(x_j)\|^2)^{1/2})u_j\\
    &\quad  + 7\eta^2 L_g^2 C_f^2\sum_{j=1}^{t-1}(1-\frac{2L\mu \eta}{L+\mu})^{t-j} \EX_A[\|u_j^{k,\nu}-g_S(x_j^{k,\nu})\|^2]+7\eta^2 L_g^2 C_f^2\sum_{j=1}^{t-1}(1-\frac{2L\mu \eta}{L+\mu})^{t-j} \EX_A[\|u_j-g_S(x_j)\|^2]\\
    &\quad + 7\eta^2 L_f^2\sum_{j=1}^{t-1}(1-\frac{2L\mu \eta}{L+\mu})^{t-j} \EX_A[\|v_j^{k,\nu}-\nabla g_S(x_j^{k,\nu})\|^2] +  7\eta^2 L_f^2 \sum_{j=1}^{t-1}(1-\frac{2L\mu \eta}{L+\mu})^{t-j} \EX_A[\|\nabla g_S(x_j) - v_j\|^2]\\
    &\quad+\frac{4\eta L_gL_f}{n}\sum_{j=1}^{t-1}(1-\frac{2L\mu \eta}{L+\mu})^{t-j} \EX_A[\|x_j - x_{j}^{k,\nu}\|] + \frac{4\eta^2 L_f^2 L_g^2}{n}\sum_{j=1}^{t-1}(1-\frac{2L\mu \eta}{L+\mu})^{t-j} + 14\eta^2 L_g^2\sigma_f^2\sum_{j=1}^{t-1}(1-\frac{2L\mu \eta}{L+\mu})^{t-j} ,
\end{aligned}
\end{equation*}

Furthermore, define 
\begin{equation*}
\begin{aligned}
    \alpha_j &\leq 2L_gC_f\eta\sum_{j=0}^{t-1}(1-\frac{2L\mu \eta}{L+\mu})^{t-j}((\EX_A[\| u_j^{k,\nu} - g_S(x_j^{k,\nu})\|^2])^{1/2} + (\EX_A[\| u_j - g_S(x_j)\|^2])^{1/2}) \\
    &\quad + 2L_f\eta\sum_{j=0}^{t-1}(1-\frac{2L\mu \eta}{L+\mu})^{t-j}((\EX_A\|v_t^{k,\nu} - \nabla g_S(x_j^{k,\nu})\|^2)^{1/2} + (\EX_A\|v_j - \nabla g_S(x_j)\|^2)^{1/2})\\
    &+\frac{4\eta L_gL_f}{n}(1-\frac{2L\mu \eta}{L+\mu})^{t-j-1},
\end{aligned}
\end{equation*}
	and 
	\begin{equation*}
		\begin{aligned}
			S_t &\leq  7\eta^2 L_g^2 C_f^2\sum_{j=1}^{t-1}(1-\frac{2L\mu \eta}{L+\mu})^{t-j} \EX_A[\|u_j^{k,\nu}-g_S(x_j^{k,\nu})\|^2]+7\eta^2 L_g^2 C_f^2\sum_{j=1}^{t-1}(1-\frac{2L\mu \eta}{L+\mu})^{t-j} \EX_A[\|u_j-g_S(x_j)\|^2]\\
			&\quad + 7\eta^2 L_f^2\sum_{j=1}^{t-1}(1-\frac{2L\mu \eta}{L+\mu})^{t-j} \EX_A[\|v_j^{k,\nu}-\nabla g_S(x_j^{k,\nu})\|^2] +  7\eta^2 L_f^2 \sum_{j=1}^{t-1}(1-\frac{2L\mu \eta}{L+\mu})^{t-j} \EX_A[\|\nabla g_S(x_j) - v_j\|^2]\\
			&\quad + \frac{4\eta^2 L_f^2 L_g^2}{n}\sum_{j=1}^t(1-\frac{2L\mu \eta}{L+\mu})^{t-j-1} + 14\eta^2 L_g^2\sigma_f^2\sum_{j=1}^{t-1}(1-\frac{2L\mu \eta}{L+\mu})^{t-j} ,
		\end{aligned}
	\end{equation*}
	using Lemma \ref{lemma:recursion lemma}, we have 
	\begin{equation*}
		\begin{aligned}
			u_t & \leq 3\eta L_g C_f(\sum_{j=1}^{t-1}(1-\frac{2L\mu \eta}{L+\mu})^{t-j} \EX_A[\|u_j^{k,\nu}-g_S(x_j^{k,\nu})\|^2])^{\frac{1}{2}}+3\eta L_g C_f(\sum_{j=1}^{t-1}(1-\frac{2L\mu \eta}{L+\mu})^{t-j} \EX_A[\|u_j-g_S(x_j)\|^2])^{\frac{1}{2}}\\
			&\quad + 3\eta L_f(\sum_{j=1}^{t-1}(1-\frac{2L\mu \eta}{L+\mu})^{t-j} \EX_A[\|v_j^{k,\nu}-\nabla g_S(x_j^{k,\nu})\|^2])^{\frac{1}{2}} +  3\eta L_f( \sum_{j=1}^{t-1}(1-\frac{2L\mu \eta}{L+\mu})^{t-j} \EX_A[\|\nabla g_S(x_j) - v_j\|^2])^{\frac{1}{2}}\\
			&\quad + 2L_gC_f\eta\sum_{j=0}^{t-1}(1-\frac{2L\mu \eta}{L+\mu})^{t-j}((\EX_A[\| u_j^{k,\nu} - g_S(x_j^{k,\nu})\|^2])^{1/2} + (\EX_A[\| u_j - g_S(x_j)\|^2])^{1/2}) \\
			&\quad + 2L_f\eta\sum_{j=0}^{t-1}(1-\frac{2L\mu \eta}{L+\mu})^{t-j}((\EX_A\|v_j^{k,\nu} - \nabla g_S(x_j^{k,\nu})\|^2)^{1/2} + (\EX_A\|v_j - \nabla g_S(x_j)\|^2)^{1/2})\\
			&\quad + \frac{2L_gL_f(L+\mu)}{n} +  \sqrt{\frac{2L_f^2L_g^2\eta(L+\mu)}{L\mu n }} + \sqrt{\frac{14L_g^2\sigma_f^2\eta(L+\mu)}{L\mu}},
		\end{aligned}
	\end{equation*}
	where we use the inequality that $(\sum_{i=1}^ka_i)^{1/2} \leq \sum_{i=1}^k(a_i)^{1/2}$ and \eqref{Eq:sum_(1-a)^t-j}. Then, with $T$ iterations, we have 
	\begin{equation*}
		\begin{aligned}
			\mathbb{E}_{A}[\|x_{T}-x_{T}^{k, \nu}\|]
			&\leq  6L_g C_f \eta \sup_S (\sum_{j=0}^{T-1}(1-\frac{2L\mu \eta}{L+\mu})^{T-j-1}\EX_A[\|u_j-g_S(x_j)\|^2])^{1/2}\\
			&\quad+ 6L_f\eta \sup_S(\sum_{j=0}^{T-1}(1-\frac{2L\mu \eta}{L+\mu})^{T-j-1}\EX_A[\|v_j - \nabla g_S(x_j)\|^2])^{1/2}\\
			&\quad + 4L_g C_f \eta \sup_S \sum_{j=0}^{T-1}(1-\frac{2L\mu \eta}{L+\mu})^{T-j-1}(\EX_A[\|u_j-g_S(x_j)\|^2])^{1/2}\\
			&\quad + 4L_f\eta \sup_S\sum_{j=0}^{T-1}(1-\frac{2L\mu \eta}{L+\mu})^{T-j-1}(\EX_A[\|v_j - \nabla g_S(x_j)\|^2])^{1/2} \\
			&\quad +  \frac{2L_gL_f(L+\mu)}{n} +  \sqrt{\frac{2L_f^2L_g^2\eta(L+\mu)}{L\mu n }} + \sqrt{\frac{14L_g^2\sigma_f^2\eta(L+\mu)}{L\mu}}.
		\end{aligned}
	\end{equation*}
	
	\noindent{\bf Estimation of $\mathbb{E}_{A}\bigl[\|x_{t+1}-x_{t+1}^{l,\omega}\|\bigr]$}  
	
	Similarly, we will consider two cases: $j_t \neq l$ and $j_t = l$.
	
	\textbf{\quad Case 1 ($j_t \neq l $). } 

	Similarly, we have 
	\begin{equation*}
	\begin{aligned}
		&\|x_{t+1} - x_{t+1}^{l,\omega}\|\\
		& \leq (1- \frac{2L\mu \eta_t}{L+\mu})\|x_t-x_t^{l,\omega}\|^2  +  2L_g C_f \eta_t \| u_t^{l,\omega} - g_S(x_t^{l,\omega})\|\cdot\|x_t^{l,\omega} - x_t \| + 2L_f\eta_t\|v_t^{l,\omega} - \nabla g_S(x_t^{l,\omega})\|\cdot\|x_t^{l,\omega} - x_t \| \\
		&\quad+ 2L_gC_f\eta_t\|u_t-g_S(x_t)\|\cdot \|x_t^{l,\omega}-x_t\| + 2L_f\eta_t\|u_t-g_S(x_t)\|\cdot \|x_t^{l,\omega}-x_t\|\\
		&\quad+ 7\eta_t^2L_g^2\| \nabla f_{\nu_{i_t}}(u_t^{l,\omega}) -\nabla f_S(u_t^{l,\omega})\|^2  + 7\eta_t^2L_g^2C_f^2\|u_t^{l,\omega}-g_S(x_t^{l,\omega})\|^2+ 7\eta_t^2L_f^2\|v_t^{l,\omega}-\nabla g_S(x_t^{l,\omega})\|^2\\
		& \quad + 7\eta_t^2L_f^2\|\nabla g_S(x_t) - v_t\|^2 + 7\eta_t^2L_g^2C_f^2\|u_t-g_S(x_t)\|^2 + 7\eta_t^2L_g^2\| \nabla f_{\nu_{i_t}}(u_t) -\nabla f_S(u_t)\|^2.
	\end{aligned}
\end{equation*}
where the inequality holds by $\eta_t \leq \frac{2}{7(L+\mu)}$.

	\textbf{\quad Case 2 ($j_t = l $). } We have 
	\begin{equation*}
		\|x_{t+1} - x_{t+1}^{l,\omega}\|^2 \leq \|x_t - x_t^{l,\omega}\|^2 + 4\eta_t L_g L_f\|x_t - x_t^{l,\omega}\| + 4\eta_t^2L_g^2L_f^2.
	\end{equation*}
	Combining the above two cases, we have 
	\begin{equation*}
		\begin{aligned}
			&\EX_A[\|x_{t+1} - x_{t+1}^{l,\omega}\|^2] \\
			&\leq (1- \frac{2L\mu \eta_t}{L+\mu})\|x_t-x_t^{l,\omega}\|^2  +  2L_g C_f \eta_t \| u_t^{l,\omega} - g_S(x_t^{l,\omega})\|\cdot\|x_t^{l,\omega} - x_t \| + 2L_f\eta_t\|v_t^{l,\omega} - \nabla g_S(x_t^{l,\omega})\|\cdot\|x_t^{l,\omega} - x_t \| \\
			&\quad+ 2L_gC_f\eta_t\|u_t-g_S(x_t)\|\cdot \|x_t^{l,\omega}-x_t\| + 2L_f\eta_t\|u_t-g_S(x_t)\|\cdot \|x_t^{l,\omega}-x_t\|\\
			&\quad+ 7\eta_t^2L_g^2\| \nabla f_{\nu_{i_t}}(u_t^{l,\omega}) -\nabla f_S(u_t^{l,\omega})\|^2  + 7\eta_t^2L_g^2C_f^2\|u_t^{l,\omega}-g_S(x_t^{l,\omega})\|^2+ 7\eta_t^2L_f^2\|v_t^{l,\omega}-\nabla g_S(x_t^{l,\omega})\|^2\\
			& \quad + 7\eta_t^2L_f^2\|\nabla g_S(x_t) - v_t\|^2 + 7\eta_t^2L_g^2C_f^2\|u_t-g_S(x_t)\|^2 + 7\eta_t^2L_g^2\| \nabla f_{\nu_{i_t}}(u_t) -\nabla f_S(u_t)\|^2\\
			&\quad + 4 \eta_{t} L_{g} L_{f} \mathbb{E}_{A}[\|x_{t}-x_{t}^{l, \omega}\| \1_{[j_{t}=l]}]+4 \eta_{t}^{2} L_{f}^{2} L_{g}^{2} \mathbb{E}_{A}[\1_{[j_{t}=l]}].
		\end{aligned}
	\end{equation*}
	Setting $\eta_t = \eta$,  telescoping above inequality from 1 to $t$ we have 
	\begin{equation*}
		\begin{aligned}
			&\EX_A[\|x_{t+1} - x_{t+1}^{l,\omega}\|^2] \\
			&\leq  2L_gC_f\eta\sum_{j=0}^{t-1}(1-\frac{2L\mu \eta}{L+\mu})^{t-j}((\EX_A[\| u_j^{l,\omega} - g_S(x_j^{l,\omega})\|^2])^{1/2} + (\EX_A[\| u_j - g_S(x_j)\|^2])^{1/2}) (\EX_A[\|x_j - x_j^{l,\omega}\|^2])^{1/2}\\
			&\quad + 2L_f\eta\sum_{j=0}^{t-1}(1-\frac{2L\mu \eta}{L+\mu})^{t-j}((\EX_A\|v_j^{l,\omega} - \nabla g_S(x_j^{l,\omega})\|^2)^{1/2} + (\EX_A\|v_j - \nabla g_S(x_j)\|^2)^{1/2})(\EX_A[\|x_j - x_j^{l,\omega}\|^2])^{1/2}\\
			&\quad  + 7\eta^2 L_g^2 C_f^2\sum_{j=1}^t(1-\frac{2L\mu \eta}{L+\mu})^{t-j} \EX_A[\|u_j^{l,\omega}-g_S(x_j^{l,\omega})\|^2]+7\eta^2 L_g^2 C_f^2\sum_{j=1}^t(1-\frac{2L\mu \eta}{L+\mu})^{t-j} \EX_A[\|u_j-g_S(x_j)\|^2]\\
			&\quad + 7\eta^2 L_f^2\sum_{j=1}^t(1-\frac{2L\mu \eta}{L+\mu})^{t-j} \EX_A[\|v_j^{l,\omega}-\nabla g_S(x_j^{l,\omega})\|^2] +  7\eta^2 L_f^2 \sum_{j=1}^t(1-\frac{2L\mu \eta}{L+\mu})^{t-j} \EX_A[\|\nabla g_S(x_j) - v_j\|^2]\\
			&\quad+\frac{4\eta L_gL_f}{m}\sum_{j=1}^t(1-\frac{2L\mu \eta}{L+\mu})^{t-j} \EX_A[\|x_j - x_{j}^{l,\omega}\|] + \frac{4\eta^2 L_f^2 L_g^2}{m}\sum_{j=1}^t(1-\frac{2L\mu \eta}{L+\mu})^{t-j}+ 14\eta^2 L_g^2\sigma_f^2\sum_{j=1}^{t-1}(1-\frac{2L\mu \eta}{L+\mu})^{t-j} ,
		\end{aligned}
	\end{equation*}

Then with $T$ iterations, we have 
\begin{equation}
\begin{aligned}
	\mathbb{E}_{A}[\|x_{T}-x_{T}^{l, \omega}\|]
	&\leq  6L_g C_f \eta \sup_S (\sum_{j=0}^{T-1}(1-\frac{2L\mu \eta}{L+\mu})^{T-j-1}\EX_A[\|u_j-g_S(x_j)\|^2])^{1/2}\\
 &+ 6L_f\eta \sup_S(\sum_{j=0}^{T-1}(1-\frac{2L\mu \eta}{L+\mu})^{T-j-1}\EX_A[\|v_j - \nabla g_S(x_j)\|^2])^{1/2}\\
	&\quad + 4L_g C_f \eta \sup_S \sum_{j=0}^{T-1}(1-\frac{2L\mu \eta}{L+\mu})^{T-j-1}(\EX_A[\|u_j-g_S(x_j)\|^2])^{1/2}\\
 &+ 4L_f\eta \sup_S\sum_{j=0}^{T-1}(1-\frac{2L\mu \eta}{L+\mu})^{T-j-1}(\EX_A[\|v_j - \nabla g_S(x_j)\|^2])^{1/2} \\
			&\quad +  \frac{2L_gL_f(L+\mu)}{m} +  \sqrt{\frac{2L_f^2L_g^2\eta(L+\mu)}{L\mu m }} + \sqrt{\frac{14L_g^2\sigma_f^2\eta(L+\mu)}{L\mu}}.
		\end{aligned}
	\end{equation}

	Now we combine the above results  for estimating $\mathbb{E}_{A}[\|x_{T}-x_{T}^{k, \nu}\|]$ and $\mathbb{E}_{A}[\|x_{T}-x_{T}^{l, \omega}\|]$, we have 
	\begin{equation}\label{Eq:ineq_holds_two_level_sc}
		\begin{aligned}
			\epsilon_\nu + \epsilon_\omega  
			& \leq 20L_g C_f \eta \sup_S \sum_{j=0}^{T-1}(1-\frac{2L\mu \eta}{L+\mu})^{T-j-1}(\EX_A[\|u_j-g_S(x_j)\|^2])^{1/2}\\
			& \quad + 20L_f\eta \sup_S\sum_{j=0}^{T-1}(1-\frac{2L\mu \eta}{L+\mu})^{T-j-1}(\EX_A[\|v_j - \nabla g_S(x_j)\|^2])^{1/2}+ \sqrt{\frac{2L_f^2L_g^2\eta(L+\mu)}{L\mu m }} \\
			&\quad  + \frac{2(L+ \mu)L_gL_f}{L\mu m} + \sqrt{\frac{2L_f^2L_g^2\eta(L+\mu)}{L\mu n }} + \frac{2(L+ \mu)L_gL_f}{L\mu n} + 8\sqrt{\frac{L_g^2\sigma_f^2\eta(L+\mu)}{L\mu}}.
		\end{aligned}
	\end{equation}
	Now we will illustrate why the second inequality of above holds true. According to Lemma \ref{lemma:u_t_bound_two}, we have 
	\begin{equation*}
		\begin{aligned}
			&(\sum_{j=0}^{T-1}(1-\frac{2L\mu \eta}{L+\mu})^{T-j-1}\EX_A[\|u_j-g_S(x_j)\|^2])^{1/2}\\
			& \leq (\sum_{j=0}^{T-1}(1-\frac{2L\mu \eta}{L+\mu})^{T-j-1}((\frac{c}{e})^{c}(j\beta)^{-c})\EX[\|u_0-g_S(x_0)\|^2]  + 2\sigma_g^2\beta + \frac{2L_g^4L_f^2\eta^2}{\beta} )^{1/2}\\
			& \leq (\sum_{j=0}^{T-1}(1-\frac{2L\mu \eta}{L+\mu})^{T-j-1}(2\sigma_g^2\beta + \frac{2L_g^4L_f^2\eta^2}{\beta} ))^{\frac{1}{2}} + ((\sum_{j=0}^{T-1}(1-\frac{2L\mu \eta}{L+\mu})^{T-j-1}(\frac{c}{e})^{c}(j\beta)^{-c})\EX[\|u_0-g_S(x_0)\|^2])^{\frac{1}{2}}\\
			& \leq \frac{2\sigma_g\sqrt{\beta(L+\mu)}}{\sqrt{ L \mu \eta}} + \frac{2L_g^2L_f\eta\sqrt{L+\mu}}{\sqrt{L \mu \eta\beta}} + (\frac{c}{e})^{\frac{c}{2}} \frac{\sqrt{U_u}(L+\mu)}{L \mu} T^{-\frac{c}{2}} \beta^{-\frac{c}{2}}.
		\end{aligned}
	\end{equation*}
	Likewise,
	\begin{equation}\label{Eq:106}
		\begin{aligned}
			&\sum_{j=0}^{T-1}(1-\frac{2L\mu \eta}{L+\mu})^{T-j-1}(\EX_A[\|u_j-g_S(x_j)\|^2])^{1/2}\\
			& \leq \sum_{j=0}^{T-1}(1-\frac{2L\mu \eta}{L+\mu})^{T-j-1}((\frac{c}{e})^{c}(j\beta)^{-c})\EX[\|u_0-g_S(x_0)\|^2]  + 2\sigma_g^2\beta + \frac{2L_g^4L_f^2\eta^2}{\beta})^{\frac{1}{2}}\\
			& \leq \sum_{j=0}^{T-1}(1-\frac{2L\mu \eta}{L+\mu})^{T-j-1}(\sigma_g\sqrt{2\beta} + \frac{L_g^2L_f\eta\sqrt{2}}{\sqrt{\beta}}) + (\frac{c}{e})^{\frac{c}{2}}\sqrt{U_u}\sum_{j=0}^{T-1}(1-\frac{2L\mu \eta}{L+\mu})^{T-j-1}(j\beta)^{-\frac{c}{2}}\\
			& \leq \frac{(L+\mu)\sigma_g\sqrt{2\beta}}{2L\mu\eta} + \frac{(L+\mu)\sqrt{2}L_g^2L_f}{2L\mu\sqrt{\beta}} + (\frac{c}{e})^{\frac{c}{2}} \frac{\sqrt{U_u}(L+\mu)}{L \mu\eta} T^{-\frac{c}{2}} \beta^{-\frac{c}{2}}.
		\end{aligned}
	\end{equation}
	According to the above two inequalities, we can get the dominating term is $\sum_{j=0}^{T-1}(1-\frac{2L\mu \eta}{L+\mu})^{T-j-1}(\EX_A[\|u_j-g_S(x_j)\|^2])^{1/2}$, then the inequality \eqref{Eq:ineq_holds_two_level_sc} holds true. The treatment of the other items is similar, so we won't go into details. Since often we have $\eta \leq \min (\frac{1}{n}, \frac{1}{m})$, then we have
	\begin{equation}\label{Eq:107}
		\begin{aligned}
			\epsilon_\nu + \epsilon_\omega  \leq O\Big(&  L_g C_f \eta \sup_S \sum_{j=0}^{T-1}(1-\frac{2L\mu \eta}{L+\mu})^{T-j-1}(\EX_A[\|u_j-g_S(x_j)\|^2])^{1/2}\\
			& \quad + L_f\eta \sup_S\sum_{j=0}^{T-1}(1-\frac{2L\mu \eta}{L+\mu})^{T-j-1}(\EX_A[\|v_j - \nabla g_S(x_j)\|^2])^{1/2} \\
			&\quad  + \frac{(L+ \mu)L_gL_f}{L\mu m} + \frac{(L+ \mu)L_gL_f}{L\mu n} + L_g\sigma_f \sqrt{\frac{L+\mu}{L\mu}}\sqrt{\eta}\Big).
		\end{aligned}
	\end{equation}
	This completes the proof.
\end{proof}

\begin{corollary}[Two-level Optimization]\label{cor:2_two_level}Consider Algorithm \ref{alg_cover} with $\eta_t = \eta \leq1/\bigl(4L+4\mu\bigr)$, and $\beta_t=\beta< \min{\{1/8C_f^2, 1\}}$ for any $t\in [0,T-1]$ and the output  $A(S)  =x_{T}$.  Then, we have the following results
\begin{equation*}
	\epsilon_\nu + \epsilon_\omega  \leq O( (T \beta)^{-\frac{c}{2}} + \beta^{\frac{1}{2}}+ \eta^{\frac{1}{2}}  + \eta\beta^{-\frac{1}{2}  }+ n^{-1} +m^{-1}).
\end{equation*}
\end{corollary}

\begin{proof}[proof of Corollary \ref{cor:2_two_level}]
	Next, we move on to the Corollary \ref{cor:2_two_level}. Combining \eqref{Eq:106} and \eqref{Eq:107} we have 
	\begin{equation*}
		\epsilon_\nu + \epsilon_\omega  \leq O(n^{-1}+m^{-1} + \beta^{1/2} + \eta^{1/2}+ \eta\beta^{-1/2}+T^{-\frac{c}{2}} \beta^{-\frac{c}{2}}).
	\end{equation*}
The proof is completed.
\end{proof}

Before giving the detailed proof, we first give a useful lemma.
\begin{lemma}\label{lemma:for_two_level_theorem_optimization_sc}
	Let Assumption \ref{ass:Lipschitz continuous}(ii), \ref{ass:bound variance} (ii) and \ref{ass:Smoothness and Lipschitz continuous gradient} (ii) hold for the empirical risk $F_S$, and $F_S$ is $\mu$-strongly convex, for Algorithm \ref{alg_cover}, we have 
	\begin{equation*}
	\begin{aligned}
		\EX_A[F_{S}(x_{t+1})|\F_t]& \leq \EX_A[F_{S}(x_{t})|\F_t] -\frac{\eta_t}{2}\|\nabla F_S(x_t)\|^2   + L_f^2\eta_t\EX_A[\|v_t - \nabla g_S(x_t)\|^2|\F_t] \\
		&\quad+ L_g^2C_f^2\eta_t\EX_A[\|u_t-g_S(x_t)\|^2|\F_t] + \frac{LL_g^2L_f^2\eta_t^2}{2}.
	\end{aligned}
\end{equation*}
	where $\EX_{A} $  denotes  the expectation taken with respect to the randomness of the algorithm, and $\F_t$ is the $\sigma$-field generated by \(\{\omega_{j_0}, \ldots, \omega_{j_{t-1}}, \nu_{i_0}, \ldots, \nu_{i_{t- 1}}\}\).
\end{lemma}
\begin{proof}
	According to the smoothness of $F_S$, we have 
	\begin{equation*}
		\begin{aligned}
			F_{S}(x_{t+1})& \leq F_{S}(x_{t})+\langle\nabla F_{S}(x_{t}), x_{t+1}-x_{t}\rangle+\frac{L}{2}\|x_{t+1}-x_{t}\|^{2}\\
			& \leq F_{S}(x_{t}) - \eta_t\langle\nabla F_{S}(x_{t}), \nabla g_S(x_t)\nabla f_S(g_S(x_t))\rangle + \frac{L\eta_t^2}{2}\| v_t \nabla f_{\nu_{j_t}} (u_t)\|^2 + \theta_t,
		\end{aligned}
	\end{equation*}
	where $\theta_t = \eta_t \langle \nabla F_S(x_t), g_S(x_t)\nabla f_S(g_S(x_t)) - v_t \nabla f_{\nu_{j_t}} (u_t) \rangle$.
	As for the term $\theta_t$, we have 
	\begin{equation*}
		\begin{aligned}
			\EX_A[\theta_t|\F_t] & =  \eta_t\EX_A[\langle  \nabla F_S(x_t),\nabla g_S(x_t)\nabla f_S(g_S(x_t)) - v_t \nabla f_{\nu_{j_t}} (u_t) \rangle|\F_t] \\
			& = \eta_t \langle \nabla F_S(x_t),\nabla g_S(x_t)\nabla f_S(g_S(x_t)) -v_t \nabla f_S (g_S(x_t)) \rangle|\F_t]\\
			&\quad+ \eta_t\EX_A[ \langle \nabla F_S(x_t), v_t \nabla f_S (g_S(x_t)) - v_t \nabla f_S(u_t)\rangle|\F_t] + \eta_t\EX_A[ \langle \nabla F_S(x_t), v_t \nabla f_S(u_t) - v_t \nabla f_{\nu_{j_t}}(u_t) \rangle|\F_t]\\
			&\leq \eta_t \|\nabla F_S(x_t)\| \cdot \|\nabla f_S(g_S(x_t)) \| \cdot \| v_t - \nabla g_S(x_t)\| + \eta_t \|\nabla F_S(x_t)\| \cdot \| v_t \| \cdot \|\nabla f_S(g_S(x_t)) -\nabla f_S(u_t)\| \\
			&\leq \eta_t L_f\|\nabla F_S(x_t)\| \cdot \|v_t - \nabla g_S(x_t)\| + \eta_t L_g C_f \|\nabla F_S(x_t)\| \cdot \|u_t - g_S(x_t)\| \\
			&\leq \frac{\eta_t}{2}\|\nabla F_S(x_t)\|^2 + L_f^2\eta_t\|v_t - \nabla g_S(x_t)\|^2 + L_g^2C_f^2\eta_t\|u_t-g_S(x_t)\|^2,
		\end{aligned}
	\end{equation*}
	where the last inequality holds by Cauchy-Schwarz inequality. Combining above two inequalities,  let $\F_t$ be the $\sigma$-field generated by \(\{\omega_{j_0}, \ldots, \omega_{j_{t-1}}, \nu_{i_0}, \ldots, \nu_{i_{t- 1}}\}\),  we have 
	\begin{equation*}
		\begin{aligned}
			\EX_A[F_{S}(x_{t+1})|\F_t]& \leq \EX_A[F_{S}(x_{t})|\F_t] -\frac{\eta_t}{2}\|\nabla F_S(x_t)\|^2   + L_f^2\eta_t\EX_A[\|v_t - \nabla g_S(x_t)\|^2|\F_t] \\
			&\quad+ L_g^2C_f^2\eta_t\EX_A[\|u_t-g_S(x_t)\|^2|\F_t] + \frac{LL_g^2L_f^2\eta_t^2}{2}.
		\end{aligned}
	\end{equation*}
	Then we complete the proof.
\end{proof}

\begin{proof}[proof of Theorem \ref{thm:opt_sconvex}]
	We begin to give the detailed proof of Theorem \ref{thm:opt_sconvex}. Note that strong convexity implies the Polyak-\L ojasiewicz (PL) inequality
	\begin{equation*}
		\frac{1}{2}\|\nabla F_{S}(x)\|^{2} \geq \mu(F_{S}(x)-F_{S}(x_{*}^{S})), \quad \forall x.
	\end{equation*}
	Then according to Lemma \ref{lemma:for_two_level_theorem_optimization_sc} and PL condition,  subtracting both sides with $F_S(x_*^S)$  we have  
	\begin{equation*}
		\begin{aligned}
			\EX_A[F_{S}(x_{t+1}) - F_{S}(x_*^S)]& \leq  (1-\mu\eta_t) \EX_A[F_{S}(x_{t}) - F_{S}(x_*^S)] + L_f^2\eta_t\EX_A[\|v_t - \nabla g_S(x_t)\|^2|\F_t] \\
			&\quad+ L_g^2C_f^2\eta_t\EX_A[\|u_t-g_S(x_t)\|^2|\F_t] + \frac{LL_g^2L_f^2\eta_t^2}{2}.
		\end{aligned}
	\end{equation*}
	Setting $\eta_t = \eta$ and $\beta_t = \beta$, using Lemma \ref{lemma:u_t_bound_two} and \ref{lemma:iv_t_bound} , we have 
	\begin{equation*}
		\begin{aligned}
			&\EX_A[F_{S}(x_{t+1}) - F_{S}(x_*^S)]\\
			& \leq  (1-\mu\eta) \EX_A[F_{S}(x_{t}) - F_{S}(x_*^S)]  + \frac{LL_g^2L_f^2\eta^2}{2}  + L_g^2C_f^2\eta((\frac{c}{e})^{c}U(t\beta)^{-c} + 2\sigma_g^2\beta + \frac{2L_g^4L_f^2\eta^2}{\beta}) \\
			&\quad+  L_f^2\eta((\frac{c}{e})^{c}V(t\beta)^{-c} + 2\sigma_{g'}^2\beta + \frac{2L_g^4L_f^2\eta^2}{\beta}).
		\end{aligned}
	\end{equation*}
	Telescoping the above inequality from 1 to $T-1$ we have 
	\begin{equation*}
		\begin{aligned}
			&\EX[F_{S}(x_{T}) - F_{S}(x_*^S)]\\
			& \leq (1-\mu\eta)^{T-1} \EX[F_{S}(x_{1}) - F_{S}(x_*^S)]  + \frac{LL_g^2L_f^2\eta^2}{2}\sum_{t=1}^{T-1}(1-\mu \eta)^{T-t-1} \\
			&\quad +(\frac{c}{e})^{c}\beta^{-c}(L_g^2C_f^2\eta U+ L_f^2\eta V) \sum_{t=1}^{T-1} t^{-c}(1-\mu \eta)^{T-t-1} \\
			&\quad + (2L_g^2C_f^2\sigma_{g}^2\beta\eta+ 2L_f^2\sigma_{g'}^2\beta\eta + \frac{2L_g^6C_f^2L_f^2\eta^3}{\beta} + \frac{2L_g^4L_f^4\eta^3}{\beta} )\sum_{t=1}^{T-1} (1-\mu \eta)^{T-t-1}.
		\end{aligned}
	\end{equation*}
	For $t=0$, we have 
	\begin{equation*}
		\begin{aligned}
			\EX_A[F_S(x_1) - F_S(x_*^S)]\leq (1-\mu \eta)\EX_A[F_S(x_0) - F_S(x_*^S)] + L_g^2C_f^2\eta U  + L_f^2\eta V  + \frac{LL_g^2L_f^2\eta^2}{2}.
		\end{aligned}
	\end{equation*}
	Then combining the above two inequality we have 
	\begin{equation*}
		\begin{aligned}
			&\EX[F_{S}(x_{T}) - F_{S}(x_*^S)]\\
			& \leq (1-\mu \eta)^T\EX_A[F_S(x_0) - F_S(x_*^S)] + \frac{LL_g^2L_f^2\eta^2}{2}\sum_{t=1}^{T}(1-\mu \eta)^{T-t} + (L_g^2C_f^2\eta U  + L_f^2\eta V )(1-\mu\eta)^{T-1}\\
			&\quad +(\frac{c}{e})^{c}\beta^{-c}(L_g^2C_f^2\eta U+ L_f^2\eta V) \sum_{t=1}^{T-1} t^{-c}(1-\mu \eta)^{T-t-1} \\
			&\quad + (2L_g^2C_f^2\sigma_{g}^2\beta\eta+ 2L_f^2\sigma_{g'}^2\beta\eta + \frac{2L_g^6C_f^2L_f^2\eta^3}{\beta} + \frac{2L_g^4L_f^4\eta^3}{\beta} )\sum_{t=1}^{T-1} (1-\mu \eta)^{T-t-1}.
		\end{aligned}
	\end{equation*}
	According to the fact that $\sum_{t=1}^{T}(1-\mu \eta)^{T-t} \leq \frac{1}{\mu \eta}$, using  Lemma \ref{lemma:weighted_avg}, we have 
	
	\begin{equation*}
		\sum_{t=1}^{T-1}\left(1-\mu \eta\right)^{T-t-1} t^{-c} \leq \frac{\sum_{t=1}^{T-1}\left(1-\mu \eta\right)^{T-t-1}}{T-1} \sum_{t=1}^{T-1} t^{-c} \leq \frac{1}{T \mu \eta} \sum_{t=1}^{T-1} t^{-c}.
	\end{equation*}
Then we can get
	\begin{equation}\label{Eq:F_s-F_s*_sc}
		\begin{aligned}
			\EX[F_{S}(x_{T}) - F_{S}(x_*^S)]
			& \leq (\frac{c}{e \mu})^{c}(\eta T)^{-c} D_x +  \frac{LL_g^2L_f^2\eta}{2\mu} +(L_g^2C_f^2\eta U  + L_f^2\eta V )(\frac{c}{e\mu})^c(\eta T)^{-c}\\
			&\quad  + (\frac{c}{e})^{c}\beta^{-c}(L_g^2C_f^2\eta U+ L_f^2\eta V)T^{-1}\mu^{-1} \sum_{t=1}^{T-1} t^{-c}\\
			& \quad + \frac{2L_g^2C_f^2\sigma_{g}^2\beta}{\mu} + \frac{2L_f^2\sigma_{g'}^2\beta}{\mu} + \frac{2L_g^6C_f^2L_f^2\eta^2}{\beta\mu}+ \frac{2L_g^4L_f^4\eta^2}{\beta\mu}.
		\end{aligned}
	\end{equation}
	According to $\sum_{t=1}^Tt^{-z} = O(T^{1-z})$ for $z \in (-1,0) \cup(-\infty,-1)$ and $\sum_{t=1}^Tt^{-1} = O(\log T)$, as long as $c \neq 1$ we have 
	\begin{equation*}
		\begin{aligned}
			\EX[F_{S}(x_{T}) - F_{S}(x_*^S)]
			& \leq O\Big(D_x (\eta T)^{-c}  + LL_g^2L_f^2\eta + (L_g^2C_f^2 U  + L_f^2 V )(\eta T)^{-c}\eta \\
			&~~~~~~~~~~~+ (L_g^2C_f^2\eta U+ L_f^2\eta V) (\beta T)^{-c} + (L_g^2C_f^2\sigma_g^2 + L_f^2\sigma_{g'}^2)\beta+ (L_g^6C_f^2L_f^2 +L_f^4L_g^4) \eta^2\beta^{-1}\Big).
		\end{aligned}
	\end{equation*}
	The proof is completed.
\end{proof}

\begin{proof}[proof of Theorem \ref{thm:Excess_Risk_Bound_s_convex}]
	Putting \eqref{Eq:106} into \eqref{Eq:ineq_holds_two_level_sc}, we have 
	\begin{equation*}
		\begin{aligned}
			&\mathbb{E}_{A}[\|x_{T}-x_{T}^{k, \nu}\|]+ 4\mathbb{E}_{A}[\|x_{T}-x_{T}^{l, \omega}\|]\\
			& \leq 50L_g C_f \eta \sup_S (\frac{(L+\mu)\sigma_g\sqrt{2\beta}}{2L\mu\eta} + \frac{(L+\mu)\sqrt{2}L_g^2L_f}{2L\mu\sqrt{\beta}} + (\frac{c}{e})^{\frac{c}{2}} \frac{\sqrt{U}(L+\mu)}{L \mu\eta} T^{-\frac{c}{2}} \beta^{-\frac{c}{2}})\\
			& \quad + 50L_f\eta \sup_S(\frac{(L+\mu)\sigma_{g'}\sqrt{2\beta}}{2L\mu\eta} + \frac{(L+\mu)\sqrt{2}L_g^2L_f}{2L\mu\sqrt{\beta}} + (\frac{c}{e})^{\frac{c}{2}} \frac{\sqrt{V}(L+\mu)}{L \mu\eta} T^{-\frac{c}{2}} \beta^{-\frac{c}{2}}) \\
			&\quad + 4\sqrt{\frac{2L_f^2L_g^2\eta(L+\mu)}{L\mu m }} + \frac{8(L+ \mu)L_gL_f}{L\mu m} + \sqrt{\frac{2L_f^2L_g^2\eta(L+\mu)}{L\mu n }} + \frac{2(L+ \mu)L_gL_f}{L\mu n}.
		\end{aligned}
	\end{equation*}
From Theorem \ref{theorem:general_two_level} we have 
\begin{equation*}
\begin{aligned}
    &\EX [F(x_T) - F_S(x_T)]\\
    & \leq 50 C_fL_fL_g^2 \eta \sup_S (\frac{(L+\mu)\sigma_g\sqrt{2\beta}}{2L\mu\eta} + \frac{(L+\mu)\sqrt{2}L_g^2L_f}{2L\mu\sqrt{\beta}} + (\frac{c}{e})^{\frac{c}{2}} \frac{\sqrt{U_u}(L+\mu)}{L \mu\eta} T^{-\frac{c}{2}} \beta^{-\frac{c}{2}})\\
    & \quad + 50L_f^2L_g\eta \sup_S(\frac{(L+\mu)\sigma_{g'}\sqrt{2\beta}}{2L\mu\eta} + \frac{(L+\mu)\sqrt{2}L_g^2L_f}{2L\mu\sqrt{\beta}} + (\frac{c}{e})^{\frac{c}{2}} \frac{\sqrt{U_{\iv}}(L+\mu)}{L \mu\eta} T^{-\frac{c}{2}} \beta^{-\frac{c}{2}}) \\
    &\quad + 4L_f^2L_g^2\sqrt{\frac{2\eta(L+\mu)}{L\mu m }} + \frac{8(L+ \mu)L_g^2L_f^2}{L\mu m} + L_f^2L_g^2\sqrt{\frac{2\eta(L+\mu)}{L\mu n }} + \frac{2(L+ \mu)L_g^2L_f^2}{L\mu n}\\
    & + L_{f} \sqrt{m^{-1} \mathbb{E}_{S, A}[\operatorname{Var}_{\omega}(g_{\omega}(A(S)))]}.
\end{aligned}
\end{equation*}
	Combining  \eqref{Eq:F_s-F_s*_sc} and above inequality,  using $F_S(x_*^S) \leq F_S(x_*)$ we have 
	
	\begin{equation*}
		\begin{aligned}
			&\mathbb{E}_{S, A}\left[F(A(S))-F\left(x_{*}\right)\right]\\
			& \leq 50 C_fL_fL_g^2 \eta \sup_S (\frac{(L+\mu)\sigma_g\sqrt{2\beta}}{2L\mu\eta} + \frac{(L+\mu)\sqrt{2}L_g^2L_f}{2L\mu\sqrt{\beta}} + (\frac{c}{e})^{\frac{c}{2}} \frac{\sqrt{U_u}(L+\mu)}{L \mu\eta} T^{-\frac{c}{2}} \beta^{-\frac{c}{2}})\\
			& \quad + 50L_f^2L_g\eta \sup_S(\frac{(L+\mu)\sigma_{g'}\sqrt{2\beta}}{2L\mu\eta} + \frac{(L+\mu)\sqrt{2}L_g^2L_f}{2L\mu\sqrt{\beta}} + (\frac{c}{e})^{\frac{c}{2}} \frac{\sqrt{U_{\iv}}(L+\mu)}{L \mu\eta} T^{-\frac{c}{2}} \beta^{-\frac{c}{2}}) \\
			&\quad + 4L_f^2L_g^2\sqrt{\frac{2\eta(L+\mu)}{L\mu m }} + \frac{8(L+ \mu)L_g^2L_f^2}{L\mu m} + L_f^2L_g^2\sqrt{\frac{2\eta(L+\mu)}{L\mu n }} + \frac{2(L+ \mu)L_g^2L_f^2}{L\mu n}\\
			&\quad +  L_{f} \sqrt{m^{-1} \mathbb{E}_{S, A}[\operatorname{Var}_{\omega}(g_{\omega}(A(S)))]} + (\frac{c}{e \mu})^{c}(\eta T)^{-c} D_x +  \frac{LL_g^2L_f^2\eta}{2\mu} +(L_g^2C_f^2\eta U  + L_f^2\eta V )(\frac{c}{e\mu})^c(\eta T)^{-c}\\
			&\quad  + (\frac{c}{e})^{c}\beta^{-c}(L_g^2C_f^2\eta U+ L_f^2\eta V)T^{-1}\mu^{-1} \sum_{t=1}^{T-1} t^{-c} + \frac{2L_g^2C_f^2\sigma_{g}^2\beta}{\mu} + \frac{2L_f^2\sigma_{g'}^2\beta}{\mu} + \frac{2L_g^6C_f^2L_f^2\eta^2}{\beta\mu}+ \frac{2L_g^4L_f^4\eta^2}{\beta\mu}.
		\end{aligned}
	\end{equation*}
	Setting $\eta = T^{-a}, \beta = T^{-b}$,  since often we have $\eta \leq \min (\frac{1}{n}, \frac{1}{m})$, then we have 
	\begin{equation*}
		\begin{aligned}
			\mathbb{E}_{S, A}\left[F(A(S))-F\left(x_{*}\right)\right]
			& \leq O(T^{-\frac{b}{2}} + T^{\frac{b}{2}-a} + T^{\frac{c}{2}(b-1)}+ m^{-\frac{1}{2}} + m^{-1} + T^{-c(1-a)} + T^{-c(1-b)}\\
			&\quad + T^{-a} + T^{-b} + T^{b-a} + T^{b-2a} + m^{-\frac{1}{2}}T^{-\frac{a}{2}} + n^{-\frac{1}{2}}T^{-\frac{a}{2}} ).
		\end{aligned}
	\end{equation*}
Setting $c =3$, the dominating terms are $\mathcal{O}(T^{\frac{b}{2}- a}), \mathcal{O}(T^{-\frac{b}{2}}), \quad \mathcal{O}(T^{3(b- 1)}), \mathcal{O}(T^{-\frac{a}{2}})$, and$\mathcal{O}(T^{6(a- 1)})$.
  
Setting $a = b = \frac{6}{7}$, we have 
\begin{equation*}
	\mathbb{E}_{S, A}\left[F(A(S))-F\left(x_{*}\right)\right]=O\left(T^{-\frac{3}{7}}\right).
\end{equation*}
	
Setting $T = O(\max\{n^{7/6},m^{7,6}\})$, we have the following 
\begin{equation*}
	\mathbb{E}_{S, A}\left[F(A(S))-F\left(x_{*}\right)\right]=O\left(\frac{1}{\sqrt{n}}+\frac{1}{\sqrt{m}}\right).
\end{equation*}
The proof is completed.
\end{proof}

%% file: appendix_theorem_multi_level.tex
\section{$K$-level Stochastic Optimizations}\label{App:Multi}
\begin{lemma}[lemma 6 in \cite{jiang2022optimal}]\label{lemma:v_t_recursive_mul}
	Let Assumption \ref{ass:Lipschitz continuous}(iii), \ref{ass:bound variance} (iii) and \ref{ass:Smoothness and Lipschitz continuous gradient} (iii) hold for the empirical risk.  $x_t, u_t^{(i)}$ and $v_t^{(i)}$ are generated by Algorithm \ref{alg_svmr_multi} for any $i \in [1,K]$, then we have 
	\begin{equation}
		\mathbb{E}[\|v_{t}^{(i)}-\nabla f_{i,S}(u_{t}^{(i-1)})\|^{2}] \leq(1-\beta_{t}) \mathbb{E}[\|v_{t-1}^{(i)}-\nabla f_{i,S}(u_{t-1}^{(i-1)})\|^{2}]+2 \beta_{t}^{2} \sigma_{J}^{2}+2 L_f^{2} \mathbb{E}[\|u_{t}^{(i-1)}-u_{t-1}^{(i-1)}\|^{2}]
	\end{equation}
\end{lemma}

\begin{lemma}[lemma 6 in \cite{jiang2022optimal}]\label{lemma:u_t_recursive_mul}
	Let Assumption \ref{ass:Lipschitz continuous}(iii), \ref{ass:bound variance} (iii) and \ref{ass:Smoothness and Lipschitz continuous gradient} (iii) hold for the empirical risk. $x_t, u_t^{(i)}$ and $v_t^{(i)}$ are generated by Algorithm \ref{alg_svmr_multi} for any $i \in [1,K]$, then we have 
	\begin{equation}
		\mathbb{E}[\|u_{t}^{(i)}- f_{i,S}(u_{t}^{(i-1)})\|^{2}] \leq(1-\beta_{t}) \mathbb{E}[\|u_{t-1}^{(i)}- f_{i,S}(u_{t-1}^{(i-1)})\|^{2}]+2 \beta_{t}^{2} \sigma_{f}^{2}+2 L_f^{2} \mathbb{E}[\|u_{t}^{(i-1)}-u_{t-1}^{(i-1)}\|^{2}]
	\end{equation}
\end{lemma}

\begin{lemma}[lemma 7 in \cite{jiang2022optimal}]\label{lemma:u_t_layer_recursive_mul}
	Let Assumption \ref{ass:Lipschitz continuous}(iii), \ref{ass:bound variance} (iii) and \ref{ass:Smoothness and Lipschitz continuous gradient} (iii) hold for the empirical risk. $x_t, u_t^{(i)}$ and $v_t^{(i)}$ are generated by Algorithm \ref{alg_svmr_multi} for any $i \in [1,K]$. Then for any $P \in [1, K]$, we have 
	\begin{equation}
		\sum_{i=1}^{P} \mathbb{E}[\|u_{t+1}^{(i-1)}-u_{t}^{(i-1)}\|^{2}] \leq(\sum_{i=1}^{P}(2 L_f^{2})^{i-1})(\mathbb{E}[\|x_{t+1} - x_t\|^{2}]+2 \beta_{t+1}^{2} \sigma_{f}^{2} P+2 \beta_{t+1}^{2} P \sum_{i=1}^{P} \mathbb{E}[\|u_{t}^{(i)}-f_{i}(u_{t}^{(i-1)})\|^{2}]).
	\end{equation}
\end{lemma}

\begin{lemma}\label{lemma:u_t_bound_mul}
	Let Assumption \ref{ass:Lipschitz continuous}(iii), \ref{ass:bound variance} (iii) and \ref{ass:Smoothness and Lipschitz continuous gradient} (iii) hold for the empirical risk. $x_t, u_t^{(i)}$ and $v_t^{(i)}$ are generated by Algorithm \ref{alg_svmr_multi} for any $i \in [1,K]$, let $ 0 <\eta_t = \eta <  1$ and let $ 0<  \beta_{t} =\beta < \max{\{ 1, 1/(4K\sum_{i=1}^K(2L_f^2)^{i}\}}$ we have 
	\begin{equation*}
		\begin{aligned}
			&\sum_{i=1}^K \mathbb{E}[\|u_{t}^{(i)}- f_{i,S}(u_{t}^{(i-1)})\|^{2}]\\
			& \leq \sum_{i=1}^K (\frac{c}{e})^{c}(\frac{t\beta}{2})^{-c} \mathbb{E}[\|u_{1}^{(i)}- f_{i,S}(x_{0})\|^{2}]+ 4\beta \sigma_{f}^{2} K ((\sum_{i=1}^{K}(2 L_f^{2})^{i}) + 1) + \frac{2\sum_{i=1}^{K}(2 L_f^{2})^{i}\eta^2 L_f^K}{\beta}.
		\end{aligned}
	\end{equation*}
\end{lemma}

\begin{proof}[proof of Lemma \ref{lemma:u_t_bound_mul}]
	Now we give the detailed proof of Lemma \ref{lemma:u_t_bound_mul}. According to Lemma \ref{lemma:u_t_recursive_mul} and \ref{lemma:u_t_layer_recursive_mul}, we have 
	\begin{equation}
		\begin{aligned}
			&\sum_{i=1}^K \mathbb{E}[\|u_{t}^{(i)}- f_{i,S}(u_{t}^{(i-1)})\|^{2}]\\
			& \leq \sum_{i=1}^K(1-\beta_{t}) \mathbb{E}[\|u_{t-1}^{(i)}- f_{i,S}(u_{t-1}^{(i-1)})\|^{2}]+2 \beta_{t}^{2} \sigma_{f}^{2} K\\
			&+(\sum_{i=1}^{K}(2 L_f^{2})^{i})(\mathbb{E}[\|x_{t} - x_{t-1}\|^{2}]+2 \beta_{t}^{2} \sigma_{f}^{2} K+2 \beta_{t}^{2} K \sum_{i=1}^{K} \mathbb{E}[\|u_{t-1}^{(i)}-f_{i}(u_{t-1}^{(i-1)})\|^{2}]).
		\end{aligned}
	\end{equation}
	According to the setting that $\beta_{t} \leq  \max{\{ 1, 1/(4K\sum_{i=1}^K(2L_f^2)^{i}\}}$, we have 
	\begin{equation}
		\begin{aligned}
			&\sum_{i=1}^K \mathbb{E}[\|u_{t}^{(i)}- f_{i,S}(u_{t}^{(i-1)})\|^{2}]\\
			& \leq \sum_{i=1}^K(1-\frac{\beta_{t}}{2}) \mathbb{E}[\|u_{t-1}^{(i)}- f_{i,S}(u_{t-1}^{(i-1)})\|^{2}]+2 \beta_{t}^{2} \sigma_{f}^{2} K+(\sum_{i=1}^{K}(2 L_f^{2})^{i})(\mathbb{E}[\|x_{t} - x_{t-1}\|^{2}]+2 \beta_{t}^{2} \sigma_{f}^{2} K)\\
			& \leq \sum_{i=1}^K(1-\frac{\beta_{t}}{2}) \mathbb{E}[\|u_{t-1}^{(i)}- f_{i,S}(u_{t-1}^{(i-1)})\|^{2}]+2 \beta_{t}^{2} \sigma_{f}^{2} K ((\sum_{i=1}^{K}(2 L_f^{2})^{i}) + 1)+\sum_{i=1}^{K}(2 L_f^{2})^{i}\eta_t^2 L_f^K.
		\end{aligned}
	\end{equation}
	Then using Lemma \ref{lemma:recursion lemma}, setting $\eta_t = \eta$ and $\beta_t = \beta$, similar to the proof of Lemma \ref{lemma:u_t_bound_two}, we have 
	\begin{equation}
		\begin{aligned}
			&\sum_{i=1}^K \mathbb{E}[\|u_{t}^{(i)}- f_{i,S}(u_{t}^{(i-1)})\|^{2}]\\
			& \leq \sum_{i=1}^K\prod_{j=1}^{t}(1-\frac{\beta_j}{2})  \mathbb{E}[\|u_{1}^{(i)}- f_{i,S}(x_{0})\|^{2}] + 4\beta \sigma_{f}^{2} K ((\sum_{i=1}^{K}(2 L_f^{2})^{i}) + 1) + \frac{2\sum_{i=1}^{K}(2 L_f^{2})^{i}\eta^2 L_f^K}{\beta}.
		\end{aligned}
	\end{equation}
	Note that $\prod_{i=K}^N\leq \exp(-\sum_{i=K}^N\beta_i)$ for all $K\leq N$ and $\beta_i \geq 0$, then we have 
	\begin{equation}
		\begin{aligned}
			&\sum_{i=1}^K \mathbb{E}[\|u_{t}^{(i)}- f_{i,S}(u_{t}^{(i-1)})\|^{2}]\\
			& \leq \sum_{i=1}^K \exp(-\frac{\beta i}{2})  \mathbb{E}[\|u_{1}^{(i)}- f_{i,S}(x_{0})\|^{2}] +4\beta \sigma_{f}^{2} K ((\sum_{i=1}^{K}(2 L_f^{2})^{i}) + 1) + \frac{2\sum_{i=1}^{K}(2 L_f^{2})^{i}\eta^2 L_f^K}{\beta}\\
			& \leq \sum_{i=1}^K (\frac{c}{e})^{c}(\frac{t\beta}{2})^{-c} \mathbb{E}[\|u_{1}^{(i)}- f_{i,S}(x_{0})\|^{2}]+ 4\beta \sigma_{f}^{2} K ((\sum_{i=1}^{K}(2 L_f^{2})^{i}) + 1) + \frac{2\sum_{i=1}^{K}(2 L_f^{2})^{i}\eta^2 L_f^K}{\beta}.
		\end{aligned}
	\end{equation}
	Then we finish the proof.
	
	\begin{lemma}\label{lemma:v_t_bound_mul}
		Let Assumption \ref{ass:Lipschitz continuous}(iii), \ref{ass:bound variance} (iii) and \ref{ass:Smoothness and Lipschitz continuous gradient} (iii) hold for the empirical risk.  $x_t, u_t^{(i)}$ and $v_t^{(i)}$ are generated by Algorithm \ref{alg_svmr_multi} for any $i \in [1,K]$, let $ 0 <\eta_t = \eta <  1$ and let $ 0<  \beta_{t} =\beta < \max{\{ 1, 1/(8K\sum_{i=1}^K(2L_f^2)^{i}\}}$ we have 
		\begin{equation*}
			\begin{aligned}
				&\sum_{i=1}^K  \mathbb{E}[\|v_{t}^{(i)}-\nabla  f_{i,S}(u_{t}^{(i-1)})\|^{2}]  \\
				& \leq 	\sum_{i=1}^K (\frac{c}{e})^{c}(\frac{t\beta}{2})^{-c}  (\mathbb{E}[\|u_{1}^{(i)}- f_{i,S}(x_{0})\|^{2}] + \mathbb{E}[\|v_{1}^{(i)}- \nabla f_{i,S}(x_{0})\|^{2}] ) + \frac{4(\sum_{i=1}^{K}(2 L_f^{2})^{i})\eta^2 L_f^K}{\beta} \\
				&+4 \beta K  (\sigma_{f}^{2} + \sigma_{J}^{2} + 2 \sigma_{f}^2(\sum_{i=1}^{K}(2 L_f^{2})^{i})  ).
			\end{aligned}
		\end{equation*}
	\end{lemma}
	
	\begin{proof}[proof of Lemma \ref{lemma:v_t_bound_mul}]
		Now we give the detailed proof of Lemma \ref{lemma:v_t_bound_mul}. According to Lemma \ref{lemma:u_t_recursive_mul}, \ref{lemma:u_t_layer_recursive_mul} and \ref{lemma:v_t_recursive_mul}, we have 
		\begin{equation*}
			\begin{aligned}
				&\sum_{i=1}^K ( \mathbb{E}[\|v_{t}^{(i)}-\nabla  f_{i,S}(u_{t}^{(i-1)})\|^{2}] +\mathbb{E}[\|u_{t}^{(i)}- f_{i,S}(u_{t}^{(i-1)})\|^{2}]) \\
				& \leq 	 \sum_{i=1}^K(1-\beta_{t}) \mathbb{E}[\|u_{t-1}^{(i)}- f_{i,S}(u_{t-1}^{(i-1)})\|^{2}]+2 \beta_{t}^{2} \sigma_{f}^{2} K +  \sum_{i=1}^K(1-\beta_{t}) \mathbb{E}[\|v_{t-1}^{(i)}-\nabla f_{i,S}(u_{t-1}^{(i-1)})\|^{2}]+2 \beta_{t}^{2} \sigma_{J}^{2} K\\
				&+2(\sum_{i=1}^{K}(2 L_f^{2})^{i})(\mathbb{E}[\|x_{t} - x_{t-1}\|^{2}]+2 \beta_{t}^{2} \sigma_{f}^{2} K+2 \beta_{t}^{2} K \sum_{i=1}^{K} \mathbb{E}[\|u_{t-1}^{(i)}-f_{i}(u_{t-1}^{(i-1)})\|^{2}]).
			\end{aligned}
		\end{equation*}
		According to the setting that  $ 0 <\eta_t = \eta <  1$ and let $ 0<  \beta_{t} =\beta < \max{\{ 1, 1/(8K\sum_{i=1}^K(2L_f^2)^{i}\}}$ we have 
		\begin{equation*}
			\begin{aligned}
				&\sum_{i=1}^K ( \mathbb{E}[\|v_{t}^{(i)}-\nabla  f_{i,S}(u_{t}^{(i-1)})\|^{2}] +\mathbb{E}[\|u_{t}^{(i)}- f_{i,S}(u_{t}^{(i-1)})\|^{2}]) \\
				& \leq 	 \sum_{i=1}^K(1-\frac{\beta_{t}}{2}) \mathbb{E}[\|u_{t-1}^{(i)}- f_{i,S}(u_{t-1}^{(i-1)})\|^{2}] +  \sum_{i=1}^K(1-\beta_{t}) \mathbb{E}[\|u_{t-1}^{(i)}- f_{i,S}(u_{t-1}^{(i-1)})\|^{2}]\\
				&+2(\sum_{i=1}^{K}(2 L_f^{2})^{i})\eta^2 L_f^K  +2 \beta_{t}^{2} K (\sigma_{f}^{2} + \sigma_{J}^{2} + 2 \sigma_{f}^2(\sum_{i=1}^{K}(2 L_f^{2})^{i})  ) .
			\end{aligned}
		\end{equation*}
	\end{proof}
	Using Lemma \ref{lemma:recursion lemma}, we have 
	\begin{equation*}
		\begin{aligned}
			&\sum_{i=1}^K ( \mathbb{E}[\|v_{t}^{(i)}-\nabla  f_{i,S}(u_{t}^{(i-1)})\|^{2}] +\mathbb{E}[\|u_{t}^{(i)}- f_{i,S}(u_{t}^{(i-1)})\|^{2}]) \\
			& \leq 	\sum_{i=1}^K\prod_{j=1}^{t}(1-\frac{\beta_j}{2}) (\mathbb{E}[\|u_{1}^{(i)}- f_{i,S}(x_{0})\|^{2}] + \mathbb{E}[\|v_{1}^{(i)}- \nabla f_{i,S}(x_{0})\|^{2}] ) + \frac{4(\sum_{i=1}^{K}(2 L_f^{2})^{i})\eta^2 L_f^K}{\beta} \\
			&+4 \beta K  (\sigma_{f}^{2} + \sigma_{J}^{2} + 2 \sigma_{f}^2(\sum_{i=1}^{K}(2 L_f^{2})^{i})  ).
		\end{aligned}
	\end{equation*}
	Then we have 
	
	\begin{equation*}
		\begin{aligned}
			&\sum_{i=1}^K  \mathbb{E}[\|v_{t}^{(i)}-\nabla  f_{i,S}(u_{t}^{(i-1)})\|^{2}]  \\
			& \leq 	\sum_{i=1}^K (\frac{c}{e})^{c}(\frac{t\beta}{2})^{-c}  (\mathbb{E}[\|u_{1}^{(i)}- f_{i,S}(x_{0})\|^{2}] + \mathbb{E}[\|v_{1}^{(i)}- \nabla f_{i,S}(x_{0})\|^{2}] ) + \frac{4(\sum_{i=1}^{K}(2 L_f^{2})^{i})\eta^2 L_f^K}{\beta} \\
			&+4 \beta K  (\sigma_{f}^{2} + \sigma_{J}^{2} + 2 \sigma_{f}^2(\sum_{i=1}^{K}(2 L_f^{2})^{i})  ).
		\end{aligned}
	\end{equation*}
	This complete the proof.
\end{proof}

\begin{proof}[proof of Theorem \ref{theorem:general_multi_level}]
	
	\begin{equation}\label{Eq:mulit_general_orgin}
		\begin{aligned}
			&\EX_{S,A}[F(A(S))-F_S(A(S))]\\
			& = \EX_{S,A}[\EX_{\nu^{(K)}}[f_K^{\nu^{(K)}}(\EX_{\nu^{(K-1)}}[f_{K-1}^{\nu^{(K-1)}}]\cdots\EX_{\nu^{(1)}}[f_{1}^{\nu^{(1)}}(A(S))])] \\
			& ~~~~~~~~~~~~~~~~~~~~~~~~~~~~~~~~~~~~~~~~~~~~~~~~~~~- \frac{1}{n_K}\sum_{i_K=1}^{n_K}f_K^{\nu_{i_K}^{(K)}}(\frac{1}{n_{K-1}}\sum_{i_{K-1} = 1}^{n_{K-1}}f_{K-1}^{\nu_{i_{K-1}}^{(K-1)}}\cdots (\frac{1}{n_1}\sum_{i_1 = 1}^{n_1}f_1^{\nu_{i_1}^{(1)}}(A(S))))]\\
			& =\EX_{S,A}[ \EX_{\nu^{(K)}}[f_K^{\nu^{(K)}}(\EX_{\nu^{(K-1)}}[f_{K-1}^{\nu^{(K-1)}}]\cdots\EX_{\nu^{(1)}}[f_{1}^{\nu^{(1)}}(A(S))])] \\
			& ~~~~~~~~~~~~~~~~~~~~~~~~~~~~~~~~~~~~~~~~~~~~~~~~~~~- \frac{1}{n_K}\sum_{i_K=1}^{n_K}f_K^{\nu_{i_K}^{(K)}}(\EX_{\nu^{(K-1)}}[f_{K-1}^{\nu^{(K-1)}}]\cdots \EX_{\nu^{(1)}}[f_{1}^{\nu^{(1)}}(A(S))])\\
			& + \EX_{S,A}[\frac{1}{n_K}\sum_{i_K=1}^{n_K}f_K^{\nu_{i_K}^{(K)}}(\EX_{\nu^{(K-1)}}[f_{K-1}^{\nu^{(K-1)}}]\cdots \EX_{\nu^{(1)}}[f_{1}^{\nu^{(1)}}(A(S))]) \\
			& ~~~~~~~~~~~~~~~~~~~~~~~~~~~~~~~~~~~~~~~~~~~~~~~~~~~ - \frac{1}{n_K}\sum_{i_K=1}^{n_K}f_K^{\nu_{i_K}^{(K)}}(\frac{1}{n_{K-1}}\sum_{i_{K-1} = 1}^{n_{K-1}}f_{K-1}^{\nu_{i_{K-1}}^{(K-1)}}\cdots \EX_{\nu^{(1)}}[f_{1}^{\nu^{(1)}}(A(S))])]\\
			&\quad\vdots\\
			& +\EX_{S,A}[\frac{1}{n_K}\sum_{i_K=1}^{n_K}f_K^{\nu_{i_K}^{(K)}}(\frac{1}{n_{K-1}}\sum_{i_{K-1} = 1}^{n_{K-1}}f_{K-1}^{\nu_{i_{K-1}}^{(K-1)}}\cdots \EX_{\nu^{(1)}}[f_{1}^{\nu^{(1)}}(A(S))]) \\
			& ~~~~~~~~~~~~~~~~~~~~~~~~~~~~~~~~~~~~~~~~~~~~~~~~~~~ - \frac{1}{n_K}\sum_{i_K=1}^{n_K}f_K^{\nu_{i_K}^{(K)}}(\frac{1}{n_{K-1}}\sum_{i_{K-1} = 1}^{n_{K-1}}f_{K-1}^{\nu_{i_{K-1}}^{(K-1)}}\cdots (\frac{1}{n_1}\sum_{i_1 = 1}^{n_1}f_1^{\nu_{i_1}^{(1)}}(A(S))))].
		\end{aligned}
	\end{equation}
	Now we estimate the terms of the RHS.  Define $S^{(i)} = \{\nu^{(1)}_1,\cdots, \nu^{(1)}_{n_1}, \cdots, \nu^{(i)'}_1, \cdots,\nu^{(i)'}_{n_i}, \cdots, \nu^{(K)}_1,\cdots, \nu^{(K)}_{n_1} \}$, where $i \in [1, K]$.
	
	For the first term, we have 
	\begin{equation*}
		\begin{aligned}
			&\EX_{S,A}[ \EX_{\nu^{(K)}}[f_K^{\nu^{(K)}}(\EX_{\nu^{(K-1)}}[f_{K-1}^{\nu^{(K-1)}}]\cdots\EX_{\nu^{(1)}}[f_{1}^{\nu^{(1)}}(A(S))])] \\
			& ~~~~~~~~~~~~~~~~~~~~~~~~~~~~~~~~~~~~~~~~~~~~~~~~~~~- \frac{1}{n_K}\sum_{i_K=1}^{n_K}f_K^{\nu_{i_K}^{(K)}}(\EX_{\nu^{(K-1)}}[f_{K-1}^{\nu^{(K-1)}}]\cdots \EX_{\nu^{(1)}}[f_{1}^{\nu^{(1)}}(A(S))])\\
			& \leq \EX_{S,A,S^{(K)}}\Big[\frac{1}{n_K}\sum_{i_K=1}^{n_K}f_K^{\nu_{i_K}^{(K)}}(\EX_{\nu^{(K-1)}}[f_{K-1}^{\nu^{(K-1)}}]\cdots \EX_{\nu^{(1)}}[f_{1}^{\nu^{(1)}}(A(S^{i,K}))]) \\
			&~~~~~~~~~~~~~~~~~~~~~~~~~~~~~~~~~~~~~~~~~~~~~~~~~~~-\frac{1}{n_K}\sum_{i_K=1}^{n_K}f_K^{\nu_{i_K}^{(K)}}(\EX_{\nu^{(K-1)}}[f_{K-1}^{\nu^{(K-1)}}]\cdots \EX_{\nu^{(1)}}[f_{1}^{\nu^{(1)}}(A(S))])\Big]\\
			& \leq L_f^K \|A(S^{i,K}) - A(S)\|\\
			& \leq L_f^K \epsilon_{K}
		\end{aligned}
	\end{equation*}
	
	For the second term, we have 
	\begin{equation*}
		\begin{aligned}
			& \EX_{S,A}[\frac{1}{n_K}\sum_{i_K=1}^{n_K}f_K^{\nu_{i_K}^{(K)}}(\EX_{\nu^{(K-1)}}[f_{K-1}^{\nu^{(K-1)}}]\cdots \EX_{\nu^{(1)}}[f_{1}^{\nu^{(1)}}(A(S))]) \\
			& ~~~~~~~~~~~~~~~~~~~~~~~~~~~~~~~~~~~~~~~~~~~~~~~~~~~ - \frac{1}{n_K}\sum_{i_K=1}^{n_K}f_K^{\nu_{i_K}^{(K)}}(\frac{1}{n_{K-1}}\sum_{i_{K-1} = 1}^{n_{K-1}}f_{K-1}^{\nu_{i_{K-1}}^{(K-1)}}\cdots \EX_{\nu^{(1)}}[f_{1}^{\nu^{(1)}}(A(S))])]\\
			&\leq L_f\EX_{S,A}[\|f_{K-1}(f_{K-2}\circ \cdots \circ f_1(A(S))) - \frac{1}{n_{K-1}}\sum_{i_{K-1}=1}^{n_{K-1}}f_{K-1}^{\nu_{i_{K-1}}^{(K-1)}}(f_{K-2}\circ \cdots \circ f_1(A(S)))\|].
		\end{aligned}
	\end{equation*}
	Besides, 
	\begin{equation*}
		\begin{aligned}
			&f_{K-1}(f_{K-2}\circ \cdots \circ f_1(A(S))) - \frac{1}{n_{K-1}}\sum_{i_{K-1}=1}^{n_{K-1}}f_{K-1}^{\nu_{i_{K-1}}^{(K-1)}}(f_{K-2}\circ \cdots \circ f_1(A(S)))\\
			& = \frac{1}{n_{K-1}}\sum_{j=1}^{n_{K-1}}\EX_{v^{(K-1)},v^{(K-1)'}_j}[f_{K-1}^{v^{(K-1)}}(f_{K-2}\circ \cdots \circ f_1(A(S))) - f_{K-1}^{v^{(K-1)}}(f_{K-2}\circ \cdots \circ f_1(A(S^{j,(K-1)})))]\\
			& + \frac{1}{n_{K-1}}\sum_{j=1}^{n_{K-1}}\EX_{v^{(K-1)'}_j}[\EX_{v^{(K-1)}}[f_{K-1}^{v^{(K-1)}}(f_{K-2}\circ \cdots \circ f_1(A(S^{j,(K-1)})))] \\
			& ~~~~~~~~~~~~~~~~~~~~~~~~~~~~~~~~~~~~~~~~~~~~~~~~~~~ ~~~~~~~~~~~~~~~~~~~~~~~~~~~~~~~~~~~~~ -  f_{K-1}^{v^{(K-1)}_j}(f_{K-2}\circ \cdots \circ f_1(A(S^{j,(K-1)}))) ]\\
			& +  \frac{1}{n_{K-1}}\sum_{j=1}^{n_{K-1}}\EX_{v^{(K-1)'}_j}[f_{K-1}^{v^{(K-1)}_j}(f_{K-2}\circ \cdots \circ f_1(A(S^{j,(K-1)}))) - f_{K-1}^{v^{(K-1)}_j}(f_{K-2}\circ \cdots \circ f_1(A(S)))].
		\end{aligned}
	\end{equation*}
	Note that $S$ and $S^{j,(K-1)}$ differ by a single example. By the assumption on stability and Definition \ref{def:stability}, we have 
	
	\begin{equation}\label{Eq:multi_term2_orgin}
		\begin{aligned}
			&\EX_{S,A}[\|f_{K-1}(f_{K-2}\circ \cdots \circ f_1(A(S))) - \frac{1}{n_{K-1}}\sum_{i_{K-1}=1}^{n_{K-1}}f_{K-1}^{\nu_{i_{K-1}}^{(K-1)}}(f_{K-2}\circ \cdots \circ f_1(A(S)))\|]\\
			& \leq 2L_f^{K-1} \epsilon_{K-1} + \EX_{S,A} [\frac{1}{n_{K-1}}\|\sum_{j=1}^{n_{K-1}}\EX_{v^{(K-1)'}_j}[\EX_{v^{(K-1)}}[f_{K-1}^{v^{(K-1)}}(f_{K-2}\circ \cdots \circ f_1(A(S^{j,(K-1)})))] \\
			& ~~~~~~~~~~~~~~~~~~~~~~~~~~~~~~~~~~~~~~~~~~~~~~~~~~~~~~~~~~~~~~~~~~~~~~~~~~~~~~~~~~~~~~~~ -  f_{K-1}^{v^{(K-1)}_j}(f_{K-2}\circ \cdots \circ f_1(A(S^{j,(K-1)})))]\|].  
		\end{aligned}
	\end{equation}
	Next step, we need to estimate the second term of above inequality. We denote
	\begin{equation*}
		\xi_j(S) = \EX_{v^{(K-1)'}_j}[\EX_{v^{(K-1)}}[f_{K-1}^{v^{(K-1)}}(f_{K-2}\circ \cdots \circ f_1(A(S^{j,(K-1)})))]  -  f_{K-1}^{v^{(K-1)}_j}(f_{K-2}\circ \cdots \circ f_1(A(S^{j,(K-1)})))].
	\end{equation*} 
	Notice that 
	\begin{equation*}
		\mathbb{E}_{S, A}[\|\sum_{j=1}^{n_{K-1}} \xi_{j}(S)\|^{2}]=\mathbb{E}_{S, A}[\sum_{j=1}^{n_{K-1}}\|\xi_{j}(S)\|^{2}]+\sum_{j, i \in[n_{K-1}]: j \neq i} \mathbb{E}_{S, A}[\langle\xi_{j}(S), \xi_{i}(S)\rangle].
	\end{equation*}
	Using Cauchy-Schwartz inequality, we have
	\begin{equation*}
		\begin{aligned}
			&\mathbb{E}_{S, A}[\sum_{j=1}^{n_{K-1}}\|\xi_{j}(S)\|^{2}]\\
			& = \mathbb{E}_{S, A}[\sum_{j=1}^{n_{K-1}}\| \EX_{v^{(K-1)'}_j}[\EX_{v^{(K-1)}}[f_{K-1}^{v^{(K-1)}}(f_{K-2}\circ \cdots \circ f_1(A(S^{j,(K-1)})))] \\
			& ~~~~~~~~~~~~~~~~~~~~~~~~~~~~~~~~~~~~~~~~~~~~~~~~~~~~~~~~~~~~~~~~~~~~~~~~~~~~~~~~~~~~~~~~-  f_{K-1}^{v^{(K-1)}_j}(f_{K-2}\circ \cdots \circ f_1(A(S^{j,(K-1)})))]\|^{2}] \\
			& \leq \mathbb{E}_{S, A}[\sum_{j=1}^{n_{K-1}}\| \EX_{v^{(K-1)}}[f_{K-1}^{v^{(K-1)}}(f_{K-2}\circ \cdots \circ f_1(A(S^{j,(K-1)})))]-  f_{K-1}^{v^{(K-1)}_j}(f_{K-2}\circ \cdots \circ f_1(A(S^{j,(K-1)})))\|^{2}]\\
			& = \mathbb{E}_{S, A}[\sum_{j=1}^{n_{K-1}}\| \EX_{v^{(K-1)}}[f_{K-1}^{v^{(K-1)}}(f_{K-2}\circ \cdots \circ f_1(A(S))]-  f_{K-1}^{v^{(K-1)}_j}(f_{K-2}\circ \cdots \circ f_1(A(S))\|^{2}]\\
			& = n_{K-1} \mathbb{E}_{S, A} [\operatorname{Var}_{K-1}(A(S)],
		\end{aligned}
	\end{equation*}
	where $\operatorname{Var}_{K-1}(A(S) =  \EX_{v^{(K-1)}}[\|f_{K-1}(f_{K-2}\circ \cdots \circ f_1(A(S))-  f_{K-1}^{v^{(K-1)}}(f_{K-2}\circ \cdots \circ f_1(A(S))\|^{2}] $.
	
	Next, we will estimate the  term $\sum_{j, i \in[n_{K-1}]: j \neq i} \mathbb{E}_{S, A}[\langle\xi_{j}(S), \xi_{i}(S)\rangle]$. We first define 
	\begin{equation*}
		\begin{aligned}
			S^{i,K-1} &= \{\nu^{(1)}_1,\cdots, \nu^{(1)}_{n_1}, \cdots,\nu^{(K-1)}_{i-1}, \nu^{(K-1)'}_i, \nu^{(K-1)}_{i+1}, \cdots, \nu^{(K-1)}_{n_{K-1}}, \nu^{(K)}_1,\cdots, \nu^{(K)}_{n_K} \}\\
			S^{i,j,K-1} &= \{\nu^{(1)}_1,\cdots, \nu^{(1)}_{n_1}, \cdots,\nu^{(K-1)}_{i-1}, \nu^{(K-1)'}_i, \nu^{(K-1)}_{i+1}, \cdots, \nu^{(K-1)}_{j-1}, \nu^{(K-1)'}_j, \nu^{(K-1)}_{j+1},\cdots, \nu^{(K)}_{n_{K}} \}.
		\end{aligned}
	\end{equation*}
	Due to the symmetry between $\nu^{(K-1)}$ and $\nu^{(K-1)}_j$, we have 
	\begin{equation*}
		\EX_{\nu^{(K-1)}_j}[\xi_j(S)] = 0, \forall j \in [1, n_{K-1}].
	\end{equation*}
	
	If $j\neq i$, then we have 
	\begin{equation*}
		\begin{aligned}
			\mathbb{E}_{S, A}[\langle\xi_{j}(S^{i, K-1}), \xi_{i}(S)\rangle] & =\mathbb{E}_{S, A} \mathbb{E}_{\nu^{(K-1)}_i}[\langle\xi_{j}(S^{i, K-1}), \xi_{i}(S)\rangle] \\
			& =\mathbb{E}_{S, A}[\langle\xi_{j}(S^{i, K-1}), \mathbb{E}_{\nu^{(K-1)}_i}[\xi_{i}(S)]\rangle]=0,
		\end{aligned}
	\end{equation*}
	In a similar way, we can get for any $j \neq i$
	\begin{equation*}
		\begin{aligned}
			\mathbb{E}_{S, A}[\langle \xi_{j}(S), \xi_{i}(S^{j, K-1})\rangle] & =\mathbb{E}_{S, A} \mathbb{E}_{\nu^{(K-1)}_j}[\langle \xi_{j}(S), \xi_{i}(S^{j, K-1}) \rangle] \\
			& =\mathbb{E}_{S, A}[\langle \mathbb{E}_{\nu^{(K-1)}_j}[\xi_{j}(S)], \xi_{i}(S^{j, K-1}) \rangle]=0,
		\end{aligned}
	\end{equation*}
	and 
	\begin{equation*}
		\begin{aligned}
			\mathbb{E}_{S, A}[\langle \xi_{j}(S^{i, K-1}), \xi_{i}(S^{j,K-1}) \rangle] & =\mathbb{E}_{S, A} \mathbb{E}_{\nu^{(K-1)}_j}[\langle \xi_{j}(S^{i,K-1}), \xi_{i}(S^{j, K-1}) \rangle] \\
			& =\mathbb{E}_{S, A}[\langle \mathbb{E}_{\nu^{(K-1)}_j}[\xi_{j}(S^{i,K-1})], \xi_{i}(S^{j, K-1}) \rangle]=0,
		\end{aligned}
	\end{equation*}
	Combining the above identities, we have for any $i\neq j$
	\begin{equation*}
		\begin{aligned}
			&\mathbb{E}_{S, A}[\langle\xi_{j}(S), \xi_{i}(S)\rangle]\\
			& = \mathbb{E}_{S, A}[\langle\xi_{j}(S)-\xi_{j}(S^{i, K-1}), \xi_{i}(S)-\xi_{i}(S^{j, K-1})\rangle] \\
			&\leq \mathbb{E}_{S, A}[\|\xi_{j}(S)-\xi_{j}(S^{i, K-1})\| \cdot\|\xi_{i}(S)-\xi_{i}(S^{j, K-1})\|] \\
			&\leq \frac{1}{2} \mathbb{E}_{S, A}[\|\xi_{j}(S)-\xi_{j}(S^{i, K-1})\|^{2}]+\frac{1}{2} \mathbb{E}_{S, A}[\|\xi_{i}(S)-\xi_{i}(S^{j, K-1})\|^{2}].
		\end{aligned}
	\end{equation*}
	Then 
	\begin{equation*}
		\begin{aligned}
			&\mathbb{E}_{S, A}[\|\xi_{j}(S)-\xi_{j}(S^{i, K-1})\|^{2}]\\
			& = 2\EX_{S,A}[\|f_{K-1}^{v^{(K-1)}}(f_{K-2}\circ \cdots \circ f_1(A(S^{j,(K-1)}))) - f_{K-1}^{v^{(K-1)}}(f_{K-2}\circ \cdots \circ f_1(A(S^{i,j,(K-1)})))\|^2]\\
			& \quad +  2\EX_{S,A}[\|f_{K-1}^{v^{(K-1)}_j}(f_{K-2}\circ \cdots \circ f_1(A(S^{i,j,(K-1)}))) - f_{K-1}^{v^{(K-1)}_j}(f_{K-2}\circ \cdots \circ f_1(A(S^{j,(K-1)})))\|^2]\\
			& \leq 4L_f^{K-1}\epsilon_{K-1}^2.
		\end{aligned}
	\end{equation*}
	In a similar way, we can have
	\begin{equation*}
		\mathbb{E}_{S, A}[\|\xi_{i}(S)-\xi_{i}(S^{j, K-1})\|^{2}] \leq 4L_f^{K-1}\epsilon_{K-1}^2.
	\end{equation*}
	According to the above inequalities, we have 
	\begin{equation*}
		\sum_{j, i \in[n_{K-1}]: j \neq i} \mathbb{E}_{S, A}[\langle\xi_{j}(S), \xi_{i}(S)\rangle] \leq 4(n_{K-1}-1)n_{K-1} L_f^{K-1}\epsilon_{K-1}^2, \forall j \neq i.
	\end{equation*}
	Then we have 
	\begin{equation*}
		\mathbb{E}_{S, A}[\|\sum_{j=1}^{n_{K-1}} \xi_{j}(S)\|^{2}] \leq 4(n_{K-1}-1)n_{K-1} L_f^{K-1}\epsilon_{K-1}^2 + n_{K-1} \mathbb{E}_{S, A} [\operatorname{Var}_{K-1}(A(S)].
	\end{equation*}
	Therefore
	\begin{equation}\label{Eq:multi_term2_orgin_term2}
		\mathbb{E}_{S, A}[\|\sum_{j=1}^{n_{K-1}} \xi_{j}(S)\|] \leq 2n_{K-1} L_f^{K-1}\epsilon_{K-1} + \sqrt{n_{K-1} \mathbb{E}_{S, A} [\operatorname{Var}_{K-1}(A(S)]}.
	\end{equation}
	Combining \eqref{Eq:multi_term2_orgin} and \eqref{Eq:multi_term2_orgin_term2} we have 
	\begin{equation*}
		\begin{aligned}
			&\EX_{S,A}[\|f_{K-1}(f_{K-2}\circ \cdots \circ f_1(A(S))) - \frac{1}{n_{K-1}}\sum_{i_{K-1}=1}^{n_{K-1}}f_{K-1}^{\nu_{i_{K-1}}^{(K-1)}}(f_{K-2}\circ \cdots \circ f_1(A(S)))\|]\\
			& \leq 4L_f^{K-1}\epsilon_{K-1} + \sqrt{\frac{ \mathbb{E}_{S, A} [\operatorname{Var}_{K-1}(A(S)]}{n_{K-1}}}.
		\end{aligned}
	\end{equation*}
	
	Then the second term 
	\begin{equation*}
		\begin{aligned}
			& \EX_{S,A}[\frac{1}{n_K}\sum_{i_K=1}^{n_K}f_K^{\nu_{i_K}^{(K)}}(\EX_{\nu^{(K-1)}}[f_{K-1}^{\nu^{(K-1)}}]\cdots \EX_{\nu^{(1)}}[f_{1}^{\nu^{(1)}}(A(S))]) \\
			& ~~~~~~~~~~~~~~~~~~~~~~~~~~~~~~~~~~~~~~~~~~~~~~~~~~~ - \frac{1}{n_K}\sum_{i_K=1}^{n_K}f_K^{\nu_{i_K}^{(K)}}(\frac{1}{n_{K-1}}\sum_{i_{K-1} = 1}^{n_{K-1}}f_{K-1}^{\nu_{i_{K-1}}^{(K-1)}}\cdots \EX_{\nu^{(1)}}[f_{1}^{\nu^{(1)}}(A(S))])]\\
			& \leq 4L_f^K\epsilon_{K-1} + L_f\sqrt{\frac{ \mathbb{E}_{S, A} [\operatorname{Var}_{K-1}(A(S)]}{n_{K-1}}},
		\end{aligned}
	\end{equation*}
	where $\operatorname{Var}_{K-1}(A(S) = \EX_{v^{(K-1)}}[\|f_{K-1}(f_{K-2}\circ \cdots \circ f_1(A(S))-  f_{K-1}^{v^{(K-1)}}(f_{K-2}\circ \cdots \circ f_1(A(S))\|^{2}]$.
	
	Similarly, we can get, for any $t \in [2,K]$,  the $t$-th term of \eqref{Eq:mulit_general_orgin} is bounded by 
	\begin{equation*}
		4L_f^K\epsilon_{K-t+1} + L_f\sqrt{\frac{ \mathbb{E}_{S, A} [\operatorname{Var}_{K-t+1}(A(S)]}{n_{K-t+1}}},
	\end{equation*}
	where $\operatorname{Var}_{K-t+1}(A(S) = \EX_{v^{(K-t+1)}}[\|f_{K-t+1}(f_{K-t}\circ \cdots \circ f_1(A(S))-  f_{K-t+1}^{v^{(K-t+1)}}(f_{K-t}\circ \cdots \circ f_1(A(S))\|^{2}]$.
	
	Then we can conclude that 
	\begin{equation*}
		\begin{aligned}
			&\EX_{S,A}[F(A(S))-F_S(A(S))]\leq L_f^K\epsilon_K + 4L_f^K\sum_{t=2}^K\epsilon_{K-t+1}+ L_f\sum_{t=2}^K\sqrt{\frac{ \mathbb{E}_{S, A} [\operatorname{Var}_{K-t+1}(A(S)]}{n_{K-t+1}}},
		\end{aligned}
	\end{equation*}
	where $\operatorname{Var}_{K-t+1}(A(S) = \EX_{v^{(K-t+1)}}[\|f_{K-t+1}(f_{K-t}\circ \cdots \circ f_1(A(S))-  f_{K-t+1}^{v^{(K-t+1)}}(f_{K-t}\circ \cdots \circ f_1(A(S))\|^{2}]$.
	
	This completes the proof.
	
\end{proof}

\subsection{Convex setting}\label{sec:mul_convex}
\begin{proof}[proof of Theorem \ref{thm:sta_convex_mul_level}]
	Since changing one sample data can happen in any layer of the function, we define 
	\begin{equation*}
		S^{l,k} = \{\nu^{(1)}_1,\cdots, \nu^{(1)}_{n_1}, \cdots, \nu^{(k)}_1,\cdots,\nu^{(k)}_{l-1}, \nu^{(k)'}_l, \nu^{(k)}_{l+1}, \cdots,\nu^{(k)}_{n_k},\cdots, \nu^{(K)}_1,\cdots, \nu^{(K)}_{n_K} \}.
	\end{equation*}
	
	Let $\{x_{t+1}\}$, $\{u_{t+1}^{(i)}\}$ and $\{v_{t+1}^{(i)}\}$ be produced by SVMR based on $S$, where $i \in [1,K]$ and represents an estimator of the function of layer $i$.  $\{x_{t+1}^{l,k}\}$, $\{u_{t+1}^{(i),l,k}\}$ and$\{v_{t+1}^{(i),l,k}\}$ be produced by SVMR based on $S^{l,k}$. For any $l \in [1,n_k],~ k \in [1, K]$, let $x_0=x_0^{l,k}$  be starting points in $\mathcal{X}$.
	
	We begin with the estimation of the term $\EX_A[\|x_{t+1} - x_{t+1}^{l,k}\|]$. For this purpose, we will consider two cases, $i_t \neq l$ and $i_t = l$. 
	
	\textbf{\quad Case 1($i_t\neq l$ ).}  We have 
	\begin{equation}\label{Eq:case1_mul_c}
		\begin{aligned}
			\|x_{t+1} - x_{t+1}^{l,k}\|^2 &\leq \|x_t - \eta_t\prod_{i=1}^K v_{t}^{(i)} - x_t^{l,k} + \eta_t \prod_{i=1}^K  v_{t}^{(i),l,k}\|^2 \\
			& \leq  \|x_t - x_t^{l,k}\|^2 - 2\eta_t \langle \prod_{i=1}^K v_{t}^{(i)} - \prod_{i=1}^K  v_{t}^{(i),l,k}, x_t - x_{t+1}^{l,k} \rangle + \eta_t^2 \|\prod_{i=1}^K v_{t}^{(i)} - \prod_{i=1}^K  v_{t}^{(i),l,k}\|^2.
		\end{aligned}
	\end{equation}
	
	Now we  estimate the second term of above inequality.
	\begin{equation*}
		\begin{aligned}
			&- 2\eta_t \langle \prod_{i=1}^K v_{t}^{(i)} - \prod_{i=1}^K  v_{t}^{(i),l,k}, x_t - x_t^{l,k} \rangle\\
			& = -2\eta_t \langle \prod_{i=1}^{K} v_{t}^{(i)} - \nabla f_{1,S}(x_t) \cdot \prod_{i=2}^{K}v_{t}^{(i)}, x_t - x_t^{l,k} \rangle\\
			& -2\eta_t \langle \nabla f_{1,S}(x_t) \cdot \prod_{i=2}^{K}v_{t}^{(i)} - \nabla f_{1,S}(x_t) \cdot \prod_{i=3}^{K}v_{t}^{(i)} \cdot \nabla f_{2,S}(u_t^{(1)}), x_t - x_t^{l,k} \rangle\\
			& -2\eta_t \langle \nabla f_{1,S}(x_t) \cdot \prod_{i=3}^{K}v_{t}^{(i)} \cdot \nabla f_{2,S}(u_t^{(1)}) - \nabla f_{1,S}(x_t) \cdot \prod_{i=3}^{K}v_{t}^{(i)} \cdot \nabla f_{2,S}(f_{1,S}(x_t)), x_t - x_t^{l,k} \rangle \\
			& \vdots\\
			& -2\eta_t \langle \prod_{i=1}^{K} \nabla F_{i,S}(x_t)- \prod_{i=1}^{K} \nabla F_{i,S}(x_t^{l,k}), x_t - x_t^{l,k} \rangle\\
			&  -2\eta_t \langle \prod_{i=1}^{K} \nabla F_{i,S}(x_t^{l,k}) -  \prod_{i=2}^{K} \nabla F_{i,S}(x_t^{l,k}) \cdot v_t^{(1),l,k} , x_t - x_t^{l,k} \rangle \\
			& -2\eta_t \langle   \prod_{i=2}^{K} \nabla F_{i,S}(x_t^{l,k}) \cdot v_t^{(1),l,k}  -  \prod_{i=3}^{K} \nabla F_{i,S}(x_t^{l,k}) \cdot v_t^{(1),l,k} \cdot \nabla f_{2,S}(u_t^{(1),l,k}) , x_t - x_t^{l,k} \rangle \\
			& \vdots\\
			&  -2\eta_t \langle  \prod_{i=1}^{K-1} v_t^{(i),l,k} \cdot \nabla f_{K,S}(u_t^{(K-1),l,k}) - \prod_{i=1}^K  v_{t}^{(i),l,k}, x_t - x_t^{l,k} \rangle.
		\end{aligned}
	\end{equation*}
	From the above inequality, we decompose it to $K(K+1)+1 \leq K(K+2)$ terms, where $K$  is the number of layers of the function. Using Assumption \ref{ass:Smoothness and Lipschitz continuous gradient} (iii) we can get 
	\begin{equation}\label{Eq:case1_mul_c_similar}
		\begin{aligned}
			&- 2\eta_t \langle \prod_{i=1}^K v_{t}^{(i)} - \prod_{i=1}^K  v_{t}^{(i),l,k}, x_t - x_t^{l,k} \rangle\\
			& \leq 2\eta_t L_f^{K-1}\|v_t^{(1)} - \nabla f_{1,S}(x_t)\|\cdot \|x_t - x_t^{l,k}\|\\
			& \quad+ (2\eta_t L_f^{K-1}\|v_t^{(2)}-\nabla f_{2,S}(u_t^{(1)})\|  + 2\eta_t L_f^{K}\|u_t^{(1)} - f_{1,S}(x_t)\|) \cdot \|x_t - x_t^{l,k}\|\\ 
			& \quad\vdots\\
			& \quad+ (2\eta_t L_f^{m_2}\|v_t^{(K)}-\nabla f_{K,S}(u_t^{(K-1)})\| +\cdots + 2\eta_t L_f^{K-1 + (K-1)K/2}\|u_t^{(1)} - f_{1,S}(x_t)\|) \cdot \|x_t - x_t^{l,k}\| \\
			&\quad- \frac{2\eta_t}{L}\| \prod_{i=1}^{K} \nabla F_{i,S}(x_t)- \prod_{i=1}^{K} \nabla F_{i,S}(x_t^{l,k})\|\\
			&\quad +( 2\eta_t L_f^{K-1 + (K-1)K/2} \| u_t^{(1),l,k}  x_t^{l,k}\| + \cdots  2\eta_t  L_f^{(K-1)K/2} \|  \nabla f_{K,S}(u_t^{(K-1),l,k})   -    \nabla v_t^{(K),l,k} \|) \cdot \|x_t - x_t^{l,k}\| \\
			&\quad\vdots\\
			&\quad + ( 2\eta_t L_f^{K}\|u_t^{(1),l,k} - f_{1,S}(x_t^{l,k})\| +  2\eta_t L_f^{K-1}\|v_t^{(2),l,k}-\nabla f_{2,S}(u_t^{(1),l,k})\|  ) \cdot \|x_t - x_t^{l,k}\|  \\
			& \quad +  2\eta_t L_f^{K-1}\|v_t^{(1),l,k} - \nabla f_{1,S}(x_t^{l,k})\|\cdot \|x_t - x_t^{l,k}\|.\\
		\end{aligned}
	\end{equation}
	Conclude above inequality, we have 
	\begin{equation*}
		\begin{aligned}
			&- 2\eta_t \langle \prod_{i=1}^K v_{t}^{(i)} - \prod_{i=1}^K  v_{t}^{(i),l,k}, x_t - x_t^{l,k} \rangle\\
			& \leq 2\eta_t \sum_{i=1}^{K}\sum_{j=1}^{i-1} L_f^{K-j+\frac{1}{2}(i-1)i} (\|u_{t}^{(j)} - f_{j,S}(u_{t}^{(j-1)})\| + \|u_{t}^{(j),l,k} - f_{j,S}(u_{t}^{(j-1),l,k})\|) \cdot   \|x_t - x_t^{l,k}\|\\
			&\quad + 2\eta_t \sum_{i=1}^K  L_f^{K-i + \frac{1}{2}(i-1)i}  (\|v_t^{(i)} - \nabla f_{i,S}(u_t^{(i-1)})\| + \|v_t^{(i),k,l} - \nabla f_{i,S}(u_t^{(i-1),k,l})\|)\cdot  \|x_t - x_t^{l,k}\|  \\
			&\quad - \frac{2\eta_t}{L}\| \prod_{i=1}^{K} \nabla F_{i,S}(x_t)- \prod_{i=1}^{K} \nabla F_{i,S}(x_t^{l,k})\|.
		\end{aligned}
	\end{equation*}
	Now we consider the third term of \eqref{Eq:case1_mul_c}. Similar to the \eqref{Eq:case1_mul_c_similar}, we have 
	\begin{equation*}
		\begin{aligned}
			&\|\prod_{i=1}^K v_{t}^{(i)} - \prod_{i=1}^K  v_{t}^{(i),l,k}\|\\
			& \leq  L_f^{K-1}\|v_t^{(1)} - \nabla f_{1,S}(x_t)\|\cdot \|x_t - x_t^{l,k}\|\\
			& \quad+ ( L_f^{K-1}\|v_t^{(2)}-\nabla f_{2,S}(u_t^{(1)})\|  +  L_f^{K}\|u_t^{(1)} - f_{1,S}(x_t)\|) \cdot \|x_t - x_t^{l,k}\|\\ 
			& \quad\vdots\\
			& \quad+ ( L_f^{m_2}\|v_t^{(K)}-\nabla f_{K,S}(u_t^{(K-1)})\| +\cdots +  L_f^{K-1 + (K-1)K/2}\|u_t^{(1)} - f_{1,S}(x_t)\|) \cdot \|x_t - x_t^{l,k}\| \\
			&\quad + \| \prod_{i=1}^{K} \nabla F_{i,S}(x_t)- \prod_{i=1}^{K} \nabla F_{i,S}(x_t^{l,k})\|\\
			&\quad +(  L_f^{K-1 + (K-1)K/2} \| u_t^{(1),l,k}  x_t^{l,k}\| + \cdots    L_f^{(K-1)K/2} \|  \nabla f_{K,S}(u_t^{(K-1),l,k})   -    \nabla v_t^{(K),l,k} \|) \cdot \|x_t - x_t^{l,k}\| \\
			&\quad\vdots\\
			&\quad + (  L_f^{K}\|u_t^{(1),l,k} - f_{1,S}(x_t^{l,k})\| +   L_f^{K-1}\|v_t^{(2),l,k}-\nabla f_{2,S}(u_t^{(1),l,k})\|  ) \cdot \|x_t - x_t^{l,k}\|  \\
			& \quad +   L_f^{K-1}\|v_t^{(1),l,k} - \nabla f_{1,S}(x_t^{l,k})\|\cdot \|x_t - x_t^{l,k}\|.\\
		\end{aligned}
	\end{equation*}
	Taking square on both sides of the above inequality, we have that
	\begin{equation*}
		\begin{aligned}
			&\eta_t^2\|\prod_{i=1}^K v_{t}^{(i)} - \prod_{i=1}^K  v_{t}^{(i),l,k}\|^2\\
			& \leq  L_f^{2K-2} K(K+2) \eta_t^2 \|v_t^{(1)} - \nabla f_{1,S}(x_t)\|^2\\
			& \quad+ L_f^{2K-2} K(K+2) \eta_t^2 \|v_t^{(2)}-\nabla f_{2,S}(u_t^{(1)})\|^2  +  L_f^{2K} K(K+2) \eta_t^2 \|u_t^{(1)} - f_{1,S}(x_t)\|^2\\ 
			& \quad\vdots\\
			& \quad+  L_f^{m} K(K+2) \eta_t^2 \|v_t^{(K)}-\nabla f_{K,S}(u_t^{(K-1)})\|^2 +\cdots +  L_f^{K-1 + (K-1)K} K(K+2) \eta_t^2  \|u_t^{(1)} - f_{1,S}(x_t)\|^2 \\
			&\quad + K(K+2) \eta_t^2 \| \prod_{i=1}^{K} \nabla F_{i,S}(x_t)- \prod_{i=1}^{K} \nabla F_{i,S}(x_t^{l,k})\|^2\\
			&\quad + L_f^{2K-2 + (K-1)K} K(K+2) \eta_t^2 \| u_t^{(1),l,k}  x_t^{l,k}\|^2 + \cdots    L_f^{(K-1)K} K(K+2) \eta_t^2 \|  \nabla f_{K,S}(u_t^{(K-1),l,k})   -    \nabla v_t^{(K),l,k} \|^2\| \\
			&\quad\vdots\\
			&\quad +   L_f^{2K} K(K+2) \eta_t^2 \|u_t^{(1),l,k} - f_{1,S}(x_t^{l,k})\|^2 +   L_f^{2K-2} K(K+2) \eta_t^2 \|v_t^{(2),l,k}-\nabla f_{2,S}(u_t^{(1),l,k})\|^2   \\
			& \quad +   L_f^{2K-2} K(K+2) \eta_t^2 \|v_t^{(1),l,k} - \nabla f_{1,S}(x_t^{l,k})\|^2.\\
		\end{aligned}
	\end{equation*}
	where we have used the fact that $(\sum_{i=1}^Ka_i)^2  \leq \sum_{i=1}^K K a_i^2$. 
	Then we can conclude that 
	\begin{equation*}
		\begin{aligned}
			&\eta_t^2\|\prod_{i=1}^K v_{t}^{(i)} - \prod_{i=1}^K  v_{t}^{(i),l,k}\|^2\\
			& \leq \eta_t^2 \sum_{i=1}^{K}\sum_{j=1}^{i-1} L_f^{2K-2j+(i-1)i} (\|u_{t}^{(j)} - f_{j,S}(u_{t}^{(j-1)})\|^2 + \|u_{t}^{(j),l,k} - f_{j,S}(u_{t}^{(j-1),l,k})\|^2)\\
			&\quad + \eta_t^2 \sum_{i=1}^K  L_f^{2K-2i + (i-1)i}  (\|v_t^{(i)} - \nabla f_{i,S}(u_t^{(i-1)})\|^2 + \|v_t^{(i),k,l} - \nabla f_{i,S}(u_t^{(i-1),k,l})\|^2)  \\
			&\quad + (K+2)K\eta_t^2\| \prod_{i=1}^{K} \nabla F_{i,S}(x_t)- \prod_{i=1}^{K} \nabla F_{i,S}(x_t^{l,k})\|. 
		\end{aligned}
	\end{equation*}
	
	Putting above inequality into \eqref{Eq:case1_mul_c}, according to  $\eta_t \leq \frac{2}{LK(K+2)}$,  we have 
	\begin{equation*}
		\begin{aligned}
			&\|x_{t+1} - x_{t+1}^{l,k}\|^2\\
			& \leq  2\eta_t \sum_{i=1}^{K}\sum_{j=1}^{i-1} L_f^{K-j+\frac{1}{2}(i-1)i} (\|u_{t}^{(j)} - f_{j,S}(u_{t}^{(j-1)})\| + \|u_{t}^{(j),l,k} - f_{j,S}(u_{t}^{(j-1),l,k})\|) \cdot   \|x_t - x_t^{l,k}\|\\
			&\quad + 2\eta_t \sum_{i=1}^K  L_f^{K-i + \frac{1}{2}(i-1)i}  (\|v_t^{(i)} - \nabla f_{i,S}(u_t^{(i-1)})\| + \|v_t^{(i),k,l} - \nabla f_{i,S}(u_t^{(i-1),k,l})\|)\cdot  \|x_t - x_t^{l,k}\|  \\
			&\quad + \eta_t^2 \sum_{i=1}^{K}\sum_{j=1}^{i-1} L_f^{2K-2j+(i-1)i} (\|u_{t}^{(j)} - f_{j,S}(u_{t}^{(j-1)})\|^2 + \|u_{t}^{(j),l,k} - f_{j,S}(u_{t}^{(j-1),l,k})\|^2)\\
			&\quad + \eta_t^2 \sum_{i=1}^K  L_f^{2K-2i + (i-1)i}  (\|v_t^{(i)} - \nabla f_{i,S}(u_t^{(i-1)})\|^2 + \|v_t^{(i),k,l} - \nabla f_{i,S}(u_t^{(i-1),k,l})\|^2)  +  \|x_t - x_t^{l,k}\|^2.
		\end{aligned}
	\end{equation*}
	
	\textbf{\quad Case 2 ($i_t =  l$ ).}  We have   
	\begin{equation}\label{Eq:case2_mul_c}
		\begin{aligned}
			\|x_{t+1} - x_{t+1}^{l,k}\| & =  \|x_t - \eta_t\prod_{i=1}^K v_{t}^{(i)} - x_t^{l,k} + \eta_t \prod_{i=1}^K  v_{t}^{(i),l,k}\| \\
			& \leq  \|x_t - x_t^{l,k}\|  + \eta_t \|\prod_{i=1}^K v_{t}^{(i)} - \prod_{i=1}^K  v_{t}^{(i),l,k}\| \leq \|x_t - x_t^{l,k}\| + 2\eta_t L_f^{K}.
		\end{aligned}
	\end{equation}
	Therefore, we have 
	\begin{equation*}
		\|x_{t+1} - x_{t+1}^{l,k}\|^2 \leq \|x_t - x_t^{l,k}\|^2 + 4\eta_t L_f^{K} \|x_t - x_t^{l,k}\| + 4\eta_t^2L_f^{2K}.
	\end{equation*}
	Combining above two cases, we have
	\begin{equation*}
		\begin{aligned}
			&\|x_{t+1} - x_{t+1}^{l,k}\|^2 \\\
			& \leq  2\eta_t \sum_{i=1}^{K}\sum_{j=1}^{i-1} L_f^{K-j+\frac{1}{2}(i-1)i} (\|u_{t}^{(j)} - f_{j,S}(u_{t}^{(j-1)})\| + \|u_{t}^{(j),l,k} - f_{j,S}(u_{t}^{(j-1),l,k})\|) \cdot   \|x_t - x_t^{l,k}\|\\
			&\quad + 2\eta_t \sum_{i=1}^K  L_f^{K-i + \frac{1}{2}(i-1)i}  (\|v_t^{(i)} - \nabla f_{i,S}(u_t^{(i-1)})\| + \|v_t^{(i),k,l} - \nabla f_{i,S}(u_t^{(i-1),k,l})\|)\cdot  \|x_t - x_t^{l,k}\|  \\
			&\quad + \eta_t^2 \sum_{i=1}^{K}\sum_{j=1}^{i-1} L_f^{2K-2j+(i-1)i} (\|u_{t}^{(j)} - f_{j,S}(u_{t}^{(j-1)})\|^2 + \|u_{t}^{(j),l,k} - f_{j,S}(u_{t}^{(j-1),l,k})\|^2)\\
			&\quad + \eta_t^2 \sum_{i=1}^K  L_f^{2K-2i + (i-1)i}  (\|v_t^{(i)} - \nabla f_{i,S}(u_t^{(i-1)})\|^2 + \|v_t^{(i),k,l} - \nabla f_{i,S}(u_t^{(i-1),k,l})\|^2)  +  \|x_t - x_t^{l,k}\|^2\\
			&\quad +  4\eta_t L_f^{K} \|x_t - x_t^{l,k}\| \cdot \1_{[i_t = l]} + 4\eta_t^2L_f^{2K} \cdot \1_{[i_t = l]}.
		\end{aligned}
	\end{equation*}
	According to 
	\begin{equation*}
		\mathbb{E}_{A}[\|x_{t}-x_{t}^{l, k}\| \1_{[i_{t}=l]}]=\mathbb{E}_{A}[\|x_{t}-x_{t}^{l, k}\| \mathbb{E}_{i_{t}}[\1_{[i_{t}=l]}]]=\frac{1}{n_k} \mathbb{E}_{A}[\|x_{t}-x_{t}^{l, k}\|] \leq \frac{1}{n_k}(\mathbb{E}_{A}[\|x_{t}-x_{t}^{l, k}\|^{2}])^{1 / 2},
	\end{equation*}
	note that $\|x_{0}-x_{0}^{l, k}\|^2 = 0$, we have 
	\begin{equation*}
		\begin{aligned}
			&\EX_A\|x_{t+1} - x_{t+1}^{l,k}\|^2 \\
			& \leq  2\eta_t \sum_{i=1}^{K}\sum_{j=1}^{i-1} L_f^{K-j+\frac{1}{2}(i-1)i} (\EX_A\|u_{t}^{(j)} - f_{j,S}(u_{t}^{(j-1)})\|^2)^{1/2} \cdot   (\EX_A\|x_t - x_t^{l,k}\|^2)^{1/2}\\
			& \quad +  2\eta_t \sum_{i=1}^{K}\sum_{j=1}^{i-1} L_f^{K-j+\frac{1}{2}(i-1)i} (\EX_A \|u_{t}^{(j),l,k} - f_{j,S}(u_{t}^{(j-1),l,k})\|^2)^{1/2} \cdot  (\EX_A\|x_t - x_t^{l,k}\|^2)^{1/2}\\
			&\quad + 2\eta_t \sum_{i=1}^K  L_f^{K-i + \frac{1}{2}(i-1)i}  (\EX_A\|v_t^{(i)} - \nabla f_{i,S}(u_t^{(i-1)})\|^2)^{1/2}\cdot  (\EX_A\|x_t - x_t^{l,k}\|^2)^{1/2} \\
			&\quad + 2\eta_t \sum_{i=1}^K  L_f^{K-i + \frac{1}{2}(i-1)i}  ( \EX_A\|v_t^{(i),k,l} - \nabla f_{i,S}(u_t^{(i-1),k,l})\|^2)^{1/2}\cdot  (\EX_A\|x_t - x_t^{l,k}\|^2)^{1/2}  \\
			&\quad + \eta_t^2 \sum_{i=1}^{K}\sum_{j=1}^{i-1} L_f^{2K-2j+(i-1)i} (\EX_A\|u_{t}^{(j)} - f_{j,S}(u_{t}^{(j-1)})\|^2 + \EX_A\|u_{t}^{(j),l,k} - f_{j,S}(u_{t}^{(j-1),l,k})\|^2)\\
			&\quad + \eta_t^2 \sum_{i=1}^K  L_f^{2K-2i + (i-1)i}  (\EX_A\|v_t^{(i)} - \nabla f_{i,S}(u_t^{(i-1)})\|^2 + \EX_A\|v_t^{(i),k,l} - \nabla f_{i,S}(u_t^{(i-1),k,l})\|^2)  \\
			&\quad +  \EX_A\|x_t - x_t^{l,k}\|^2 +  \frac{4\eta_t L_f^{K}}{n_k} (\EX_A\|x_t - x_t^{l,k}\|^2)^{1/2}+ \frac{4\eta_t^2L_f^{2K}}{n_k} .
		\end{aligned}
	\end{equation*}
	Telescoping from 0 to $t-1$, according to $\|x_{0}-x_{0}^{l, k}\|^2 = 0$,  we have 
	\begin{equation*}
		\begin{aligned}
			&\EX_A\|x_{t} - x_{t}^{l,k}\|^2 \\
			& \leq  2 \sum_{s=1}^{t-1}\sum_{i=1}^{K}\sum_{j=1}^{i-1}\eta_s L_f^{K-j+\frac{1}{2}(i-1)i} (\EX_A\|u_{s}^{(j)} - f_{j,S}(u_{s}^{(j-1)})\|^2)^{1/2} \cdot   (\EX_A\|x_s - x_s^{l,k}\|^2)^{1/2}\\
			& \quad +  2 \sum_{s=1}^{t-1} \sum_{i=1}^{K}\sum_{j=1}^{i-1} \eta_s L_f^{K-j+\frac{1}{2}(i-1)i} (\EX_A \|u_{s}^{(j),l,k} - f_{j,S}(u_{s}^{(j-1),l,k})\|^2)^{1/2} \cdot  (\EX_A\|x_s - x_s^{l,k}\|^2)^{1/2}\\
			&\quad + 2 \sum_{s=1}^{t-1}\sum_{i=1}^K \eta_s L_f^{K-i + \frac{1}{2}(i-1)i}  (\EX_A\|v_s^{(i)} - \nabla f_{i,S}(u_s^{(i-1)})\|^2)^{1/2}\cdot  (\EX_A\|x_s - x_s^{l,k}\|^2)^{1/2} \\
			&\quad + 2  \sum_{s=1}^{t-1} \sum_{i=1}^K \eta_s L_f^{K-i + \frac{1}{2}(i-1)i}  ( \EX_A\|v_s^{(i),k,l} - \nabla f_{i,S}(u_s^{(i-1),k,l})\|^2)^{1/2}\cdot  (\EX_A\|x_s - x_s^{l,k}\|^2)^{1/2}  \\
			&\quad + \sum_{s=1}^{t-1} \sum_{i=1}^{K}\sum_{j=1}^{i-1} \eta_s^2 L_f^{2K-2j+(i-1)i} (\EX_A\|u_{s}^{(j)} - f_{j,S}(u_{s}^{(j-1)})\|^2 + \EX_A\|u_{s}^{(j),l,k} - f_{j,S}(u_{s}^{(j-1),l,k})\|^2)\\
			&\quad + \sum_{s=1}^{t-1}  \sum_{i=1}^K \eta_s^2 L_f^{2K-2i + (i-1)i}  (\EX_A\|v_s^{(i)} - \nabla f_{i,S}(u_s^{(i-1)})\|^2 + \EX_A\|v_s^{(i),k,l} - \nabla f_{i,S}(u_s^{(i-1),k,l})\|^2)  \\
			&\quad +  \sum_{s=1}^{t-1} \frac{4\eta_s L_f^{K}}{n_k} (\EX_A\|x_s - x_s^{l,k}\|^2)^{1/2}+  \sum_{s=1}^{t-1} \frac{4\eta_s^2L_f^{2K}}{n_k} .
		\end{aligned}
	\end{equation*} 
	Similarly, for notational convenience, denote $u_t = (\EX_A\|x_{t} - x_{t}^{l,k}\|^2)^{1/2}$, and letting 
	\begin{equation*}
		\begin{aligned}
			S_t &= \sum_{s=1}^{t-1} \sum_{i=1}^{K}\sum_{j=1}^{i-1} \eta_s^2 L_f^{2K-2j+(i-1)i}(\EX_A\|u_{s}^{(j)} - f_{j,S}(u_{s}^{(j-1)})\|^2 + \EX_A\|u_{s}^{(j),l,k} - f_{j,S}(u_{s}^{(j-1),l,k})\|^2)\\
			&\quad + \sum_{s=1}^{t-1}  \sum_{i=1}^K \eta_s^2 L_f^{2K-2i + (i-1)i}  (\EX_A\|v_s^{(i)} - \nabla f_{i,S}(u_s^{(i-1)})\|^2 + \EX_A\|v_s^{(i),k,l} - \nabla f_{i,S}(u_s^{(i-1),k,l})\|^2)  \\
			&\quad +  \sum_{s=1}^{t-1} \frac{4\eta_s^2L_f^{2K}}{n_k} ,\\
			\alpha_s &=  \frac{4\eta_s L_f^{K}}{n_k}  + 2 \sum_{i=1}^{K}\sum_{j=1}^{i-1}\eta_s L_f^{K-j+\frac{1}{2}(i-1)i} (\EX_A\|u_{s}^{(j)} - f_{j,S}(u_{s}^{(j-1)})\|^2)^{1/2} \\
			& \quad +  2  \sum_{i=1}^{K}\sum_{j=1}^{i-1} \eta_s L_f^{K-j+\frac{1}{2}(i-1)i} (\EX_A \|u_{s}^{(j),l,k} - f_{j,S}(u_{s}^{(j-1),l,k})\|^2)^{1/2} \\
			&\quad + 2 \sum_{i=1}^K \eta_s L_f^{K-i + \frac{1}{2}(i-1)i}  (\EX_A\|v_s^{(i)} - \nabla f_{i,S}(u_s^{(i-1)})\|^2)^{1/2} \\
			&\quad + 2  \sum_{i=1}^K \eta_s L_f^{K-i + \frac{1}{2}(i-1)i}  ( \EX_A\|v_s^{(i),k,l} - \nabla f_{i,S}(u_s^{(i-1),k,l})\|^2)^{1/2}.
		\end{aligned}
	\end{equation*}
	According to Lemma \ref{lemma:recursion lemma}, we have 
	\begin{equation*}
		\begin{aligned}
			u_t &\leq \sqrt{S_t} + \sum_{s=1}^{t-1}\alpha_s\\
			& \leq \sum_{s=1}^{t-1} \sum_{i=1}^{K}\sum_{j=1}^{i-1} \eta_s L_f^{K-j+(i-1)i/2} ((\EX_A\|u_{s}^{(j)} - f_{j,S}(u_{s}^{(j-1)})\|^2)^{1/2} +( \EX_A\|u_{s}^{(j),l,k} - f_{j,S}(u_{s}^{(j-1),l,k})\|^2)^{1/2})\\
			&\quad + \sum_{s=1}^{t-1}  \sum_{i=1}^K \eta_s L_f^{K-i + (i-1)i/2}  ((\EX_A\|v_s^{(i)} - \nabla f_{i,S}(u_s^{(i-1)})\|^2)^{1/2} +( \EX_A\|v_s^{(i),k,l} - \nabla f_{i,S}(u_s^{(i-1),k,l})\|^2)^{1/2})  \\
			&\quad+  (\sum_{s=1}^{t-1} \frac{4\eta_s^2L_f^{2K}}{n_k})^{1/2} + \sum_{s=1}^{t-1}\frac{4\eta_s L_f^{K}}{n_k}  + 2 \sum_{s=1}^{t-1} \sum_{i=1}^{K}\sum_{j=1}^{i-1}\eta_s L_f^{K-j+\frac{1}{2}(i-1)i} (\EX_A\|u_{s}^{(j)} - f_{j,S}(u_{s}^{(j-1)})\|^2)^{1/2} \\
			& \quad +  2 \sum_{s=1}^{t-1} \sum_{i=1}^{K}\sum_{j=1}^{i-1} \eta_s L_f^{K-j+\frac{1}{2}(i-1)i} (\EX_A \|u_{s}^{(j),l,k} - f_{j,S}(u_{s}^{(j-1),l,k})\|^2)^{1/2} \\
			&\quad + 2 \sum_{s=1}^{t-1} \sum_{i=1}^K \eta_s L_f^{K-i + \frac{1}{2}(i-1)i}  (\EX_A\|v_s^{(i)} - \nabla f_{i,S}(u_s^{(i-1)})\|^2)^{1/2} \\
			&\quad + 2 \sum_{s=1}^{t-1} \sum_{i=1}^K \eta_s L_f^{K-i + \frac{1}{2}(i-1)i}  ( \EX_A\|v_s^{(i),k,l} - \nabla f_{i,S}(u_s^{(i-1),k,l})\|^2)^{1/2},
		\end{aligned}
	\end{equation*}
	where the inequality holds by $(\sum_{i=1}^{K}a_i)^{1/2} \leq \sum_{i=1}^{K}(a_i)^{1/2}$. Besides, if we let $\eta_t = \eta$, then it's easy to get 
	\begin{equation*}
		\begin{aligned}
			&\sum_{s=1}^{t-1} \sum_{i=1}^{K}\sum_{j=1}^{i-1} \eta_s L_f^{K-j+(i-1)i/2} (\EX_A\|u_{s}^{(j)} - f_{j,S}(u_{s}^{(j-1)})\|^2)^{1/2}\\
			& ~~~~~~~~~~~~~~~~~~~~~~~~~~~~~~~~  \leq \sup_S \eta \sum_{s=1}^{t-1} \sum_{i=1}^{K}\sum_{j=1}^{i-1}  L_f^{K-j+(i-1)i/2} (\EX_A\|u_{s}^{(j)} - f_{j,S}(u_{s}^{(j-1)})\|^2)^{1/2},
		\end{aligned}
	\end{equation*}
	and 
	\begin{equation*}
		\begin{aligned}
			&\sum_{s=1}^{t-1} \sum_{i=1}^{K}\sum_{j=1}^{i-1} \eta_s L_f^{K-j+\frac{1}{2}(i-1)i} (\EX_A \|u_{s}^{(j),l,k} - f_{j,S}(u_{s}^{(j-1),l,k})\|^2)^{1/2}\\
			& ~~~~~~~~~~~~~~~~~~~~~~~~~~~~~~~~  \leq \sup_S \eta \sum_{s=1}^{t-1} \sum_{i=1}^{K}\sum_{j=1}^{i-1}  L_f^{K-j+(i-1)i/2} (\EX_A\|u_{s}^{(j)} - f_{j,S}(u_{s}^{(j-1)})\|^2)^{1/2}.
		\end{aligned}
	\end{equation*}
	This inequality is also true for $v_s^{(j)}$ and $v_s^{(j),l,k}$.  Consequently, with $T$ iterations, we obtain that
	\begin{equation*}
		\begin{aligned}
			u_T &\leq  6\sup_S \eta \sum_{s=1}^{T-1} \sum_{i=1}^{K}\sum_{j=1}^{i-1}  L_f^{K-j+(i-1)i/2} (\EX_A\|u_{s}^{(j)} - f_{j,S}(u_{s}^{(j-1)})\|^2)^{1/2} \\
			&\quad +  6 \sup_S \eta \sum_{s=1}^{T-1} \sum_{i=1}^K  L_f^{K-i + \frac{1}{2}(i-1)i}  (\EX_A\|v_s^{(i)} - \nabla f_{i,S}(u_s^{(i-1)})\|^2)^{1/2} +  (\sum_{s=1}^{T-1} \frac{4\eta_s^2L_f^{2K}}{n_k})^{1/2} + \frac{4\eta L_f^{K} T}{n_k}\\
			& \leq 6\sup_S \eta \sum_{s=1}^{T-1} \sum_{i=1}^{K}\sum_{j=1}^{i-1}  L_f^{K-j+(i-1)i/2} (\EX_A\|u_{s}^{(j)} - f_{j,S}(u_{s}^{(j-1)})\|^2)^{1/2} \\
			&\quad +  6 \sup_S \eta \sum_{s=1}^{T-1} \sum_{i=1}^K  L_f^{K-i + \frac{1}{2}(i-1)i}  (\EX_A\|v_s^{(i)} - \nabla f_{i,S}(u_s^{(i-1)})\|^2)^{1/2}  + \frac{6\eta L_f^{K} T}{n_k},
		\end{aligned}
	\end{equation*}
	where the last inequality holds by the fact that we often have $T \geq n_{k}$, for any $k \in [1, K]$. Besides 
	\begin{equation*}
		\EX_A[\|x_T - X_T^{l,k}\|] \leq u_T  = (\EX_A[\|x_T - X_T^{l,k}\|]^{2})^{1/2}.
	\end{equation*}
	Then we can get the result for the $k$-th layer
	\begin{equation*}
		\begin{aligned}
			\EX_A[\|x_T - X_T^{l,k}\|] & = O\Big(\sup_S \eta \sum_{s=1}^{T-1} \sum_{i=1}^{K}\sum_{j=1}^{i-1}  L_f^{K-j+(i-1)i/2} (\EX_A\|u_{s}^{(j)} - f_{j,S}(u_{s}^{(j-1)})\|^2)^{1/2} \\
			&\quad ~~~~~~~~~~~~~~~+  \sup_S \eta \sum_{s=1}^{T-1} \sum_{i=1}^K  L_f^{K-i + \frac{1}{2}(i-1)i}  (\EX_A\|v_s^{(i)} - \nabla f_{i,S}(u_s^{(i-1)})\|^2)^{1/2}  + \frac{\eta L_f^{K} T}{n_k}\Big),
		\end{aligned}
	\end{equation*}
	where $k \in [1,K]$. Then we have 
	\begin{equation}\label{Eq:eplison_mul_sum}
		\begin{aligned}
			\sum_{k=1}^K \EX_{A}[\|\epsilon_k\|] & = O(\sup_S \eta \sum_{s=1}^{T-1} \sum_{i=1}^{K}\sum_{j=1}^{i-1}  L_f^{K-j+(i-1)i/2} (\EX_A\|u_{s}^{(j)} - f_{j,S}(u_{s}^{(j-1)})\|^2)^{1/2} \\
			&\quad +   \sup_S \eta \sum_{s=1}^{T-1} \sum_{i=1}^K  L_f^{K-i + \frac{1}{2}(i-1)i}  (\EX_A\|v_s^{(i)} - \nabla f_{i,S}(u_s^{(i-1)})\|^2)^{1/2}  + \sum_{k=1}^K\frac{\eta L_f^{K} T}{n_k}).
		\end{aligned}
	\end{equation} 
	This completes the proof.
\end{proof}

\begin{corollary}[$K$-level Optimization]\label{cor:1_mul_level}
	Consider SVMR in Algorithm~\ref{alg_svmr_multi} with $0 < \eta_t = \eta < 2/LK(K+1)$ and let $ 0< \beta_{t} =\beta < \max{\big\{1, \frac{1}{(4K\sum_{i=1}^K(2L_f^2)^{i}}\big\}}$ for any $t\in [0,T-1]$. With the output  $A(S) =x_{T}$, then we have
	\begin{equation*}
		\begin{aligned}
		\sum_{k=1}^K \epsilon_k = 	O\bigg( \eta T \big(  (\beta T)^{-\frac{c}{2}} + \beta^{1/2} + \eta \beta^{-1/2}\big) + \eta T \sum_{k=1}^K\frac{1}{n_k} \bigg).
		\end{aligned}
	\end{equation*}
\end{corollary}

Now we give the proof of Corollary \ref{cor:1_mul_level}.
\begin{proof}[proof of Corollary \ref{cor:1_mul_level}.]
	According to \eqref{Eq:eplison_mul_sum}, we have 
	\begin{equation*}
		\begin{aligned}
			\sum_{k=1}^K \epsilon_k & = O(\sup_S \eta \sum_{s=1}^{T-1} \sum_{i=1}^{K}\sum_{j=1}^{i-1}  (\EX_A\|u_{s}^{(j)} - f_{j,S}(u_{s}^{(j-1)})\|^2)^{1/2} \\
			&\quad +  \sup_S \eta \sum_{s=1}^{T-1} \sum_{i=1}^K   (\EX_A\|v_s^{(i)} - \nabla f_{i,S}(u_s^{(i-1)})\|^2)^{1/2}  + \sum_{k=1}^K\frac{6\eta L_f^{K} T}{n_k}).
		\end{aligned}
	\end{equation*} 
	
	According to Lemma \ref{lemma:u_t_bound_mul} and \ref{lemma:v_t_bound_mul}, we can get  
	\begin{equation*}
		\begin{aligned}
			\sum_{k=1}^K \epsilon_k & = O(\sum_{k=1}^K\frac{\eta T}{n_k} + \eta \sum_{s=1}^{T-1}((s\beta)^{-c/2} + \frac{\eta}{\sqrt{\beta}} + \sqrt{\beta})  )\\
			& = O(\sum_{k=1}^K\frac{\eta T}{n_k} + \eta T^{-c / 2+1} \beta^{-c / 2}+\eta^{2} \beta^{-\frac{1}{2}} T+\eta \beta^{1 / 2} T ).
		\end{aligned}
	\end{equation*}
	This complete the proof.
\end{proof}

Before give the detailed proof of Theorem \ref{thm:opt_convex}, we first give a useful lemma.
\begin{lemma}\label{lemma:opt_mul_c}
	Let Assumption \ref{ass:Lipschitz continuous}(iii), \ref{ass:bound variance} (iii) and \ref{ass:Smoothness and Lipschitz continuous gradient} (iii) hold for the empirical risk $F_S$, for SVMR, we have for any $\gamma_t >0$ and $\lambda_t > 0$ we have
	\begin{equation*}
		\begin{aligned}
			&\EX_A[\|x_{t+1} - x_*^S\|^2 | \F_t] \\
			& \leq \|x_t -  x_*^S\|^2 + L_f^K\eta_t^2 - 2\eta_t (F_S(x_t) - F_S(x_*^S)) +  \gamma_t \eta_t \sum_{i=1}^{K}\sum_{j=1}^{i-1} L_f^{K-j+\frac{1}{2}(i-1)i}\|x_t -  x_*^S\|^2 \\
			&\quad + \frac{\eta_t \sum_{i=1}^{K}\sum_{j=1}^{i-1} L_f^{K-j+\frac{1}{2}(i-1)i}\EX_{A}[ \|u_{t}^{(j)} - f_{j,S}(u_{t}^{(j-1)})\|^2|\F_t] }{\gamma_t} \\
			&\quad + \frac{\eta_t \sum_{i=1}^K  L_f^{K-i + \frac{1}{2}(i-1)i} \EX_{A}[ \|v_t^{(i)} - \nabla f_{i,S}(u_t^{(i-1)})\|^2|\F_t]}{\lambda_t} + \lambda_t \eta_t \sum_{i=1}^K  L_f^{K-i + \frac{1}{2}(i-1)i} \|x_t -  x_*^S\|^2.
		\end{aligned}
	\end{equation*}
	
\end{lemma}
\begin{proof}
	According to the update rule of SVMR, we have 
	\begin{equation*}
		\begin{aligned}
			\|x_{t+1} - x_*^S\|^2 & = \|x_t - \eta_t \prod_{i=1}^K v_t^{(i)} - x_*^S\|^2\\
			& \leq \|x_t -  x_*^S\|^2 + L_f^K\eta_t^2 - 2\eta_t \langle x_t -  x_*^S, \prod_{i=1} ^{K} \nabla F_{i,S}(x_t)\rangle + u_t ,
		\end{aligned}
	\end{equation*}
	where $u_t = 2\eta_t \langle x_t -  x_*^S, \prod_{i=1} ^{K} \nabla F_{i,S}(x_t) - \prod_{i=1}^K v_t^{(i)} \rangle$.
	Let $\F_t$ be the $\sigma$ field generated by $S$. Taking expectation with respect
	to the internal randomness of the algorithm and using Assumption \ref{ass:Lipschitz continuous} (iii), we have 
	\begin{equation*}
		\begin{aligned}
			&\EX_A[\|x_{t+1} - x_*^S\|^2 \|\F_t] \\
			& \leq \|x_t -  x_*^S\|^2 + L_f^K\eta_t^2 - 2\eta_t \EX_{A}[ \langle  x_t -  x_*^S, \prod_{i=1} ^{K} \nabla F_{i,S}(x_t)\rangle | \F_t] + \EX_A[u_t|\F_t]\\
			& = \|x_t -  x_*^S\|^2 + L_f^K\eta_t^2 - 2\eta_t \langle x_t -  x_*^S, \nabla F_S(x_t)\rangle + \EX_A[u_t|\F_t]\\
			& \leq \|x_t -  x_*^S\|^2 + L_f^K\eta_t^2 - 2\eta_t (F_S(x_t) - F_S(x_*^S)) + \EX_A[u_t|\F_t],
		\end{aligned}
	\end{equation*}
	where the last inequality comes from the convexity of $F_S$. Now we handle the term $\EX_A[u_t|\F_t]$.
	\begin{equation*}
		\begin{aligned}
			u_t & = 2\eta_t \langle x_t -  x_*^S, \prod_{i=1} ^{K} \nabla F_{i,S}(x_t) - \prod_{i=1}^K v_t^{(i)} \rangle\\
			& =  2\eta_t \langle \prod_{i=1}^{K} \nabla F_{i,S}(x_t) -  \prod_{i=2}^{K} \nabla F_{i,S}(x_t) \cdot v_t^{(1)} , x_t -  x_*^S \rangle \\
			&  + 2\eta_t \langle   \prod_{i=2}^{K} \nabla F_{i,S}(x_t) \cdot v_t^{(1)}  -  \prod_{i=3}^{K} \nabla F_{i,S}(x_t) \cdot v_t^{(1)} \cdot \nabla f_{2,S}(u_t^{(1)}) , x_t -  x_*^S \rangle \\
			&  +2\eta_t \langle  \prod_{i=3}^{K} \nabla F_{i,S}(x_t) \cdot v_t^{(1)} \cdot \nabla f_{2,S}(u_t^{(1)}) -  \prod_{i=3}^{K} \nabla F_{i,S}(x_t) \cdot v_t^{(1)} \cdot  v_t^{(2)}, x_t -  x_*^S \rangle \\
			& \vdots\\
			&  + 2\eta_t \langle  \prod_{i=1}^{K-1} v_t^{(i)} \cdot \nabla f_{K,S}(u_t^{(K-1)}) - \prod_{i=1}^K  v_{t}^{(i)}, x_t -  x_*^S \rangle.
		\end{aligned}
	\end{equation*}
	Conclude above inequality, we have 
	\begin{equation*}
		\begin{aligned}
			&- 2\eta_t \langle \prod_{i=1}^K v_{t}^{(i)} - \prod_{i=1}^K  v_{t}^{(i),l,k}, x_t - x_t^{l,k} \rangle\\
			& \leq 2\eta_t \sum_{i=1}^{K}\sum_{j=1}^{i-1} L_f^{K-j+\frac{1}{2}(i-1)i}\|u_{t}^{(j)} - f_{j,S}(u_{t}^{(j-1)})\| \cdot   \|x_t - x_t^{l,k}\|\\
			&\quad + 2\eta_t \sum_{i=1}^K  L_f^{K-i + \frac{1}{2}(i-1)i} \|v_t^{(i)} - \nabla f_{i,S}(u_t^{(i-1)})\|\cdot  \|x_t - x_t^{l,k}\|.
		\end{aligned}
	\end{equation*}
	
	Conclude above inequality, for any  $\gamma_t >0$ and $\lambda_t > 0$ we have 
	\begin{equation*}
		\begin{aligned}
			&2\eta_t \langle x_t -  x_*^S, \prod_{i=1} ^{K} \nabla F_{i,S}(x_t) - \prod_{i=1}^K v_t^{(i)} \rangle\\
			& \leq 2\eta_t \sum_{i=1}^{K}\sum_{j=1}^{i-1} L_f^{K-j+\frac{1}{2}(i-1)i} \|u_{t}^{(j)} - f_{j,S}(u_{t}^{(j-1)})\|  \cdot   \|x_t -  x_*^S\|\\
			&\quad + 2\eta_t \sum_{i=1}^K  L_f^{K-i + \frac{1}{2}(i-1)i}  \|v_t^{(i)} - \nabla f_{i,S}(u_t^{(i-1)})\| \cdot  \|x_t -  x_*^S\| \\
			& \leq \frac{\eta_t \sum_{i=1}^{K}\sum_{j=1}^{i-1} L_f^{K-j+\frac{1}{2}(i-1)i} \|u_{t}^{(j)} - f_{j,S}(u_{t}^{(j-1)})\|^2 }{\gamma_t} +  \gamma_t \eta_t \sum_{i=1}^{K}\sum_{j=1}^{i-1} L_f^{K-j+\frac{1}{2}(i-1)i}\|x_t -  x_*^S\|^2\\
			&\quad + \frac{\eta_t \sum_{i=1}^K  L_f^{K-i + \frac{1}{2}(i-1)i}  \|v_t^{(i)} - \nabla f_{i,S}(u_t^{(i-1)})\|}{\lambda_t} + \lambda_t \eta_t \sum_{i=1}^K  L_f^{K-i + \frac{1}{2}(i-1)i} \|x_t -  x_*^S\|^2.
		\end{aligned}
	\end{equation*}
	Then we can get 
	\begin{equation*}
		\begin{aligned}
			&\EX_A[\|x_{t+1} - x_*^S\|^2 | \F_t] \\
			& \leq \|x_t -  x_*^S\|^2 + L_f^K\eta_t^2 - 2\eta_t (F_S(x_t) - F_S(x_*^S)) +  \gamma_t \eta_t \sum_{i=1}^{K}\sum_{j=1}^{i-1} L_f^{K-j+\frac{1}{2}(i-1)i}\|x_t -  x_*^S\|^2 \\
			&\quad + \frac{\eta_t \sum_{i=1}^{K}\sum_{j=1}^{i-1} L_f^{K-j+\frac{1}{2}(i-1)i}\EX_{A}[ \|u_{t}^{(j)} - f_{j,S}(u_{t}^{(j-1)})\|^2|\F_t] }{\gamma_t} \\
			&\quad + \frac{\eta_t \sum_{i=1}^K  L_f^{K-i + \frac{1}{2}(i-1)i} \EX_{A}[ \|v_t^{(i)} - \nabla f_{i,S}(u_t^{(i-1)})\|^2|\F_t]}{\lambda_t} + \lambda_t \eta_t \sum_{i=1}^K  L_f^{K-i + \frac{1}{2}(i-1)i} \|x_t -  x_*^S\|^2.
		\end{aligned}
	\end{equation*}
	This completes the proof.
\end{proof}

Then we give the detailed proof of Theorem \ref{thm:opt_convex}.
\begin{proof}[proof of Theorem \ref{thm:opt_convex}]
	According to Lemma \ref{lemma:opt_mul_c}, setting $\eta_t = \eta$, $\beta_{t} = \beta$ and $\lambda_t = \gamma_t = \sqrt{\beta} $,  we have 
	
	\begin{equation*}
		\begin{aligned}
			&\EX_A[\|x_{t+1} - x_*^S\|^2 ] \\
			& \leq \EX_A[\|x_t -  x_*^S\|^2] + L_f^K\eta^2 - 2\eta \EX_{A}[F_S(x_t) - F_S(x_*^S)] +  \eta\sqrt{\beta} \sum_{i=1}^{K}\sum_{j=1}^{i-1} L_f^{K-j+\frac{1}{2}(i-1)i}\EX_{A}[\|x_t -  x_*^S\|^2] \\
			&\quad + \frac{\eta \sum_{i=1}^{K}\sum_{j=1}^{i-1} L_f^{K-j+\frac{1}{2}(i-1)i}\EX_{A}[ \|u_{t}^{(j)} - f_{j,S}(u_{t}^{(j-1)})\|^2] }{\sqrt{\beta}} \\
			&\quad + \frac{\eta \sum_{i=1}^K  L_f^{K-i + \frac{1}{2}(i-1)i} \EX_{A}[ \|v_t^{(i)} - \nabla f_{i,S}(u_t^{(i-1)})\|^2]}{\sqrt{\beta}} + \sqrt{\beta} \eta \sum_{i=1}^K  L_f^{K-i + \frac{1}{2}(i-1)i} \EX_{A}[\|x_t -  x_*^S\|^2]\\
			& \leq \EX_A[\|x_t -  x_*^S\|^2] + L_f^K\eta^2 - 2\eta \EX_{A}[F_S(x_t) - F_S(x_*^S)]  \\
			&\quad + \frac{\eta K L_f^{m_1} \sum_{j=i}^{K-1}\EX_{A}[ \|u_{t}^{(j)} - f_{j,S}(u_{t}^{(j-1)})\|^2] }{\sqrt{\beta}} +  \frac{\eta L_f^{m_2} \sum_{i=1}^K  \EX_{A}[ \|v_t^{(i)} - \nabla f_{i,S}(u_t^{(i-1)})\|^2]}{\sqrt{\beta}}  \\
			&\quad +  \eta\sqrt{\beta} (K^2L_f^{m_1} + KL_f^{m_2})\EX_{A}[\|x_t -  x_*^S\|^2],
		\end{aligned}
	\end{equation*}
	where $L_f^{m_1} = \max\{ L_f^{K-j+\frac{1}{2}(i-1)i}\}$ for any $i,j\in[1,K]$ and $L_f^{m_2} = \max\{ L_f^{K-i + \frac{1}{2}(i-1)i}\}$ for any $i\in[1,K]$. Using Lemma \ref{lemma:u_t_bound_mul} and \ref{lemma:v_t_bound_mul} we have 
	\begin{equation*}
		\begin{aligned}
			&\EX_A[\|x_{t+1} - x_*^S\|^2 ] \\
			& \leq \EX_A[\|x_t -  x_*^S\|^2] + L_f^K\eta^2 - 2\eta \EX_{A}[F_S(x_t) - F_S(x_*^S)]  \\
			&\quad + \eta K L_f^{m_1} \Big(\sum_{i=1}^K (\frac{c}{e})^{c}(\frac{t\beta}{2})^{-c} \mathbb{E}[\|u_{1}^{(i)}- f_{i,S}(x_{0})\|^{2}]\\
			&~~~~~~~~~~~~~~~~~~~~~~~~~~~~~~~~~~~~~~~~~~~~~~~~~~~ + 4\beta \sigma_{f}^{2} K ((\sum_{i=1}^{K}(2 L_f^{2})^{i}) + 1) + \frac{2\sum_{i=1}^{K}(2 L_f^{2})^{i}\eta^2 L_f^K}{\beta}\Big)/\sqrt{\beta} \\
			&\quad+  \eta  L_f^{m_2} \Big(\sum_{i=1}^K (\frac{c}{e})^{c}(\frac{t\beta}{2})^{-c}  (\mathbb{E}[\|u_{1}^{(i)}- f_{i,S}(x_{0})\|^{2}] + \mathbb{E}[\|v_{1}^{(i)}- \nabla f_{i,S}(x_{0})\|^{2}] )\\
			&~~~~~~~~~~~~~~~~~~~~~~~~~~~~~~~~~~~~~~~~~~~~~~~~~~~ + \frac{4(\sum_{i=1}^{K}(2 L_f^{2})^{i})\eta^2 L_f^K}{\beta} +4 \beta K (\sigma_{f}^{2} + \sigma_{J}^{2} + 2 \sigma_{f}^2(\sum_{i=1}^{K}(2 L_f^{2})^{i}) ) \Big) / \sqrt{\beta}  \\
			&\quad +  \eta\sqrt{\beta} (K^2L_f^{m_1} + KL_f^{m_2})\EX_{A}[\|x_t -  x_*^S\|^2].
		\end{aligned}
	\end{equation*}
	Rearranging and telescoping the above inequality from 1 to $T$ we have
	
	\begin{equation*}
		\begin{aligned}
			&\sum_{t=1}^T 2\eta \EX_{A}[F_S(x_t) - F_S(x_*^S)] \\
			& \leq  D_x +  L_f^K\eta^2 T + \sum_{t=1}^T \eta K L_f^{m_1} \Big(\sum_{i=1}^K (\frac{c}{e})^{c}(\frac{t\beta}{2})^{-c} \mathbb{E}[\|u_{1}^{(i)}- f_{i,S}(x_{0})\|^{2}]\\
			&~~~~~~~~~~~~~~~~~~~~~~~~~~~~~~~~~~~~~~~~~~~~~~~~~~~ + 4\beta \sigma_{f}^{2} K ((\sum_{i=1}^{K}(2 L_f^{2})^{i}) + 1) + \frac{2\sum_{i=1}^{K}(2 L_f^{2})^{i}\eta^2 L_f^K}{\beta}\Big)/\sqrt{\beta} \\
			&\quad+ \sum_{t=1}^T \eta L_f^{m_2}  \Big(\sum_{i=1}^K (\frac{c}{e})^{c}(\frac{t\beta}{2})^{-c}  (\mathbb{E}[\|u_{1}^{(i)}- f_{i,S}(x_{0})\|^{2}] + \mathbb{E}[\|v_{1}^{(i)}- \nabla f_{i,S}(x_{0})\|^{2}] )\\
			&~~~~~~~~~~~~~~~~~~~~~~~~~~~~~~~~~~~~~~~~~~~~~~~~~~~ + \frac{4(\sum_{i=1}^{K}(2 L_f^{2})^{i})\eta^2 L_f^K}{\beta} +4 \beta K (\sigma_{f}^{2} + \sigma_{J}^{2} + 2\sigma_{f}^2(\sum_{i=1}^{K}(2 L_f^{2})^{i}) )\Big) / \sqrt{\beta}  \\
			&\quad +  \eta\sqrt{\beta} (K^2L_f^{m_1} + KL_f^{m_2})D_x T.
		\end{aligned}
	\end{equation*}
	
	Then denote  $L_f^m = \max\{L_f^{m_1}, L_f^{m_2}\}$ we can get 
	
	\begin{equation}\label{Eq:fs_t-f_s*t_mul}
		\begin{aligned}
			&  \EX_{A}[F_S(x_T) - F_S(x_*^S)] \\
			& \leq O\Big( D_x (\eta T)^{-1} +  L_f^K\eta + (\sum_{i=1}^K U_{i})  L_f^{m} \beta^{-1/2 - c} T^{-1}\sum_{t=1}^T   t^{-c} +L_f^{m}   \sigma_{f}^{2}  ((\sum_{i=1}^{K}( L_f^{2})^{i}) + 1) \beta^{1/2}\\
			&\quad+  L_f^{m} \sum_{i=1}^{K}( L_f^{2})^{i}\eta^2 \beta^{-3/2} + (\sum_{i=1}^K U_{i}+V_{i})  L_f^{m} \beta^{-1/2 - c} T^{-1} \sum_{t=1}^T   t^{-c} + L_f^{m} (\sum_{i=1}^{K}( L_f^{2})^{i})\eta^2  \beta ^{-3/2} \\
			&\quad+  L_f^{m}   (\sigma_{f}^{2} + \sigma_{J}^{2} + \sigma_{f}^2(\sum_{i=1}^{K}( L_f^{2})^{i}) ) \beta^{1/2}   +  D_x L_f^{m} \beta^{1/2} \Big).
		\end{aligned}
	\end{equation}
	Noting that \(\sum_{t= 1}^{T} t^{-z}= O(T^{1- z})\) for \(z\in (0, 1)\cup (1, \infty)\) and \(\sum_{t= 1}^{T} t^{-1}= O(\log T)\), as long as \(c> 2\) we get
	
	\begin{equation*}
		\begin{aligned}
			&  \EX_{A}[F_S(x_T) - F_S(x_*^S)] \\
			& \leq O\Big( D_x (\eta T)^{-1} +  L_f^K\eta  +  L_f^{m}(\sum_{i=1}^K U_{i}+V_{i})  \beta^{-1/2 - c} T^{-c}\\
			&\quad + (  L_f^{m}   (\sigma_{f}^{2} + \sigma_{J}^{2} + \sigma_{f}^2(\sum_{i=1}^{K}( L_f^{2})^{i}) )   +  D_x L_f^m ) \beta^{1/2}  + L_f^{m} (\sum_{i=1}^{K}( L_f^{2})^{i})\eta^2  \beta ^{-3/2}\Big).
		\end{aligned}
	\end{equation*}
	This complete the proof.
	
\end{proof}

\begin{proof}[proof of Theorem \ref{thm:Excess_Risk_Bound_convex}]
	According to \eqref{Eq:eplison_mul_sum}, we have 
	
	\begin{equation*}
		\begin{aligned}
			\sum_{k=1}^K \|x_t - x_t ^{l,k}\| & \leq 6K\sup_S \eta \sum_{s=1}^{t-1} \sum_{i=1}^{K}\sum_{j=1}^{i-1}  L_f^{K-j+(i-1)i/2} (\EX_A\|u_{s}^{(j)} - f_{j,S}(u_{s}^{(j-1)})\|^2)^{1/2} \\
			&\quad +  6K \sup_S \eta \sum_{s=1}^{t-1} \sum_{i=1}^K  L_f^{K-i + \frac{1}{2}(i-1)i}  (\EX_A\|v_s^{(i)} - \nabla f_{i,S}(u_s^{(i-1)})\|^2)^{1/2}  + \sum_{k=1}^K\frac{6\eta L_f^{K} t}{n_k}\\
			&  \leq 6K^2 L_f^{m}\sup_S \eta \sum_{s=1}^{t-1}\sum_{j=1}^{K-1}  (\EX_A\|u_{s}^{(j)} - f_{j,S}(u_{s}^{(j-1)})\|^2)^{1/2}\\
			&\quad + 6KL_f^{m}\sup_S\sum_{s=1}^{t-1} \sum_{i=1}^K   (\EX_A\|v_s^{(i)} - \nabla f_{i,S}(u_s^{(i-1)})\|^2)^{1/2}  + \sum_{k=1}^K\frac{6\eta L_f^{K} t}{n_k}\\
		\end{aligned}
	\end{equation*} 
	Then using Lemma \ref{lemma:u_t_bound_mul} and \ref{lemma:v_t_bound_mul} we have

	\begin{equation*}
		\begin{aligned}
			&\sum_{k=1}^K \|x_t - x_t ^{l,k}\| \\
			& \leq  6K^2 L_f^{m}\sup_S \eta \sum_{s=1}^{t-1}  (\sum_{i=1}^K (\frac{c}{e})^{c}(\frac{s\beta}{2})^{-c} \mathbb{E}[\|u_{1}^{(i)}- f_{i,S}(x_{0})\|^{2}]+ 4\beta \sigma_{f}^{2} K ((\sum_{i=1}^{K}(2 L_f^{2})^{i}) + 1)\\
			&\quad~~~~~~~~~~~~~~~~~~~~~~~~~~~~~~~~~~~~~~~~~~~~~~~~ + \frac{2\sum_{i=1}^{K}(2 L_f^{2})^{i}\eta^2 L_f^K}{\beta})^{1/2}   + \sum_{k=1}^K\frac{6\eta L_f^{K} t}{n_k}\\
			& \quad + 6KL_f^{m}\sup_S\sum_{s=1}^{t-1}(\sum_{i=1}^K (\frac{c}{e})^{c}(\frac{s\beta}{2})^{-c}  (\mathbb{E}[\|u_{1}^{(i)}- f_{i,S}(x_{0})\|^{2}] + \mathbb{E}[\|v_{1}^{(i)}- \nabla f_{i,S}(x_{0})\|^{2}] ) \\
			&\quad~~~~~~~~~~~~~~~~~~~~~~~~~~~~~~~~~~~~~~~~~~~~~~~~+ \frac{4(\sum_{i=1}^{K}(2 L_f^{2})^{i})\eta^2 L_f^K}{\beta} +4 \beta K (\sigma_{f}^{2} + \sigma_{J}^{2} + 2\sigma_{f}^2(\sum_{i=1}^{K}(2 L_f^{2})^{i}) ) )^{1/2}.
		\end{aligned}
	\end{equation*} 
	Thus we get 
	\begin{equation*}
		\begin{aligned}
			&\sum_{k=1}^K\|x_t - x_t ^{l,k}\| \\
			& \leq  6K^2 L_f^{m} \eta \sqrt{\sum_{i=1}^KU_{i}(\frac{2c}{e})^c}\beta^{-c/2}\sum_{s=1}^t s^{-c/2} + 6K^2L_f^m \sqrt{2\sum_{i=1}^{K}(2 L_f^{2})^{i} L_f^K}\cdot \eta^2 \beta^{-1/2}t + \sum_{k=1}^K\frac{6\eta L_f^{K} t}{n_k}  \\
			& \quad  + 12 K^2 L_f^{m}\sqrt{  ((\sigma_{f}^{2} + \sigma_{J}^{2}) + 2(\sum_{i=1}^{K}(2 L_f^{2})^{i}) + \sigma_{f}^2) K} \cdot \beta^{1/2} t\eta  +  6KL_f^{m}\eta \sqrt{\sum_{i=1}^K(U_{i} + V_{i})(\frac{2c}{e})^c}\beta^{-c/2}\sum_{s=1}^t s^{-c/2} .
		\end{aligned}
	\end{equation*}
	According to Theorem \ref{theorem:general_multi_level}, we have 
	
	\begin{equation*}
		\begin{aligned}
			&\EX_{S,A}[F(x_t)-F_S(x_t))]\\
			&\leq L_f^K\epsilon_K+  4L_f^K\sum_{k=1}^{K-1}\|x_t - x_t ^{l,k}\|+ L_f\sum_{k=1}^{K-1}\sqrt{\frac{ \mathbb{E}_{S, A} [\operatorname{Var}_{k}(A(S)]}{n_{k}}}\\
			& \leq    24K^2 L_f^{m+K} \eta \sqrt{\sum_{i=1}^KU_{i}(\frac{2c}{e})^c}\beta^{-c/2}\sum_{s=1}^t s^{-c/2} + 24K^2 L_f^{m+K}\sqrt{2\sum_{i=1}^{K}(2 L_f^{2})^{i} L_f^K}\cdot \eta^2 \beta^{-1/2}t + \sum_{k=1}^K\frac{24\eta L_f^{2K} t}{n_k}  \\
			& \quad  + 48K^2 L_f^{m+K}\sqrt{  ((\sigma_{f}^{2} + \sigma_{J}^{2}) + 2(\sum_{i=1}^{K}(2 L_f^{2})^{i}) + \sigma_{f}^2) K} \cdot \beta^{1/2} t\eta \\
			&\quad +  24K^2 L_f^{m+K}\eta \sqrt{\sum_{i=1}^K(U_{i} + V_{i})(\frac{2c}{e})^c}\beta^{-c/2}\sum_{s=1}^t s^{-c/2}+ L_f\sum_{k=1}^{K-1}\sqrt{\frac{ \mathbb{E}_{S, A} [\operatorname{Var}_{k}(A(S)]}{n_{k}}}.
		\end{aligned}
	\end{equation*}
	According to \eqref{Eq:fs_t-f_s*t_mul}, we have 
	\begin{equation*}
		\begin{aligned}
			&\sum_{t=1}^T\EX_{S,A}[F(x_t)-F_S(x_*^S))]\\
			& \leq  D_x\eta^{-1} +  L_f^K\eta T + \sum_{t=1}^T  K L_f^{m} (\sum_{i=1}^K (\frac{c}{e})^{c}(\frac{t\beta}{2})^{-c} \mathbb{E}[\|u_{1}^{(i)}- f_{i,S}(x_{0})\|^{2}]\\
			&~~~~~~~~~~~~~~~~~~~~~~~~~~~~~~~~~~~~~~~~~~~~~~~~ + 4\beta \sigma_{f}^{2} K ((\sum_{i=1}^{K}(2 L_f^{2})^{i}) + 1) + \frac{2\sum_{i=1}^{K}(2 L_f^{2})^{i}\eta^2 L_f^K}{\beta})/\sqrt{\beta}  \\
			&\quad+ \sum_{t=1}^T  L_f^{m}  (\sum_{i=1}^K (\frac{c}{e})^{c}(\frac{t\beta}{2})^{-c}  (\mathbb{E}[\|u_{1}^{(i)}- f_{i,S}(x_{0})\|^{2}] + \mathbb{E}[\|v_{1}^{(i)}- \nabla f_{i,S}(x_{0})\|^{2}] )\\
			&~~~~~~~~~~~~~~~~~~~~~~~~~~~~~~~~~~~~~~~~~~~~~~~~ + \frac{4(\sum_{i=1}^{K}(2 L_f^{2})^{i})\eta^2 L_f^K}{\beta} +4 \beta  ((\sigma_{f}^{2} + \sigma_{J}^{2}) + 2(\sum_{i=1}^{K}(2 L_f^{2})^{i}) + \sigma_{f}^2) K) / \sqrt{\beta}  \\
			&\quad +  \sqrt{\beta} K^2 L_f^m D_x T +   24K^2 L_f^{m+K} \eta \sqrt{\sum_{i=1}^KU_{i}(\frac{2c}{e})^c}\beta^{-c/2}\sum_{s=1}^t s^{-c/2} \\
			&\quad+ 24K^2 L_f^{m+K}\sqrt{2\sum_{i=1}^{K}(2 L_f^{2})^{i} L_f^K}\cdot \eta^2 \beta^{-1/2}t + \sum_{k=1}^K\frac{24\eta L_f^{2K} t}{n_k}  \\
			& \quad  + 48K^2 L_f^{m+K}\sqrt{  ((\sigma_{f}^{2} + \sigma_{J}^{2}) + 2(\sum_{i=1}^{K}(2 L_f^{2})^{i}) + \sigma_{f}^2) K} \cdot \beta^{1/2} t\eta \\
			&\quad +  24K^2 L_f^{m+K}\eta \sqrt{\sum_{i=1}^K(U_{i} + V_{i})(\frac{2c}{e})^c}\beta^{-c/2}\sum_{s=1}^t s^{-c/2}+ L_f\sum_{k=1}^{K-1}\sqrt{\frac{ \mathbb{E}_{S, A} [\operatorname{Var}_{k}(A(S)]}{n_{k}}}.
		\end{aligned}
	\end{equation*}
	
	Noting that \(\sum_{t= 1}^{T} t^{-z}= O(T^{1- z})\) for \(z\in (-1, 0)\cup (-\infty, -1)\) and \(\sum_{t= 1}^{T} t^{-1}= O(\log T)\), we have
	\begin{equation*}
		\sum_{t= 1}^T \sum_{j= 1}^T j^{-\frac{c}{2}}= O(\sum_{t= 1}^T t^{1- \frac{c}{2}} (\log t)^{\1_{c= 2}})= O(T^{2- \frac{c}{2}} (\log T)^{\1_{c= 2}}).
	\end{equation*}
	Setting $\eta = T^{-a}$ and $\beta = T^{-b}$ we can get 
	
	\begin{equation*}
		\begin{aligned}
			&\sum_{t=1}^T\EX_{S,A}[F(x_t)-F_S(x_*^S))]\\
			& \leq O( T^{a}  + T^{1-a} + T^{1-(1-b)c +\frac{b}{2}}(\log T)^{\1_{c=1}} + T^{-b/2} + T^{1+3b/2-2a} + T^{2-a}\sum_{k=1}^Kn_k^{-1} + T^{1-b/2} + T^{2-a-b/2} \\
			& \qquad + T^{2-a-c/2(1-b)}(\log T)^{\1_{c=1}} + T^{2-2a+1/2b} + T\sum_{k=1}^K n_k^{-1/2}).
		\end{aligned}
	\end{equation*}
	Dividing both side of above inequality with $T$, then from the choice of $A(S)$ we have 
	\begin{equation*}
		\begin{aligned}
			&\EX_{S,A}[F(A(S))-F(x_*))]\\
			& \leq O( T^{a-1}  + T^{-a} + T^{-(1-b)c +\frac{b}{2}}(\log T)^{\1_{c=1}} + T^{-b/2-1} + T^{3b/2-2a} + T^{1-a}\sum_{k=1}^Kn_k^{-1} + T^{-b/2} + T^{1-a-b/2}\\
			&\qquad + T^{1-a-c/2(1-b)}(\log T)^{\1_{c=1}} + T^{1-2a+1/2b} + \sum_{k=1}^K n_k^{-1/2} ).
		\end{aligned}
	\end{equation*}
	As long as we have \(c> 4\), the dominating terms are $ O(T^{1- a- \frac{b}{2}})$, $ O(T^{1+ \frac{b}{2}- 2a})$,  $ O(T^{1-a}\sum_{k=1}^Kn_k^{-1})$, $  O(T^{a-1})$,  and $ O(T^{\frac{3}{2}b- 2a}).$ 
	Setting $a= b= \frac{4}{5}$, we have 
	\begin{equation*}
		\EX_{S,A}\Big[F(A(S)) - F(x_*) \Big]= O(T^{-\frac{1}{5}}+ T^{\frac{1}{5}}\sum_{k=1}^Kn_k^{-1} + \sum_{k=1}^K n_k^{-1/2}).
	\end{equation*}
	
	Letting  $T= O(\max\{n_1^{2.5},\cdots, n_K^{2.5}\})$ we have  the following 
	\begin{equation*}
		\EX_{S,A}\Big[F(A(S)) - F(x_*) \Big] = O(\sum_{k=1}^K n_k^{-1/2}).
	\end{equation*}
	This complete the proof.
\end{proof}

\subsection{Strongly Convex setting}
Similarly, since changing one sample data can happen in any layer of the function, we keep the same notations as in Section \ref{sec:mul_convex}.

\textbf{\quad Case 1($i_t\neq l$ ).}  We have 
\begin{equation}\label{Eq:case1_mul_sc}
	\begin{aligned}
		\|x_{t+1} - x_{t+1}^{l,k}\|^2 &\leq \|x_t - \eta_t\prod_{i=1}^K v_{t}^{(i)} - x_t^{l,k} + \eta_t \prod_{i=1}^K  v_{t}^{(i),l,k}\|^2 \\
		& \leq  \|x_t - x_t^{l,k}\|^2 - 2\eta_t \langle \prod_{i=1}^K v_{t}^{(i)} - \prod_{i=1}^K  v_{t}^{(i),l,k}, x_t - x_{t+1}^{l,k} \rangle + \eta_t^2 \|\prod_{i=1}^K v_{t}^{(i)} - \prod_{i=1}^K  v_{t}^{(i),l,k}\|^2.
	\end{aligned}
\end{equation}

Now we  estimate the second term of above inequality. we decompose it to $K(K+2)$ terms. According to the strongly convexity of $F_S(\cdot)$, we have 
\begin{equation*}
	\begin{aligned}
		&\langle \prod_{i=1}^{K} \nabla F_{i,S}(x_t)- \prod_{i=1}^{K} \nabla F_{i,S}(x_t^{l,k}), x_t - x_t^{l,k} \rangle\\
		& \geq  \frac{L\mu}{L+\mu}\|x_t - x_t^{l,k}\|^2 + \frac{1}{L+\mu} \|  \prod_{i=1}^{K} \nabla F_{i,S}(x_t)- \prod_{i=1}^{K} \nabla F_{i,S}(x_t^{l,k})\|^2.
	\end{aligned}
\end{equation*}

Using Assumption \ref{ass:Smoothness and Lipschitz continuous gradient} (iii) and strong convexity, similar to convex setting we can get 
\begin{equation}
	\begin{aligned}
		&- 2\eta_t \langle \prod_{i=1}^K v_{t}^{(i)} - \prod_{i=1}^K  v_{t}^{(i),l,k}, x_t - x_t^{l,k} \rangle\\
		& \leq 2\eta_t L_f^{K-1}\|v_t^{(1)} - \nabla f_{1,S}(x_t)\|\cdot \|x_t - x_t^{l,k}\|\\
		& \quad+ (2\eta_t L_f^{K-1}\|v_t^{(2)}-\nabla f_{2,S}(u_t^{(1)})\|  + 2\eta_t L_f^{K}\|u_t^{(1)} - f_{1,S}(x_t)\|) \cdot \|x_t - x_t^{l,k}\|\\ 
		& \quad\vdots\\
		& \quad+ (2\eta_t L_f^{m_2}\|v_t^{(K)}-\nabla f_{K,S}(u_t^{(K-1)})\| +\cdots + 2\eta_t L_f^{K-1 + (K-1)K/2}\|u_t^{(1)} - f_{1,S}(x_t)\|) \cdot \|x_t - x_t^{l,k}\| \\
		&\quad- \frac{2\eta_tL\mu}{L+\mu}\|x_t - x_t^{l,k}\|^2 - \frac{2\eta_t}{L+\mu} \|  \prod_{i=1}^{K} \nabla F_{i,S}(x_t)- \prod_{i=1}^{K} \nabla F_{i,S}(x_t^{l,k})\|^2\\
		&\quad +( 2\eta_t L_f^{K-1 + (K-1)K/2} \| u_t^{(1),l,k}  x_t^{l,k}\| + \cdots  2\eta_t  L_f^{(K-1)K/2} \|  \nabla f_{K,S}(u_t^{(K-1),l,k})   -    \nabla v_t^{(K),l,k} \|) \cdot \|x_t - x_t^{l,k}\| \\
		&\quad\vdots\\
		&\quad + ( 2\eta_t L_f^{K}\|u_t^{(1),l,k} - f_{1,S}(x_t^{l,k})\| +  2\eta_t L_f^{K-1}\|v_t^{(2),l,k}-\nabla f_{2,S}(u_t^{(1),l,k})\|  ) \cdot \|x_t - x_t^{l,k}\|  \\
		& \quad +  2\eta_t L_f^{K-1}\|v_t^{(1),l,k} - \nabla f_{1,S}(x_t^{l,k})\|\cdot \|x_t - x_t^{l,k}\|.\\
	\end{aligned}
\end{equation}
Conclude above inequality, we have 
\begin{equation*}
	\begin{aligned}
		&- 2\eta_t \langle \prod_{i=1}^K v_{t}^{(i)} - \prod_{i=1}^K  v_{t}^{(i),l,k}, x_t - x_t^{l,k} \rangle\\
		& \leq 2\eta_t \sum_{i=1}^{K}\sum_{j=1}^{i-1} L_f^{K-j+\frac{1}{2}(i-1)i} (\|u_{t}^{(j)} - f_{j,S}(u_{t}^{(j-1)})\| + \|u_{t}^{(j),l,k} - f_{j,S}(u_{t}^{(j-1),l,k})\|) \cdot   \|x_t - x_t^{l,k}\|\\
		&\quad + 2\eta_t \sum_{i=1}^K  L_f^{K-i + \frac{1}{2}(i-1)i}  (\|v_t^{(i)} - \nabla f_{i,S}(u_t^{(i-1)})\| + \|v_t^{(i),k,l} - \nabla f_{i,S}(u_t^{(i-1),k,l})\|)\cdot  \|x_t - x_t^{l,k}\|  \\
		&\quad - \frac{2\eta_tL\mu}{L+\mu}\|x_t - x_t^{l,k}\|^2 - \frac{2\eta_t}{L+\mu} \|  \prod_{i=1}^{K} \nabla F_{i,S}(x_t)- \prod_{i=1}^{K} \nabla F_{i,S}(x_t^{l,k})\|^2.
	\end{aligned}
\end{equation*}
Changing the assumption of convexity to strong convexity does not affect the third term on the right side of \eqref{Eq:case1_mul_sc}, so we have
\begin{equation*}
	\begin{aligned}
		&\eta_t^2\|\prod_{i=1}^K v_{t}^{(i)} - \prod_{i=1}^K  v_{t}^{(i),l,k}\|^2\\
		& \leq \eta_t^2 \sum_{i=1}^{K}\sum_{j=1}^{i-1} L_f^{2K-2j+(i-1)i} (\|u_{t}^{(j)} - f_{j,S}(u_{t}^{(j-1)})\|^2 + \|u_{t}^{(j),l,k} - f_{j,S}(u_{t}^{(j-1),l,k})\|^2)\\
		&\quad + \eta_t^2 \sum_{i=1}^K  L_f^{2K-2i + (i-1)i}  (\|v_t^{(i)} - \nabla f_{i,S}(u_t^{(i-1)})\|^2 + \|v_t^{(i),k,l} - \nabla f_{i,S}(u_t^{(i-1),k,l})\|^2)  \\
		&\quad + (K+2)K\eta_t^2\| \prod_{i=1}^{K} \nabla F_{i,S}(x_t)- \prod_{i=1}^{K} \nabla F_{i,S}(x_t^{l,k})\|. 
	\end{aligned}
\end{equation*}
By setting  $\eta_t \leq \frac{2}{(L+\mu)(K+2)K }$ we have 
\begin{equation*}
	\begin{aligned}
		&\|x_{t+1} - x_{t+1}^{l,k}\|^2 \\
		& \leq  (1-\frac{2\eta_tL\mu}{L+\mu})\|x_t - x_t^{l,k}\|^2 \\
		& \quad + 2\eta_t \sum_{i=1}^{K}\sum_{j=1}^{i-1} L_f^{K-j+\frac{1}{2}(i-1)i} (\|u_{t}^{(j)} - f_{j,S}(u_{t}^{(j-1)})\| + \|u_{t}^{(j),l,k} - f_{j,S}(u_{t}^{(j-1),l,k})\|) \cdot   \|x_t - x_t^{l,k}\|\\
		&\quad + 2\eta_t \sum_{i=1}^K  L_f^{K-i + \frac{1}{2}(i-1)i}  (\|v_t^{(i)} - \nabla f_{i,S}(u_t^{(i-1)})\| + \|v_t^{(i),k,l} - \nabla f_{i,S}(u_t^{(i-1),k,l})\|)\cdot  \|x_t - x_t^{l,k}\|  \\
		&\quad + \eta_t^2 \sum_{i=1}^{K}\sum_{j=1}^{i-1} L_f^{2K-2j+(i-1)i} (\|u_{t}^{(j)} - f_{j,S}(u_{t}^{(j-1)})\|^2 + \|u_{t}^{(j),l,k} - f_{j,S}(u_{t}^{(j-1),l,k})\|^2)\\
		&\quad + \eta_t^2 \sum_{i=1}^K  L_f^{2K-2i + (i-1)i}  (\|v_t^{(i)} - \nabla f_{i,S}(u_t^{(i-1)})\|^2 + \|v_t^{(i),k,l} - \nabla f_{i,S}(u_t^{(i-1),k,l})\|^2). 
	\end{aligned}
\end{equation*}

\textbf{\quad Case 2 ($i_t =  l$ ).}  We have   
\begin{equation}\label{Eq:case2_mul_sc}
	\begin{aligned}
		\|x_{t+1} - x_{t+1}^{l,k}\| & =  \|x_t - \eta_t\prod_{i=1}^K v_{t}^{(i)} - x_t^{l,k} + \eta_t \prod_{i=1}^K  v_{t}^{(i),l,k}\| \\
		& \leq  \|x_t - x_t^{l,k}\|  + \eta_t \|\prod_{i=1}^K v_{t}^{(i)} - \prod_{i=1}^K  v_{t}^{(i),l,k}\| \leq \|x_t - x_t^{l,k}\| + 2\eta_t L_f^{K}.
	\end{aligned}
\end{equation}
Therefore, we have 
\begin{equation*}
	\|x_{t+1} - x_{t+1}^{l,k}\|^2 \leq \|x_t - x_t^{l,k}\|^2 + 4\eta_t L_f^{K} \|x_t - x_t^{l,k}\| + 4\eta_t^2L_f^{2K}.
\end{equation*}
Combining above two cases, and taking the expectation w.r.t. $A$ we have
\begin{equation*}
	\begin{aligned}
		&\EX_{A}[\|x_{t+1} - x_{t+1}^{l,k}\|^2] \\
		& \leq  (1-\frac{2\eta_tL\mu}{L+\mu})\|x_t - x_t^{l,k}\|^2 \\
		& \quad + 2\eta_t \sum_{i=1}^{K}\sum_{j=1}^{i-1} L_f^{K-j+\frac{1}{2}(i-1)i} (\|u_{t}^{(j)} - f_{j,S}(u_{t}^{(j-1)})\| + \|u_{t}^{(j),l,k} - f_{j,S}(u_{t}^{(j-1),l,k})\|) \cdot   \|x_t - x_t^{l,k}\|\\
		&\quad + 2\eta_t \sum_{i=1}^K  L_f^{K-i + \frac{1}{2}(i-1)i}  (\|v_t^{(i)} - \nabla f_{i,S}(u_t^{(i-1)})\| + \|v_t^{(i),k,l} - \nabla f_{i,S}(u_t^{(i-1),k,l})\|)\cdot  \|x_t - x_t^{l,k}\|  \\
		&\quad + \eta_t^2 \sum_{i=1}^{K}\sum_{j=1}^{i-1} L_f^{2K-2j+(i-1)i} (\|u_{t}^{(j)} - f_{j,S}(u_{t}^{(j-1)})\|^2 + \|u_{t}^{(j),l,k} - f_{j,S}(u_{t}^{(j-1),l,k})\|^2)\\
		&\quad + \eta_t^2 \sum_{i=1}^K  L_f^{2K-2i + (i-1)i}  (\|v_t^{(i)} - \nabla f_{i,S}(u_t^{(i-1)})\|^2 + \|v_t^{(i),k,l} - \nabla f_{i,S}(u_t^{(i-1),k,l})\|^2)  +  \|x_t - x_t^{l,k}\|^2\\
		&\quad +  4\eta_t L_f^{K} \|x_t - x_t^{l,k}\| \cdot \1_{[i_t = l]} + 4\eta_t^2L_f^{2K} \cdot \1_{[i_t = l]}.
	\end{aligned}
\end{equation*}

Then setting $\eta_t =\eta$ and using Lemma \ref{lemma:general_recursive} we can get
\begin{equation*}
	\begin{aligned}
		&\EX_A\|x_{t} - x_{t}^{l,k}\|^2 \\
		& \leq  2 \sum_{s=1}^{t-1}\sum_{i=1}^{K}\sum_{j=1}^{i-1}(1-\frac{2\eta L\mu}{L+\mu})^{t-s}  \eta L_f^{K-j+\frac{1}{2}(i-1)i} (\EX_A\|u_{s}^{(j)} - f_{j,S}(u_{s}^{(j-1)})\|^2)^{1/2} \cdot   (\EX_A\|x_s - x_s^{l,k}\|^2)^{1/2}\\
		& \quad +  2 \sum_{s=1}^{t-1} \sum_{i=1}^{K}\sum_{j=1}^{i-1} (1-\frac{2\eta L\mu}{L+\mu})^{t-s} \eta L_f^{K-j+\frac{1}{2}(i-1)i} (\EX_A \|u_{s}^{(j),l,k} - f_{j,S}(u_{s}^{(j-1),l,k})\|^2)^{1/2} \cdot  (\EX_A\|x_s - x_s^{l,k}\|^2)^{1/2}\\
		&\quad + 2 \sum_{s=1}^{t-1}\sum_{i=1}^K (1-\frac{2\eta L\mu}{L+\mu})^{t-s} \eta L_f^{K-i + \frac{1}{2}(i-1)i}  (\EX_A\|v_s^{(i)} - \nabla f_{i,S}(u_s^{(i-1)})\|^2)^{1/2}\cdot  (\EX_A\|x_s - x_s^{l,k}\|^2)^{1/2} \\
		&\quad + 2  \sum_{s=1}^{t-1} \sum_{i=1}^K (1-\frac{2\eta L\mu}{L+\mu})^{t-s} \eta L_f^{K-i + \frac{1}{2}(i-1)i}  ( \EX_A\|v_s^{(i),k,l} - \nabla f_{i,S}(u_s^{(i-1),k,l})\|^2)^{1/2}\cdot  (\EX_A\|x_s - x_s^{l,k}\|^2)^{1/2}  \\
		&\quad + \sum_{s=1}^{t-1} \sum_{i=1}^{K}\sum_{j=1}^{i-1} (1-\frac{2\eta L\mu}{L+\mu})^{t-s} \eta^2 L_f^{2K-2j+(i-1)i} (\EX_A\|u_{s}^{(j)} - f_{j,S}(u_{s}^{(j-1)})\|^2 + \EX_A\|u_{s}^{(j),l,k} - f_{j,S}(u_{s}^{(j-1),l,k})\|^2)\\
		&\quad + \sum_{s=1}^{t-1}  \sum_{i=1}^K (1-\frac{2\eta L\mu}{L+\mu})^{t-s} \eta^2 L_f^{2K-2i + (i-1)i}  (\EX_A\|v_s^{(i)} - \nabla f_{i,S}(u_s^{(i-1)})\|^2 + \EX_A\|v_s^{(i),k,l} - \nabla f_{i,S}(u_s^{(i-1),k,l})\|^2)  \\
		&\quad +  \sum_{s=1}^{t-1} (1-\frac{2\eta L\mu}{L+\mu})^{t-s} \frac{4\eta L_f^{K}}{n_k} (\EX_A\|x_s - x_s^{l,k}\|^2)^{1/2}+  \sum_{s=1}^{t-1} (1-\frac{2\eta L\mu}{L+\mu})^{t-s} \frac{4\eta^2L_f^{2K}}{n_k} .
	\end{aligned}
\end{equation*}  
Similarly, setting $u_t = (\EX_A\|x_{t} - x_{t}^{l,k}\|^2)^{1/2}$, 
\begin{equation*}
	\begin{aligned}
		S_t& = \sum_{s=1}^{t-1} \sum_{i=1}^{K}\sum_{j=1}^{i-1} (1-\frac{2\eta L\mu}{L+\mu})^{t-s} \eta^2 L_f^{2K-2j+(i-1)i} (\EX_A\|u_{s}^{(j)} - f_{j,S}(u_{s}^{(j-1)})\|^2 + \EX_A\|u_{s}^{(j),l,k} - f_{j,S}(u_{s}^{(j-1),l,k})\|^2)\\
		&\quad + \sum_{s=1}^{t-1}  \sum_{i=1}^K (1-\frac{2\eta L\mu}{L+\mu})^{t-s} \eta^2 L_f^{2K-2i + (i-1)i}  (\EX_A\|v_s^{(i)} - \nabla f_{i,S}(u_s^{(i-1)})\|^2 + \EX_A\|v_s^{(i),k,l} - \nabla f_{i,S}(u_s^{(i-1),k,l})\|^2)  \\
		& \quad + \sum_{s=1}^{t-1} (1-\frac{2\eta L\mu}{L+\mu})^{t-s} \frac{4\eta^2L_f^{2K}}{n_k},
	\end{aligned}
\end{equation*}
and 
\begin{equation*}
	\begin{aligned}
		\alpha_s& = 2 \sum_{i=1}^{K}\sum_{j=1}^{i-1}(1-\frac{2\eta L\mu}{L+\mu})^{t-s}  \eta L_f^{K-j+\frac{1}{2}(i-1)i} (\EX_A\|u_{s}^{(j)} - f_{j,S}(u_{s}^{(j-1)})\|^2)^{1/2} \\
		& \quad +  2  \sum_{i=1}^{K}\sum_{j=1}^{i-1} (1-\frac{2\eta L\mu}{L+\mu})^{t-s} \eta L_f^{K-j+\frac{1}{2}(i-1)i} (\EX_A \|u_{s}^{(j),l,k} - f_{j,S}(u_{s}^{(j-1),l,k})\|^2)^{1/2} \\
		&\quad + 2 \sum_{i=1}^K (1-\frac{2\eta L\mu}{L+\mu})^{t-s} \eta L_f^{K-i + \frac{1}{2}(i-1)i}  (\EX_A\|v_s^{(i)} - \nabla f_{i,S}(u_s^{(i-1)})\|^2)^{1/2} \\
		&\quad + 2   \sum_{i=1}^K (1-\frac{2\eta L\mu}{L+\mu})^{t-s} \eta L_f^{K-i + \frac{1}{2}(i-1)i}  ( \EX_A\|v_s^{(i),k,l} - \nabla f_{i,S}(u_s^{(i-1),k,l})\|^2)^{1/2} +  \sum_{s=1}^{t-1} (1-\frac{2\eta L\mu}{L+\mu})^{t-s} \frac{4\eta L_f^{K}}{n_k}.
	\end{aligned}
\end{equation*}
Then according to Lemma \ref{lemma:recursion lemma}, we have 
\begin{equation*}
	\begin{aligned}
		u_t & \leq \sqrt{S_t} + \sum_{s=1}^{t-1}\alpha_s\\
		& \leq (\sum_{s=1}^{t-1} \sum_{i=1}^{K}\sum_{j=1}^{i-1} (1-\frac{2\eta L\mu}{L+\mu})^{t-s} \eta^2 L_f^{2K-2j+(i-1)i} (\EX_A\|u_{s}^{(j)} - f_{j,S}(u_{s}^{(j-1)})\|^2 + \EX_A\|u_{s}^{(j),l,k} - f_{j,S}(u_{s}^{(j-1),l,k})\|^2))^{1/2}\\
		&\quad + (\sum_{s=1}^{t-1}  \sum_{i=1}^K (1-\frac{2\eta L\mu}{L+\mu})^{t-s} \eta^2 L_f^{2K-2i + (i-1)i}  (\EX_A\|v_s^{(i)} - \nabla f_{i,S}(u_s^{(i-1)})\|^2 + \EX_A\|v_s^{(i),k,l} - \nabla f_{i,S}(u_s^{(i-1),k,l})\|^2))^{1/2}  \\
		&\quad + 2 \sum_{s=1}^{t-1} \sum_{i=1}^{K}\sum_{j=1}^{i-1}(1-\frac{2\eta L\mu}{L+\mu})^{t-s}  \eta L_f^{K-j+\frac{1}{2}(i-1)i} (\EX_A\|u_{s}^{(j)} - f_{j,S}(u_{s}^{(j-1)})\|^2)^{1/2} \\
		& \quad +  2 \sum_{s=1}^{t-1} \sum_{i=1}^{K}\sum_{j=1}^{i-1} (1-\frac{2\eta L\mu}{L+\mu})^{t-s} \eta L_f^{K-j+\frac{1}{2}(i-1)i} (\EX_A \|u_{s}^{(j),l,k} - f_{j,S}(u_{s}^{(j-1),l,k})\|^2)^{1/2} \\
		&\quad + 2 \sum_{s=1}^{t-1} \sum_{i=1}^K (1-\frac{2\eta L\mu}{L+\mu})^{t-s} \eta L_f^{K-i + \frac{1}{2}(i-1)i}  (\EX_A\|v_s^{(i)} - \nabla f_{i,S}(u_s^{(i-1)})\|^2)^{1/2} \\
		&\quad + 2 \sum_{s=1}^{t-1}  \sum_{i=1}^K (1-\frac{2\eta L\mu}{L+\mu})^{t-s} \eta L_f^{K-i + \frac{1}{2}(i-1)i}  ( \EX_A\|v_s^{(i),k,l} - \nabla f_{i,S}(u_s^{(i-1),k,l})\|^2)^{1/2}\\
		& \quad + \sqrt{\frac{2\eta L_f^{2K}(L+\mu)}{n_kL\mu}} + \frac{2 L_f^{K}(L+\mu)}{n_kL\mu},
	\end{aligned}
\end{equation*}
where the last inequality holds by 
\begin{equation*}
	\sum_{s=1}^{t-1} (1-\frac{2\eta L\mu}{L+\mu})^{t-s} \frac{4\eta L_f^{K}}{n_k}\leq \frac{4\eta L_f^{K}}{n_k} \cdot \frac{L+\mu}{2\eta L \mu} = \frac{2L_f^{K}(L+\mu)}{n_kL\mu}.
\end{equation*}
Next, we will discuss which one is the dominant one, $(\sum_{s=1}^{t-1} \sum_{i=1}^{K}\sum_{j=1}^{i-1} (1-\frac{2\eta L\mu}{L+\mu})^{t-s} \eta^2 L_f^{2K-2j+(i-1)i} \EX_A\|u_{s}^{(j)} - f_{j,S}(u_{s}^{(j-1)})\|^2)^{1/2}$ or $ \sum_{s=1}^{t-1} \sum_{i=1}^{K}\sum_{j=1}^{i-1}(1-\frac{2\eta L\mu}{L+\mu})^{t-s}  \eta L_f^{K-j+\frac{1}{2}(i-1)i} (\EX_A\|u_{s}^{(j)} - f_{j,S}(u_{s}^{(j-1)})\|^2)^{1/2}$. According to Lemma \ref{lemma:u_t_bound_mul} we have
\begin{equation*}
	\begin{aligned}
		&(\sum_{s=1}^{t-1} \sum_{i=1}^{K}\sum_{j=1}^{i-1} (1-\frac{2\eta L\mu}{L+\mu})^{t-s} \eta^2 L_f^{2K-2j+(i-1)i} \EX_A\|u_{s}^{(j)} - f_{j,S}(u_{s}^{(j-1)})\|^2)^{1/2}\\
		& \leq \sqrt{K}\eta L_f^m ( \sum_{s=1}^{t-1}\sum_{j=1}^{K-1} (1-\frac{2\eta L\mu}{L+\mu})^{t-s}  \EX_A\|u_{s}^{(j)} - f_{j,S}(u_{s}^{(j-1)})\|^2)^{1/2}\\
		& \leq \sqrt{K}\eta L_f^m( \sum_{s=1}^{t-1}(1-\frac{2\eta L\mu}{L+\mu})^{t-s}(\sum_{i=1}^K (\frac{2c}{e})^{c}(s\beta)^{-c} \mathbb{E}[\|u_{1}^{(i)}- f_{i,S}(x_{0})\|^{2}]\\
		&\quad ~~~~~~~~~~~~~~~~~~~~~~~~~~~~~~~~~~~~~~~~~~~~~+ 4\beta \sigma_{f}^{2} K ((\sum_{i=1}^{K}(2 L_f^{2})^{i}) + 1) + \frac{2\sum_{i=1}^{K}(2 L_f^{2})^{i}\eta^2 L_f^K}{\beta})^{1/2}\\
		& \leq   \sqrt{K} L_f^{m} \sqrt{ (\frac{2c}{e})^c\sum_{i=1}^KU_{i} } \frac{\sqrt{(L+\mu)\eta}}{\sqrt{2L\mu}} T^{-\frac{c}{2}}\beta^{-\frac{c}{2}} +  2\sigma_fK\sqrt{ \frac{ L_f^{m}(\sum_{i=1}^{K}(2 L_f^{2})^{i}) + 1)(L+\mu)\eta}{2L\mu}} \cdot  \beta^{1/2}\\
		& \quad + \sqrt{\frac{K L_f^{m} \sum_{i=1}^K( 2 L_f^{2})^{i} L_f^K(L+\mu)}{L\mu}} \eta^{3/2} \beta^{-1/2},
	\end{aligned}
\end{equation*}
where the inequality holds by Lemma \ref{lemma:weighted_avg}. As for the later, 
\begin{equation}\label{Eq:bound_mul_dominate}
	\begin{aligned}
		& \sum_{s=1}^{t-1} \sum_{i=1}^{K}\sum_{j=1}^{i-1}(1-\frac{2\eta L\mu}{L+\mu})^{t-s}  \eta L_f^{K-j+\frac{1}{2}(i-1)i} (\EX_A\|u_{s}^{(j)} - f_{j,S}(u_{s}^{(j-1)})\|^2)^{1/2}\\
		& \leq K\eta L_f^m \sum_{s=1}^{t-1}\sum_{j=1}^{K-1} (1-\frac{2\eta L\mu}{L+\mu})^{t-s}  (\EX_A\|u_{s}^{(j)} - f_{j,S}(u_{s}^{(j-1)})\|^2)^{1/2}\\
		& \leq K\eta L_f^m \sum_{s=1}^{t-1}(1-\frac{2\eta L\mu}{L+\mu})^{t-s} (\sum_{i=1}^K (\frac{2c}{e})^{c}(s\beta)^{-c} \mathbb{E}[\|u_{1}^{(i)}- f_{i,S}(x_{0})\|^{2}]\\
		&\quad ~~~~~~~~~~~~~~~~~~~~~~~~~~~~~~~~~~~~~~~~~~~~~+ 4\beta \sigma_{f}^{2} K ((\sum_{i=1}^{K}(2 L_f^{2})^{i}) + 1) + \frac{2\sum_{i=1}^{K}(2 L_f^{2})^{i}\eta^2 L_f^K}{\beta})^{1/2}\\
		& \leq  K L_f^{m} \sqrt{ (\frac{2c}{e})^c}\sum_{i=1}^KU_{i}  \frac{(L+\mu)}{2L\mu} T^{-\frac{c}{2}}\beta^{-\frac{c}{2}}+  2\sigma_fK^2 L_f^{m}\sqrt{  (\sum_{i=1}^{K}(2 L_f^{2})^{i}) + 1)} \cdot \frac{(L+\mu)}{2L\mu} \beta^{1/2}\\
		& \quad + KL_f^{m}\sqrt{  2\sum_{i=1}^K( 2 L_f^{2})^{i} L_f^K} \frac{(L+\mu)}{2L\mu} \eta\beta^{-1/2}.
	\end{aligned}
\end{equation}
Comparing the above results, we can find the dominant term is $ \sum_{s=1}^{t-1} \sum_{i=1}^{K}\sum_{j=1}^{i-1}(1-\frac{2\eta L\mu}{L+\mu})^{t-s}  \eta L_f^{K-j+\frac{1}{2}(i-1)i} (\EX_A\|u_{s}^{(j)} - f_{j,S}(u_{s}^{(j-1)})\|^2)^{1/2}$.  Take a similar action for several other items and we can get 
\begin{equation}\label{Eq:bound_u_T_mul_sc}
	\begin{aligned}
		u_t & \leq \sqrt{S_t} + \sum_{s=1}^{t-1}\alpha_s\\
		& \leq 6\sum_{s=1}^{t-1} \sum_{i=1}^{K}\sum_{j=1}^{i-1} (1-\frac{2\eta L\mu}{L+\mu})^{t-s} \eta L_f^{K-j+(i-1)i/2} (\EX_A\|u_{s}^{(j)} - f_{j,S}(u_{s}^{(j-1)})\|^2)^{1/2}\\
		&\quad + 6\sum_{s=1}^{t-1}  \sum_{i=1}^K (1-\frac{2\eta L\mu}{L+\mu})^{t-s} \eta L_f^{K-i + (i-1)i/2}  (\EX_A\|v_s^{(i)} - \nabla f_{i,S}(u_s^{(i-1)})\|^2)^{1/2}  \\
		& \quad + \sqrt{\frac{2\eta L_f^{2K}(L+\mu)}{n_kL\mu}} + \frac{2 L_f^{K}(L+\mu)}{n_kL\mu}.
	\end{aligned}
\end{equation}
Since often we have $\eta \leq \min{\frac{1}{n_k}}$ for any $k \in [1,K]$. Therefore, we have 
\begin{equation*}
	\begin{aligned}
		\epsilon_k & \leq O\Big( \sum_{s=1}^{T-1} \sum_{i=1}^{K}\sum_{j=1}^{i-1} (1-\frac{2\eta L\mu}{L+\mu})^{t-s} \eta L_f^{K-j+(i-1)i/2} (\EX_A\|u_{s}^{(j)} - f_{j,S}(u_{s}^{(j-1)})\|^2)^{1/2}\\
		&\quad~~~~~~~ +\sum_{s=1}^{T-1}  \sum_{i=1}^K (1-\frac{2\eta L\mu}{L+\mu})^{t-s} \eta L_f^{K + (i-3)i/2}  (\EX_A\|v_s^{(i)} - \nabla f_{i,S}(u_s^{(i-1)})\|^2)^{1/2}   + \frac{ L_f^{K}(L+\mu)}{L\mu n_k} \Big).
	\end{aligned}
\end{equation*}
Moreover, we have 
\begin{equation*}
	\begin{aligned}
		\sum_{k=1}^K\epsilon_k & \leq O\Big( \sum_{s=1}^{T-1} \sum_{i=1}^{K}\sum_{j=1}^{i-1} (1-\frac{2\eta L\mu}{L+\mu})^{t-s} \eta L_f^{K-j+(i-1)i/2} (\EX_A\|u_{s}^{(j)} - f_{j,S}(u_{s}^{(j-1)})\|^2)^{1/2}\\
		&\quad~~~~~~~ +(\sum_{s=1}^{T-1}  \sum_{i=1}^K (1-\frac{2\eta L\mu}{L+\mu})^{t-s} \eta^2 L_f^{K+ (i-3)i/2}  (\EX_A\|v_s^{(i)} - \nabla f_{i,S}(u_s^{(i-1)})\|^2)^{1/2}   + \sum_{k=1}^K\frac{ L_f^{K}(L+\mu)}{L\mu n_k}  \Big).
	\end{aligned}
\end{equation*}
This completes the proof.

\begin{corollary}[$K$-level Optimization]\label{cor:2_mul_level}Consider Algorithm \ref{alg_cover} with  $ 0 <\eta_t = \eta <  2/(L+\mu)K(K+2)$ and let $ 0<  \beta_{t} =\beta < \max{\{ 1, 1/(4K\sum_{i=1}^K(2L_f^2)^{i}\}}$ for any $t\in [0,T-1]$ and the output  $A(S)  =x_{T}$. Then, we have the following results
	\begin{equation*}
		\begin{aligned}
			\sum_{k=1}^K \epsilon_k  \leq O((T \beta)^{-\frac{c}{2}} + \beta^{\frac{1}{2}}  + \eta\beta^{-\frac{1}{2}  } + \sum_{k=1}^Kn_k^{-1} ).
		\end{aligned}
	\end{equation*}
\end{corollary}

Next, we give the proof of Corollary \ref{cor:2_mul_level}.

\begin{proof}[proof of corollary \ref{cor:2_mul_level}]
	Putting the result \eqref{Eq:bound_mul_dominate} into \eqref{Eq:bound_u_T_mul_sc}, since often we have $\eta \leq \min{\frac{1}{n_k}}$ for any $k \in [1,K]$. Therefore, we have 
	\begin{equation*}
		\begin{aligned}
			\sum_{k=1}^K \epsilon_k & \leq O(T^{-\frac{c}{2}}\beta^{-\frac{c}{2}} + \beta^{1/2} + \eta\beta^{-1/2} + \sum_{k=1}^Kn_k^{-1} ).
		\end{aligned}
	\end{equation*}
	This complete the proof.
\end{proof}
Before  giving the proof of Theorem \ref{thm:opt_sconvex}, we first give a useful lemma.
\begin{lemma}\label{lemma:mul_sc}
	Let Assumption \ref{ass:Lipschitz continuous}(iii), \ref{ass:bound variance} (iii) and \ref{ass:Smoothness and Lipschitz continuous gradient} (iii) hold, $F_S$ is $\mu$-strongly convex, then for SVMR, we have for any $x$
	\begin{equation*}
	\begin{aligned}
		\EX_{A}[F_S(x_{t+1})|\F_t] & \leq \EX_{A}[F_S(x_{t})|\F_t]  -\frac{\eta_{t}}{2}\|\nabla F_S(x_t)\|^2 +\frac{\eta_t^2L_f^K}{2} \\
		&\quad +   4K^4L_f^{m} \eta_t \sum_{j=1}^{K-1} \EX_{A}[ \|u_{t}^{(j)} - f_{j,S}(u_{t}^{(j-1)})\|^2]|\F_t]\\
		&\quad + 4K^2 L_f^{m} \eta_t \sum_{i=1}^{K-1}  \EX_{A}[ \|v_t^{(i)} - \nabla f_{i,S}(u_t^{(i-1)})\|^2)|\F_t].
	\end{aligned}
\end{equation*}
	where $\EX_{A}$ A denotes the expectation taken with respect to the randomness of the algorithm, and $\F_t$ is the $\sigma$-field generated by $S$.
\end{lemma}
\begin{proof}[proof of Lemma \ref{lemma:mul_sc}]
	According to the Assumption \ref{ass:Smoothness and Lipschitz continuous gradient} (iii) we have 
	\begin{equation*}
		\begin{aligned}
			F_S(x_{t+1}) & \leq F_S(x_t) + \langle \nabla F_S(x_t),x_{t+1}-x_t \rangle +\frac{1}{2}\|x_{t+1}-x_t\|^2\\
			& \leq F_S(x_t)  - \eta_t \langle \nabla F_S(x_t), \prod_{i=1}^K v_{t}^{(i)} \rangle +\frac{1}{2}\|x_{t+1}-x_t\|^2\\
			&  = F_S(x_t)  - \eta_t \langle \nabla F_S(x_t),  \prod_{i=1}^K \nabla F_{i,S}(x_t) \rangle +\frac{1}{2}\|x_{t+1}-x_t\|^2 - u_t,
		\end{aligned}
	\end{equation*}
	where $u_t =  \eta_t \langle \nabla F_S(x_t), \prod_{i=1}^K v_{t}^{(i)} - \prod_{i=1}^K \nabla F_{i,S}(x_t) \rangle $.
	
	Let $\F_t$ be the $\sigma$-field generated by $S$. Taking expectation with respect to the randomness of the algorithm conditioned on $\F_t$, we have
	\begin{equation*}
		\begin{aligned}
			\EX_{A}[F_S(x_{t+1})|\F_t] & \leq \EX_{A}[F_S(x_{t})|\F_t] - \eta_t \|\nabla F_S(x_t)\|^2 + \frac{\eta_t^2L_f^K}{2} - \EX_{A}[u_t|\F_t].
		\end{aligned}
	\end{equation*}
	Now we bound the term $ \EX_{A}[u_t|\F_t]$. 
	\begin{equation*}
		\begin{aligned}
			-\EX_{A}[u_t|\F_t] & = 	\EX_{A}[\eta_t \langle \nabla F_S(x_t), \prod_{i=1}^K \nabla F_{i,S}(x_t) - \prod_{i=1}^K v_{t}^{(i)}   \rangle  |\F_t]\\
			& =  \EX_{A}[\eta_t \langle \prod_{i=1}^{K} \nabla F_{i,S}(x_t) -  \prod_{i=2}^{K} \nabla F_{i,S}(x_t) \cdot v_t^{(1)} , \nabla F_S(x_t) \rangle  |\F_t] \\
			&  +\EX_{A}[\eta_t\langle   \prod_{i=2}^{K} \nabla F_{i,S}(x_t) \cdot v_t^{(1)}  -  \prod_{i=3}^{K} \nabla F_{i,S}(x_t) \cdot v_t^{(1)} \cdot \nabla f_{2,S}(u_t^{(1)}) , \nabla F_S(x_t) \rangle  |\F_t] \\
			&  + \EX_{A}[\eta_t\langle  \prod_{i=3}^{K} \nabla F_{i,S}(x_t) \cdot v_t^{(1)} \cdot \nabla f_{2,S}(u_t^{(1)}) -  \prod_{i=3}^{K} \nabla F_{i,S}(x_t) \cdot v_t^{(1)} \cdot  v_t^{(2)}, \nabla F_S(x_t) \rangle  |\F_t] \\
			& \vdots\\
			&  + \EX_{A}[\eta_t\langle  \prod_{i=1}^{K-1} v_t^{(i)} \cdot \nabla f_{K,S}(u_t^{(K-1)}) - \prod_{i=1}^K  v_{t}^{(i)}, \nabla F_S(x_t) \rangle  |\F_t].
		\end{aligned}
	\end{equation*}
	Concluding the above inequality, using Assumption \ref{ass:Lipschitz continuous} (iii)  we have
	\begin{equation*}
		\begin{aligned}
			-\EX_{A}[u_t|\F_t] 
			& \leq \eta_t \sum_{i=1}^{K}\sum_{j=1}^{i-1} L_f^{K-j+\frac{1}{2}(i-1)i} \EX_{A}[ \|u_{t}^{(j)} - f_{j,S}(u_{t}^{(j-1)})\|  \cdot   \|\nabla F_S(x_t)\| |\F_t]\\
			&\quad + \eta_t \sum_{i=1}^K  L_f^{K-i + \frac{1}{2}(i-1)i} \EX_{A}[ \|v_t^{(i)} - \nabla f_{i,S}(u_t^{(i-1)})\| \cdot  \|\nabla F_S(x_t)\| |\F_t]  \\
			& \leq K L_f^{m} \eta_t \sum_{j=1}^{K-1}\EX_{A}[ \|u_{t}^{(j)} - f_{j,S}(u_{t}^{(j-1)})\|  \cdot   \|\nabla F_S(x_t)\| |\F_t]\\
			&\quad +  L_f^{m}\eta_t  \sum_{i=1}^{K-1} \EX_{A}[ \|v_t^{(i)} - \nabla f_{i,S}(u_t^{(i-1)})\| \cdot  \|\nabla F_S(x_t)\| |\F_t].
		\end{aligned}
	\end{equation*}
	According to Cauchy-Schwartz inequality, we have
	\begin{equation*}
		\begin{aligned}
			-\EX_{A}[u_t|\F_t]
			&\leq  K L_f^{m} \eta_t \sum_{j=1}^{K-1} ( \frac{\|F_S(x_t)\|^2}{\gamma_t} + \gamma_t \EX_{A}[ \|u_{t}^{(j)} - f_{j,S}(u_{t}^{(j-1)})\|^2]|\F_t]\\
			&\quad +  L_f^{m}\eta_t \sum_{i=1}^{K-1}  ( \frac{\|F_S(x_t)\|^2}{\lambda_t} + \lambda_t \EX_{A}[ \|v_t^{(i)} - \nabla f_{i,S}(u_t^{(i-1)})\|^2)|\F_t].
		\end{aligned}
	\end{equation*}
	Therefore we have 
	\begin{equation*}
		\begin{aligned}
			\EX_{A}[F_S(x_{t+1})|\F_t] & \leq \EX_{A}[F_S(x_{t})|\F_t] - \eta_t \|\nabla F_S(x_t)\|^2 + \frac{\eta_t^2L_f^K}{2} \\
			&\quad + K L_f^{m} \eta_t  (K \frac{\|F_S(x_t)\|^2}{\gamma_t} + \sum_{j=1}^{K-1} \gamma_t \EX_{A}[ \|u_{t}^{(j)} - f_{j,S}(u_{t}^{(j-1)})\|^2]|\F_t]\\
			&\quad +  L_f^{m}\eta_t   (K \frac{\|F_S(x_t)\|^2}{\lambda_t} + \sum_{i=1}^{K-1} \lambda_t \EX_{A}[ \|v_t^{(i)} - \nabla f_{i,S}(u_t^{(i-1)})\|^2)|\F_t].
		\end{aligned}
	\end{equation*}
	Setting $\gamma_t = 4K^2L_f^{m}$ and $\lambda_t = 4K L_f^{m} $, we have 
	\begin{equation*}
		\begin{aligned}
			\EX_{A}[F_S(x_{t+1})|\F_t] & \leq \EX_{A}[F_S(x_{t})|\F_t]  -\frac{\eta_{t}}{2}\|\nabla F_S(x_t)\|^2 +\frac{\eta_t^2L_f^K}{2} \\
			&\quad +   4K^4L_f^{m} \eta_t \sum_{j=1}^{K-1} \EX_{A}[ \|u_{t}^{(j)} - f_{j,S}(u_{t}^{(j-1)})\|^2]|\F_t]\\
			&\quad + 4K^2 L_f^{m} \eta_t \sum_{i=1}^{K-1}  \EX_{A}[ \|v_t^{(i)} - \nabla f_{i,S}(u_t^{(i-1)})\|^2)|\F_t].
		\end{aligned}
	\end{equation*}
	This complete the proof.
\end{proof}
Next, we will give the detailed proof of Theorem \ref{thm:opt_sconvex}.

\begin{proof}[proof of Theorem \ref{thm:opt_sconvex}]
	Note that strong convexity implies the Polyak-\L ojasiewicz (PL) inequality
	\begin{equation*}
		\frac{1}{2}\|\nabla F_{S}(x)\|^{2} \geq \mu(F_{S}(x)-F_{S}(x_{*}^{S})), \quad \forall x.
	\end{equation*}
	Then according to Lemma \ref{lemma:mul_sc} and PL condition, subtracting both sides with $F_S(x_*^S)$  we have 
	\begin{equation*}
		\begin{aligned}
			\EX_{A}[F_S(x_{t+1})& - F_S(x_*^S)  ] \leq (1-\mu \eta_t) \EX_{A}[F_S(x_{t}) - F_{S}(x_{*}^{S})]  +\frac{\eta_t^2L_f^K}{2} \\
			&\quad +   4K^4L_f^{m} \eta_t \sum_{j=1}^{K-1} \EX_{A}[ \|u_{t}^{(j)} - f_{j,S}(u_{t}^{(j-1)})\|^2]] + 4K^2 L_f^{m} \eta_t \sum_{i=1}^{K-1}  \EX_{A}[ \|v_t^{(i)} - \nabla f_{i,S}(u_t^{(i-1)})\|^2)].
		\end{aligned}
	\end{equation*}
	By setting $\eta_t = \eta $, $\beta_{t} = \beta$, according to Lemma \ref{lemma:u_t_bound_mul} and Lemma \ref{lemma:v_t_bound_mul} we have 
	\begin{equation*}
		\begin{aligned}
			&\EX_{A}[F_S(x_{t+1}) - F_S(x_*^S)  ] \\
			& \leq (1-\mu \eta) \EX_{A}[F_S(x_{t}) - F_{S}(x_{*}^{S})]  +\frac{\eta^2 L_f^K}{2} \\
			&\quad +   4K^4L_f^{m} \eta (\sum_{i=1}^K (\frac{c}{e})^{c}(\frac{t\beta}{2})^{-c} \mathbb{E}[\|u_{1}^{(i)}- f_{i,S}(x_{0})\|^{2}]+ 4\beta \sigma_{f}^{2} K ((\sum_{i=1}^{K}(2 L_f^{2})^{i}) + 1) + \frac{2\sum_{i=1}^{K}(2 L_f^{2})^{i}\eta^2 L_f^K}{\beta}.)\\
			&\quad + 4K^2 L_f^{m} \eta (\sum_{i=1}^K (\frac{c}{e})^{c}(\frac{t\beta}{2})^{-c}  (\mathbb{E}[\|u_{1}^{(i)}- f_{i,S}(x_{0})\|^{2}] + \mathbb{E}[\|v_{1}^{(i)}- \nabla f_{i,S}(x_{0})\|^{2}] ) + \frac{4(\sum_{i=1}^{K}(2 L_f^{2})^{i})\eta^2 L_f^K}{\beta} \\
			&\quad~~~~~~~~~~~~~~~~~~~~~~~~~~~~~~~~+4 \beta K  (\sigma_{f}^{2} + \sigma_{J}^{2} + 2 \sigma_{f}^2(\sum_{i=1}^{K}(2 L_f^{2})^{i})  )).
		\end{aligned}
	\end{equation*}
	Telescoping the above inequality from 1 to $T-1$, we have 
	
	\begin{equation*}
		\begin{aligned}
			& \EX_{A}[F_S(x_{T}) - F_S(x_*^S)  ] \\
			& \leq (1-\mu \eta)^{T-1} \EX_{A}[F_S(x_{1}) - F_{S}(x_{*}^{S})]  +\frac{\eta^2 L_f^K}{2}\sum_{t=1}^{T-1}(1-\mu\eta)^{T-t-1} +  4K^4L_f^{m} (\frac{2c}{e})^{c} \eta \beta^{-c} (\sum_{i=1}^K U_{i}) \sum_{t=1}^{T-1} t^{-c}(1-\mu\eta)^{T-t-1} \\
			&\quad +16K^5 L_f^{m} \sigma_{f}^2 ((\sum_{i=1}^{K}(2 L_f^{2})^{i}) + 1) \eta\beta \sum_{t=1}^{T-1}(1-\mu\eta)^{T-t-1} + \frac{ 8K^4L_f^{m} \sum_{i=1}^{K}(2 L_f^{2})^{i} \eta^3}{\beta} \sum_{t=1}^{T-1}(1-\mu\eta)^{T-t-1} \\
			&\quad+ 4K^2L_f^{m} (\frac{2c}{e})^{c} \eta \beta^{-c} (\sum_{i=1}^K (U_{i} + V_{i}) ) \sum_{t=1}^{T-1} t^{-c}(1-\mu\eta)^{T-t-1}+  \frac{ 16K^2L_f^{m} \sum_{i=1}^{K}(2 L_f^{2})^{i} \eta^3}{\beta} \sum_{t=1}^{T-1}(1-\mu\eta)^{T-t-1} \\
			&\quad +   16K^3 L_f^{m} (\sigma_{f}^{2} + \sigma_{J}^{2} + 2 \sigma_{f}^2(\sum_{i=1}^{K}(2 L_f^{2})^{i})  ) \eta\beta \sum_{t=1}^{T-1}(1-\mu\eta)^{T-t-1}.
		\end{aligned}
	\end{equation*}
	
	For $t=0$, we have 
	\begin{equation*}
		\begin{aligned}
			&\EX_{A}[F_S(x_{1}) - F_S(x_*^S)  ] \\
			& \leq (1-\mu \eta) \EX_{A}[F_S(x_{0}) - F_{S}(x_{*}^{S})] +\frac{\eta^2L_f^K}{2} + 4K^4L_f^{m} \eta \sum_{j=1}^{K-1}U_{i}  +  4K^2 L_f^{m} \eta \sum_{j=1}^{K-1}(U_{i} + V_{i}).
		\end{aligned}
	\end{equation*}
	Then combining above two cases, we have 
	
	\begin{equation*}
		\begin{aligned}
			&\EX_{A}[F_S(x_{T}) - F_S(x_*^S)  ] \\
			& \leq (1-\mu \eta)^T \EX_{A}[F_S(x_{0}) - F_{S}(x_{*}^{S})]  +\frac{\eta^2 L_f^K}{2}\sum_{t=1}^{T}(1-\mu\eta)^{T-t}  +  4K^4L_f^{m} (\frac{2c}{e})^{c} \eta \beta^{-c} (\sum_{i=1}^K U_{i}) \sum_{t=1}^{T-1} t^{-c}(1-\mu\eta)^{T-t-1} \\
			&\quad +16K^5 L_f^{m} \sigma_{f}^2 ((\sum_{i=1}^{K}(2 L_f^{2})^{i}) + 1) \eta\beta \sum_{t=1}^{T-1}(1-\mu\eta)^{T-t-1} + \frac{ 8K^4L_f^{m} \sum_{i=1}^{K}(2 L_f^{2})^{i} \eta^3}{\beta} \sum_{t=1}^{T-1}(1-\mu\eta)^{T-t-1} \\
			&\quad+ 4K^2L_f^{m} (\frac{2c}{e})^{c} \eta \beta^{-c} (\sum_{i=1}^K (U_{i} + V_{i}) ) \sum_{t=1}^{T-1} t^{-c}(1-\mu\eta)^{T-t-1}+  \frac{ 16K^2L_f^{m} \sum_{i=1}^{K}(2 L_f^{2})^{i} \eta^3}{\beta} \sum_{t=1}^{T-1}(1-\mu\eta)^{T-t-1} \\
			&\quad +   16K^3 L_f^{m} (\sigma_{f}^{2} + \sigma_{J}^{2} + 2 \sigma_{f}^2(\sum_{i=1}^{K}(2 L_f^{2})^{i})  ) \eta\beta \sum_{t=1}^{T-1}(1-\mu\eta)^{T-t-1}\\
			& \quad + (4K^4L_f^{m} \eta \sum_{j=1}^{K-1}U_{i}  +  4K^2 L_f^{m} \eta \sum_{j=1}^{K-1}(U_{i} + V_{i})) (1-\mu \eta)^{T-1} .
		\end{aligned}
	\end{equation*}
	
	Then from Lemma \ref{lemma:weighted_avg}, we have
	\begin{equation*}
		\sum_{t=1}^{T-1}(1-\mu \eta)^{T-t-1} t^{-c} \leq \frac{\sum_{t=1}^{T-1}(1-\mu \eta)^{T-t-1}}{T-1} \sum_{t=1}^{T-1} t^{-c} \leq \frac{1}{T \mu \eta} \sum_{t=1}^{T-1} t^{-c}
	\end{equation*}
	Therefore,
	\begin{equation*}
		\begin{aligned}
			&\EX_{A}[F_S(x_{T}) - F_S(x_*^S)  ] \\
			& \leq (\frac{c}{e\mu})^{c}(\eta T)^{-c}D_x +\frac{\eta L_f^K}{\mu}  \\
			&\quad +  \frac{4K^4L_f^{m} (\frac{2c}{e})^{c} \beta^{-c}  (\sum_{i=1}^K U_{i})   }{T\mu} \sum_{t=1}^{T-1} t^{-c} +\frac{16K^5 L_f^{m} \sigma_{f}^2 ((\sum_{i=1}^{K}(2 L_f^{2})^{i}) + 1) \beta}{\mu}    \\
			&\quad + \frac{ 8K^4L_f^{m} \sum_{i=1}^{K}(2 L_f^{2})^{i} \eta^2}{\beta\mu}  + \frac{ 4K^2L_f^{m} (\frac{2c}{e})^{c}\beta^{-c} (\sum_{i=1}^K (U_{i} + V_{i}) )}{T\mu} \sum_{t=1}^{T-1} t^{-c}\\
			&\quad +  \frac{ 16K^2L_f^{m} \sum_{i=1}^{K}(2 L_f^{2})^{i} \eta^2}{\beta\mu}   +    \frac{16K^3 L_f^{m} (\sigma_{f}^{2} + \sigma_{J}^{2} + 2 \sigma_{f}^2(\sum_{i=1}^{K}(2 L_f^{2})^{i})  ) \beta}{\mu} \\
			& \quad + (4K^4L_f^{m}  \sum_{j=1}^{K-1}U_{i}  +  4K^2 L_f^{m}  \sum_{j=1}^{K-1}(U_{i} + V_{i})) (\frac{c}{e\mu})^{c}\eta(\eta T)^{-c} .
		\end{aligned}
	\end{equation*}
	Moreover, note that \(\sum_{t= 1}^{T} t^{-z}= O(T^{1- z})\) for \(z\in (0, 1)\cup (1, \infty)\) and \(\sum_{t= 1}^{T} t^{-1}= O(\log T)\). As long as \(c\neq 1\) we get
	
	\begin{equation*}
		\begin{aligned}
			\EX_{A}[F_S(x_{T}) - F_S(x_*^S)  ] 
			& = O\Big((\eta T)^{-c}D_x +\eta L_f^K  +  L_f^{m}   (\sum_{i=1}^K U_{i})(\beta T)^{-c} + L_f^{m} \sigma_{f}^2 ((\sum_{i=1}^{K}( L_f^{2})^{i}) + 1) \beta    \\
			&\quad + L_f^{m} \sum_{i=1}^{K}( L_f^{2})^{i} \eta^2\beta^{-1}  + L_f^{m}  (\sum_{i=1}^K (U_{i} + V_{i}) ) (\beta T)^{-c} \\
			&\quad +  L_f^{m} \sum_{i=1}^{K}(L_f^{2})^{i} \eta^2\beta^{-1}   +    L_f^{m} (\sigma_{f}^{2} + \sigma_{J}^{2} +  \sigma_{f}^2(\sum_{i=1}^{K}( L_f^{2})^{i})  ) \beta \\
			& \quad + (L_f^{m}  \sum_{j=1}^{K-1}U_{i}  + L_f^{m}  \sum_{j=1}^{K-1}(U_{i} + V_{i})) \eta(\eta T)^{-c} \Big).
		\end{aligned}
	\end{equation*}
	By rearranging the above inequality, we can obtain
	\begin{equation*}
		\begin{aligned}
			\EX_{A}[F_S(x_{T}) - F_S(x_*^S)] 
			& \leq O\Big((\eta T)^{-c}D_x +\eta L_f^K + L_f^{m} \sum_{i=1}^K (U_{i} + V_{i}) (\beta T)^{-c}\\
			&\quad  + L_f^{m} (\sigma_{f}^{2} + \sigma_{J}^{2} +  \sigma_{f}^2(\sum_{i=1}^{K}( L_f^{2})^{i})  ) \beta  \\
			&\quad +  L_f^{m} \sum_{i=1}^{K}(L_f^{2})^{i} \eta^2\beta^{-1} +  L_f^{m}  \sum_{j=1}^{K-1}(U_{i} + V_{i}) \eta(\eta T)^{-c} \Big).
		\end{aligned}
	\end{equation*}
	The proof is completed.
\end{proof}

\begin{proof}[proof of Theorem \ref{thm:Excess_Risk_Bound_s_convex}]
	Combining  Theorem \ref{theorem:general_multi_level},  and Theorem  \ref{thm:opt_sconvex} we have
	\begin{equation*}
		\begin{aligned}
			&\EX_{S,A}[F(x_T)-F_S(x_T)]\\
			&\leq L_f^K \epsilon_K +  4L_f^K\sum_{t=1}^{K-1}\epsilon_{t}+ L_f\sum_{t=2}^K\sqrt{\frac{ \mathbb{E}_{S, A} [\operatorname{Var}_{K-t+1}(A(S)]}{n_{K-t+1}}}\\
			& \leq 12 L_f^K\sqrt{K  L_f^{m+K} (\frac{2c}{e})^c\sum_{i=1}^KU_{i} }  \frac{(L+\mu)}{L\mu} T^{-\frac{c}{2}}\beta^{-\frac{c}{2}} + L_f\sum_{t=2}^K\sqrt{\frac{ \mathbb{E}_{S, A} [\operatorname{Var}_{K-t+1}(A(S)]}{n_{K-t+1}}}\\
			&\quad+  24\sigma_fK L_f^K\sqrt{  L_f^{m+K}(\sum_{i=1}^{K}(2 L_f^{2})^{i}) + 1)} \cdot \frac{(L+\mu)}{L\mu} \beta^{1/2} + 12L_f^K\sqrt{K L_f^{m+K} 2\sum_{i=1}^K( 2 L_f^{2})^{i} L_f^K} \frac{(L+\mu)}{L\mu} \eta\beta^{-1/2}\\
			&\quad + 12L_f^K\sqrt{L_f^{m+K} (\frac{2c}{3})^c (\sum_{k=1}^{K}(U_{i} + v_{1}^{(i)} )) } \frac{L+\mu}{L\mu} T^{-\frac{c}{2}}\beta^{-\frac{c}{2}} +  24L_f^K\sqrt{K L_f^{m+K} \sum_{i=1}^K( 2 L_f^{2})^{i} L_f^K} \frac{(L+\mu)}{L\mu} \eta\beta^{-1/2}\\
			&\quad + 24L_f^K\sqrt{    (\sigma_{f}^{2} + \sigma_{J}^{2} + 2 \sigma_{f}^2(\sum_{i=1}^{K}(2 L_f^{2})^{i})  )} \frac{(L+\mu)}{L\mu}\beta^{1/2}  + \sum_{k=1}^K\sqrt{\frac{2\eta L_f^{2K}(L+\mu)}{n_kL\mu}} + \sum_{k=1}^K\frac{2 L_f^{K}(L+\mu)}{n_kL\mu},
		\end{aligned}
	\end{equation*}
	Then according to Theorem \ref{thm:opt_sconvex}, we have 
	\begin{equation*}
		\begin{aligned}
			&\EX_{A}[F(A(S)) - F(x_*)]\\
			& \leq 12 L_f^K\sqrt{K  L_f^{m+K} (\frac{2c}{e})^c\sum_{i=1}^KU_{i} }  \frac{(L+\mu)}{L\mu} T^{-\frac{c}{2}}\beta^{-\frac{c}{2}} + L_f\sum_{t=2}^K\sqrt{\frac{ \mathbb{E}_{S, A} [\operatorname{Var}_{K-t+1}(A(S)]}{n_{K-t+1}}}\\
			&\quad+  24\sigma_fK L_f^K\sqrt{  L_f^{m+K}(\sum_{i=1}^{K}(2 L_f^{2})^{i}) + 1)} \cdot \frac{(L+\mu)}{L\mu} \beta^{1/2} + 12L_f^K\sqrt{K L_f^{m+K} 2\sum_{i=1}^K( 2 L_f^{2})^{i} L_f^K} \frac{(L+\mu)}{L\mu} \eta\beta^{-1/2}\\
			&\quad + 12L_f^K\sqrt{L_f^{m+K} (\frac{2c}{3})^c (\sum_{k=1}^{K}(U_{i} + v_{1}^{(i)} )) } \frac{L+\mu}{L\mu} T^{-\frac{c}{2}}\beta^{-\frac{c}{2}} +  24L_f^K\sqrt{K L_f^{m+K} \sum_{i=1}^K( 2 L_f^{2})^{i} L_f^K} \frac{(L+\mu)}{L\mu} \eta\beta^{-1/2}\\
			&\quad + 24L_f^K\sqrt{    (\sigma_{f}^{2} + \sigma_{J}^{2} + 2 \sigma_{f}^2(\sum_{i=1}^{K}(2 L_f^{2})^{i})  )} \frac{(L+\mu)}{L\mu}\beta^{1/2}  + \sum_{k=1}^K\sqrt{\frac{2\eta L_f^{2K}(L+\mu)}{n_kL\mu}} + \sum_{k=1}^K\frac{2 L_f^{K}(L+\mu)}{n_kL\mu}\\
			& \quad + (\eta T)^{-c}D_x +\eta L_f^K   +  L_f^{m}   (\sum_{i=1}^K U_{i})(\beta T)^{-c} + L_f^{m} \sigma_{f}^2 ((\sum_{i=1}^{K}(2 L_f^{2})^{i}) + 1) \beta    \\
			&\quad + L_f^{m} \sum_{i=1}^{K}( L_f^{2})^{i} \eta^2\beta^{-1}  + L_f^{m}  (\sum_{i=1}^K (U_{i} + V_{i}) ) (\beta T)^{-c} \\
			&\quad +  L_f^{m} \sum_{i=1}^{K}(L_f^{2})^{i} \eta^2\beta^{-1}   +    L_f^{m} (\sigma_{f}^{2} + \sigma_{J}^{2} +  \sigma_{f}^2(\sum_{i=1}^{K}( L_f^{2})^{i})  ) \beta \\
			& \quad + (L_f^{m}  \sum_{j=1}^{K-1}U_{i}  + L_f^{m}  \sum_{j=1}^{K-1}(U_{i} + V_{i})) \eta(\eta T)^{-c} .
		\end{aligned}
	\end{equation*}
	Setting $\eta = T^{-a} \beta = T^{-b}$, we have 
	\begin{equation*}
		\begin{aligned}
			&\EX_{A}[F(A(S)) - F(x_*)]\\
			& \leq O(T^{\frac{c}{2}(b-1)} + T^{-\frac{b}{2}} +T^{\frac{b}{2}-a} + \sum_{i=1}^K n_k^{-1} + T^{-c(1-a)} +T^{-a} + + T^{-c(1-b)} +T^{-b} +T^{b-2a} +T^{-c(1-a)-a} ).
		\end{aligned}
	\end{equation*}
	Setting $c = 3$, then the dominating terms are $O(T^{\frac{b}{2}-a}), \quad O(T^{-\frac{b}{2}}), \quad O(T^{\frac{c}{2}(b-1)}), \quad O(T^{-\frac{a}{2}}), \text { and } \quad O(T^{-c(1-a)})$. Then setting  $a= b= \frac{6}{7}$ we have  
	\begin{equation*}
		\EX_{A}[F(A(S)) - F(x_*)] = O(T^{-\frac{3}{7}}).
	\end{equation*}
	Then setting  $T= O(\max\{n_1^{\frac{7}{6}},\cdots, n_K^{\frac{7}{6}}\}) $,  we have 
	\begin{equation*}
		\EX_{A}[F(A(S)) - F(x_*)] = O(\sum_{k=1}^K\frac{1}{\sqrt{n_k}}).
	\end{equation*}
	Then we complete the proof.
\end{proof}